\documentclass{article}

\usepackage[utf8]{inputenc} %
\usepackage[T1]{fontenc}    %
\usepackage{hyperref}       %
\usepackage{url}            %
\usepackage{booktabs}       %
\usepackage{amsfonts}       %
\usepackage{nicefrac}       %
\usepackage{microtype}      %
\usepackage{xcolor}         %
\usepackage{array}
\usepackage{fancyvrb}

\usepackage{natbib} %
\usepackage{mathtools} %
\usepackage{booktabs} %
\usepackage{tikz} %

\usepackage{microtype}
\usepackage{graphicx}
\usepackage{subcaption} %

\usepackage[arxiv]{icml2022}

\usepackage{algpseudocode}

\renewcommand{\textuparrow}{$\uparrow$}
\renewcommand{\textdownarrow}{$\downarrow$}

\usepackage{amsthm}
\usepackage{amssymb}
\usepackage{thmtools,thm-restate}
\usepackage{amsfonts} %
\usepackage[capitalize,noabbrev]{cleveref}
\usepackage{listings}
\usepackage{xcolor}
\usepackage{mathcommand}
\usepackage{multirow}
\usepackage{multicol}
\usepackage{pifont}
\usepackage{rotating}
\newcommand{\cmark}{\ding{51}}%
\newcommand{\xmark}{\ding{55}}%

\definecolor{solarized@base03}{HTML}{002B36}
\definecolor{solarized@base02}{HTML}{073642}
\definecolor{solarized@base01}{HTML}{586e75}
\definecolor{solarized@base00}{HTML}{657b83}
\definecolor{solarized@base0}{HTML}{839496}
\definecolor{solarized@base1}{HTML}{93a1a1}
\definecolor{solarized@base2}{HTML}{EEE8D5}
\definecolor{solarized@base3}{HTML}{FDF6E3}
\definecolor{solarized@yellow}{HTML}{B58900}
\definecolor{solarized@orange}{HTML}{CB4B16}
\definecolor{solarized@red}{HTML}{DC322F}
\definecolor{solarized@magenta}{HTML}{D33682}
\definecolor{solarized@violet}{HTML}{6C71C4}
\definecolor{solarized@blue}{HTML}{268BD2}
\definecolor{solarized@cyan}{HTML}{2AA198}
\definecolor{solarized@green}{HTML}{859900}

\lstdefinestyle{mystyle}{
    backgroundcolor=\color{solarized@base3},
    rulesepcolor=\color{solarized@base03},
    numberstyle=\tiny\color{solarized@base01},
    keywordstyle=\color{solarized@green},
    stringstyle=\color{solarized@cyan}\ttfamily,
    identifierstyle=\color{solarized@blue},
    commentstyle=\color{solarized@magenta},
    emphstyle=\color{solarized@red},
    basicstyle=\ttfamily\scriptsize\color{solarized@base0},
    breakatwhitespace=false,         
    breaklines=true,                 
    captionpos=b,                    
    keepspaces=true,                 
    numbers=none,                    
    numbersep=5pt,                  
    showspaces=false,                
    showstringspaces=false,
    showtabs=false,                  
    tabsize=2
}

\crefname{lstlisting}{listing}{listings}
\Crefname{lstlisting}{Listing}{Listings}

\usepackage{multirow}
\usepackage{enumitem}
\setitemize{noitemsep,topsep=0pt,parsep=0pt,partopsep=0pt}
\setenumerate{noitemsep,topsep=0pt,parsep=0pt,partopsep=0pt}

\theoremstyle{plain}
\newtheorem{theorem}{Theorem}[section]

\newtheorem{lemma}[theorem]{Lemma}
\newtheorem{corollary}[theorem]{Corollary}
\theoremstyle{definition}
\newtheorem{definition}[theorem]{Definition}

\newtheorem{observation}[theorem]{Observation}
\theoremstyle{remark}

\DeclareMathOperator{\opE}{\mathbb{E}}
\DeclareMathOperator{\opH}{\mathbb{H}}
\DeclareMathOperator{\opHt}{\mathbb{H}_\theta}
\DeclareMathOperator{\opI}{\mathbb{I}}
\DeclareMathOperator{\opp}{p}
\DeclareMathOperator{\ophp}{\hat{p}}
\DeclareMathOperator{\oppt}{p_\theta}
\DeclareMathOperator{\opq}{q}
\DeclareMathOperator{\opVar}{Var}
\DeclareMathOperator{\opCov}{Cov}
\newcommand{\Var}[1]{\opVar[ #1 ]}
\newcommand{\Cov}[1]{\opCov[ #1 ]}
\newcommand{\Entropy}[1]{\opH [ #1 ]}
\newcommand{\Hc}[2]{\opH [ #1 \mathbin{\vert} #2 ]}
\newcommand{\Ht}[1]{\opHt [ #1 ]}
\newcommand{\Htc}[2]{\opHt [ #1 \mathbin{\vert} #2 ]}

\newcommand{\MIc}[3]{\opI [ #1 ; #2 \mathbin{\vert} #3 ]}
\newcommand{\prob}[1]{\opp ( #1 )}
\newcommand{\probc}[2]{\opp ( #1 \mathbin{\vert} #2 )}
\newcommand{\hpprob}[1]{\ophp ( #1 )}
\newcommand{\hpprobc}[2]{\ophp ( #1 \mathbin{\vert} #2 )}
\newcommand{\qprob}[1]{\opq ( #1 )}
\newcommand{\qprobc}[2]{\opq ( #1 \mathbin{\vert} #2 )}
\newcommand{\ptprob}[1]{\oppt ( #1 )}
\newcommand{\ptprobc}[2]{\oppt ( #1 \mathbin{\vert} #2 )}
\newcommand{\sqprob}[2]{\opq#2 ( #1 )}

\newcommand{\CrossEntropy}[2]{\opH ( #1 \mathbin{\vert \vert} #2 )}
\newcommand{\Kale}[2]{D_\mathrm{KL} ( #1 \mathbin{\vert \vert} #2 )}

\newcommand{\N}{\mathcal{N}}

\newcommand{\normaldistpdf}[3]{\N(#1;\,#2,\,#3)}

\newcommand{\E}[2]{\opE_{#1} \left [ #2 \right ]}
\newcommand{\chainedE}[2]{\opE_{#1} {#2}}

\newcommand{\data}{\mathcal{D}}

\definecolor{sns-orange}{HTML}{ff7f0e}
\definecolor{sns-ambiguous}{HTML}{00528C}
\definecolor{sns-nonambiguous}{HTML}{2CA9FF}
\definecolor{sns-blue}{HTML}{1f77b4}

\graphicspath{{./figs/}}

\everypar{\looseness=-1}

\VerbatimFootnotes

\begin{document}

\twocolumn[
\icmltitle{Deep Deterministic Uncertainty: A Simple Baseline}

\icmlsetsymbol{equal}{*}

\begin{icmlauthorlist}
\icmlauthor{Jishnu Mukhoti}{equal,oat,torr}
\icmlauthor{Andreas Kirsch}{equal,oat}
\icmlauthor{Joost van Amersfoort}{oat}
\icmlauthor{Philip H.S. Torr}{torr}
\icmlauthor{Yarin Gal}{oat}
\end{icmlauthorlist}

\icmlaffiliation{oat}{OATML, 
Department of Computer Science,
University of Oxford}
\icmlaffiliation{torr}{Torr Vision Group,
Department of Engineering Science,
University of Oxford}

\icmlcorrespondingauthor{Jishnu Mukhoti}{jishnu.mukhoti@eng.ox.ac.uk}
\icmlcorrespondingauthor{Andreas Kirsch}{andreas.kirsch@cs.ox.ac.uk}

\icmlkeywords{Machine Learning, ICML}

\vskip 0.3in
]%

\printAffiliationsAndNotice{\icmlEqualContribution} %

\begin{abstract}
Reliable uncertainty from deterministic single-forward pass models is sought after because conventional methods of uncertainty quantification are computationally expensive.
We take two complex single-forward-pass uncertainty approaches, DUQ and SNGP, and examine whether they mainly rely on a well-regularized feature space. Crucially, without using their more complex methods for estimating uncertainty, a single softmax neural net with such a feature-space, achieved via residual connections and spectral normalization, \emph{outperforms} DUQ and SNGP's epistemic uncertainty predictions using simple Gaussian Discriminant Analysis \emph{post-training} as a separate feature-space density estimator---without fine-tuning on OoD data, feature ensembling, or input pre-procressing.
This conceptually simple \emph{Deep Deterministic Uncertainty (DDU)} baseline can also be used to disentangle aleatoric and epistemic uncertainty and performs as well as Deep Ensembles, the state-of-the art for uncertainty prediction, on several OoD benchmarks 
(CIFAR-10/100 vs SVHN/Tiny-ImageNet, ImageNet vs ImageNet-O) 
as well as in active learning settings across different model architectures, yet is \emph{computationally cheaper}. %
\end{abstract}
\section{Introduction}
\label{sec:intro}

Two types of uncertainty are often of interest in ML: \emph{epistemic uncertainty}, which is inherent to the model, caused by a lack of training data, and hence reducible with more data, and \emph{aleatoric uncertainty}, caused by inherent noise or ambiguity in data, and hence irreducible with more data \citep{der2009aleatory}.
Disentangling these two 
is critical for applications such as active learning \citep{gal2017deep} or detection of out-of-distribution (OoD) samples \citep{hendrycks2016baseline}: in active learning, we wish to avoid inputs with high aleatoric but low epistemic uncertainty, and in OoD detection, we wish to avoid mistaking ambiguous in-distribution (iD) examples as OoD.
This is particularly challenging for noisy and ambiguous datasets found in safety-critical applications like autonomous driving \citep{huang2020autonomous} and medical diagnosis \citep{esteva2017dermatologist, filos2019systematic}.

\textbf{Related Work.} Most well-known methods of uncertainty quantification in deep learning \citep{blundell2015weight, gal2016dropout, lakshminarayanan2017simple, wen2020batchensemble, dusenberry2020efficient} require multiple forward passes at test time.
Amongst these, Deep Ensembles have generally performed best in uncertainty prediction \citep{ovadia2019can}, but 
their significant memory and compute burden at training and test time hinders their adoption in real-life and mobile applications.
Consequently, there has been an increased interest in uncertainty quantification using deterministic single forward-pass neural networks which have a smaller footprint and lower latency.
Among these approaches, \citet{lee2018simple} uses Mahalanobis distances to quantify uncertainty by fitting a class-wise Gaussian distribution (with shared covariance matrices) on the feature space of a pre-trained ResNet encoder.
They do not consider the structure of the underlying feature-space, which might explain why their competitive results require input perturbations, ensembling GMM densities from multiple layers, and fine-tuning on OoD hold-out data.

\begin{table*}[!t]
    \centering
    \caption{\emph{OoD detection performance of different baselines using a Wide-ResNet-28-10 architecture with the CIFAR-10 vs SVHN/CIFAR-100/Tiny-ImageNet and CIFAR-100 vs SVHN/Tiny-ImageNet dataset pairs averaged over 25 runs.} \emph{SN:} Spectral Normalisation, \emph{JP:} Jacobian Penalty. The best deterministic single-forward pass method and the best method overall are in bold for each metric.}
    \label{table:ood_wrn}
    \resizebox{\linewidth}{!}{%
    \begin{tabular}{cccccccccc}
    \toprule
    \textbf{\small Train Dataset} & \textbf{\small Method} & \textbf{\small Penalty} & {\small Aleatoric Uncertainty} &
    \textbf{\small Epistemic Uncertainty} & {\small Accuracy (\textuparrow)} & {\small ECE (\textdownarrow)} & \textbf{\small AUROC SVHN (\textuparrow)} & \textbf{\small AUROC CIFAR-100 (\textuparrow)} & \textbf{\small AUROC Tiny-ImageNet (\textuparrow)}\\
    \midrule
    \multirow{7}{*}{CIFAR-10} & Softmax & - & \multirow{2}{*}{Softmax Entropy} & Softmax Entropy & \multirow{2}{*}{$95.98\pm0.02$} & \multirow{2}{*}{$\mathbf{0.85\pm0.02}$} & $94.44\pm0.43$ & $89.39\pm0.06$ & $88.42\pm0.05$ \\
    &Energy-based {\scriptsize \citep{liu2020energy}} & - && Softmax Density &&& $94.56\pm0.51$ & $88.89\pm0.07$ & $88.11\pm0.06$\\
    &DUQ {\scriptsize \citep{van2020simple}} & JP & Kernel Distance & Kernel Distance & $94.6\pm0.16$ & $1.55\pm0.08$ & $93.71\pm0.61$ & $85.92\pm0.35$ & $86.83\pm0.12$\\
    &SNGP {\scriptsize \citep{liu2020simple}} & SN & Predictive Entropy & Predictive Entropy & $\mathbf{96.04\pm0.09}$ & $1.8\pm0.1$ & $94.0\pm1.3$ & $91.13\pm0.15$ & $89.97\pm0.19$\\
    &\textbf{DDU (ours)} & \textbf{SN} & \textbf{Softmax Entropy} & \textbf{GMM Density} & $95.97\pm0.03$&$\mathbf{0.85\pm0.04}$&$\mathbf{97.86\pm0.19}$&$\mathbf{91.34\pm0.04}$ & $\mathbf{91.07\pm0.05}$\\
    \cmidrule{2-10}
    &5-Ensemble & \multirow{2}{*}{-} & \multirow{2}{*}{Predictive Entropy} & Predictive Entropy & \multirow{2}{*}{$\mathbf{96.59\pm0.02}$}&\multirow{2}{*}{$\mathbf{0.76\pm0.03}$}&$97.73\pm0.31$&$\mathbf{92.13\pm0.02}$ & $90.06\pm0.03$\\
    &{\scriptsize \citep{lakshminarayanan2017simple}} &&& Mutual Information &&&$97.18\pm0.19$&$91.33\pm0.03$ & $90.90\pm0.03$\\
    \midrule
    &&&&& {\small Accuracy (\textuparrow)} & {\small{ECE (\textdownarrow)}} & \multicolumn{2}{c}{\textbf{\small AUROC SVHN (\textuparrow)}} & \textbf{\small AUROC Tiny-ImageNet (\textuparrow)} \\
    \cmidrule{6-10}
    \multirow{6}{*}{CIFAR-100} & Softmax & - & \multirow{2}{*}{Softmax Entropy} & Softmax Entropy & \multirow{2}{*}{$80.26\pm0.06$}&\multirow{2}{*}{$4.62\pm0.06$}& \multicolumn{2}{c}{$77.42\pm0.57$} & $81.53\pm0.05$ \\
    &Energy-based {\scriptsize \citep{liu2020energy}} & - && Softmax Density &&&\multicolumn{2}{c}{$78\pm0.63$} & $81.33\pm0.06$ \\
    &SNGP {\scriptsize \citep{liu2020simple}} & SN & Predictive Entropy & Predictive Entropy & $80.00\pm0.11$ & $4.33\pm0.01$ & \multicolumn{2}{c}{$85.71\pm0.81$} & $78.85\pm0.43$ \\
    &\textbf{DDU (ours)} & \textbf{SN} & \textbf{Softmax Entropy} &  \textbf{GMM Density} & $\mathbf{80.98\pm0.06}$&$\mathbf{4.10\pm0.08}$& \multicolumn{2}{c}{$\mathbf{87.53\pm0.62}$}& $\mathbf{83.13\pm0.06}$ \\
    \bottomrule
    \cmidrule{2-10}
    &5-Ensemble & \multirow{2}{*}{-} & \multirow{2}{*}{Predictive Entropy} & Predictive Entropy &\multirow{2}{*}{$\mathbf{82.79\pm0.10}$}&\multirow{2}{*}{$\mathbf{3.32\pm0.09}$}& \multicolumn{2}{c}{$79.54\pm0.91$} & $82.95\pm0.09$\\
    &{\scriptsize \citep{lakshminarayanan2017simple}} &&& Mutual Information &&& \multicolumn{2}{c}{$77.00\pm1.54$} & $82.82\pm0.04$\\
    \end{tabular}}
\end{table*}
\begin{figure*}[!t]
    \centering
    \begin{subfigure}{0.20\linewidth}
        \centering
        \includegraphics[width=\linewidth]{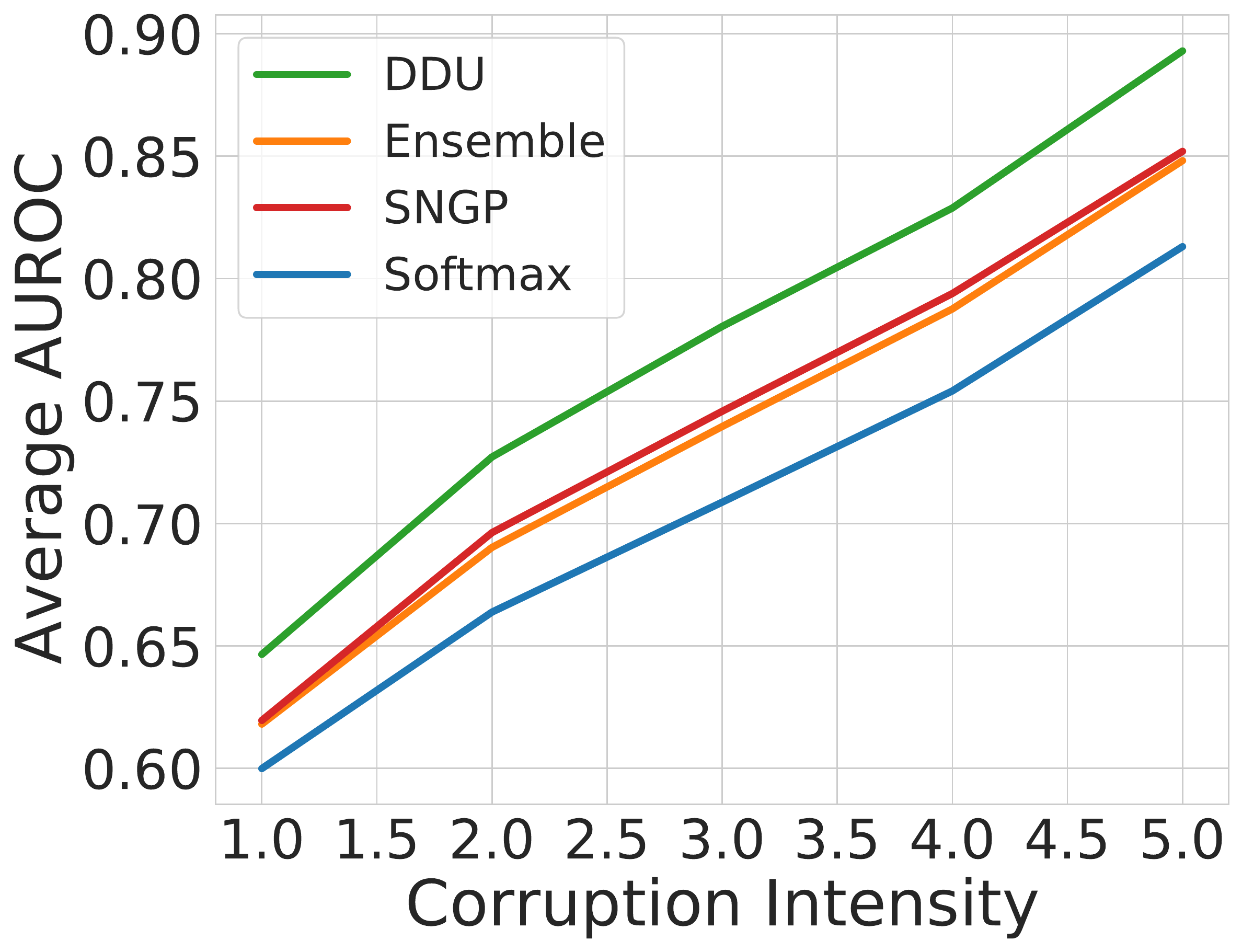}
        \caption{Wide-ResNet-28-10}
        \label{subfig:cifar10_c_wide_resnet}
    \end{subfigure}
    \begin{subfigure}{0.20\linewidth}
        \centering
        \includegraphics[width=\linewidth]{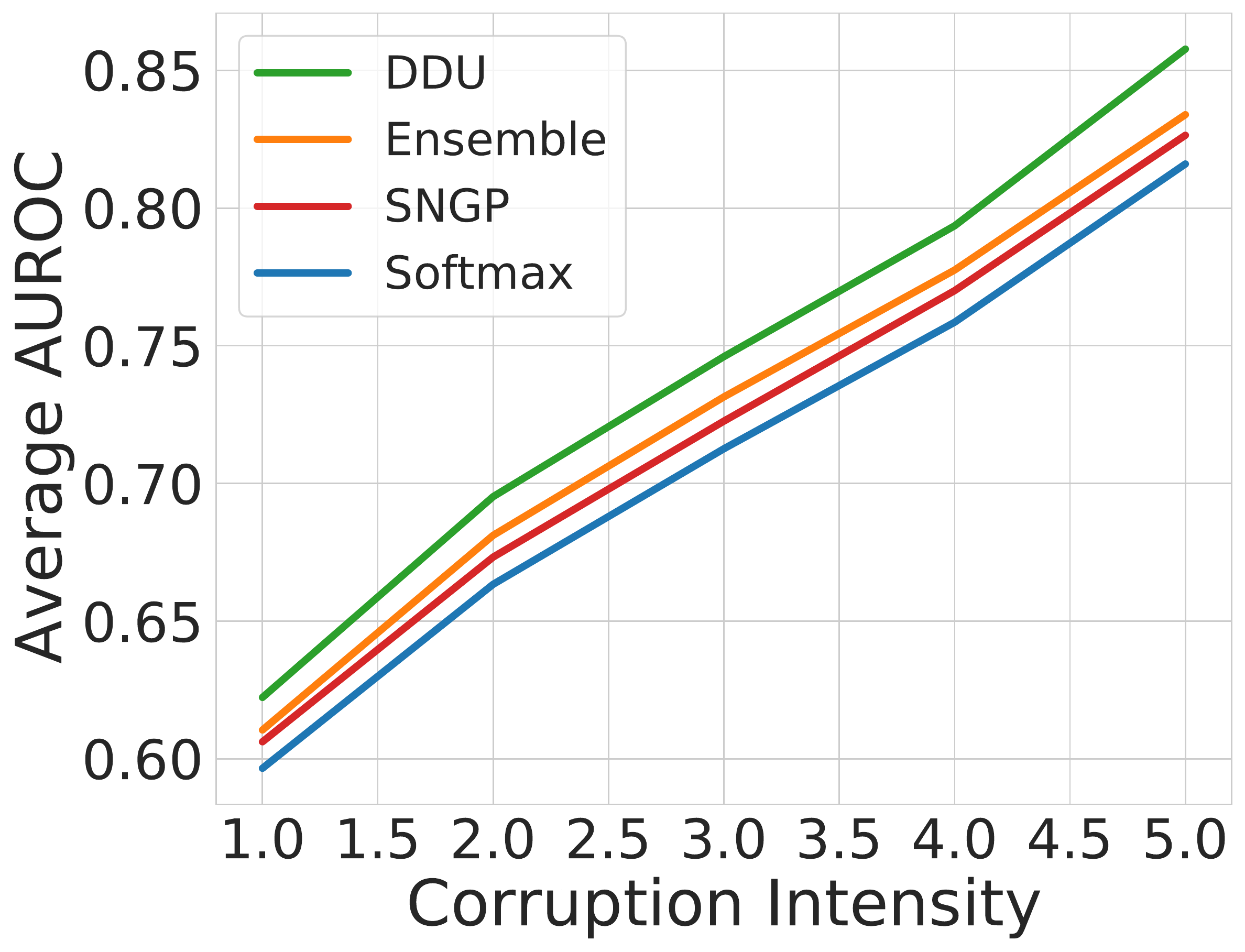}
        \caption{ResNet-50}
        \label{subfig:cifar10_c_resnet50}
    \end{subfigure} 
    \begin{subfigure}{0.20\linewidth}
        \centering
        \includegraphics[width=\linewidth]{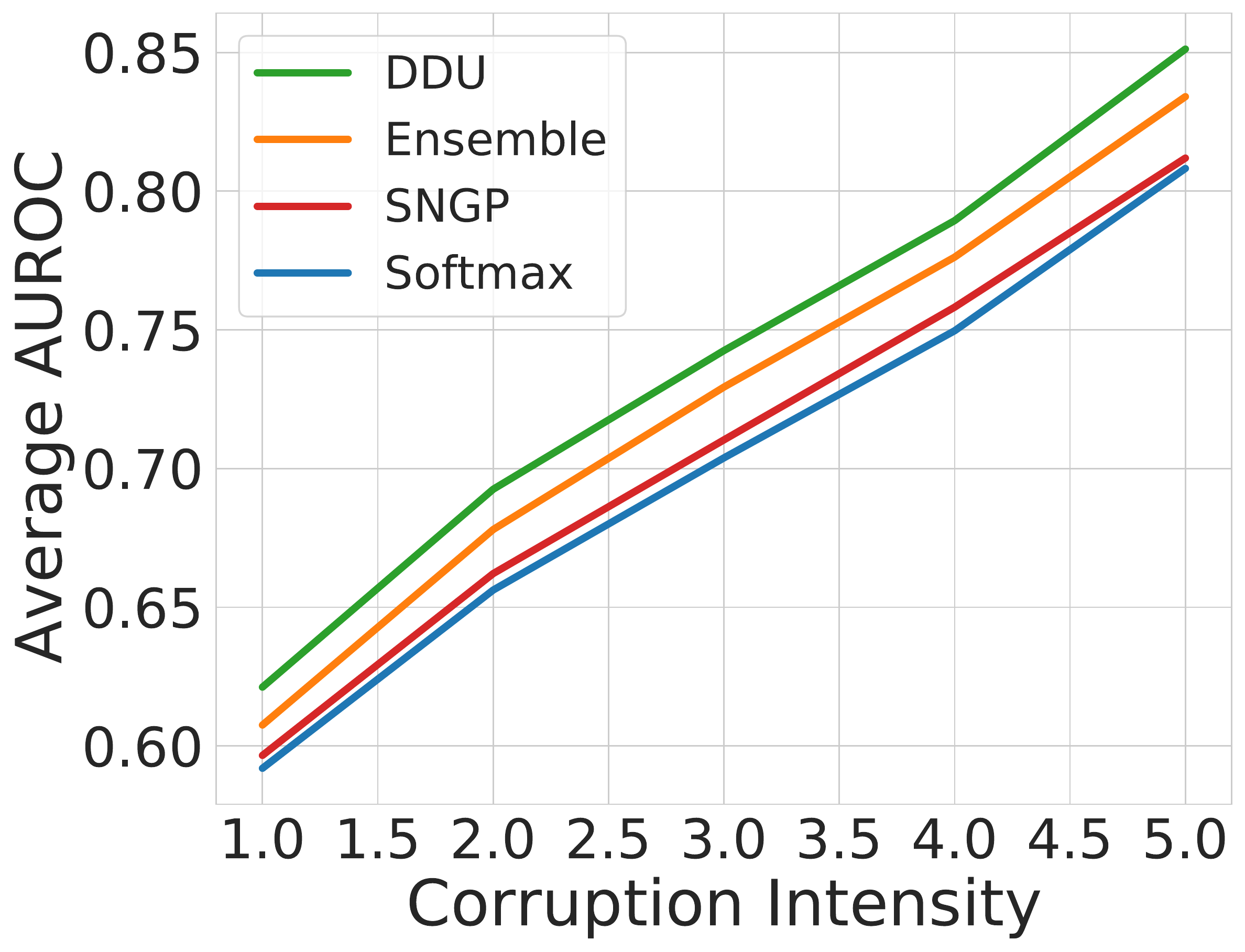}
        \caption{ResNet-110}
        \label{subfig:cifar10_c_resnet110}
    \end{subfigure} 
    \begin{subfigure}{0.20\linewidth}
        \centering
        \includegraphics[width=\linewidth]{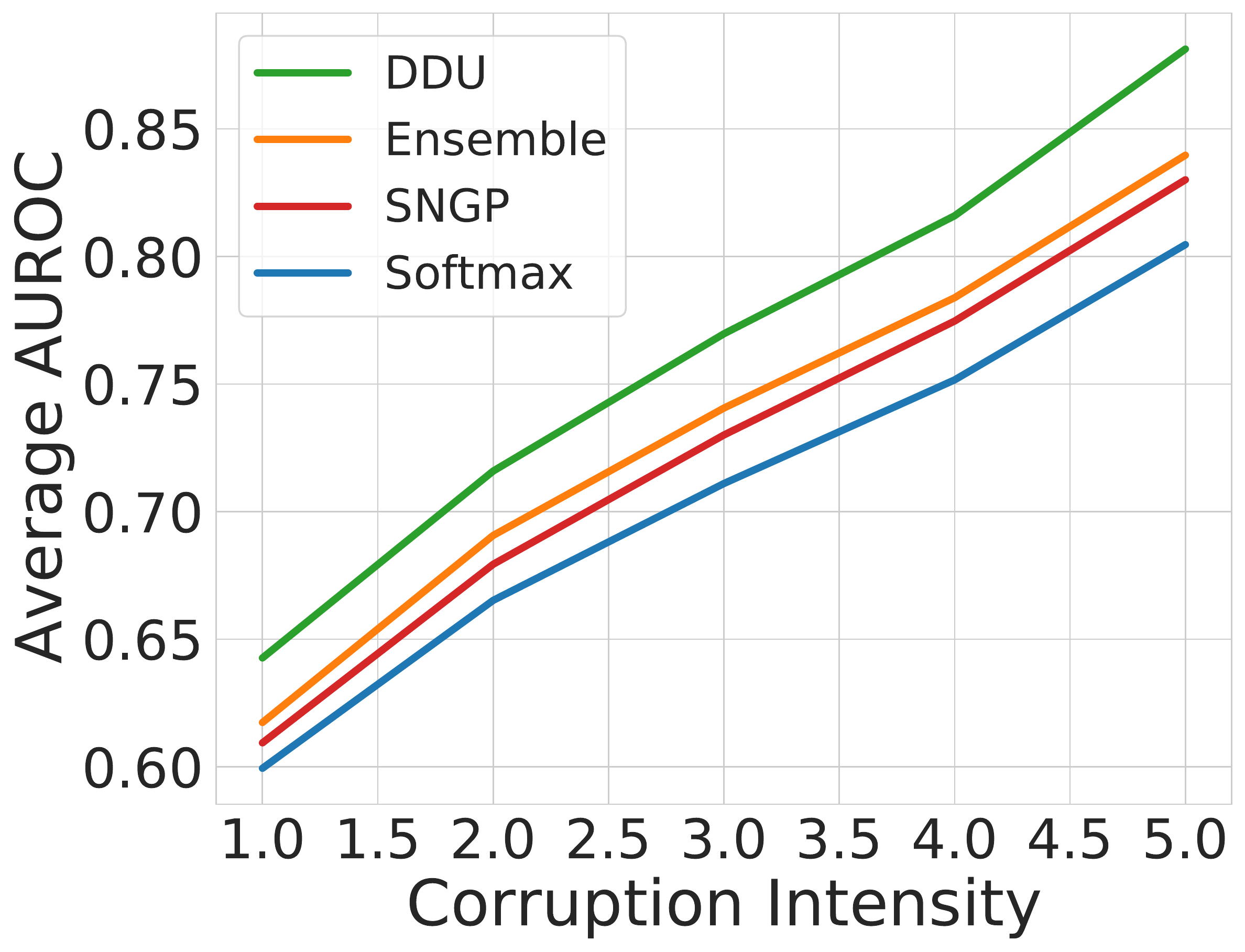}
        \caption{DenseNet-121}
        \label{subfig:cifar10_c_densenet121}
    \end{subfigure} 
    \vspace{-0.5em} %
    \caption{
    AUROC vs corruption intensity averaged over all corruption types in CIFAR-10-C for architectures: Wide-ResNet-28-10, ResNet-50, ResNet-110 and DenseNet-121 and baselines: Softmax Entropy, Ensemble (using Predictive Entropy as uncertainty), SNGP, and DDU feature density.
    More details in \S\ref{sec:experiments_ood_detection} and more model architectures and ablations in \S\ref{app:exp_details} in the appendix. 
    }
    \label{fig:cifar10_c_results}
    \vspace{-0.5em}
\end{figure*}

\begin{figure}[t!]
    \centering
    \begin{subfigure}{\linewidth}
        \includegraphics[width=\linewidth]{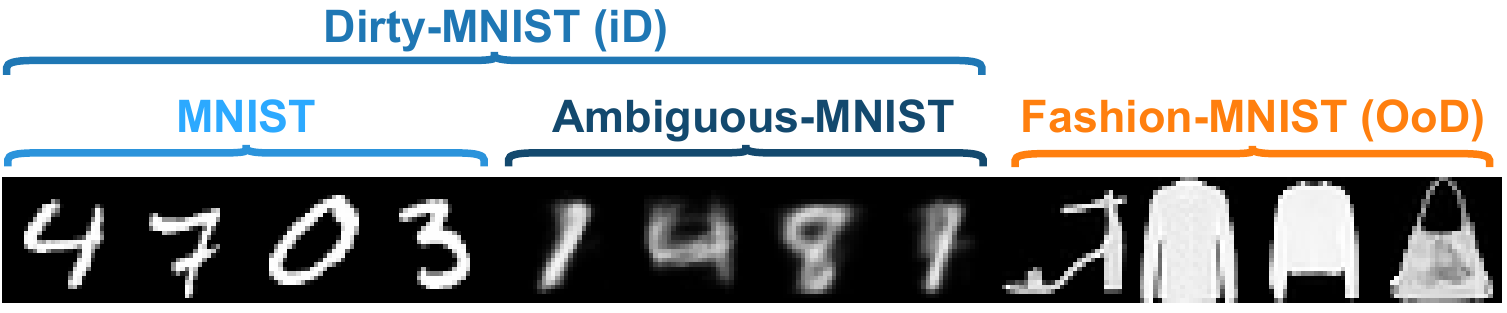}%
        \caption{{\color{sns-blue} Dirty-MNIST (iD)} and {\color{sns-orange} Fashion-MNIST (OoD)}}
        \label{fig:intro_sample_viz}
    \end{subfigure}
    \begin{subfigure}{\linewidth}
        \begin{subfigure}{0.33\linewidth}
            \centering
            \includegraphics[width=\linewidth]{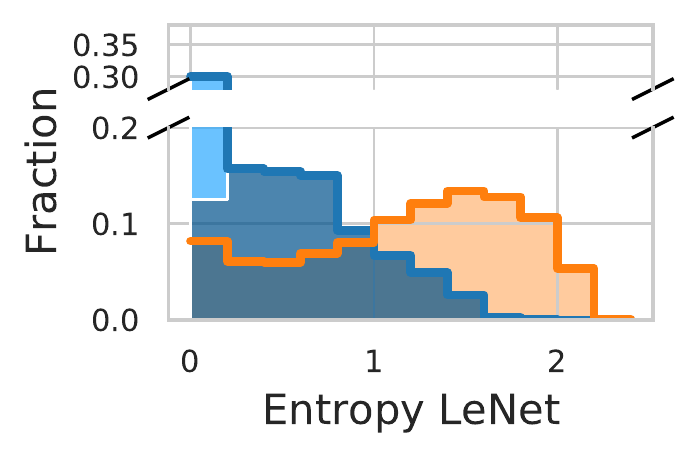}%
        \end{subfigure}%
        \begin{subfigure}{0.33\linewidth}
            \centering
            \includegraphics[width=\linewidth]{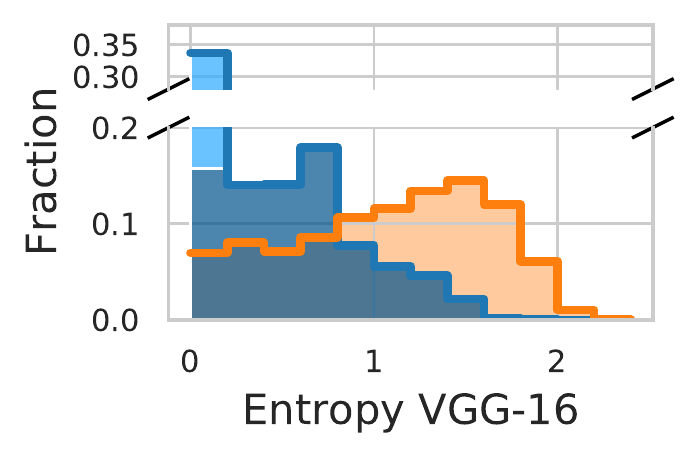}%
        \end{subfigure}%
        \begin{subfigure}{0.33\linewidth}
            \centering
            \includegraphics[width=\linewidth]{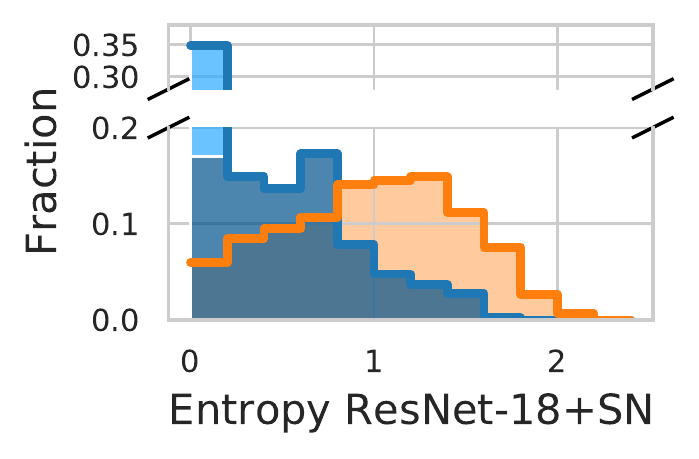}%
        \end{subfigure}%
        \subcaption{{Softmax entropy}}\label{fig:intro_softmax_ent}
    \end{subfigure}
    \begin{subfigure}{\linewidth}
        \begin{subfigure}{0.33\linewidth}
            \centering
            \includegraphics[width=\linewidth]{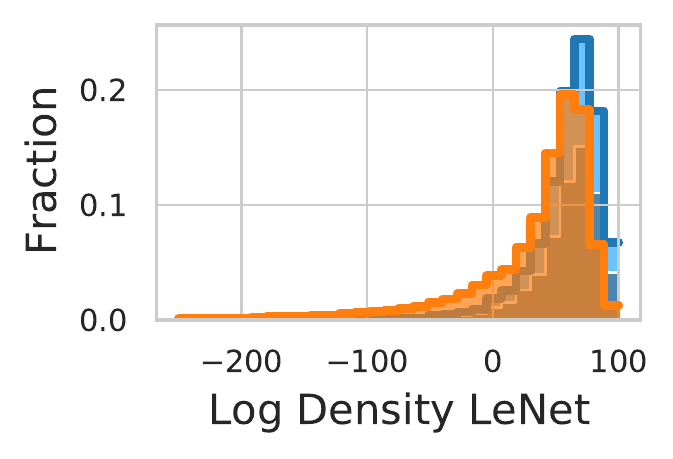}%
        \end{subfigure}%
        \begin{subfigure}{0.33\linewidth}
            \centering
            \includegraphics[width=\linewidth]{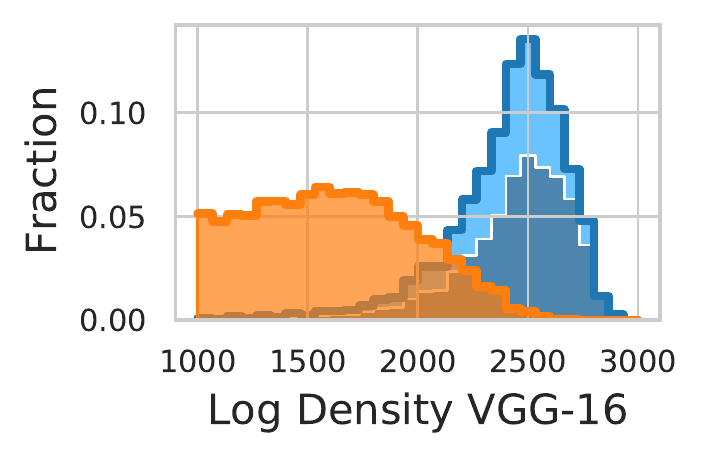}%
        \end{subfigure}%
        \begin{subfigure}{0.33\linewidth}
            \centering
            \includegraphics[width=\linewidth]{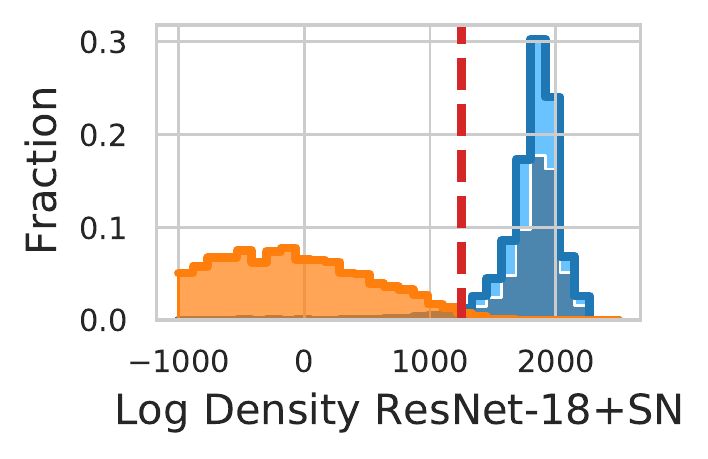}%
        \end{subfigure}%
        \subcaption{{Feature-space density}}\label{fig:intro_gmm}
    \end{subfigure}%
    \caption{
    \emph{
    Disentangling aleatoric and epistemic uncertainty on {\color{sns-blue}Dirty-MNIST (iD)} and {\color{sns-orange} Fashion-MNIST (OoD)} \textbf{\subref{fig:intro_sample_viz}} requires using \emph{softmax entropy} \textbf{\subref{fig:intro_softmax_ent}} and \emph{feature-space density (GMM)} \textbf{\subref{fig:intro_gmm}} with a well-regularized feature space (\emph{ResNet-18+SN} vs \emph{LeNet} \& \emph{VGG-16} without smoothness \& sensitivity).}
    \textbf{\subref{fig:intro_softmax_ent}:}
    Softmax entropy captures aleatoric uncertainty for iD data (Dirty-MNIST), thereby separating {\color{sns-nonambiguous}unambiguous MNIST samples} and {\color{sns-ambiguous}Ambiguous-MNIST samples}. However, iD and OoD are confounded: softmax entropy has arbitrary values for OoD, indistinguishable from iD.
    \textbf{\subref{fig:intro_gmm}:}
    With a well-regularized feature space (DDU with ResNet-18+SN), iD and OoD densities do not overlap, capturing epistemic uncertainty. However, without such feature space (LeNet \& VGG-16), feature density suffers from \emph{feature collapse}: iD and OoD densities overlap. Generally, feature-space density confounds unambiguous and ambiguous iD samples as their densities overlap.
    }
    \label{fig:intro_histograms}
    \vspace{-0.5em}
\end{figure}

\textbf{DUQ \& SNGP.} Two recent works in single forward-pass uncertainty, DUQ \citep{van2020simple} and SNGP \citep{liu2020simple}, propose distance-aware output layers, in the form of RBFs (radial basis functions) or GPs (Gaussian processes), and introduce additional inductive biases in the feature extractor using a Jacobian penalty \citep{gulrajani2017improved} or spectral normalisation \citep{miyato2018spectral}, respectively, which encourage smoothness and sensitivity in the latent space.
These methods perform well and are almost competitive with Deep Ensembles on OoD benchmarks.
However, they require training to be changed substantially, and introduce additional hyper-parameters due to the specialised output layers used at training.
Furthermore, DUQ and SNGP cannot disentangle aleatoric and epistemic uncertainty.
In DUQ, the feature representation of an ambiguous data point, high on aleatoric uncertainty, will be in between two centroids, but due to the exponential decay of the RBF it will seem far from both and thus have uncertainty similar to epistemically uncertain data points that are far from all centroids.
In SNGP, the predictive variance is computed using a mean-field approximation of the softmax likelihood, which cannot be disentangled, or using MC samples of the softmax likelihood.
In theory, the MC samples allow disentangling the uncertainty (see \Cref{eq:BALD}), but this requires modelling the covariance between the classes, which is not the case in SNGP.

We provide a more extensive review of related work in \S\ref{sec:related work}.

\textbf{Contributions.} Firstly, we investigate the question whether complex methods to estimate uncertainty like in DUQ and SNGP are necessary beyond feature-space regularization that encourages bi-Lipschitzness. When we use spectral normalisation like SNGP does, the short answer is an empirical no.
Indeed, with a well-regularized feature space using spectral normalisation, we find that we can fit a GDA \emph{after training}, similar to \citet{lee2018simple}, as feature-space density estimator to capture epistemic uncertainty---but, unlike \citet{lee2018simple}, who do not place any constraints on the feature space, we do not require training on ``OoD'' hold-out data, feature ensembling, and input pre-processing to obtain good performance (see \Cref{table:ood_wrn}). The regularizing effect of spectral normalisation seems sufficient to not need these additional steps, resulting in a conceptually simpler method.

Secondly, we investigate how to disentangle aleatoric and epistemic uncertainty. This is something that DUQ and SNGP do not address directly. As we only fit GDA after training, the original softmax layer is trained using cross-entropy as a proper scoring rule \citep{gneiting2007strictly} and can be temperature-scaled to provide good in-distribution calibration and aleatoric uncertainty.

This combination of using GDA for epistemic uncertainty and the softmax predictive distribution for aleatoric uncertainty after training with feature-space regularisation, e.g.\ using spectral normalisation, provides a simple baseline which we call \emph{Deep Deterministic Uncertainty (DDU)}.

Altogether, DDU performs as well as a Deep Ensemble's epistemic uncertainty \citep{lakshminarayanan2017simple} and outperforms SNGP and DUQ \citep{van2020simple, liu2020simple} ---with no changes to the model architecture beyond spectral normalisation---on several OoD benchmarks and active learning settings. It also outperforms regular softmax neural networks which might not capture epistemic uncertainty well, as illustrated in \Cref{fig:intro_histograms}.

\textbf{Additional Insights.} Beyond an empirical investigation and the description of DDU, we also provide several additional insights on potential pitfalls for practitioners. 
First, predictive entropy confounds aleatoric and epistemic uncertainty (\cref{fig:intro_softmax_ent}). This can be an issue in active learning in particular. Yet, this issue is often not visible for standard benchmark datasets without aleatoric noise. To examine this failure in more detail, we introduce a new dataset, Dirty-MNIST, which showcases the issue more clearly than artificially curated datasets like MNIST or CIFAR-10. \emph{Dirty-MNIST} is a modified version of MNIST \citep{lecun1998gradient} with additional ambiguous digits (Ambiguous-MNIST) with multiple plausible labels and thus higher aleatoric uncertainty (\cref{fig:intro_sample_viz}).
Secondly, the softmax entropy of a deterministic model, while being high for ambiguous points with high aleatoric uncertainty, might not be consistent for points with high epistemic uncertainty for models trained with maximum likelihood, i.e.\ the softmax entropy for the same OoD sample might be low, high or anything in between for different models trained on the same data (\cref{fig:intro_softmax_ent}).

\textbf{Feature-Space Regularization.} Feature-space density can be a well-performing and simpler approach\footnote{\citet{pearce2021understanding} argue for softmax confidence and entropy in their paper, yet feature-space density performs better in their experiments, too.} to estimate epistemic uncertainty (see \cref{fig:intro_gmm}). 
Crucially, the feature space needs to be well-regularized \citep{liu2020simple}:
without \emph{smoothness} and \emph{sensitivity}, feature-space density alone might not separate iD from OoD data, possibly explaining the limited empirical success of previous approaches which attempt to use feature-space density \citep{postels2020quantifying}.
This can be seen in \cref{fig:intro_gmm} where the feature-space density of a VGG-16 or LeNet model are not able to differentiate iD Dirty-MNIST from OoD Fashion-MNIST while a ResNet-18 with spectral normalization can do so better.

\textbf{Scope.} %
Our focus is on obtaining a well-regularized feature space using spectral normalization in common model architectures with residual connections, following \citep{liu2020simple}. 
Unsupervised methods using contrastive learning \citep{winkens2020contrastive} might also obtain such a feature space by training on very large datasets, but access to them is generally limited or training on them very expensive \citep{sun2017revisiting}.
Similarly, we only use GDA for estimating the feature-space density as it is straight-forward to implement and does not require performing expectation maximization or variational inference like other density estimators. Normalizing flows \citep{dinh2015nice} or other more complex density estimators might provide even better density estimates, of course. However, GDA is a very simple method and already sufficient to outperform other more complex approaches and obtain good results.
As the amount of training data available grows and feature extractors improve, the quality of feature representations might improve as well---an underlying hypothesis of this paper is that simple approaches will remain more applicable than more complex ones as our empirical results suggest.

\begin{figure*}[t]
    \begin{minipage}{0.59\linewidth}
    \centering
    \begin{subfigure}{0.32\linewidth}
        \centering
        \includegraphics[width=\linewidth]{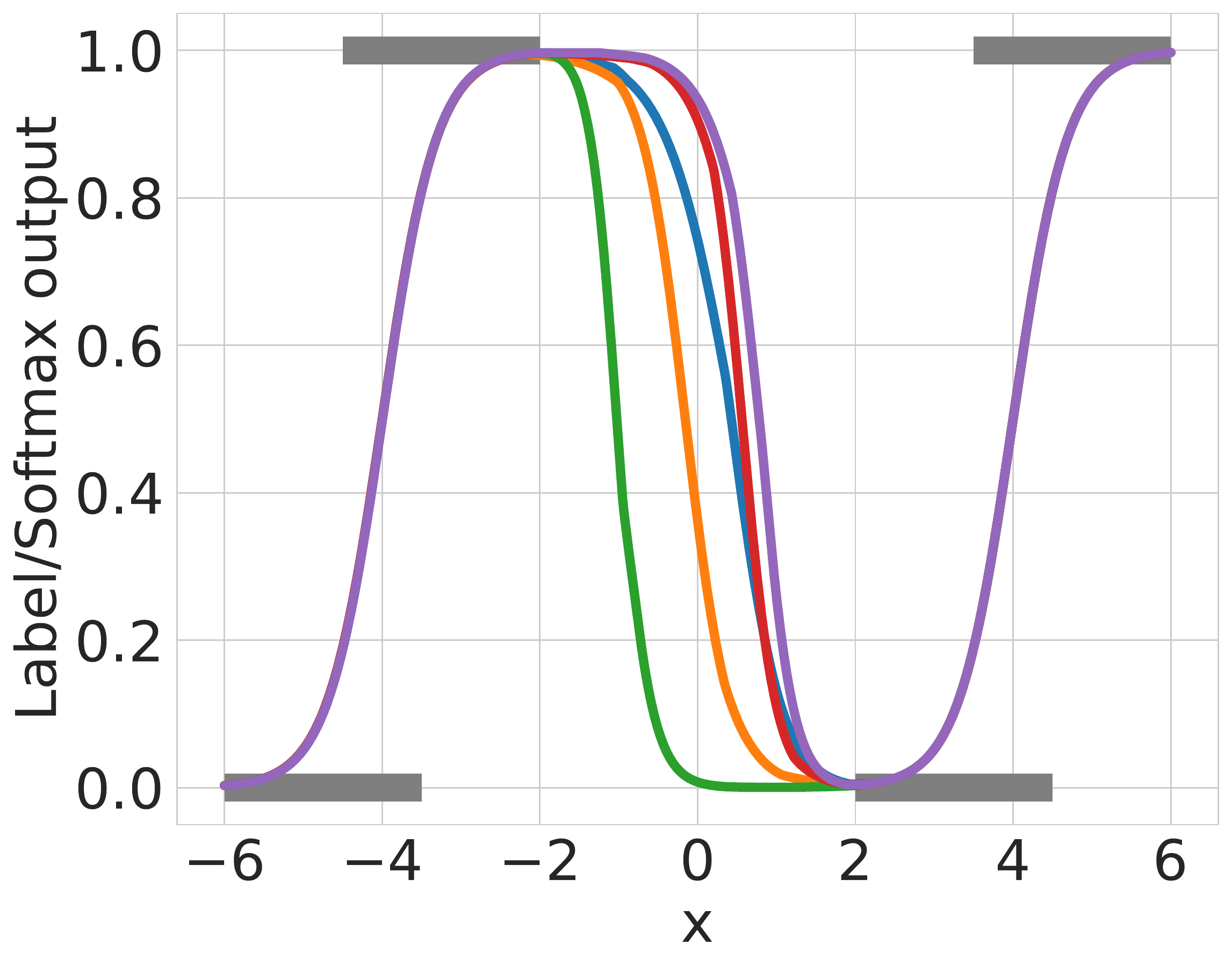}%
        \caption{Softmax output}
        \label{subfig:softmax_output}
    \end{subfigure}
    \begin{subfigure}{0.32\linewidth}
        \centering
        \includegraphics[width=\linewidth]{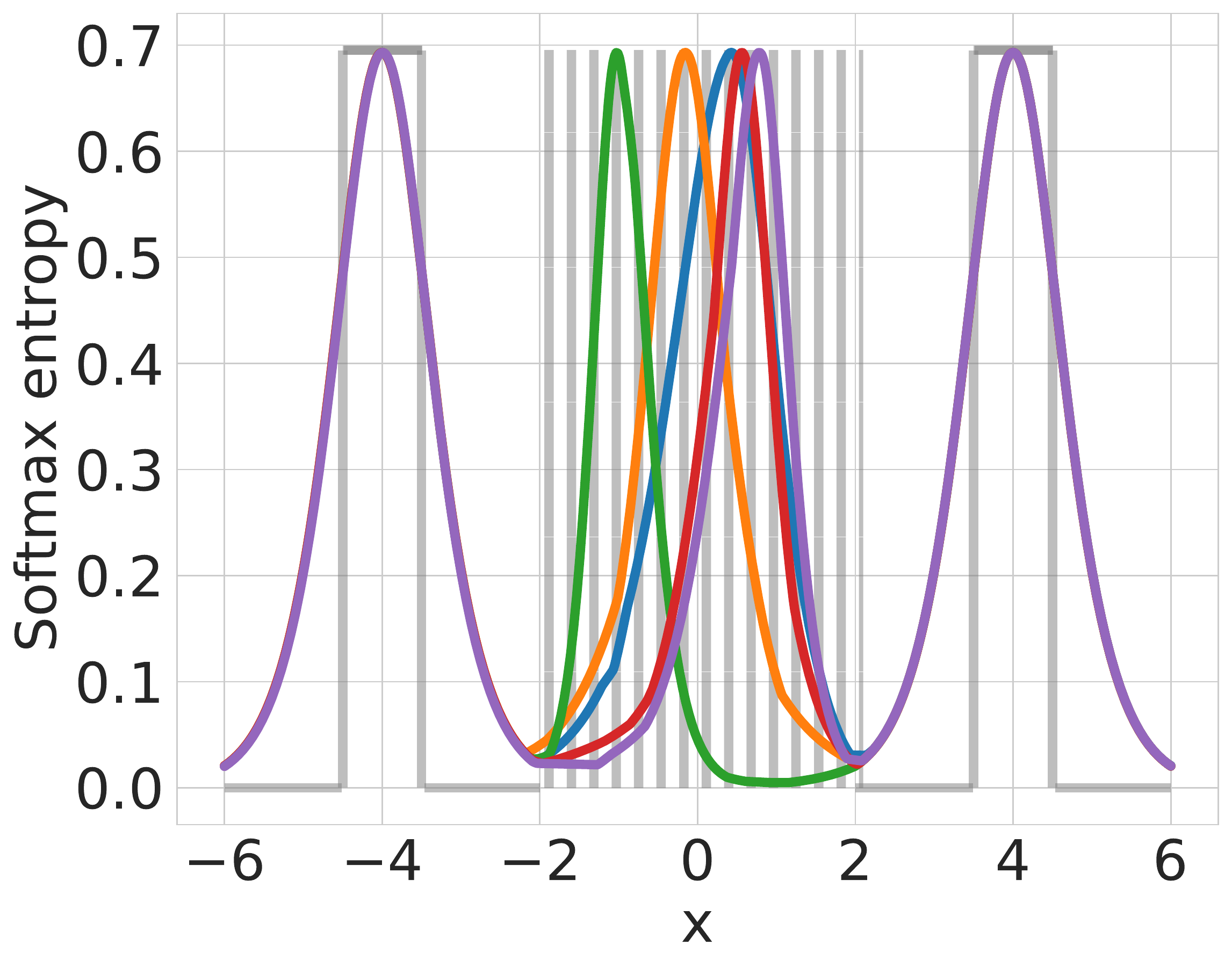}%
        \caption{Softmax entropy}
        \label{subfig:softmax_entropy}
    \end{subfigure}
    \begin{subfigure}{0.32\linewidth}
        \centering
        \includegraphics[width=\linewidth]{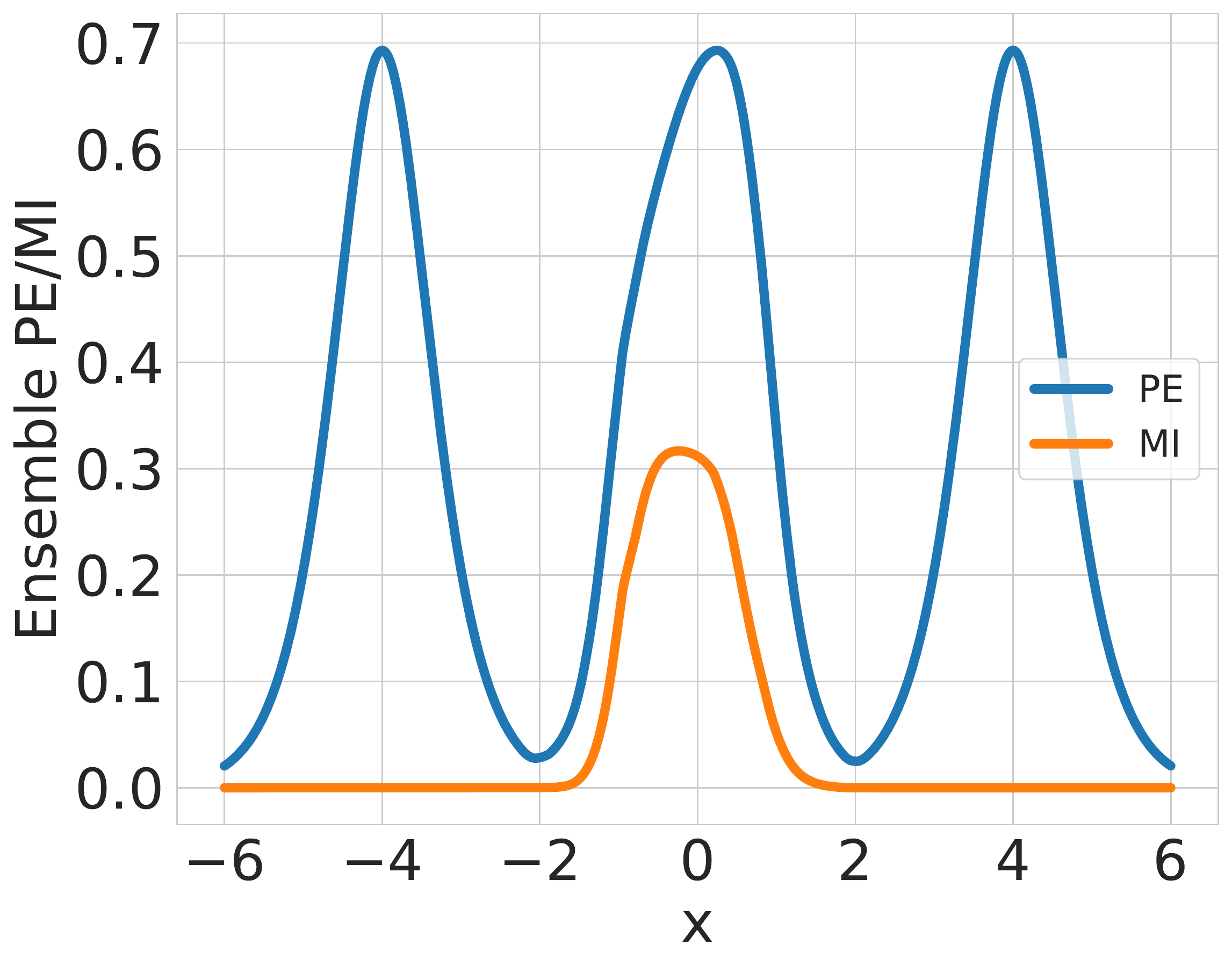}%
        \caption{5-Ensemble}
        \label{subfig:pe_mi}
    \end{subfigure}%
    \caption{
    \emph{Softmax outputs \& entropies for 5 softmax models along with the predictive entropy (PE) and mutual information (MI) for the resulting 5-Ensemble.}
    \subref{subfig:softmax_output} and \subref{subfig:softmax_entropy} show that the softmax entropy is only reliably high for ambiguous iD points ($\pm$3.5--4.5),
    whereas it can be low or high for OoD points (-2--2). The different colors are the different ensemble components. Similarly, \subref{subfig:pe_mi} shows that the MI of the ensemble is only high for OoD,
    whereas the PE is high for both OoD and for regions of ambiguity. See \S\ref{app:5_ensemble}.
    }
    \label{fig:lewis_vis}
\end{minipage}
\hfill
\begin{minipage}{0.39\linewidth}
    \centering
    \begin{subfigure}{0.49\linewidth}
        \centering
        \includegraphics[width=\linewidth]{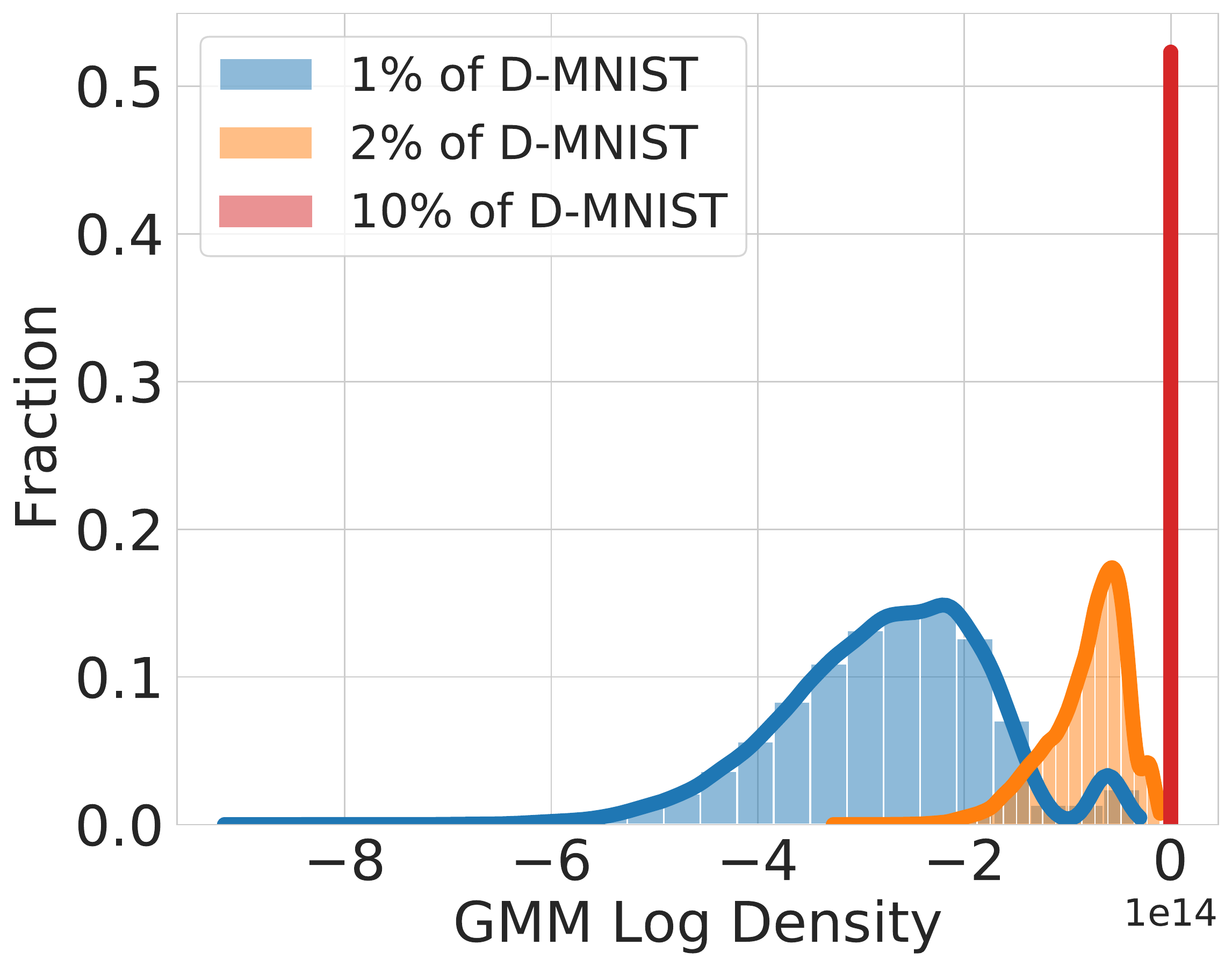}%
        \caption{Epistemic}
        \label{subfig:kendall_viz_gmm_density}
    \end{subfigure}
    \begin{subfigure}{0.49\linewidth}
        \centering
        \includegraphics[width=\linewidth]{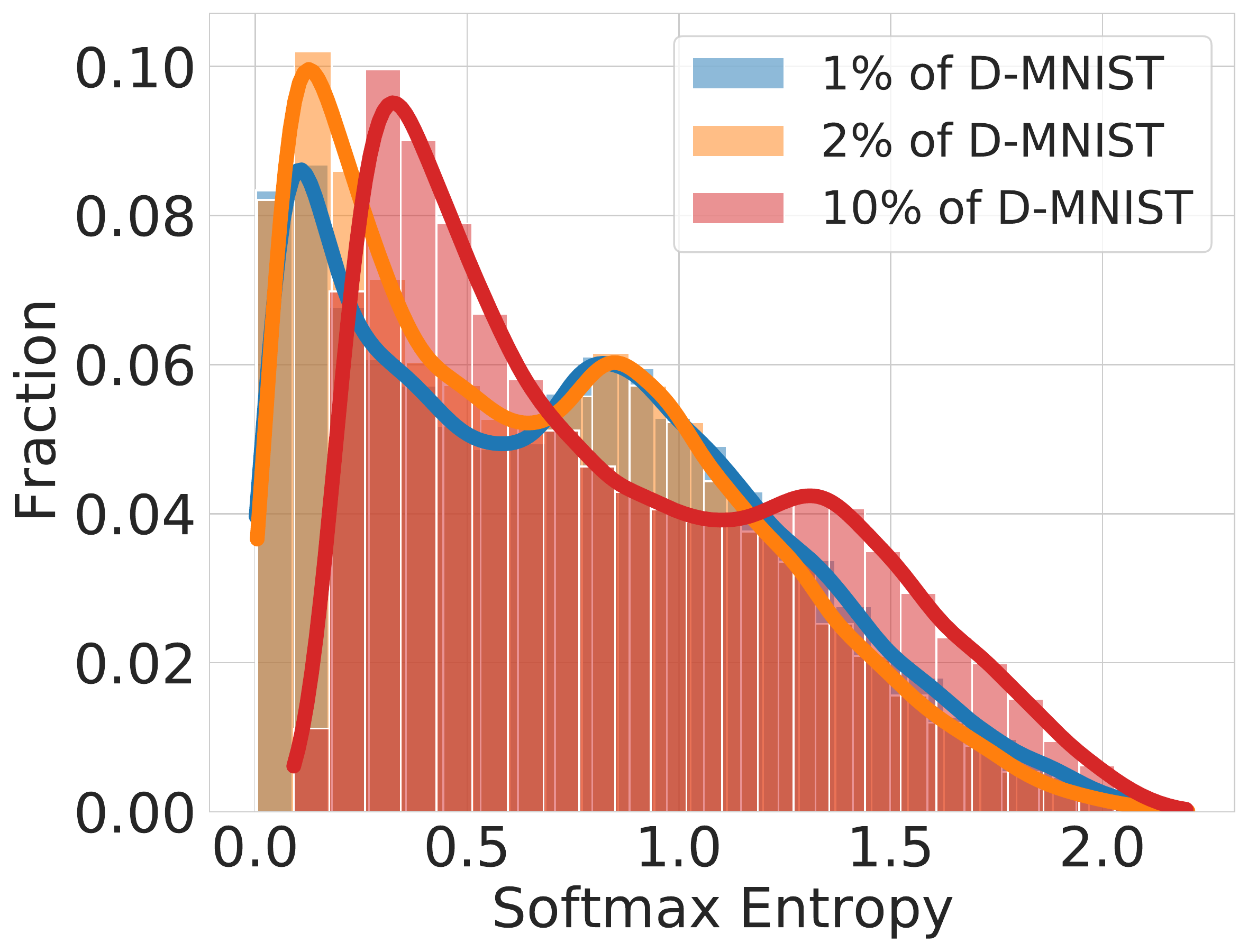}%
        \caption{Aleatoric}
        \label{subfig:kendall_viz_softmax_entropy}
    \end{subfigure}%
    \caption{
    \emph{Epistemic and aleatoric uncertainty of ResNet-18+SN models trained on increasingly large subsets of DirtyMNIST.}
    The feature-space density increases while the softmax entropy stays roughly the same, consistent with epistemic and aleatoric uncertainty being reducible and irreducible with more data, respectively. See \S\ref{app:kendall_viz} for a discussion on this.
    }%
    \label{fig:kendall_viz}
\end{minipage}
\vspace{-0.5em}
\end{figure*}

\vspace{-0.5em}
\section{Background}
\label{sec:background}
In this section, we review concepts important for quantifying uncertainty.

\textbf{Epistemic Uncertainty} at point $x$ is a quantity which is high for a previously unseen $x$, and decreases when $x$ is added to the training set and the model is updated \citep{kendall2017uncertainties}. This conforms with using mutual information in Bayesian models and deep ensembles \citep{kirsch2019batchbald} and feature-space density in deterministic models as surrogates for epistemic uncertainty \citep{postels2020quantifying} as we examine below---and depicted in \Cref{subfig:kendall_viz_gmm_density} (also \S\ref{app:kendall_viz}). 

\textbf{Aleatoric Uncertainty} at point $x$ is a quantity which is high for ambiguous or noisy samples \citep{kendall2017uncertainties}, i.e\ if multiple labels were to be observed at $x$, aleatoric will be high. It does not decrease with more data---depicted in \Cref{subfig:kendall_viz_softmax_entropy}. Note that aleatoric uncertainty is only meaningful in-distribution,
as, by definition, it quantifies the level of ambiguity between the different classes which might be observed for input $x$\footnote{If the probability of observing $x$ under the data generating distribution is zero, $p(y|x)=\frac{p(x, y)}{p(x)}$, and hence, the entropy as a measure of aleatoric uncertainty, is not defined.}.

\textbf{Bayesian Models} \citep{neal2012bayesian, mackay1992bayesian} provide a principled way of measuring uncertainty. Starting with a prior distribution $p(\omega)$ over model parameters $\omega$, they infer a posterior $p(\omega|\mathcal{D})$, given the training data $\mathcal{D}$. The predictive distribution $p(y|x, \mathcal{D})$ for a given input $x$ is computed via marginalisation over the posterior: $p(y|x, \mathcal{D}) = \mathbb{E}_{\omega \sim p(\omega|\mathcal{D})}[p(y|x, \omega)]$. Its predictive entropy $\mathbb{H}[Y|x, \mathcal{D}]$ upper-bounds the epistemic uncertainty, where epistemic uncertainty is quantified as the mutual information $\mathbb{I}[Y; \omega|x, \mathcal{D}]$ (\emph{expected information gain}) between parameters $\omega$ and output $y$ \citep{Gal2016Uncertainty, smith2018understanding}:
\begin{align}
    \underbrace{\mathbb{H}[Y|x, \mathcal{D}]}_\text{predictive} = \underbrace{\mathbb{I}[Y; \omega|x, \mathcal{D}]}_\text{epistemic} + \underbrace{\mathbb{E}_{\probc{\omega}{\data}} [\mathbb{H}[Y|x, \omega]]}_\text{aleatoric (for iD $x$)}.
    \label{eq:BALD}
\end{align}
Predictive uncertainty will be high whenever either epistemic uncertainty is high,
or when aleatoric uncertainty is high.
The intractability of exact Bayesian inference in deep learning has led to the development of methods for approximate inference \citep{hinton1993keeping, hernandez2015probabilistic, blundell2015weight, gal2016dropout}. In practice, however, these methods are either unable to scale to large datasets and model architectures, suffer from low uncertainty quality, or require expensive Monte-Carlo sampling.

\textbf{Deep Ensembles} 
are an ensemble of neural networks which average the models' softmax outputs. Uncertainty is then estimated as the entropy of this averaged softmax vector.
Despite incurring a high computational overhead at training and test time, Deep Ensembles, along with recent extensions \citep{smith2018understanding,wen2020batchensemble, dusenberry2020efficient} form the state-of-the-art in uncertainty quantification in deep learning.

\textbf{Deterministic Models} produce a softmax distribution $p(y|x, \omega)$, and commonly either the \emph{softmax confidence} $\max_c p(y=c|x, \omega)$ or the \emph{softmax entropy} $\mathbb{H}[Y|x, \omega]$ are used as a measure of uncertainty \citep{hendrycks2016baseline}. 
Popular approaches to improve these metrics include pre-processing of inputs and post-hoc calibration methods \citep{liang2017enhancing, guo2017calibration}, alternative objective functions \citep{lee2017training, devries2018learning}, and exposure to outliers \citep{hendrycks2018deep}. However, these methods suffer from several shortcomings including failing to perform under distribution shift \citep{ovadia2019can}, requiring significant changes to the training setup, and assuming the availability of OoD samples during training (which many applications do not have access to).

\textbf{Feature-Space Distances} \citep{lee2018simple, van2020simple, liu2020simple}  and \textbf{Feature-Space Density} \citep{postels2020quantifying, liu2020energy} offer a different approach for estimating uncertainty in deterministic models. Following the definition above, epistemic uncertainty must decrease when previously unseen samples are added to the training set, and
feature-space distance and density methods realise this by estimating distance or density, respectively, to training data in the feature space---see again \Cref{subfig:kendall_viz_gmm_density}.
A previously unseen point with high distance (low density), once added to the training data, will have low distance (high density).
Hence, they can be used as a proxy for epistemic uncertainty---under important assumptions about the feature space as detailed below.
None of these methods, however, is competitive with the state-of-the-art, Deep Ensembles, in uncertainty quantification, potentially for the reasons discussed next.

\textbf{Feature Collapse}
\citep{van2020simple} is a reason as to why distance and density estimation in the feature space may fail to capture epistemic uncertainty out of the box:
feature extractors  might map the features of OoD inputs to iD regions in the feature space \citep[c.f. Figure 2]{van2021improving}.

\textbf{Smoothness \& Sensitivity} can be encouraged to prevent feature collapse by subjecting the feature extractor $f_\theta$, with parameters $\theta$ to a \emph{bi-Lipschitz constraint}:
\begin{equation*}
    \label{eq:bi-lipschitz}
    K_L \; d_I(\mathbf{\mathrm{x}}_1, \mathbf{\mathrm{x}}_2) \leq d_F(f_\theta(\mathbf{\mathrm{x}}_1), f_\theta(\mathbf{\mathrm{x}}_2)) \leq K_U \; d_I(\mathbf{\mathrm{x}}_1, \mathbf{\mathrm{x}}_2),
\end{equation*}
for all inputs, $\mathbf{\mathrm{x}}_1$ and  $\mathbf{\mathrm{x}}_2$, where $d_I$ and $d_F$ denote metrics for the input and feature space respectively, and $K_L$ and $K_U$ the lower and upper Lipschitz constants \citep{liu2020simple}. 
The lower bound ensures \emph{sensitivity} to distances in the input space,
and the upper bound ensures \emph{smoothness} in the features, preventing them from becoming too sensitive to input variations, which, otherwise, can lead to poor generalisation and loss of robustness \citep{van2020simple}.
Methods of enouraging bi-Lipschitzness include: 
\textbf{i)} gradient penalty, by applying a two-sided penalty to the L2 norm of the Jacobian \citep{gulrajani2017improved}, and
\textbf{ii)} spectral normalisation \citep{miyato2018spectral} in models with residual connections, like ResNets \citep{he2016deep}.
\citet{smith2021convolutional} provide in-depth analysis which supports that spectral normalisation leads to bi-Lipschitzness.
Compared to the Jacobian gradient penalty used in \citep{van2020simple}, spectral normalisation is significantly faster and has more stable training dynamics.
Additionally, using a gradient penalty with residual connection leads to difficulties as discussed in \citep{liu2020simple}.

\vspace{-0.5em}
\section{{Deep Deterministic Uncertainty}}
\label{section:methods}

As introduced in \S\ref{sec:intro}, we propose to use a deterministic neural network with an appropriately regularized feature-space, using spectral normalization \citep{liu2020simple}, and to disentangle aleatoric and epistemic uncertainty by fitting a GDA after training without any additional steps (no hold-out ``OoD'' data, feature ensembling, or input pre-processing ala \citet{lee2018simple}).

\textbf{Ensuring Sensitivity \& Smoothness.} %
We ensure sensitivity and smoothness using spectral normalisation in models with residual connections. We make minor changes to the standard ResNet model architecture to further encourage sensitivity 
without sacrificing accuracy---details in \S\ref{app:more_model_architecture}.

\textbf{Disentangling Epistemic \& Aleatoric Uncertainty.} %
To quantify epistemic uncertainty, we fit a feature-space density estimator after training. We use GDA, a GMM $q(y, z)$ with a single Gaussian mixture component per class, and fit each class component by computing the empirical mean and covariance, per class, of the feature vectors $z=f_\theta(x)$, which are the outputs of the last convolutional layer of the model computed on the training samples $x$. \emph{Note that we do not require OoD data to fit these}.
Unlike the Expectation Maximization algorithm, this only requires a single pass through the training set given a trained model.

\textbf{Evaluation.} At test time, we estimate the epistemic uncertainty by evaluating the marginal likelihood of the feature representation under our density $\qprob{z} = \sum_y \qprob{z|y} \qprob{y}$.
To quantify aleatoric uncertainty for in-distribution samples, we use the entropy $\mathbb{H}[Y|x, \theta]$ of the softmax distribution $\probc{y}{x, \theta}$. %
Note that the softmax distribution thus obtained can be further calibrated using temperature scaling \citep{guo2017calibration}.
Thus, for a given input, a high feature-space density indicates low epistemic uncertainty (iD), and we can trust the aleatoric uncertainty and predictions estimated from the softmax layer.
The sample can then be either unambiguous (low softmax entropy) or ambiguous (high softmax entropy).
Conversely, a low feature-space density %
indicates high epistemic uncertainty (OoD), and we cannot trust the predictions.
The algorithm and a pseudo-code implementation %
can be found in \S\ref{app:implementation}.

\vspace{-0.5em}
\section{Experiments}
\label{sec:experiments}

We evaluate DDU's quality of epistemic uncertainty estimation in active learning \citep{cohn1996active} using MNIST, CIFAR-10 and an ambiguous version of MNIST (Dirty-MNIST). We also test DDU on challenging OoD detection settings including CIFAR-10 vs SVHN/CIFAR-100/Tiny-ImageNet/CIFAR-10-C, CIFAR-100 vs SVHN/Tiny-ImageNet and ImageNet vs ImageNet-O dataset pairings, where we outperform other deterministic single-forward-pass methods and perform on par with deep ensembles. In the appendix, we also examine DDU's performance on the well-known Two Moons toy dataset in \S\ref{app:2_moons}; we elaborate how DDU can disentangle epistemic and aleatoric uncertainty, the setting depicted in \Cref{fig:intro_histograms}, in \S\ref{app:experiments_disentangling}, and the effect of feature-space regularisation in \S\ref{app:feature_collapse_fig2}. 

\begin{figure}[!t]
    \centering
    \begin{minipage}{0.31\linewidth}
        \begin{subfigure}{\linewidth}%
            \centering
            \includegraphics[width=\linewidth]{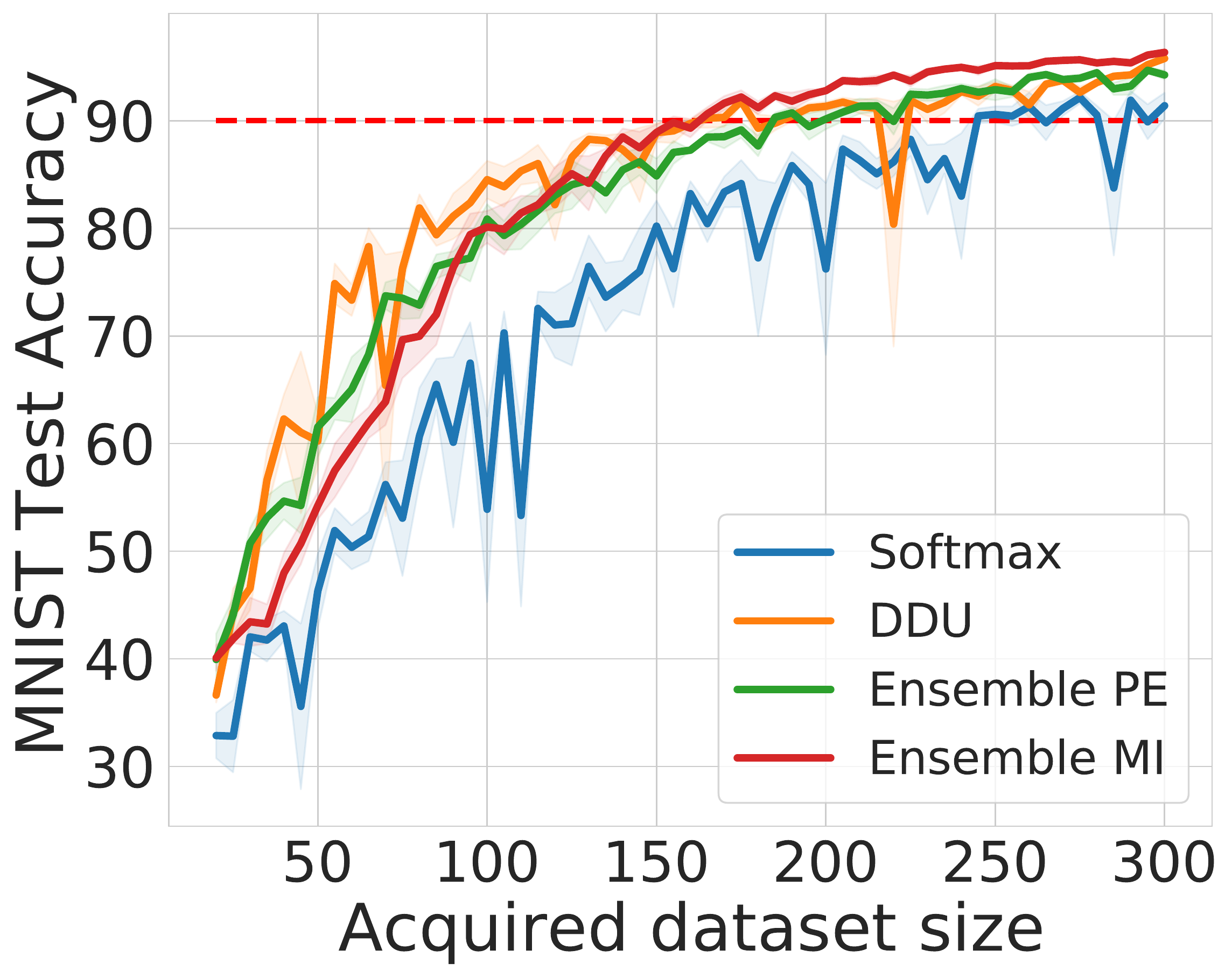}%
            \caption{MNIST}
            \label{subfig:test_accuracy_active_learning_mnist}
        \end{subfigure}
        \begin{subfigure}{\linewidth}%
            \centering
            \includegraphics[width=\linewidth]{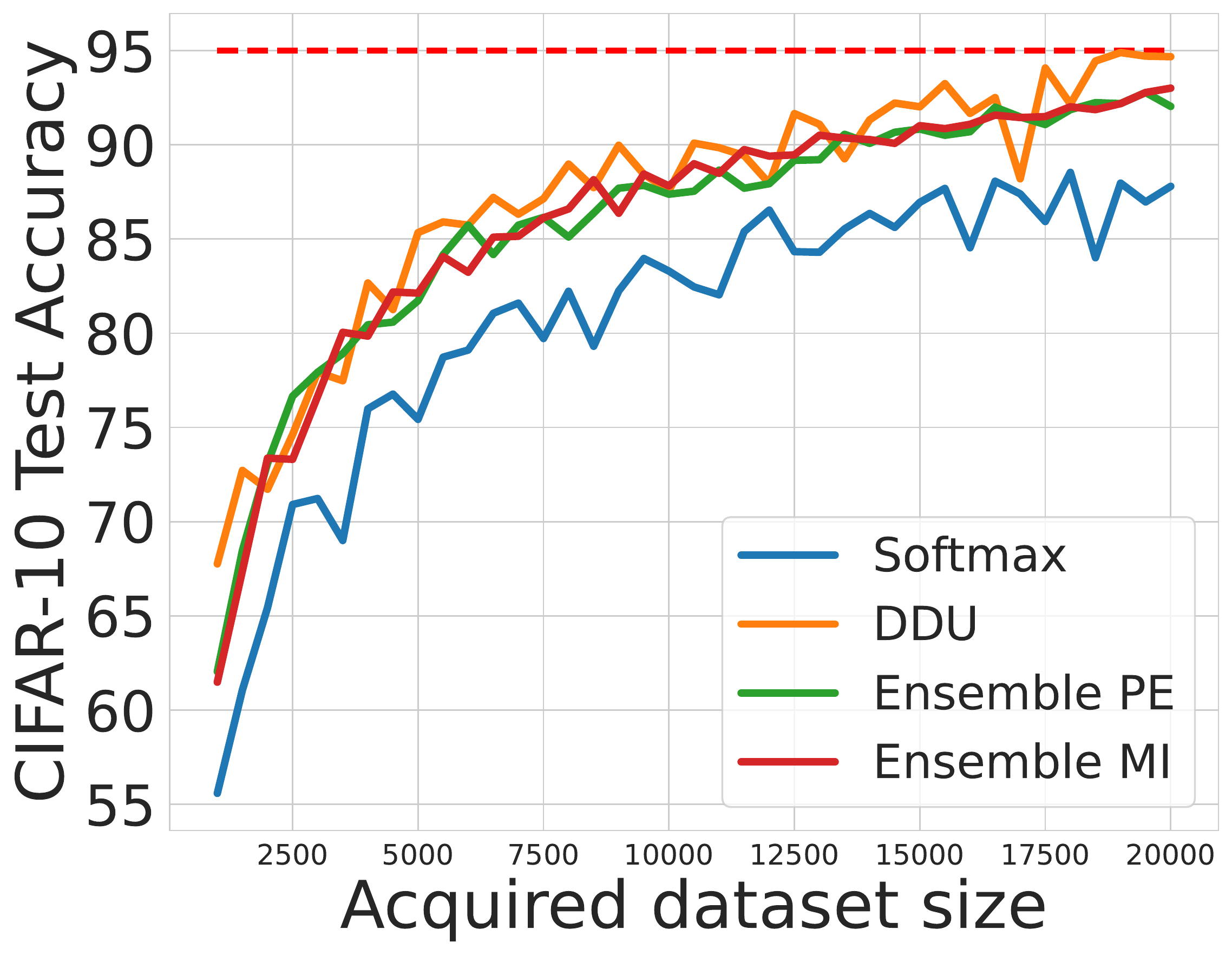}
            \caption{CIFAR-10}
            \label{subfig:test_accuracy_active_learning_cifar10}
        \end{subfigure}
    \end{minipage}
    \begin{subfigure}{0.68\linewidth}
        \centering
        \includegraphics[width=\linewidth]{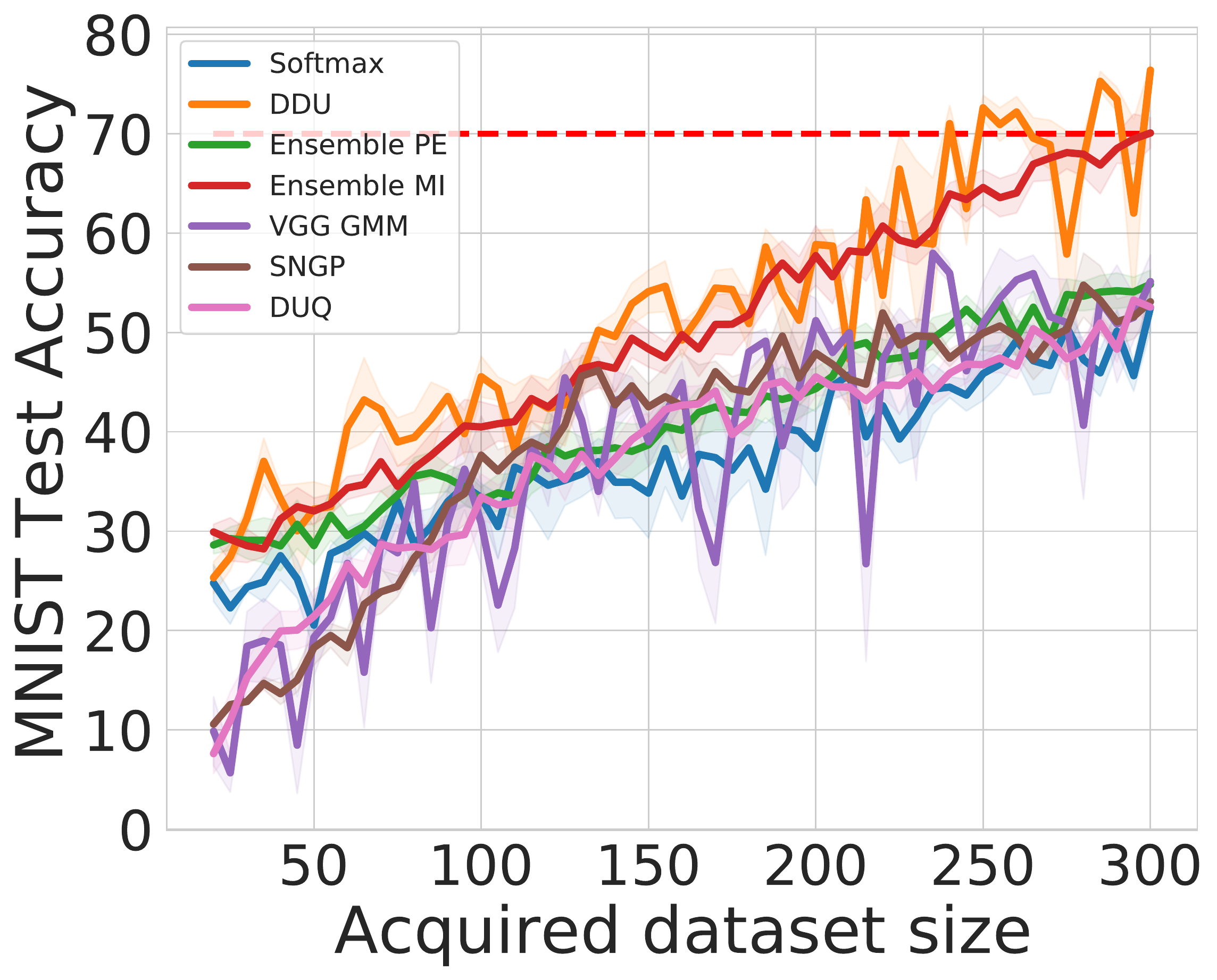}%
        \caption{Dirty-MNIST}
        \label{subfig:test_accuracy_active_learning_dirty_mnist_1000}
    \end{subfigure} 
    \caption{
    \emph{Active Learning experiments.} Acquired training set size vs test accuracy. DDU performs on par with Deep Ensembles.
    }
    \label{fig:active_learning_mnist_dirty_mnist}
    \vspace{-0.5em}
\end{figure}

\subsection{Active Learning}
\label{sec:experiments_active_learning}

We first demonstrate the quality of our uncertainty disentanglement in active learning (AL) \citep{cohn1996active}.
AL aims to train models in a data-efficient manner.
Additional training samples are iteratively acquired from a large pool of unlabelled data and labelled with the help of an expert. After each acquisition step, the model is retrained on the newly expanded training set. This is repeated until the model achieves a desirable accuracy---or when a maximum number of samples have been acquired.

Data-efficient acquisition relies on acquiring labels for the most informative samples.
This can be achieved by selecting points with high epistemic uncertainty \citep{gal2017deep}.
Conversely, repeated acquisition of points with high aleatoric uncertainty is not informative for the model and such acquisitions lead to data inefficiency.
AL, therefore, makes an excellent application for evaluating epistemic uncertainty and the ability of models to separate different sources of uncertainty.
We evaluate DDU on three different setups: \textbf{i)} with clean MNIST samples in the pool set, \textbf{ii)} with clean CIFAR-10 samples in the pool set, and \textbf{ii)} with Dirty-MNIST, having a 1:60 ratio of MNIST to Ambiguous-MNIST samples, in the pool set. In the first two setups, we compare 3 baselines: \textbf{i)} a ResNet-18 with softmax entropy as the acquisition function,
\textbf{ii)} DDU trained using a ResNet-18 with feature density as acquisition function,
and \textbf{iii)} a Deep Ensemble of 3 ResNet-18s with the predictive entropy (PE) and mutual information (MI) of the ensemble as the acquisition functions. In the last setup, in addition to the above 3 approaches, we also use \textbf{iv)} feature density of a VGG-16 instead of ResNet-18+SN as an ablation to see if feature density of a model without inductive biases performs well, \textbf{v)} SNGP and \textbf{vi)} DUQ as additional baselines.
For MNIST and Dirty-MNIST, we start with an initial training-set size of 20 randomly chosen MNIST points, and in each iteration, acquire the 5 samples with highest reported epistemic uncertainty. For each step, we train the models using Adam \citep{kingma2014adam} for 100 epochs and choose the one with the best validation set accuracy. We stop the process when the training set size reaches 300. For CIFAR-10, we start with 1000 samples and go up to 20000 samples with an acquisition size of 500 samples in each step.

\textbf{MNIST \& CIFAR-10} In \Cref{fig:active_learning_mnist_dirty_mnist}\subref{subfig:test_accuracy_active_learning_mnist} and \Cref{fig:active_learning_mnist_dirty_mnist}\subref{subfig:test_accuracy_active_learning_cifar10}, for regular curated MNIST and CIFAR-10 in the pool set, DDU clearly outperforms the deterministic softmax baseline and is competitive with Deep Ensembles. For MNIST, the softmax baseline reaches 90\% test-set accuracy at a training-set size of 245. DDU reaches 90\% accuracy at a training-set size of 160, whereas Deep Ensemble reaches the same at 185 and 155 training samples with PE and MI as the acquisition functions respectively. Note that DDU is three times faster than a Deep Ensemble, which needs to train three models independently after every acquisition.

\textbf{Dirty-MNIST.} Real-life datasets often contain observation noise and ambiguous samples.
What happens when the pool set contains a lot of such noisy samples having high aleatoric uncertainty?
In such cases, it becomes important for models to identify unseen and informative samples with high epistemic uncertainty and not with high aleatoric uncertainty.
To study this, we construct a pool set with samples from Dirty-MNIST (see \S\ref{app:dirty_mnist}). We significantly increase the proportion of ambiguous samples by using a 1:60 split of MNIST to Ambiguous-MNIST (a total of 1K MNIST and 60K Ambiguous-MNIST samples).
In \Cref{fig:active_learning_mnist_dirty_mnist}\subref{subfig:test_accuracy_active_learning_dirty_mnist_1000}, for Dirty-MNIST in the pool set, the difference in the performance of DDU and the deterministic softmax model is stark. While DDU achieves a test set accuracy of 70\% at a training set size of 240 samples, the accuracy of the softmax baseline peaks at a mere 50\%. In addition, all baselines, including SNGP, DUQ and the feature density of a VGG-16, which fail to solely capture epistemic uncertainty, are significantly outperformed by DDU and the MI baseline of the deep ensemble. However, note that DDU also performs better than Deep Ensembles with the PE acquisition function. The difference gets larger as the training set size grows: DDU's feature density and Deep Ensemble's MI solely capture epistemic uncertainty and hence, do not get confounded by iD ambiguous samples with high aleatoric uncertainty. 
\subsection{OoD Detection}
\label{sec:experiments_ood_detection}

\begin{table*}[!t]
    \centering
    \caption{\emph{OoD detection performance of different baselines using ResNet-50, Wide-ResNet-50-2 and VGG-16 architectures on ImageNet vs ImageNet-O \cite{hendrycks2021natural}.} Best AUROC scores are marked in bold.}
    \label{table:ood_imagenet}
    \resizebox{\linewidth}{!}{%
    \begin{tabular}{cccccccccc}
    \toprule
    \textbf{\small Model} & \multicolumn{2}{c}{{\small Accuracy (\textuparrow)}} & \multicolumn{2}{c}{{\small ECE (\textdownarrow)}} &
    \multicolumn{5}{c}{\textbf{\small AUROC (\textuparrow)}} \\
    & {\small Deterministic} & {\small 3-Ensemble} & {\small Deterministic} & {\small 3-Ensemble} & \textbf{\small Softmax Entropy} & \textbf{\small Energy-based Model} & \textbf{\small DDU} & \textbf{\small 3-Ensemble PE} & \textbf{\small 3-Ensemble MI} \\
    \midrule
    {\small ResNet-50} & $74.8\pm0.05$ & $76.01$ & $2.08\pm0.11$ & $2.07$ & $50.65\pm0.63$ & $53.88\pm0.80$ & $\mathbf{59.44\pm0.15}$ & $51.19$ & $55.46$ \\
    {\small Wide-ResNet-50-2} & $76.75\pm0.11$ & $77.58$ & $1.18\pm0.07$ & $1.22$ & $50.44\pm0.25$ & $54.92\pm0.43$ & $\mathbf{63.15\pm0.18}$ & $51.83$ & $57.98$ \\
    \midrule
    {\small VGG-16} & $72.48\pm0.02$ & $73.54$ & $2.62\pm0.11$ & $2.59$ & $50.51\pm0.25$ & $51.15\pm0.27$ & $51.73\pm0.21$ & $51.89$ & $\mathbf{56.56}$ \\
    \bottomrule
    \end{tabular}}
    \vspace{-0.5em} %
\end{table*}

OoD detection is an application of epistemic uncertainty quantification: if we do not train on OoD data, we expect OoD data points to have higher epistemic uncertainty than iD data.
We evaluate CIFAR-10 vs SVHN/CIFAR-100/Tiny-ImageNet/CIFAR-10-C, CIFAR-100 vs SVHN/Tiny-ImageNet and ImageNet vs ImageNet-O as iD vs OoD dataset pairs for this experiment \citep{krizhevsky2009learning, netzer2011reading,deng2009imagenet,hendrycks2019benchmarking}.
We also evaluate DDU on different architectures: Wide-ResNet-28-10, Wide-ResNet-50-2, ResNet-50, ResNet-110 and DenseNet-121 \citep{zagoruyko2016wide,he2016deep,huang2017densely}.
The training setup is described in \S\ref{app:exp_details_cifar}.
In addition to using softmax entropy of a deterministic model (\emph{Softmax}) for both aleatoric and epistemic uncertainty, we also compare with the following baselines that do not require training or fine-tuning on OoD data:
\begin{itemize}[leftmargin=*]
\item \emph{Energy-based model} \citep{liu2020energy}: We use the softmax entropy of a deterministic model as aleatoric uncertainty and the unnormalized softmax density (the logsumexp of the logits) as epistemic uncertainty \emph{without} regularisation to avoid feature collapse. We only compare with the version that does not train on OoD data.

\item \emph{DUQ \citep{van2020simple} \& SNGP \citep{liu2020simple}}: We compare with the state-of-the-art deterministic methods for uncertainty quantification including DUQ and SNGP. For SNGP, we use the exact predictive covariance computation and we use the entropy of the average of the MC softmax samples as uncertainty. For DUQ, we use the closest kernel distance. Note that for CIFAR-100, DUQ's one-vs-all objective did not converge during training and hence, we do not include the DUQ baseline for CIFAR-100.

\item \emph{5-Ensemble}: We use an ensemble of 5 networks
and compute the predictive entropy of the ensemble as both epistemic and aleatoric uncertainty and mutual information as epistemic uncertainty.

\end{itemize}
\Cref{table:ood_wrn} shows the AUROC scores for Wide-ResNet-28-10 based models on CIFAR-10 vs SVHN/CIFAR-100/Tiny-ImageNet and CIFAR-100 vs SVHN/Tiny-ImageNet along with their respective test set accuracy and test set ECE after temperature scaling. The equivalent results for the other architectures, ResNet-50, ResNet-110 and DenseNet-121 can be found in \Cref{table:ood_resnet50}, \Cref{table:ood_resnet110} and \Cref{table:ood_densenet121} respectively in the appendix. Note that for DDU, post-hoc calibration, e.g. in the form of temperature scaling \citep{guo2017calibration}, is straightforward as it does not affect the GMM density. In addition, we plot the AUROC averaged over all corruption types vs corruption intensity for CIFAR-10 vs CIFAR-10-C in \Cref{fig:cifar10_c_results}, with detailed AUROC plots for each corruption type in \Cref{fig:cifar10_c_results_wide_resnet}, \Cref{fig:cifar10_c_results_resnet50}, \Cref{fig:cifar10_c_results_resnet110} and \Cref{fig:cifar10_c_results_densenet121} of the appendix. Finally, in \Cref{table:ood_imagenet}, we present AUROC scores for models trained on ImageNet.

For OoD detection, \emph{DDU outperforms all other deterministic single-forward-pass methods, DUQ, SNGP and the energy-based model approach from \citet{liu2020energy}, on CIFAR-10 vs SVHN/CIFAR-100/Tiny-ImageNet, CIFAR-10 vs CIFAR-10-C and CIFAR-100 vs SVHN/Tiny-ImageNet, often performs on par with state-of-the-art Deep Ensembles---and even performing better in a few cases}. This holds true for all the architectures we experimented on. Similar observations can be made on ImageNet vs ImageNet-O as well.
Importantly, the great performance in OoD detection comes without compromising on the single-model test set accuracy in comparison to other deterministic methods.

Additional ablations for the CIFAR-10/100 experiments are detailed in \S\ref{app:additional_exp_results}: \Cref{table:ood_2} and \ref{table:ood_3}.
These tables along with observations in \Cref{table:ood_imagenet}, show that \emph{the feature density of a VGG-16 (i.e.\ without residual connections and spectral normalisation) is unable to beat a VGG-16 ensemble, whereas a Wide-ResNet-28-10 with spectral normalisation outperforms its corresponding ensemble in almost all the cases}. %
This result further validates the importance of having the bi-Lipschitz constraint (spectral normalisation) on the model to obtain smoothness and sensitivity. Finally, even without spectral normalisation, a Wide-ResNet-28 has the inductive bias of residual connections built into its model architecture, which can be a contributing factor towards good performance in general as residual connections already make the model sensitive to changes in the input space.

\begin{figure*}[t]
    \begin{subfigure}{0.59\linewidth}
        \centering
        \includegraphics[width=\linewidth]{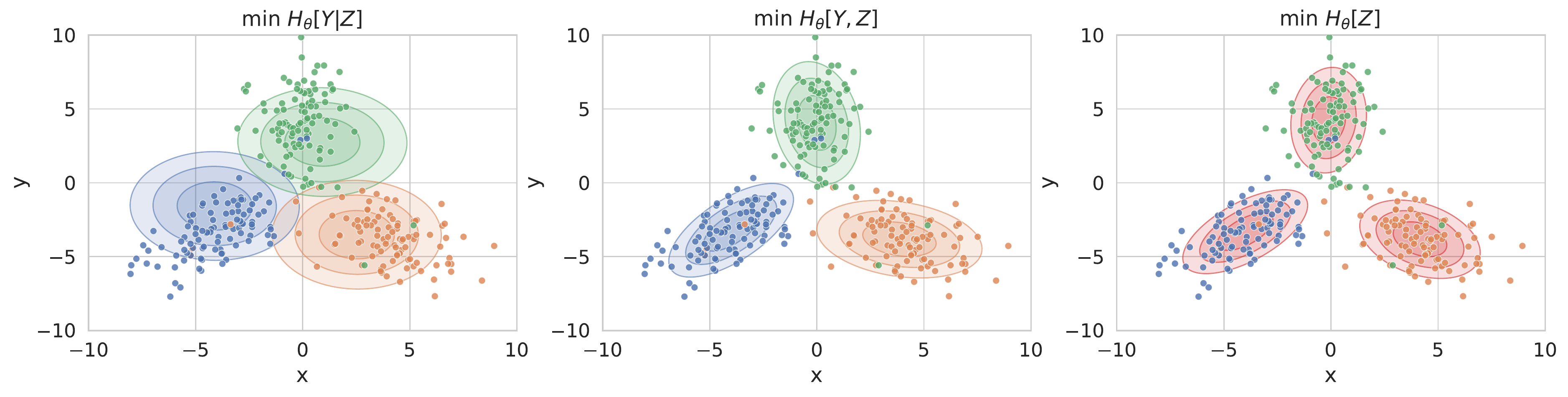}%
        \caption{
        \emph{Density.} Contours at 68.26\%, 95.44\%, and 99.7\%.}
        \label{fig:objective_mismatch_density}
    \end{subfigure}
    \hfill
    \begin{subfigure}{0.40\linewidth}
        \centering
        \includegraphics[width=\linewidth]{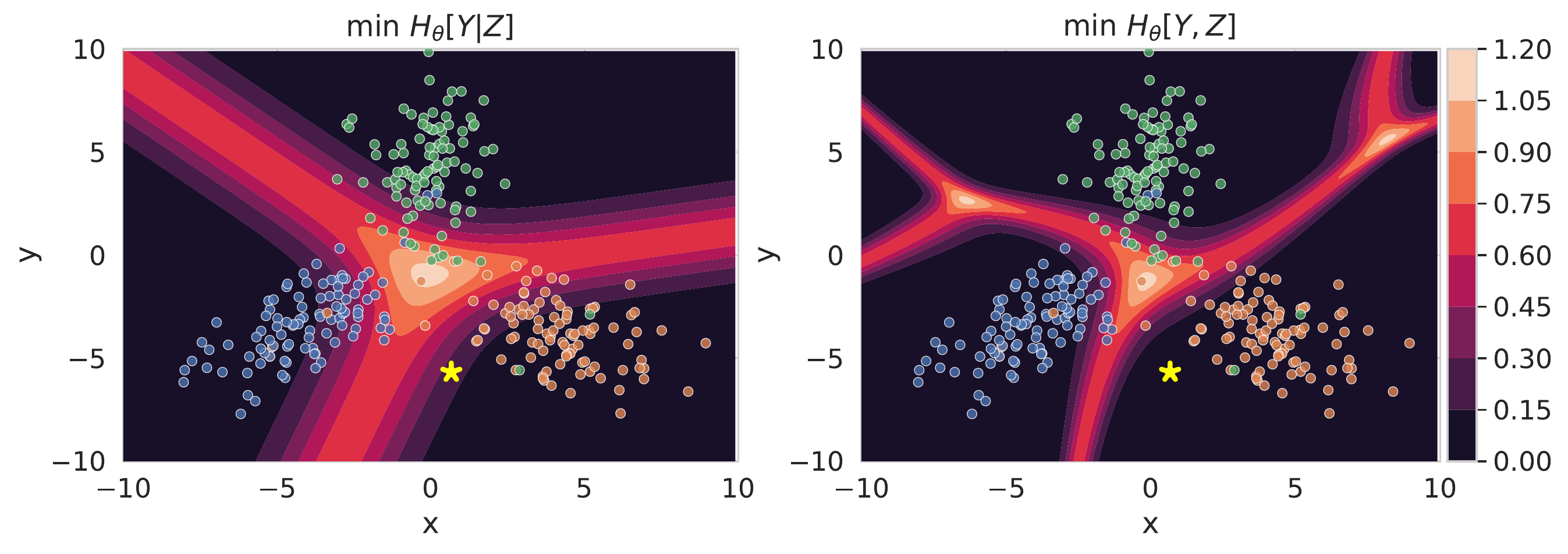}%
        \caption{\emph{Entropy.} Darker is lower.}
        \label{fig:objective_mismatch_entropy}
    \end{subfigure}
    \caption{\emph{3-component GMM fitted to a synthetic dataset with 3 different classes (differently colored) with 4\% label noise using different objectives}.
    \textbf{\subref{fig:objective_mismatch_density}:} The optimas for conditional log-likelihood $\Htc{Y}{Z}$, joint log-likelihood $\Ht{Y, Z}$, and marginalised log-likehood $\Ht{Z}$ all differ. Hence, the best calibrated model ($\Htc{Y}{Z}$) will not provide the best density estimate ($\Ht{Z}$), and vice-versa. 
    \textbf{\subref{fig:objective_mismatch_entropy}:} A mixture model that optimizes $\Ht{Y, Z}$ (GDA) does not have calibrated decision boundaries for aleatoric uncertainty: the ambiguous sample (due to label noise) marked by the yellow star has no aleatoric uncertainty under the GDA model. See \S\ref{app:objective_mismatch_toy_example} for details.
    }
    \label{fig:objective_mismatch}
    \vspace{-0.5mm}
\end{figure*}

\vspace{-3mm}
\section{{Additional Insights}}
\label{sec:motivation}

We elaborate on potential pitfalls of predictive entropy in general and softmax entropy of determinstic models in particular in the following section. Additional proofs for all statements in this section are provided in \S\ref{app:theory}.

\textbf{Potential Pitfalls of Predictive Entropy.} %
Conceptually, we note that \emph{predictive entropy confounds epistemic and aleatoric uncertainty}. 
Ensembling might also be seen as performing Bayesian Model Averaging \citep{He2020Bayesian, wilson2020bayesian}, as each ensemble member, producing a softmax output $\probc{y}{x, \omega}$, can be considered to be drawn from some distribution $\probc{\omega}{\data}$ over the trained model parameters $\omega$, which is induced by the pushforward of the  weight initialization under stochastic optimization. As a result, eq.~\eqref{eq:BALD} can also be applied to Deep Ensembles to disentangle epistemic from predictive uncertainty via computing the mutual information.
Both mutual information $\MIc{Y}{\omega}{x, \mathcal{D}}$ and predictive entropy $\Hc{Y}{x, \mathcal{D}}$ could be used to detect OoD samples. However, previous empirical findings show the predictive entropy outperforming mutual information \citep{malinin2018predictive}.
Indeed, much of recent literature only focuses on predictive entropy and the related confidence score for OoD detection (see \S\ref{app:entropy_confidence_lit_review} for a review).
This can be explained by the following observation:
\vspace{-0.75em}
\begin{observation}
When we \emph{already know} that aleatoric uncertainty or epistemic uncertainty is \emph{low} for a sample, predictive entropy is a good measure of the other quantity.
\end{observation}
\vspace{-0.5em}
Hence, the predictive entropy as an upper-bound of the mutual information can separate iD and OoD data better for curated datasets with low aleatoric uncertainty.
However, as seen in eq.~\eqref{eq:BALD}, predictive entropy can be high for both iD ambiguous samples (high aleatoric uncertainty) as well as for OoD samples (high epistemic uncertainty) (see \Cref{fig:lewis_vis}) and might \emph{not} be an effective measure for OoD detection when used with datasets that are not curated and ambiguous samples, like Dirty-MNIST in \Cref{fig:intro_histograms}. 

\textbf{Potential Pitfalls of Softmax Entropy.} The softmax entropy for deterministic models trained with maximum likelihood can be \emph{inconsistent}.
The mechanism underlying Deep Ensemble uncertainty that pushes epistemic uncertainty to be high on OoD data is the function disagreement between different ensemble components, i.e.\ arbitrary predictive extrapolations of the softmax models composing the ensemble:
\begin{restatable}{proposition}{ensemblesoftmax}
\label{pro:ensemble_softmax}
Let $x_1$ and $x_2$ be points such that $x_1$ has \textbf{higher} epistemic uncertainty than $x_2$ under the ensemble:
\begin{math}
\MIc{Y_1}{\omega}{x_1, \data} > \MIc{Y_2}{\omega}{x_2, \data} + \delta,
\end{math}
$\delta \ge 0$.
Further assume both have similar predictive entropy 
\begin{math}
    | \Hc{Y_1}{x_1, \data} - \Hc{Y_2}{x_2, \data} | \le \epsilon,
\end{math}
$\epsilon \ge 0$.
Then, there exist sets of ensemble members $\Omega$
with $\probc{\Omega}{\data} > 0$, such that for all softmax models $\omega \in \Omega$ the softmax entropy of $x_1$ is \textbf{lower} than the softmax entropy of $x_2$: $\Hc{Y_1}{x_1, \omega} < \Hc{Y_2}{x_2, \omega} - (\delta - \epsilon)$.
\end{restatable}
If a sample is assigned higher epistemic uncertainty (in the form of mutual information) by a Deep Ensemble than another sample, it will necessarily be assigned lower softmax entropy by at least one of the ensemble's members. 
As a result, a priori, we cannot know whether a softmax model preserves the order or not, and \emph{the empirical observation that the mutual information of an ensemble can quantify epistemic uncertainty well implies that the softmax entropy of a deterministic model might not}.
This can be seen in \Cref{fig:intro_softmax_ent}, \ref{fig:lewis_vis} and \S\ref{app:app_softmax_theory} where we observe the softmax entropy for OoD samples to have values which can be high, low or anywhere in between. Note \emph{not all} model architectures will behave like this, but when the mutual information of a corresponding Deep Ensemble works well empirically (for example in active learning), \Cref{pro:ensemble_softmax} holds.

\textbf{Objective Mismatch.} %
The predictive probability induced by a feature-density estimator will generally not be well-calibrated as there is an objective mismatch. %
This was overlooked in previous research on uncertainty quantification for deterministic models \citep{lee2018simple,liu2020simple,van2020simple,he2016deep, postels2020quantifying}.
Specifically, a mixture model $\qprob{y, z} = \sum_y \qprobc{z}{y} \qprob{y}$, using one component per class, cannot be optimal for both feature-space density and predictive distribution estimation as there is an \emph{objective mismatch}\footnote{This follows \citet[Ex. 4.20, p. 145]{murphy2012machine}.}:
\begin{restatable}{proposition}{objectivemismatch}
\label{pro:objectivemismatch}
For an input $x$, let $z=f_\theta(x)$ denote its feature representation in a feature extractor $f_\theta$ with parameters $\theta$. Then the following hold:
\begin{enumerate}[leftmargin=*]
    \item A discriminative classifier $\probc{y}{z}$, e.g. a softmax layer, is well-calibrated in its predictions when it maximises the conditional log-likelihood $\log \probc{y}{z}$;
    \item A feature-space density estimator $\qprob{z}$ is optimal when it maximises the marginalised log-likelihood $\log \qprob{z}$;
    \item A mixture model $\qprob{y, z} = \sum_y \qprobc{z}{y} \qprob{y}$ might not maximise both objectives, conditional log-likelihood and marginalised log-likelihood, at the same time. In the specific instance that a GMM with one component per class does maximise both, the resulting model must be a GDA (but the opposite does not hold).
\end{enumerate}
\end{restatable}
Hence, importantly, DDU uses \emph{both} a discriminative classifier (softmax layer) to capture aleatoric uncertainty for iD samples and a separate feature-density estimator to capture epistemic uncertainty even on a model trained using conditional log-likelihood, i.e.\ the usual cross-entropy objective. \Cref{fig:objective_mismatch} and \S\ref{app:objective_mismatch_toy_example} provide additional intuitions.

\section{Conclusion}
\label{sec:conclusion}
Deep Deterministic Uncertainty (DDU) can outperform state-of-the-art deterministic single-pass uncertainty methods in active learning and OoD detection by fitting a GDA for feature-space density estimation after training a model with residual connections and spectral normalization \citep{lee2018simple, liu2020simple}, and it manages to perform as well as deep ensembles in various settings.
Hence, DDU provides a very simple method to produce good epistemic and aleatoric uncertainty estimates and might be taken into consideration as an alternative to deep ensembles without requiring the complexities or computational cost of the current state-of-the-art. Reliable uncertainty quantification is an important requirement to make deep neural nets safe for deployment. Thus, we hope our work will contribute to increasing safety, reliability and trust in AI.

\newpage
\bibliography{ddu}

\begin{thebibliography}{64}
\providecommand{\natexlab}[1]{#1}
\providecommand{\url}[1]{\texttt{#1}}
\expandafter\ifx\csname urlstyle\endcsname\relax
  \providecommand{\doi}[1]{doi: #1}\else
  \providecommand{\doi}{doi: \begingroup \urlstyle{rm}\Url}\fi

\bibitem[Blundell et~al.(2015)Blundell, Cornebise, Kavukcuoglu, and
  Wierstra]{blundell2015weight}
Blundell, C., Cornebise, J., Kavukcuoglu, K., and Wierstra, D.
\newblock Weight uncertainty in neural network.
\newblock In \emph{International Conference on Machine Learning}, pp.\
  1613--1622. PMLR, 2015.

\bibitem[Buitinck et~al.(2013)Buitinck, Louppe, Blondel, Pedregosa, Mueller,
  Grisel, Niculae, Prettenhofer, Gramfort, Grobler, Layton, VanderPlas, Joly,
  Holt, and Varoquaux]{sklearn_api}
Buitinck, L., Louppe, G., Blondel, M., Pedregosa, F., Mueller, A., Grisel, O.,
  Niculae, V., Prettenhofer, P., Gramfort, A., Grobler, J., Layton, R.,
  VanderPlas, J., Joly, A., Holt, B., and Varoquaux, G.
\newblock {API} design for machine learning software: experiences from the
  scikit-learn project.
\newblock In \emph{ECML PKDD Workshop: Languages for Data Mining and Machine
  Learning}, pp.\  108--122, 2013.

\bibitem[Cohn et~al.(1996)Cohn, Ghahramani, and Jordan]{cohn1996active}
Cohn, D.~A., Ghahramani, Z., and Jordan, M.~I.
\newblock Active learning with statistical models.
\newblock \emph{Journal of artificial intelligence research}, 4:\penalty0
  129--145, 1996.

\bibitem[Cover(1999)]{cover1999elements}
Cover, T.~M.
\newblock \emph{Elements of information theory}.
\newblock John Wiley \& Sons, 1999.

\bibitem[Deng et~al.(2009)Deng, Dong, Socher, Li, Li, and
  Fei-Fei]{deng2009imagenet}
Deng, J., Dong, W., Socher, R., Li, L.-J., Li, K., and Fei-Fei, L.
\newblock Imagenet: A large-scale hierarchical image database.
\newblock In \emph{2009 IEEE conference on computer vision and pattern
  recognition}, pp.\  248--255. Ieee, 2009.

\bibitem[Der~Kiureghian \& Ditlevsen(2009)Der~Kiureghian and
  Ditlevsen]{der2009aleatory}
Der~Kiureghian, A. and Ditlevsen, O.
\newblock Aleatory or epistemic? does it matter?
\newblock \emph{Structural safety}, 31\penalty0 (2):\penalty0 105--112, 2009.

\bibitem[DeVries \& Taylor(2018)DeVries and Taylor]{devries2018learning}
DeVries, T. and Taylor, G.~W.
\newblock Learning confidence for out-of-distribution detection in neural
  networks.
\newblock \emph{arXiv preprint arXiv:1802.04865}, 2018.

\bibitem[Dinh et~al.(2015)Dinh, Krueger, and Bengio]{dinh2015nice}
Dinh, L., Krueger, D., and Bengio, Y.
\newblock Nice: Non-linear independent components estimation, 2015.

\bibitem[Dusenberry et~al.(2020)Dusenberry, Jerfel, Wen, Ma, Snoek, Heller,
  Lakshminarayanan, and Tran]{dusenberry2020efficient}
Dusenberry, M., Jerfel, G., Wen, Y., Ma, Y., Snoek, J., Heller, K.,
  Lakshminarayanan, B., and Tran, D.
\newblock Efficient and scalable bayesian neural nets with rank-1 factors.
\newblock In \emph{International conference on machine learning}, pp.\
  2782--2792. PMLR, 2020.

\bibitem[Esteva et~al.(2017)Esteva, Kuprel, Novoa, Ko, Swetter, Blau, and
  Thrun]{esteva2017dermatologist}
Esteva, A., Kuprel, B., Novoa, R.~A., Ko, J., Swetter, S.~M., Blau, H.~M., and
  Thrun, S.
\newblock Dermatologist-level classification of skin cancer with deep neural
  networks.
\newblock \emph{nature}, 542\penalty0 (7639):\penalty0 115--118, 2017.

\bibitem[Filos et~al.(2019)Filos, Farquhar, Gomez, Rudner, Kenton, Smith,
  Alizadeh, de~Kroon, and Gal]{filos2019systematic}
Filos, A., Farquhar, S., Gomez, A.~N., Rudner, T.~G., Kenton, Z., Smith, L.,
  Alizadeh, M., de~Kroon, A., and Gal, Y.
\newblock A systematic comparison of bayesian deep learning robustness in
  diabetic retinopathy tasks.
\newblock \emph{arXiv preprint arXiv:1912.10481}, 2019.

\bibitem[Gal(2016)]{Gal2016Uncertainty}
Gal, Y.
\newblock \emph{Uncertainty in Deep Learning}.
\newblock PhD thesis, University of Cambridge, 2016.

\bibitem[Gal \& Ghahramani(2016)Gal and Ghahramani]{gal2016dropout}
Gal, Y. and Ghahramani, Z.
\newblock Dropout as a bayesian approximation: Representing model uncertainty
  in deep learning.
\newblock In \emph{international conference on machine learning}, pp.\
  1050--1059, 2016.

\bibitem[Gal et~al.(2017)Gal, Islam, and Ghahramani]{gal2017deep}
Gal, Y., Islam, R., and Ghahramani, Z.
\newblock Deep bayesian active learning with image data.
\newblock In \emph{International Conference on Machine Learning}, pp.\
  1183--1192. PMLR, 2017.

\bibitem[Gneiting \& Raftery(2007)Gneiting and Raftery]{gneiting2007strictly}
Gneiting, T. and Raftery, A.~E.
\newblock Strictly proper scoring rules, prediction, and estimation.
\newblock \emph{Journal of the American statistical Association}, 102\penalty0
  (477):\penalty0 359--378, 2007.

\bibitem[Gouk et~al.(2021)Gouk, Frank, Pfahringer, and
  Cree]{gouk2018regularisation}
Gouk, H., Frank, E., Pfahringer, B., and Cree, M.~J.
\newblock Regularisation of neural networks by enforcing lipschitz continuity.
\newblock \emph{Machine Learning}, 110\penalty0 (2):\penalty0 393--416, 2021.

\bibitem[Gulrajani et~al.(2017)Gulrajani, Ahmed, Arjovsky, Dumoulin, and
  Courville]{gulrajani2017improved}
Gulrajani, I., Ahmed, F., Arjovsky, M., Dumoulin, V., and Courville, A.~C.
\newblock Improved training of wasserstein gans.
\newblock In \emph{NeurIPS}, 2017.

\bibitem[Guo et~al.(2017)Guo, Pleiss, Sun, and Weinberger]{guo2017calibration}
Guo, C., Pleiss, G., Sun, Y., and Weinberger, K.~Q.
\newblock On calibration of modern neural networks.
\newblock \emph{arXiv preprint arXiv:1706.04599}, 2017.

\bibitem[He et~al.(2020)He, Lakshminarayanan, and Teh]{He2020Bayesian}
He, B., Lakshminarayanan, B., and Teh, Y.~W.
\newblock {Bayesian Deep Ensembles via the Neural Tangent Kernel}.
\newblock In \emph{Advances in neural information processing systems}, 2020.

\bibitem[He et~al.(2016)He, Zhang, Ren, and Sun]{he2016deep}
He, K., Zhang, X., Ren, S., and Sun, J.
\newblock Deep residual learning for image recognition.
\newblock In \emph{Proceedings of the IEEE conference on computer vision and
  pattern recognition}, pp.\  770--778, 2016.

\bibitem[Hendrycks \& Dietterich(2019)Hendrycks and
  Dietterich]{hendrycks2019benchmarking}
Hendrycks, D. and Dietterich, T.
\newblock Benchmarking neural network robustness to common corruptions and
  perturbations.
\newblock \emph{arXiv preprint arXiv:1903.12261}, 2019.

\bibitem[Hendrycks \& Gimpel(2016)Hendrycks and Gimpel]{hendrycks2016baseline}
Hendrycks, D. and Gimpel, K.
\newblock A baseline for detecting misclassified and out-of-distribution
  examples in neural networks.
\newblock \emph{arXiv preprint arXiv:1610.02136}, 2016.

\bibitem[Hendrycks et~al.(2018)Hendrycks, Mazeika, and
  Dietterich]{hendrycks2018deep}
Hendrycks, D., Mazeika, M., and Dietterich, T.
\newblock Deep anomaly detection with outlier exposure.
\newblock In \emph{International Conference on Learning Representations}, 2018.

\bibitem[Hendrycks et~al.(2021)Hendrycks, Zhao, Basart, Steinhardt, and
  Song]{hendrycks2021natural}
Hendrycks, D., Zhao, K., Basart, S., Steinhardt, J., and Song, D.
\newblock Natural adversarial examples.
\newblock In \emph{Proceedings of the IEEE/CVF Conference on Computer Vision
  and Pattern Recognition}, pp.\  15262--15271, 2021.

\bibitem[Hern{\'a}ndez-Lobato \& Adams(2015)Hern{\'a}ndez-Lobato and
  Adams]{hernandez2015probabilistic}
Hern{\'a}ndez-Lobato, J.~M. and Adams, R.
\newblock Probabilistic backpropagation for scalable learning of bayesian
  neural networks.
\newblock In \emph{International Conference on Machine Learning}, pp.\
  1861--1869. PMLR, 2015.

\bibitem[Hinton \& Van~Camp(1993)Hinton and Van~Camp]{hinton1993keeping}
Hinton, G.~E. and Van~Camp, D.
\newblock Keeping the neural networks simple by minimizing the description
  length of the weights.
\newblock In \emph{Proceedings of the sixth annual conference on Computational
  learning theory}, pp.\  5--13, 1993.

\bibitem[Hsu et~al.(2020)Hsu, Shen, Jin, and Kira]{hsu2020generalized}
Hsu, Y.-C., Shen, Y., Jin, H., and Kira, Z.
\newblock Generalized odin: Detecting out-of-distribution image without
  learning from out-of-distribution data.
\newblock In \emph{CVPR}, 2020.

\bibitem[Huang et~al.(2017)Huang, Liu, Van Der~Maaten, and
  Weinberger]{huang2017densely}
Huang, G., Liu, Z., Van Der~Maaten, L., and Weinberger, K.~Q.
\newblock Densely connected convolutional networks.
\newblock In \emph{Proceedings of the IEEE conference on computer vision and
  pattern recognition}, pp.\  4700--4708, 2017.

\bibitem[Huang \& Chen(2020)Huang and Chen]{huang2020autonomous}
Huang, Y. and Chen, Y.
\newblock Autonomous driving with deep learning: A survey of state-of-art
  technologies.
\newblock \emph{arXiv preprint arXiv:2006.06091}, 2020.

\bibitem[Kendall \& Gal(2017)Kendall and Gal]{kendall2017uncertainties}
Kendall, A. and Gal, Y.
\newblock What uncertainties do we need in bayesian deep learning for computer
  vision?
\newblock In \emph{Advances in neural information processing systems}, pp.\
  5574--5584, 2017.

\bibitem[Kingma \& Ba(2015)Kingma and Ba]{kingma2014adam}
Kingma, D.~P. and Ba, J.
\newblock Adam: A method for stochastic optimization.
\newblock In \emph{ICLR}, 2015.

\bibitem[Kingma \& Welling(2014)Kingma and Welling]{kingma2013auto}
Kingma, D.~P. and Welling, M.
\newblock Auto-encoding variational bayes.
\newblock In \emph{ICLR}, 2014.

\bibitem[Kirsch et~al.(2019)Kirsch, van Amersfoort, and
  Gal]{kirsch2019batchbald}
Kirsch, A., van Amersfoort, J., and Gal, Y.
\newblock Batchbald: Efficient and diverse batch acquisition for deep bayesian
  active learning.
\newblock In \emph{Advances in Neural Information Processing Systems}, 2019.

\bibitem[Kirsch et~al.(2020)Kirsch, Lyle, and Gal]{kirsch2020unpacking}
Kirsch, A., Lyle, C., and Gal, Y.
\newblock Unpacking information bottlenecks: Unifying information-theoretic
  objectives in deep learning.
\newblock \emph{arXiv preprint arXiv:2003.12537}, 2020.

\bibitem[Kristiadi et~al.(2020)Kristiadi, Hein, and
  Hennig]{Kristiadi2020BeingBE}
Kristiadi, A., Hein, M., and Hennig, P.
\newblock Being bayesian, even just a bit, fixes overconfidence in relu
  networks.
\newblock In \emph{ICML}, 2020.

\bibitem[Krizhevsky et~al.(2009)Krizhevsky, Hinton,
  et~al.]{krizhevsky2009learning}
Krizhevsky, A., Hinton, G., et~al.
\newblock Learning multiple layers of features from tiny images.
\newblock 2009.

\bibitem[Lakshminarayanan et~al.(2017)Lakshminarayanan, Pritzel, and
  Blundell]{lakshminarayanan2017simple}
Lakshminarayanan, B., Pritzel, A., and Blundell, C.
\newblock Simple and scalable predictive uncertainty estimation using deep
  ensembles.
\newblock In \emph{Advances in neural information processing systems}, pp.\
  6402--6413, 2017.

\bibitem[LeCun et~al.(1998)LeCun, Bottou, Bengio, and
  Haffner]{lecun1998gradient}
LeCun, Y., Bottou, L., Bengio, Y., and Haffner, P.
\newblock Gradient-based learning applied to document recognition.
\newblock \emph{Proceedings of the IEEE}, 86\penalty0 (11):\penalty0
  2278--2324, 1998.

\bibitem[Lee et~al.(2018{\natexlab{a}})Lee, Lee, Lee, and
  Shin]{lee2017training}
Lee, K., Lee, H., Lee, K., and Shin, J.
\newblock Training confidence-calibrated classifiers for detecting
  out-of-distribution samples.
\newblock In \emph{International Conference on Learning Representations},
  2018{\natexlab{a}}.

\bibitem[Lee et~al.(2018{\natexlab{b}})Lee, Lee, Lee, and Shin]{lee2018simple}
Lee, K., Lee, K., Lee, H., and Shin, J.
\newblock A simple unified framework for detecting out-of-distribution samples
  and adversarial attacks.
\newblock In \emph{NeurIPS}, 2018{\natexlab{b}}.

\bibitem[Liang et~al.(2018)Liang, Li, and Srikant]{liang2017enhancing}
Liang, S., Li, Y., and Srikant, R.
\newblock Enhancing the reliability of out-of-distribution image detection in
  neural networks.
\newblock In \emph{International Conference on Learning Representations}, 2018.

\bibitem[Liu et~al.(2020{\natexlab{a}})Liu, Lin, Padhy, Tran, Bedrax-Weiss, and
  Lakshminarayanan]{liu2020simple}
Liu, J.~Z., Lin, Z., Padhy, S., Tran, D., Bedrax-Weiss, T., and
  Lakshminarayanan, B.
\newblock Simple and principled uncertainty estimation with deterministic deep
  learning via distance awareness.
\newblock In \emph{NeurIPS}, 2020{\natexlab{a}}.

\bibitem[Liu et~al.(2020{\natexlab{b}})Liu, Wang, Owens, and Li]{liu2020energy}
Liu, W., Wang, X., Owens, J., and Li, Y.
\newblock Energy-based out-of-distribution detection.
\newblock \emph{Advances in Neural Information Processing Systems}, 33,
  2020{\natexlab{b}}.

\bibitem[MacKay(1992)]{mackay1992bayesian}
MacKay, D.~J.
\newblock \emph{Bayesian methods for adaptive models}.
\newblock PhD thesis, California Institute of Technology, 1992.

\bibitem[Malinin \& Gales(2018)Malinin and Gales]{malinin2018predictive}
Malinin, A. and Gales, M.~J.
\newblock Predictive uncertainty estimation via prior networks.
\newblock In \emph{NeurIPS}, 2018.

\bibitem[Malinin et~al.(2020)Malinin, Mlodozeniec, and
  Gales]{Malinin2020EnsembleDD}
Malinin, A., Mlodozeniec, B., and Gales, M. J.~F.
\newblock Ensemble distribution distillation.
\newblock \emph{ArXiv}, abs/1905.00076, 2020.

\bibitem[Miyato et~al.(2018)Miyato, Kataoka, Koyama, and
  Yoshida]{miyato2018spectral}
Miyato, T., Kataoka, T., Koyama, M., and Yoshida, Y.
\newblock Spectral normalization for generative adversarial networks.
\newblock In \emph{International Conference on Learning Representations}, 2018.

\bibitem[Murphy(2012)]{murphy2012machine}
Murphy, K.~P.
\newblock \emph{Machine learning: a probabilistic perspective}.
\newblock MIT press, 2012.

\bibitem[Neal(2012)]{neal2012bayesian}
Neal, R.~M.
\newblock \emph{Bayesian learning for neural networks}, volume 118.
\newblock Springer Science \& Business Media, 2012.

\bibitem[Netzer et~al.(2011)Netzer, Wang, Coates, Bissacco, Wu, and
  Ng]{netzer2011reading}
Netzer, Y., Wang, T., Coates, A., Bissacco, A., Wu, B., and Ng, A.~Y.
\newblock Reading digits in natural images with unsupervised feature learning.
\newblock 2011.

\bibitem[Ovadia et~al.(2019)Ovadia, Fertig, Ren, Nado, Sculley, Nowozin,
  Dillon, Lakshminarayanan, and Snoek]{ovadia2019can}
Ovadia, Y., Fertig, E., Ren, J., Nado, Z., Sculley, D., Nowozin, S., Dillon,
  J., Lakshminarayanan, B., and Snoek, J.
\newblock Can you trust your model's uncertainty? evaluating predictive
  uncertainty under dataset shift.
\newblock In \emph{Advances in Neural Information Processing Systems}, pp.\
  13991--14002, 2019.

\bibitem[Pearce et~al.(2021)Pearce, Brintrup, and Zhu]{pearce2021understanding}
Pearce, T., Brintrup, A., and Zhu, J.
\newblock Understanding softmax confidence and uncertainty, 2021.

\bibitem[Postels et~al.(2020)Postels, Blum, Cadena, Siegwart, Van~Gool, and
  Tombari]{postels2020quantifying}
Postels, J., Blum, H., Cadena, C., Siegwart, R., Van~Gool, L., and Tombari, F.
\newblock Quantifying aleatoric and epistemic uncertainty using density
  estimation in latent space.
\newblock \emph{arXiv preprint arXiv:2012.03082}, 2020.

\bibitem[Simonyan \& Zisserman(2015)Simonyan and Zisserman]{simonyan2014very}
Simonyan, K. and Zisserman, A.
\newblock Very deep convolutional networks for large-scale image recognition.
\newblock In \emph{International Conference on Learning Representations}, 2015.

\bibitem[Smith \& Gal(2018)Smith and Gal]{smith2018understanding}
Smith, L. and Gal, Y.
\newblock {Understanding Measures of Uncertainty for Adversarial Example
  Detection}.
\newblock In \emph{UAI}, 2018.

\bibitem[Smith et~al.(2021)Smith, van Amersfoort, Huang, Roberts, and
  Gal]{smith2021convolutional}
Smith, L., van Amersfoort, J., Huang, H., Roberts, S., and Gal, Y.
\newblock Can convolutional resnets approximately preserve input distances? a
  frequency analysis perspective, 2021.

\bibitem[Sun et~al.(2017)Sun, Shrivastava, Singh, and Gupta]{sun2017revisiting}
Sun, C., Shrivastava, A., Singh, S., and Gupta, A.
\newblock Revisiting unreasonable effectiveness of data in deep learning era.
\newblock In \emph{Proceedings of the IEEE international conference on computer
  vision}, pp.\  843--852, 2017.

\bibitem[van Amersfoort et~al.(2020)van Amersfoort, Smith, Teh, and
  Gal]{van2020simple}
van Amersfoort, J., Smith, L., Teh, Y.~W., and Gal, Y.
\newblock Uncertainty estimation using a single deep deterministic neural
  network.
\newblock In \emph{International Conference on Machine Learning}, pp.\
  9690--9700. PMLR, 2020.

\bibitem[van Amersfoort et~al.(2021)van Amersfoort, Smith, Jesson, Key, and
  Gal]{van2021improving}
van Amersfoort, J., Smith, L., Jesson, A., Key, O., and Gal, Y.
\newblock Improving deterministic uncertainty estimation in deep learning for
  classification and regression.
\newblock \emph{arXiv preprint arXiv:2102.11409}, 2021.

\bibitem[Wen et~al.(2019)Wen, Tran, and Ba]{wen2020batchensemble}
Wen, Y., Tran, D., and Ba, J.
\newblock Batchensemble: an alternative approach to efficient ensemble and
  lifelong learning.
\newblock In \emph{International Conference on Learning Representations}, 2019.

\bibitem[Wilson \& Izmailov(2020)Wilson and Izmailov]{wilson2020bayesian}
Wilson, A.~G. and Izmailov, P.
\newblock Bayesian deep learning and a probabilistic perspective of
  generalization.
\newblock \emph{arXiv preprint arXiv:2002.08791}, 2020.

\bibitem[Winkens et~al.(2020)Winkens, Bunel, Roy, Stanforth, Natarajan, Ledsam,
  MacWilliams, Kohli, Karthikesalingam, Kohl, et~al.]{winkens2020contrastive}
Winkens, J., Bunel, R., Roy, A.~G., Stanforth, R., Natarajan, V., Ledsam,
  J.~R., MacWilliams, P., Kohli, P., Karthikesalingam, A., Kohl, S., et~al.
\newblock Contrastive training for improved out-of-distribution detection.
\newblock \emph{arXiv preprint arXiv:2007.05566}, 2020.

\bibitem[Xiao et~al.(2017)Xiao, Rasul, and Vollgraf]{xiao2017fashion}
Xiao, H., Rasul, K., and Vollgraf, R.
\newblock Fashion-mnist: a novel image dataset for benchmarking machine
  learning algorithms.
\newblock \emph{arXiv preprint arXiv:1708.07747}, 2017.

\bibitem[Zagoruyko \& Komodakis(2016)Zagoruyko and
  Komodakis]{zagoruyko2016wide}
Zagoruyko, S. and Komodakis, N.
\newblock Wide residual networks.
\newblock \emph{arXiv preprint arXiv:1605.07146}, 2016.

\end{thebibliography}

\clearpage

\appendix
\onecolumn

\section{Related Work}
\label{sec:related work}
Several existing approaches model uncertainty using feature-space density but underperform without fine-tuning on OoD data. Our work identifies feature collapse and objective mismatch as possible reasons for this.
Among these approaches, \citet{lee2018simple} uses Mahalanobis distances to quantify uncertainty by fitting a class-wise Gaussian distribution (with shared covariance matrices) on the feature space of a pre-trained ResNet
encoder. The competitive results they report require input perturbations, ensembling GMM densities from multiple layers, and fine-tuning on OoD hold-out data.
They do not discuss any constraints which the ResNet encoder should satisfy,
and therefore,
are vulnerable to feature collapse. %
In \Cref{fig:intro_gmm}, for example, the feature density of a LeNet and a VGG are unable to distinguish OoD from iD samples. 
\citet{postels2020quantifying} also propose a density-based estimation of aleatoric and epistemic uncertainty. Similar to \citet{lee2018simple}, they do not constrain their pre-trained ResNet encoder.
They do discuss
feature collapse though, noting that they do not address this problem. 
Moreover, they do not consider the objective mismatch that arises (see \Cref{pro:objectivemismatch} below) and use a single estimator for both epistemic and aleatoric uncertainty. Consequently, they report worse epistemic uncertainty: 74\% AUROC on CIFAR-10 vs SVHN, which we show to considerably fall behind modern approaches for uncertainty estimation in deep learning in \S\ref{sec:experiments}.
Likewise, \citet{liu2020energy} compute an unnormalized density based on the softmax logits %
without taking into account the need for inductive biases to ensure smoothness and sensitivity of the feature space.

\citet{winkens2020contrastive} use contrastive training on the feature extractor before estimating the feature-space density.
Our method is orthogonal from this work as we restrict ourselves to the supervised setting and show that the inductive biases that result in bi-Lipschitzness
\citep{van2020simple, liu2020simple} are sufficient for the feature-space density to reliably capture epistemic uncertainty.

Lastly, our method improves upon \citet{van2020simple} and \citet{liu2020simple} by alleviating the need for additional hyperparameters:
DDU only needs minimal changes from the standard softmax setup to outperform DUQ and SNGP on uncertainty benchmarks, and our GMM parameters are optimised for the already trained model using the training set. In particular, DDU does not require training or fine-tuning with OoD data.
Moreover, our insights in \S\ref{sec:motivation} explain why \citet{liu2020simple} found that a baseline that uses \emph{the softmax entropy instead of the feature-space density} of a deterministic network with bi-Lipschitz constraint underperforms.

\subsection{Predictive Entropy and Confidence in Recent Works}
\label{app:entropy_confidence_lit_review}

\Cref{table:confidence_entropy_lit_review} shows a selection of recently published papers which use entropy or confidence as OoD score. Only two papers examine using Mutual Information with Deep Ensembles as OoD score at all. None of the papers examines the possible confounding of aleatoric and epistemic uncertainty when using predictive entropy or confidence, or the consistency issues of softmax entropy (and softmax confidence), detailed in \S\ref{sec:motivation}. This list is not exhaustive, of course.

\begin{table}[!t]

\centering
\caption{\emph{A sample of recently published papers and OoD metrics.} Many recently published papers only use Predictive Entropy or Predictive Confidence  (for Deep Ensembles) or Softmax Confidence (for deterministic models) as OoD scores without addressing the possible confounding of aleatoric and epistemic uncertainty, that is ambiguous iD samples with OoD samples. Only two papers examine using Mutual Information with Deep Ensembles as OoD score at all.}
\label{table:confidence_entropy_lit_review}
\resizebox{\linewidth}{!}{%
\renewcommand{\arraystretch}{1.3} 
\begin{tabular}{m{30em} lrrrr}
\toprule
Title & Citation & Softmax Confidence & Predictive Confidence & Predictive Entropy & Mutual Information \\
\midrule
A Baseline for Detecting Misclassified and Out-of-Distribution Examples in Neural Networks & \citet{hendrycks2016baseline} & \cmark & \xmark & \xmark & \xmark \\
Deep Anomaly Detection with Outlier Exposure & \citet{hendrycks2018deep} & \cmark & \xmark & \xmark & \xmark \\
Enhancing The Reliability of Out-of-distribution Image Detection in Neural Networks & \citet{liang2017enhancing} & \cmark & \xmark & \xmark & \xmark \\
Training Confidence-calibrated Classifiers for Detecting Out-of-Distribution Samples & \citet{lee2017training} & \cmark & \xmark & \xmark & \xmark \\
Learning Confidence for Out-of-Distribution Detection in Neural Networks & \citet{devries2018learning} & \cmark & \xmark & \xmark & \xmark \\
Simple and Scalable Predictive Uncertainty Estimation using Deep Ensembles & \citet{lakshminarayanan2017simple} & \xmark & \cmark & \cmark & \xmark \\
Predictive Uncertainty Estimation via Prior Networks & \citet{malinin2018predictive} & \xmark & \cmark & \cmark & \cmark \\
Ensemble Distribution Distillation & \citet{Malinin2020EnsembleDD} & \xmark & \xmark & \cmark & \cmark \\
Generalized ODIN: Detecting Out-of-Distribution Image Without Learning From Out-of-Distribution Data & \citet{hsu2020generalized} & \cmark & \xmark & \xmark & \xmark \\
Being Bayesian, Even Just a Bit, Fixes Overconfidence in ReLU Networks & \citet{Kristiadi2020BeingBE} & \xmark & \cmark & \xmark & \xmark \\
\bottomrule
\end{tabular}
}
\end{table}

\section{Ambiguous- and Dirty-MNIST}
\label{app:dirty_mnist}
Each sample in Ambiguous-MNIST is constructed by decoding a linear combination of latent representations of 2 different MNIST digits from a pre-trained VAE \citep{kingma2013auto}. Every decoded image is assigned several labels sampled from the softmax probabilities of an off-the-shelf MNIST neural network ensemble, with points filtered based on an ensemble's MI (to remove `junk' images) and then stratified class-wise based on their softmax entropy (some classes are inherently more ambiguous, so we ``amplify'' these; we stratify per-class to try to preserve a wide spread of possible entropy values, and avoid introducing additional ambiguity which will increase all points to have highest entropy). All off-the-shelf MNIST neural networks were then discarded and new models were trained to generate Fig 1 (and as can be seen, the ambiguous points we generate indeed have high entropy regardless of the model architecture used).
We create 60K such training and 10K test images to construct Ambiguous-MNIST. Finally, the Dirty-MNIST dataset in this experiment contains MNIST and Ambiguous-MNIST samples in a 1:1 ratio (with 120K training and 20K test samples). In \Cref{fig:ambiguous_mnist_samples}, we provide some samples from Ambiguous-MNIST.

\begin{figure*}[!t]
    \centering
    \begin{subfigure}{0.80\linewidth}
        \centering
        \includegraphics[width=\linewidth]{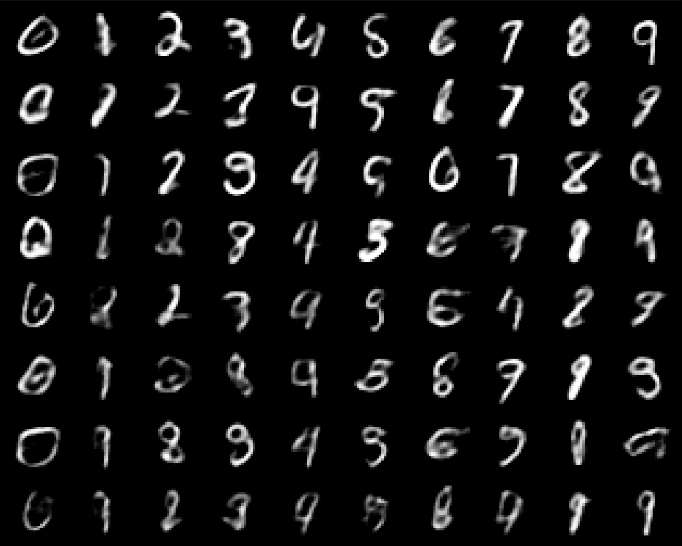}
    \end{subfigure}
    \caption{
    \emph{Samples from Ambiguous-MNIST.}
    }
    \label{fig:ambiguous_mnist_samples}
\end{figure*}

\section{Algorithm}

\begin{algorithm}[t]
    \caption{Deep Deterministic Uncertainty}
    \label{algorithm:fancy}
    \begin{algorithmic}[1]
        \State \textbf{Definitions:}

        - Regularized feature extractor $f_\theta: x \rightarrow \mathbb{R}^d$

        - Softmax output predictions: $p(y|x)$
        
        - GMM density: $q(z) = \sum_y q(z|y=c) \, q(y=c)$

        - Dataset $(X, Y)$

        \item[] %
        
        \Procedure{train}{}

        \State train NN $p(y| f_\theta(x))$ with $(X, Y)$
        
        \For{each class $c$ with samples $\mathbf{x}_c  \subset X$}
        \State $\mu_c \leftarrow \frac{1}{|\mathbf{x}_c|} \sum_{\mathbf{x}_c} f_\theta(\mathbf{x}_c)$
        
        \State $\Sigma_c \leftarrow \frac{1}{|\mathbf{x}_c| - 1} (f_\theta(\mathbf{x}_c) - \mu_c)(f_\theta(\mathbf{x}_c) - \mu_c)^T$
        \State $\pi_c \leftarrow \frac{\sum_{\mathbf{x}_c} 1}{\lvert X \rvert}$
        \EndFor
        
        \EndProcedure
        
        \item[] %
    
        \Function{disentangle\_uncertainty}{sample $x$}
        
        \State compute feature representation $z = f_\theta(x)$        
        \State compute density under GMM: $q(z) = \sum_y \qprobc{z}{y} \qprob{y}$ with $\qprobc{z}{y} \sim \mathcal{N}(\mu_y;\sigma_y), \qprob{y} = \pi_y$ 
        \State compute softmax entropy: $H_p[Y|x]$
    
        \item[] %
        
        \If {low density $q(z)$}
            \State \Return OoD
        \ElsIf {high density $q(z)$}
            \If {high entropy $H_p[Y|x]$}
                \State \Return ambiguous iD
            \ElsIf {low entropy $H_p[Y|x]$}
                \State \Return iD
            \EndIf
        \EndIf
        \EndFunction
    \end{algorithmic}
\end{algorithm}

\begin{lstlisting}[float=tp, language=Python, caption=Deep Deterministic Uncertainty Pseudo-Code, label=algorithm, mathescape=true]
# instantiate models
model = create_sensitive_smooth_model()
gda = create_gda()

# train
training_samples, training_labels = load_training_set()
model.fit(training_samples, training_labels)

training_features = model.features(training_samples)
gda.fit(training_features, training_labels)

# test
test_features = model.features(test_sample)

epistemic_uncertainty = gda.log_density(test_features)

is_ood = ood_threshold <= epistemic_uncertainty
if not is_ood:
  predictions = model.softmax_layer(test_features)
  aleatoric_uncertainty = entropy(predictions)
\end{lstlisting}

\subsection{Increasing sensitivity}
\label{app:more_model_architecture}

Using residual connections to enforce sensitivity works well in practice when the layer is defined as $x' = x + f(x)$.
However, there are several places in the network where additional spatial downsampling is done in $f(\cdot)$ (through a strided convolution), and in order to compute the residual operation $x$ needs to be downsampled as well.
These downsampling operations are crucial for managing memory consumption and generalisation.
The way this is traditionally done in ResNets is by introducing an additional function $g(\cdot)$ on the residual branch (obtaining $x' = g(x) + f(x)$) which is a strided 1x1 convolution.
In practice, the stride is set to 2 pixels, which leads to the output of $g(\cdot)$ only being dependent on the top-left pixel of each 2x2 patch, which reduces sensitivity.
We overcome this issue by making an architectural change that improves uncertainty quality without sacrificing accuracy.
We use a strided average pooling operation instead of a 1x1 convolution in $g(\cdot)$. This makes the output of $g(\cdot)$ dependent on all input pixels.
Additionally, we use leaky ReLU activation functions, which are equivalent to ReLU activations when the input is larger than 0, but below 0 they compute $p * x$ with $p = 0.01$ in practice.
These further improve sensitivity as all negative activations still propagate in the network.

\subsection{Algorithm \& Pseudo-Code Implementation}
\label{app:implementation}

The algorithm is provided in \Cref{algorithm:fancy}. A simple Python pseudo-code implementation using a scikit-learn-like API \citep{sklearn_api} is shown in \Cref{algorithm}. Note that in order to compute thresholds for low and high density or entropy, we can simply use the training set containing iD data. We set all points having density lower than $99\%$ quantile as OoD. 

\section{Experimental Details}
\label{app:exp_details}

\subsection{Dirty-MNIST}
\label{app:exp_details_dirty_mnist}

We train for 50 epochs using SGD with a momentum of 0.9 and an initial learning rate of 0.1.
The learning rate drops by a factor of 10 at training epochs 25 and 40. 
Following SNGP \citep{liu2020simple}, we apply online spectral normalisation with one step of a power iteration on the convolutional weights.
For 1x1 convolutions, we use the exact algorithm, and for 3x3 convolutions, the approximate algorithm from \citet{gouk2018regularisation}.
The coefficient for SN is a hyper-parameter which we set to 3 using cross-validation.

\subsection{OoD Detection Training Setup}
\label{app:exp_details_cifar}
We train the softmax baselines on CIFAR-10/100 for 350 epochs using SGD as the optimiser with a momentum of 0.9, and an initial learning rate of 0.1. The learning rate drops by a factor of 10 at epochs 150 and 250. We train the 5-Ensemble baseline using this same training setup. The SNGP and DUQ models were trained using the setup of SNGP and hyper-parameters mentioned in their respective papers \citep{liu2020simple, van2020simple}. For models trained on ImageNet, we train for 90 epochs with SGD optimizer, an initial learning rate of 0.1 and a weight decay of 1e-4. We use a learning rate warmup decay of 0.01 along with a step scheduler with step size of 30 and a step factor of 0.1.

\subsection{Compute Resources}
\label{app:exp_details_compute}

Each model (ResNet-18, Wide-ResNet-28-10, ResNet-50, ResNet-110, DenseNet-121 or VGG-16) used for the large scale active learning, CIFAR-10 vs SVHN/CIFAR-100/Tiny-ImageNet/CIFAR-10-C and CIFAR-100 vs SVHN/Tiny-ImageNet tasks was trained on a single Nvidia Quadro RTX 6000 GPU. Each model (LeNet, VGG-16 and ResNet-18) used to get the results in \Cref{fig:intro_histograms} and \Cref{table:auroc_tab1} was trained on a single Nvidia GeForce RTX 2060 GPU. Each model (ResNet-50, Wide-ResNet-50-2, VGG-16) trained on ImageNet was trained using 8 Nvidia Quadro RTX 6000 GPUs.

\begin{table*}[!t]
\centering
\caption{\emph{OoD detection performance of different baselines using a ResNet-50 architecture with the CIFAR-10 vs SVHN/CIFAR-100/Tiny-ImageNet and CIFAR-100 vs SVHN/Tiny-ImageNet dataset pairs averaged over 25 runs.} Note: SN stands for Spectral Normalisation, JP stands for Jacobian Penalty. We highlight the best deterministic and best method overall in bold for each metric.}
\label{table:ood_resnet50}
\resizebox{\linewidth}{!}{%
\begin{tabular}{cccccccccc}
\toprule
\textbf{\small Train Dataset} & \textbf{\small Method} & \textbf{\small Penalty} & \textbf{\small Aleatoric Uncertainty} &
\textbf{\small Epistemic Uncertainty} & \textbf{\small Test Accuracy (\textuparrow)} & \textbf{\small Test ECE (\textdownarrow)} & \textbf{\small AUROC SVHN (\textuparrow)} & \textbf{\small AUROC CIFAR-100 (\textuparrow)} & \textbf{\small AUROC Tiny-ImageNet (\textuparrow)}\\
\midrule
\multirow{7}{*}{CIFAR-10} & Softmax & - & \multirow{2}{*}{Softmax Entropy} & Softmax Entropy & \multirow{2}{*}{$\mathbf{95.04\pm0.05}$} & \multirow{2}{*}{$\mathbf{0.97\pm0.04}$} & $93.80\pm0.41$ & $88.91\pm0.07$ & $88.32\pm0.07$ \\
&Energy-based {\scriptsize \citep{liu2020energy}} & - && Softmax Density &&& $94.48\pm0.44$ & $88.84\pm0.08$ & $88.45\pm0.08$\\
&DUQ {\scriptsize \citep{van2020simple}} & JP & Kernel Distance & Kernel Distance & $94.05\pm0.11$ & $1.71\pm0.07$ & $93.14\pm0.43$ & $83.87\pm0.27$ & $84.28\pm0.26$\\
&SNGP {\scriptsize \citep{liu2020simple}} & SN & Predictive Entropy & Predictive Entropy & $94.90\pm0.11$ & $1.01\pm0.03$ & $93.15\pm0.85$ & $89.32\pm0.10$ & $88.96\pm0.13$\\
&\textbf{DDU (ours)} & SN & Softmax Entropy & GMM Density & $94.92\pm0.06$&$1\pm0.04$&$\mathbf{94.77\pm0.35}$&$\mathbf{89.98\pm0.17}$ & $\mathbf{89.12\pm0.13}$\\
\cmidrule{2-10}
&5-Ensemble & \multirow{2}{*}{-} & \multirow{2}{*}{Predictive Entropy} & Predictive Entropy & \multirow{2}{*}{$\mathbf{96.06\pm0.04}$}&\multirow{2}{*}{$1.65\pm0.07$}&$94.75\pm0.39$&$89.87\pm0.06$ & $88.69\pm0.05$\\
&{\scriptsize \citep{lakshminarayanan2017simple}} &&& Mutual Information &&&$94.09\pm0.20$&$89.76\pm0.06$ & $89.04\pm0.03$\\
\midrule
&&&&& \textbf{\small Test Accuracy (\textuparrow)} & \textbf{\small{Test ECE (\textdownarrow)}} & \multicolumn{2}{c}{\textbf{\small AUROC SVHN (\textuparrow)}} & \textbf{\small AUROC Tiny-ImageNet (\textuparrow)} \\
\cmidrule{6-10}
\multirow{6}{*}{CIFAR-100} & Softmax & - & \multirow{2}{*}{Softmax Entropy} & Softmax Entropy & \multirow{2}{*}{$77.91\pm0.09$} & \multirow{2}{*}{$4.32\pm0.10$} & \multicolumn{2}{c}{$81.32\pm0.65$} & $79.83\pm0.07$ \\
&Energy-based {\scriptsize \citep{liu2020energy}} & - && Softmax Density &&&\multicolumn{2}{c}{$82.05\pm0.69$} & $79.61\pm0.08$ \\
&SNGP {\scriptsize \citep{liu2020simple}} & SN & Predictive Entropy & Predictive Entropy & $74.73\pm0.22$ & $7.68\pm0.13$ & \multicolumn{2}{c}{$82.50\pm2.09$} & $77.05\pm0.16$ \\
&\textbf{DDU (ours)} & SN & Softmax Entropy & GMM Density & $\mathbf{79.26\pm0.16}$&$\mathbf{4.07\pm0.06}$& \multicolumn{2}{c}{$\mathbf{87.34\pm0.64}$}& $\mathbf{82.11\pm0.20}$ \\
\cmidrule{2-10}
&5-Ensemble & \multirow{2}{*}{-} & \multirow{2}{*}{Predictive Entropy} & Predictive Entropy & \multirow{2}{*}{$\mathbf{81.06\pm0.07}$}&\multirow{2}{*}{$\mathbf{3.54\pm0.12}$}& \multicolumn{2}{c}{$83.42\pm0.89$} & $77.69\pm0.12$\\
&{\scriptsize \citep{lakshminarayanan2017simple}} &&& Mutual Information &&& \multicolumn{2}{c}{$84.24\pm0.90$} & $81.59\pm0.05$\\
\bottomrule
\end{tabular}}
\end{table*}
\begin{table*}[!t]
\centering
\caption{\emph{OoD detection performance of different baselines using a ResNet-110 architecture with the CIFAR-10 vs SVHN/CIFAR-100/Tiny-ImageNet and CIFAR-100 vs SVHN/Tiny-ImageNet dataset pairs averaged over 25 runs.} Note: SN stands for Spectral Normalisation, JP stands for Jacobian Penalty. We highlight the best deterministic and best method overall in bold for each metric.}
\label{table:ood_resnet110}
\resizebox{\linewidth}{!}{%
\begin{tabular}{cccccccccc}
\toprule
\textbf{\small Train Dataset} & \textbf{\small Method} & \textbf{\small Penalty} & \textbf{\small Aleatoric Uncertainty} &
\textbf{\small Epistemic Uncertainty} & \textbf{\small Test Accuracy (\textuparrow)} & \textbf{\small Test ECE (\textdownarrow)} & \textbf{\small AUROC SVHN (\textuparrow)} & \textbf{\small AUROC CIFAR-100 (\textuparrow)} & \textbf{\small AUROC Tiny-ImageNet (\textuparrow)}\\
\midrule
\multirow{7}{*}{CIFAR-10} & Softmax & - & \multirow{2}{*}{Softmax Entropy} & Softmax Entropy & \multirow{2}{*}{$\mathbf{95.08\pm0.04}$} & \multirow{2}{*}{$1.02\pm0.04$} & $93.12\pm0.44$ & $88.7\pm0.1$ & $88.07\pm0.11$ \\
&Energy-based {\scriptsize \citep{liu2020energy}} & - && Softmax Density &&& $93.67\pm0.47$ & $88.60\pm0.11$ & $88.13\pm0.11$\\
&DUQ {\scriptsize \citep{van2020simple}} & JP & Kernel Distance & Kernel Distance & $94.32\pm0.17$ & $1.21\pm0.07$ & $94.02\pm0.45$ & $86.17\pm0.35$ & $85.24\pm0.21$\\
&SNGP {\scriptsize \citep{liu2020simple}} & SN & Predictive Entropy & Predictive Entropy & $94.85\pm0.09$ & $1.04\pm0.02$ & $93.17\pm0.53$ & $89.23\pm0.10$ & $88.80\pm0.12$\\
&\textbf{DDU (ours)} & SN & Softmax Entropy & GMM Density & $94.82\pm0.06$&$\mathbf{1.01\pm0.04}$&$\mathbf{95.48\pm0.30}$&$\mathbf{90.08\pm0.13}$ & $\mathbf{89.18\pm0.15}$\\
\cmidrule{2-10}
&5-Ensemble & \multirow{2}{*}{-} & \multirow{2}{*}{Predictive Entropy} & Predictive Entropy & \multirow{2}{*}{$\mathbf{96.18\pm0.05}$}& \multirow{2}{*}{$1.57\pm0.05$}&$95.07\pm0.45$&$\mathbf{90.23\pm0.04}$ & $89\pm0.03$\\
&{\scriptsize \citep{lakshminarayanan2017simple}} &&& Mutual Information &&&$94.72\pm0.34$&$89.69\pm0.05$ & $88.35\pm0.05$\\
\midrule
&&&&& \textbf{\small Test Accuracy (\textuparrow)} & \textbf{\small{Test ECE (\textdownarrow)}} & \multicolumn{2}{c}{\textbf{\small AUROC SVHN (\textuparrow)}} & \textbf{\small AUROC Tiny-ImageNet (\textuparrow)} \\
\cmidrule{6-10}
\multirow{6}{*}{CIFAR-100} & Softmax & - & \multirow{2}{*}{Softmax Entropy} & Softmax Entropy & \multirow{2}{*}{$78.65\pm0.10$} & \multirow{2}{*}{$3.93\pm0.13$} & \multicolumn{2}{c}{$82.04\pm0.57$} & $80.13\pm0.07$ \\
&Energy-based {\scriptsize \citep{liu2020energy}} & - && Softmax Density &&&\multicolumn{2}{c}{$82.78\pm0.60$} & $80.01\pm0.09$ \\
&SNGP {\scriptsize \citep{liu2020simple}} & SN & Predictive Entropy & Predictive Entropy & $76.16\pm0.27$ & $6.43\pm0.75$ & \multicolumn{2}{c}{$83.94\pm0.10$} & $78.54\pm0.28$ \\
&\textbf{DDU (ours)} & SN & Softmax Entropy & GMM Density & $\mathbf{78.89\pm0.17}$&$\mathbf{3.79\pm0.07}$& \multicolumn{2}{c}{$\mathbf{88.66\pm0.56}$}& $\mathbf{82.58\pm0.24}$ \\
\cmidrule{2-10}
&5-Ensemble & \multirow{2}{*}{-} & \multirow{2}{*}{Predictive Entropy} & Predictive Entropy & \multirow{2}{*}{$\mathbf{81.80\pm0.10}$} & \multirow{2}{*}{$\mathbf{3.67\pm0.11}$} & \multicolumn{2}{c}{$83.68\pm0.33$} & $81.12\pm0.13$\\
&{\scriptsize \citep{lakshminarayanan2017simple}} &&& Mutual Information &&& \multicolumn{2}{c}{$85.11\pm0.57$} & $81.94\pm0.06$\\
\bottomrule
\end{tabular}}
\end{table*}
\begin{table*}[!t]
\centering
\caption{\emph{OoD detection performance of different baselines using a DenseNet-121 architecture with the CIFAR-10 vs SVHN/CIFAR-100/Tiny-ImageNet and CIFAR-100 vs SVHN/Tiny-ImageNet dataset pairs averaged over 25 runs.} Note: SN stands for Spectral Normalisation, JP stands for Jacobian Penalty. We highlight the best deterministic and best method overall in bold for each metric.}
\label{table:ood_densenet121}
\resizebox{\linewidth}{!}{%
\begin{tabular}{cccccccccc}
\toprule
\textbf{\small Train Dataset} & \textbf{\small Method} & \textbf{\small Penalty} & \textbf{\small Aleatoric Uncertainty} &
\textbf{\small Epistemic Uncertainty} & \textbf{\small Test Accuracy (\textuparrow)} & \textbf{\small Test ECE (\textdownarrow)} & \textbf{\small AUROC SVHN (\textuparrow)} & \textbf{\small AUROC CIFAR-100 (\textuparrow)} & \textbf{\small AUROC Tiny-ImageNet (\textuparrow)}\\
\midrule
\multirow{7}{*}{CIFAR-10} & Softmax & - & \multirow{2}{*}{Softmax Entropy} & Softmax Entropy & \multirow{2}{*}{$95.16\pm0.03$} & \multirow{2}{*}{$1.10\pm0.04$} & $94\pm0.44$ & $87.55\pm0.11$ & $86.99\pm0.12$ \\
&Energy-based {\scriptsize \citep{liu2020energy}} & - && Softmax Density &&& $94.07\pm0.54$ & $86.73\pm0.15$ & $86.43\pm0.16$\\
&DUQ {\scriptsize \citep{van2020simple}} & JP & Kernel Distance & Kernel Distance & $95.02\pm0.14$ & $1.08\pm0.08$ & $94.67\pm0.41$ & $87.38\pm0.21$ & $86.72\pm0.14$\\
&SNGP {\scriptsize \citep{liu2020simple}} & SN & Predictive Entropy & Predictive Entropy & $94.31\pm0.21$ & $1.08\pm0.10$ & $94.48\pm0.34$ & $88.86\pm0.46$ & $88.40\pm0.48$\\
&\textbf{DDU (ours)} & SN & Softmax Entropy & GMM Density & $\mathbf{95.21\pm0.03}$&$\mathbf{1.05\pm0.03}$&$\mathbf{96.21\pm0.31}$&$\mathbf{90.84\pm0.06}$ & $\mathbf{89.70\pm0.06}$\\
\cmidrule{2-10}
&5-Ensemble & \multirow{2}{*}{-} & \multirow{2}{*}{Predictive Entropy} & Predictive Entropy & \multirow{2}{*}{$\mathbf{96.18\pm0.05}$}&\multirow{2}{*}{$1.07\pm0.07$}&$95.78\pm0.11$&$90.65\pm0.03$ & $89.62\pm0.06$\\
&{\scriptsize \citep{lakshminarayanan2017simple}} &&& Mutual Information &&&$95.75\pm0.10$&$90.71\pm0.04$ & $89.34\pm0.06$\\
\midrule
&&&&& \textbf{\small Test Accuracy (\textuparrow)} & \textbf{\small{Test ECE (\textdownarrow)}} & \multicolumn{2}{c}{\textbf{\small AUROC SVHN (\textuparrow)}} & \textbf{\small AUROC Tiny-ImageNet (\textuparrow)} \\
\cmidrule{6-10}
\multirow{6}{*}{CIFAR-100} & Softmax & - & \multirow{2}{*}{Softmax Entropy} & Softmax Entropy & \multirow{2}{*}{$79.02\pm0.08$} & \multirow{2}{*}{$4.11\pm0.08$} & \multicolumn{2}{c}{$85.86\pm0.42$} & $81.10\pm0.07$ \\
&Energy-based {\scriptsize \citep{liu2020energy}} & - && Softmax Density &&&\multicolumn{2}{c}{$87.09\pm0.49$} & $80.84\pm0.08$ \\
&SNGP {\scriptsize \citep{liu2020simple}} & SN & Predictive Entropy & Predictive Entropy & $79.15\pm0.15$ & $6.73\pm0.10$ & \multicolumn{2}{c}{$85.00\pm0.12$} & $79.76\pm0.15$ \\
&\textbf{DDU (ours)} & SN & Softmax Entropy & GMM Density & $\mathbf{79.15\pm0.07}$&$\mathbf{4.11\pm0.06}$& \multicolumn{2}{c}{$\mathbf{88.44\pm0.55}$}& $\mathbf{81.85\pm0.11}$ \\
\cmidrule{2-10}
&5-Ensemble & \multirow{2}{*}{-} & \multirow{2}{*}{Predictive Entropy} & Predictive Entropy & \multirow{2}{*}{$\mathbf{81.01\pm0.13}$}&\multirow{2}{*}{$4.81\pm0.05$}& \multicolumn{2}{c}{$88.32\pm0.61$} & $81.45\pm0.12$\\
&{\scriptsize \citep{lakshminarayanan2017simple}} &&& Mutual Information &&& \multicolumn{2}{c}{$88.36\pm0.17$} & $81.73\pm0.06$\\
\bottomrule
\end{tabular}}
\end{table*}
\begin{table*}[!t]
\centering
\caption{\emph{OoD detection performance of different ablations trained on CIFAR-10 using Wide-ResNet-28-10 and VGG-16 architectures with SVHN, CIFAR-100 and Tiny-ImageNet as OoD datasets averaged over 25 runs.} Note: SN stands for Spectral Normalisation. We highlight the best deterministic and best method overall in bold for each metric.}
\label{table:ood_2}
\scriptsize
\resizebox{\linewidth}{!}{%
\begin{tabular}{cccccccccccc}
\toprule
\multicolumn{5}{c}{\textbf{Ablations}} & \textbf{Aleatoric Uncertainty} &
\textbf{Epistemic Uncertainty} & \textbf{Test Accuracy (\textuparrow)} & \textbf{Test ECE (\textdownarrow)} & \textbf{AUROC SVHN (\textuparrow)} & \textbf{AUROC CIFAR-100 (\textuparrow)} & \textbf{AUROC Tiny-ImageNet (\textuparrow)} \\
\cmidrule{1-5}
\textbf{Architecture} & \textbf{Ensemble} & \textbf{Residual Connections} & \textbf{SN} & \textbf{GMM} &&&&&\\
\midrule
\multirow{10}{*}{Wide-ResNet-28-10} & \multirow{8}{*}{\xmark} & \multirow{8}{*}{\cmark} & \multirow{4}{*}{\xmark} & \multirow{2}{*}{\xmark} & \multirow{2}{*}{Softmax Entropy} & Softmax Entropy & \multirow{2}{*}{$\mathbf{95.98\pm0.02}$}&\multirow{2}{*}{$0.85\pm0.02$}&$94.44\pm0.43$&$89.39\pm0.06$&$88.42\pm0.05$ \\
   &&                         &                         &                         &                                  & Softmax Density &&&$94.56\pm0.51$&$88.89\pm0.07$&$88.11\pm0.06$ \\
                                                                                 \cmidrule{5-12}
   &&                         &                         &   \cmark                & Softmax Entropy & GMM Density   & $95.98\pm0.02$&$0.85\pm0.02$&$96.08\pm0.25$&$90.94\pm0.03$&$90.62\pm0.05$ \\
                                                       \cmidrule{4-12}
   &&                         & \multirow{4}{*}{\cmark} & \multirow{2}{*}{\xmark}  & \multirow{2}{*}{Softmax Entropy} & Softmax Entropy & \multirow{2}{*}{$95.97\pm0.03$}&\multirow{2}{*}{$0.85\pm0.04$}&$94.05\pm0.26$&$90.02\pm0.07$&$89.07\pm0.06$ \\
   &&                         &                         &                         &                                   & Softmax Density &&&$94.31\pm0.33$&$89.78\pm0.08$&$88.96\pm0.07$ \\
                                                                                 \cmidrule{5-12}
   &&                         &                         &      \cmark                   & \textbf{Softmax Entropy} &   \textbf{GMM Density}                               & $95.97\pm0.03$&$\mathbf{0.85\pm0.04}$&$\mathbf{97.86\pm0.19}$&$\mathbf{91.34\pm0.04}$&$\mathbf{91.07\pm0.05}$ \\
   \cmidrule{2-12}
   & \multirow{2}{*}{\cmark} & \multirow{2}{*}{\cmark} & \multirow{2}{*}{\xmark} & \multirow{2}{*}{\xmark} & \multirow{2}{*}{Predictive Entropy} & Predictive Entropy & \multirow{2}{*}{$\mathbf{96.59\pm0.02}$} & \multirow{2}{*}{$\mathbf{0.76\pm0.03}$} & $97.73\pm0.31$ & $\mathbf{92.13\pm0.02}$ & $90.06\pm0.03$\\
   & & & & & & Mutual Information &&& $97.18\pm0.19$ & $91.33\pm0.03$ & $90.90\pm0.03$\\
\midrule
\multirow{10}{*}{VGG-16} & \multirow{8}{*}{\xmark} & \multirow{8}{*}{\cmark} & \multirow{4}{*}{\xmark} & \multirow{2}{*}{\xmark} & \multirow{2}{*}{Softmax Entropy} & Softmax Entropy & \multirow{2}{*}{$93.63\pm0.04$}&\multirow{2}{*}{$1.64\pm0.03$}&$85.76\pm0.84$&$82.48\pm0.14$&$83.07\pm0.12$ \\
&&                         &                         &                         &                                  & Softmax Density &&&$84.24\pm1.04$&$81.91\pm0.17$&$82.82\pm0.14$ \\
                                                                             \cmidrule{5-12}
&&                         &                         &   \cmark                & Softmax Entropy & GMM Density   & $93.63\pm0.04$&$1.64\pm0.03$&$89.25\pm0.36$&$86.55\pm0.10$&$86.78\pm0.09$ \\
                                                   \cmidrule{4-12}
&&                         & \multirow{4}{*}{\cmark} & \multirow{2}{*}{\xmark}  & \multirow{2}{*}{Softmax Entropy} & Softmax Entropy & \multirow{2}{*}{$93.62\pm0.04$}&\multirow{2}{*}{$1.78\pm0.04$}&$87.54\pm0.41$&$82.71\pm0.09$&$83.33\pm0.08$ \\
&&                         &                         &                         &                                   & Softmax Density &&&$86.28\pm0.51$&$82.15\pm0.11$&$83.07\pm0.10$ \\
                                                                             \cmidrule{5-12}
&&                         &                         &      \cmark                   & Softmax Entropy &   GMM Density                               & $93.62\pm0.04$&$1.78\pm0.04$&$89.62\pm0.37$&$86.37\pm0.14$&$86.63\pm0.11$ \\
\cmidrule{2-12}
& \multirow{2}{*}{\cmark} & \multirow{2}{*}{\cmark} & \multirow{2}{*}{\xmark} & \multirow{2}{*}{\xmark} & \multirow{2}{*}{Predictive Entropy} & Predictive Entropy & \multirow{2}{*}{$94.9\pm0.05$} & \multirow{2}{*}{$2.03\pm0.03$} & $92.80\pm0.18$ & $89.01\pm0.08$ & $87.66\pm0.08$\\
& & & & & & Mutual Information &&& $91\pm0.22$ & $88.43\pm0.08$ & $88.74\pm0.05$\\
\bottomrule
\end{tabular}}
\end{table*}
\begin{table*}[!t]
\centering
\caption{\emph{OoD detection performance of different ablations trained on CIFAR-100 using Wide-ResNet-28-10 and VGG-16 architectures with SVHN and Tiny-ImageNet as the OoD dataset averaged over 25 runs.} Note: SN stands for Spectral Normalisation. We highlight the best deterministic and best method overall in bold for each metric.}
\label{table:ood_3}
\scriptsize
\resizebox{\linewidth}{!}{%
\begin{tabular}{ccccccccccc}
\toprule
\multicolumn{5}{c}{\textbf{Ablations}} & \textbf{Aleatoric Uncertainty} &
\textbf{Epistemic Uncertainty} & \textbf{Test Accuracy (\textuparrow)} & \textbf{Test ECE (\textdownarrow)} & \textbf{AUROC SVHN (\textuparrow)} & \textbf{AUROC Tiny-ImageNet (\textuparrow)} \\
\cmidrule{1-5}
\textbf{Architecture} & \textbf{Ensemble} & \textbf{Residual Connections} & \textbf{SN} & \textbf{GMM} &&&&\\
\midrule
\multirow{10}{*}{Wide-ResNet-28-10} & \multirow{8}{*}{\xmark} & \multirow{8}{*}{\cmark} & \multirow{4}{*}{\xmark} & \multirow{2}{*}{\xmark} & \multirow{2}{*}{Softmax Entropy} & Softmax Entropy & \multirow{2}{*}{$80.26\pm0.06$}&\multirow{2}{*}{$4.62\pm0.06$}&$77.42\pm0.57$&$81.53\pm0.05$ \\
       &&                         &                         &                         &                                  & Softmax Density &&&$78.00\pm0.63$&$81.33\pm0.06$ \\
                                                                                     \cmidrule{5-11}
       &&                         &                         &   \cmark                & Softmax Entropy & GMM Density   & $80.26\pm0.06$&$4.62\pm0.06$&$87.54\pm0.61$&$78.13\pm0.08$ \\
                                                           \cmidrule{4-11}
       &&                         & \multirow{4}{*}{\cmark} & \multirow{2}{*}{\xmark}  & \multirow{2}{*}{Softmax Entropy} & Softmax Entropy & \multirow{2}{*}{$80.98\pm0.06$}&\multirow{2}{*}{$4.10\pm0.08$}&$85.37\pm0.36$&$82.57\pm0.03$ \\
       &&                         &                         &                         &                                   & Softmax Density &&&$86.41\pm0.38$&$82.49\pm0.04$ \\
                                                                                     \cmidrule{5-11}
       &&                         &                         &      \cmark                   & \textbf{Softmax Entropy} &   \textbf{GMM Density}                               & $\mathbf{80.98\pm0.06}$&$\mathbf{4.10\pm0.08}$&$\mathbf{87.53\pm0.62}$&$\mathbf{83.13\pm0.06}$ \\
       \cmidrule{2-11}
       & \multirow{2}{*}{\cmark} & \multirow{2}{*}{\cmark} & \multirow{2}{*}{\xmark} & \multirow{2}{*}{\xmark} & \multirow{2}{*}{Predictive Entropy} & Predictive Entropy & \multirow{2}{*}{$\mathbf{82.79\pm0.10}$} & \multirow{2}{*}{$\mathbf{3.32\pm0.09}$} & $79.54\pm0.91$&$82.95\pm0.09$ \\
       & & & & & & Mutual Information &&& $77.00\pm1.54$ & $82.82\pm0.04$ \\
\midrule
\multirow{10}{*}{VGG-16} & \multirow{8}{*}{\xmark} & \multirow{8}{*}{\cmark} & \multirow{4}{*}{\xmark} & \multirow{2}{*}{\xmark} & \multirow{2}{*}{Softmax Entropy} & Softmax Entropy & \multirow{2}{*}{$73.48\pm0.05$}&\multirow{2}{*}{$4.46\pm0.05$}&$76.73\pm0.72$&$76.43\pm0.05$ \\
       &&                         &                         &                         &                                  & Softmax Density &&&$77.70\pm0.86$&$74.68\pm0.07$ \\
                                                                                     \cmidrule{5-11}
       &&                         &                         &   \cmark                & Softmax Entropy & GMM Density   & $73.48\pm0.05$&$4.46\pm0.05$&$75.65\pm0.95$&$74.32\pm1.73$ \\
                                                           \cmidrule{4-11}
       &&                         & \multirow{4}{*}{\cmark} & \multirow{2}{*}{\xmark}  & \multirow{2}{*}{Softmax Entropy} & Softmax Entropy & \multirow{2}{*}{$73.58\pm0.06$}&\multirow{2}{*}{$4.32\pm0.06$}&$77.21\pm0.77$&$76.59\pm0.06$ \\
       &&                         &                         &                         &                                   & Softmax Density &&&$77.76\pm0.90$&$74.86\pm0.08$ \\
                                                                                     \cmidrule{5-11}
       &&                         &                         &      \cmark                   & Softmax Entropy &   GMM Density                               & $73.58\pm0.06$&$4.32\pm0.06$&$75.99\pm1.23$&$74.06\pm1.67$ \\
       \cmidrule{2-11}
       & \multirow{2}{*}{\cmark} & \multirow{2}{*}{\cmark} & \multirow{2}{*}{\xmark} & \multirow{2}{*}{\xmark} & \multirow{2}{*}{Predictive Entropy} & Predictive Entropy & \multirow{2}{*}{$77.84\pm0.11$} & \multirow{2}{*}{$5.32\pm0.10$} & $79.62\pm0.73$&$78.66\pm06$ \\
       & & & & & & Mutual Information &&& $72.07\pm0.48$ & $76.27\pm0.05$ \\
\bottomrule
\end{tabular}}
\end{table*}

\begin{figure}[!t]
    \centering
    \begin{subfigure}{0.18\linewidth}
        \centering
        \includegraphics[width=\linewidth]{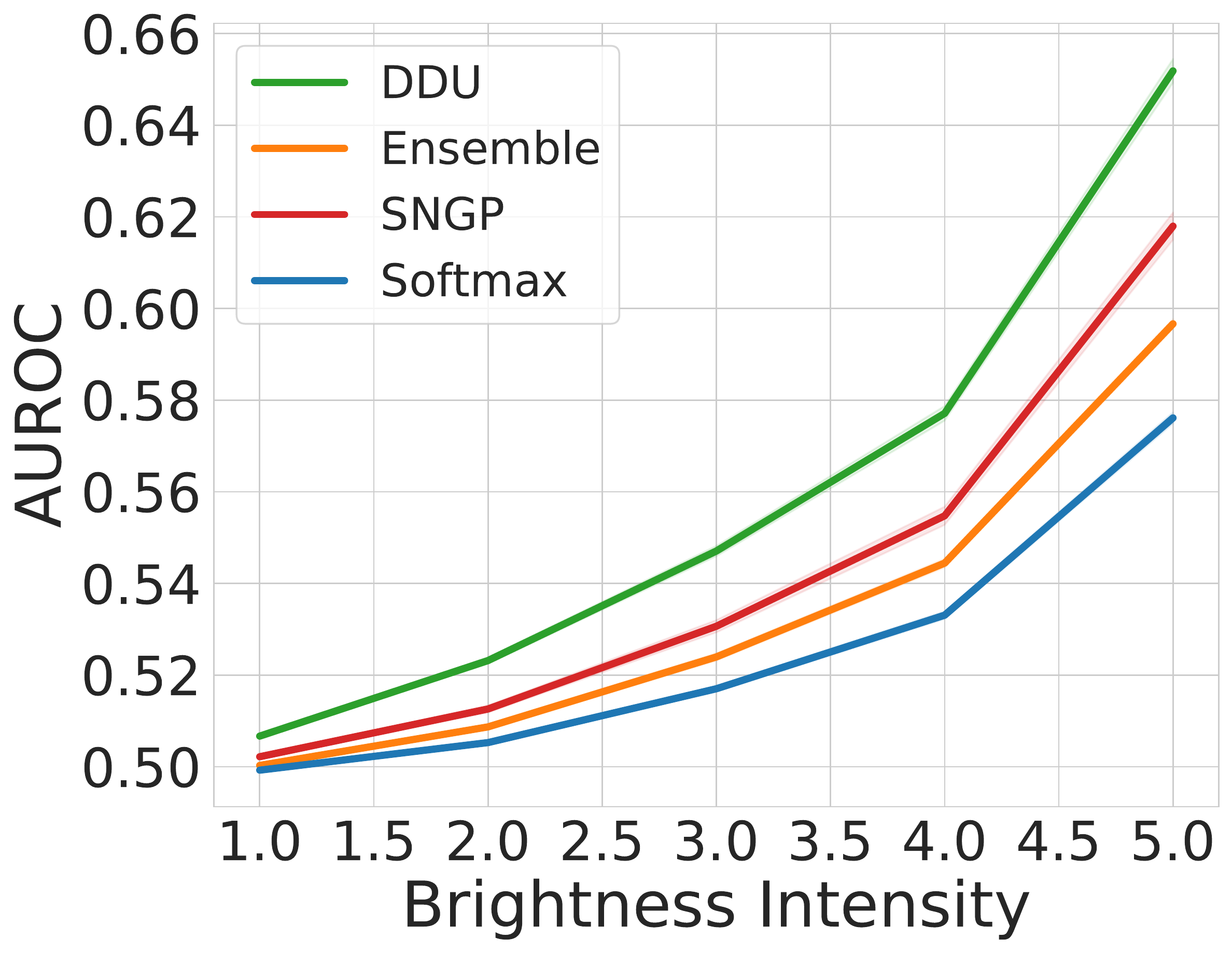}
    \end{subfigure}
    \begin{subfigure}{0.18\linewidth}
        \centering
        \includegraphics[width=\linewidth]{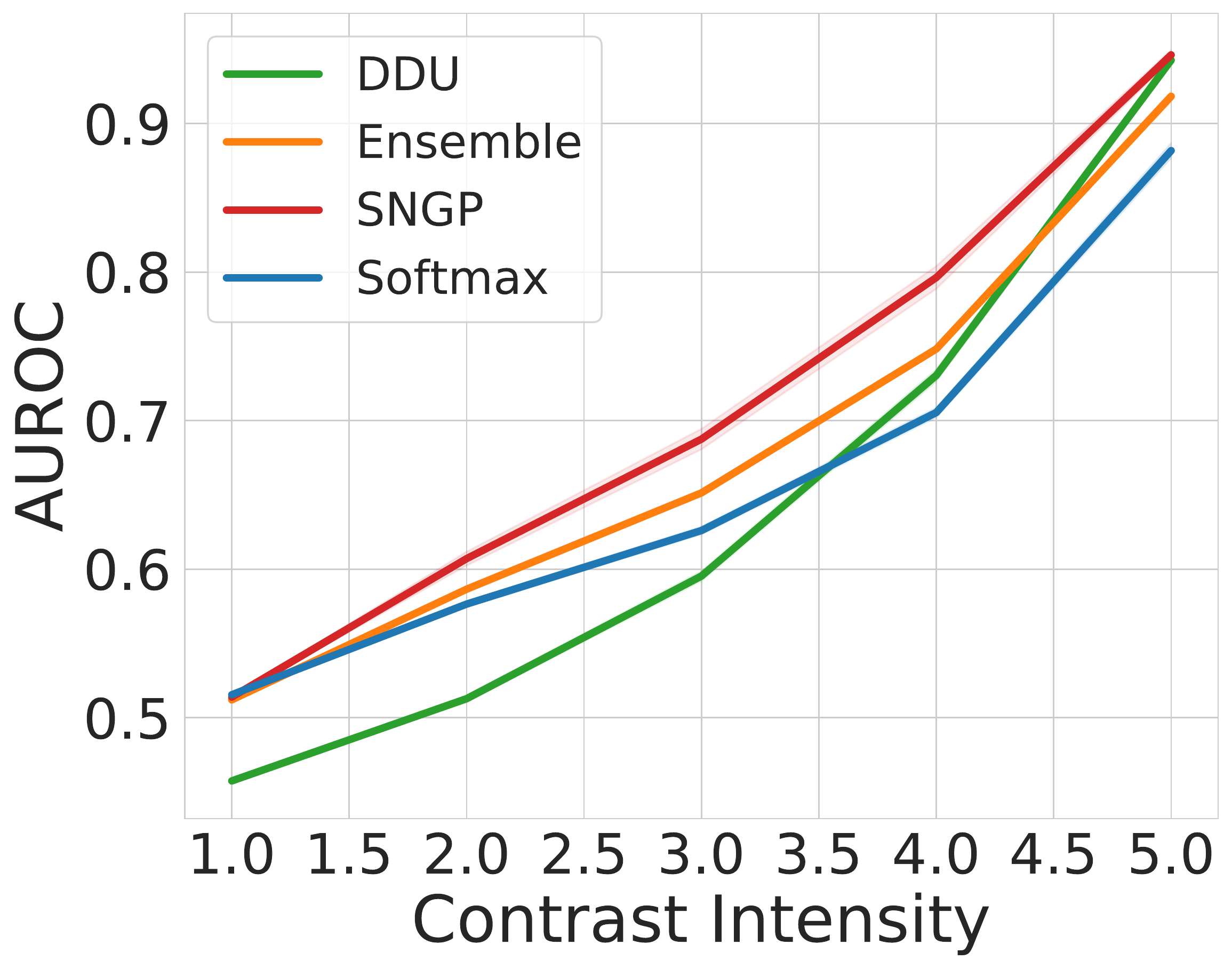}
    \end{subfigure} 
    \begin{subfigure}{0.18\linewidth}
        \centering
        \includegraphics[width=\linewidth]{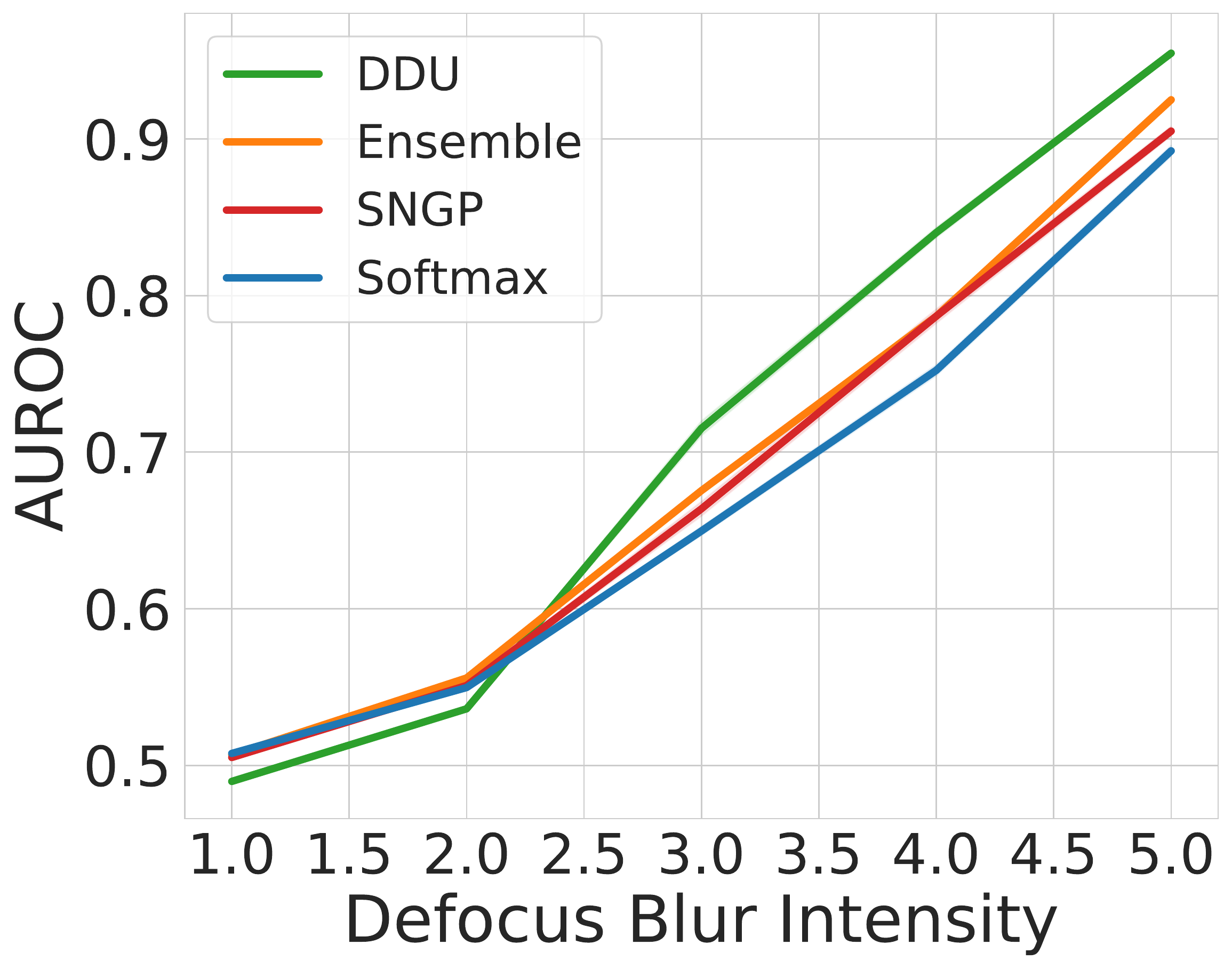}
    \end{subfigure} 
    \begin{subfigure}{0.18\linewidth}
        \centering
        \includegraphics[width=\linewidth]{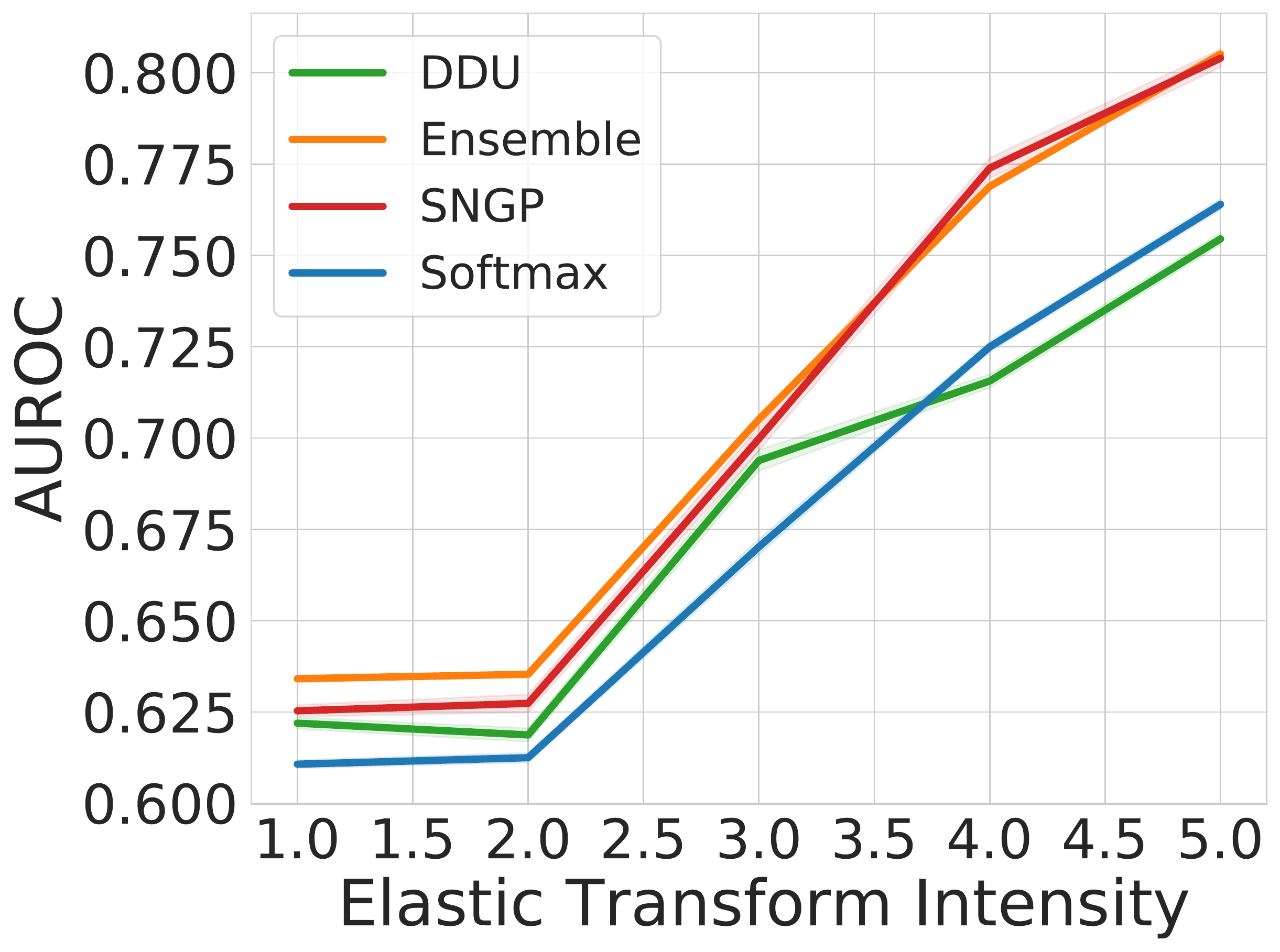}
    \end{subfigure}
    \begin{subfigure}{0.18\linewidth}
        \centering
        \includegraphics[width=\linewidth]{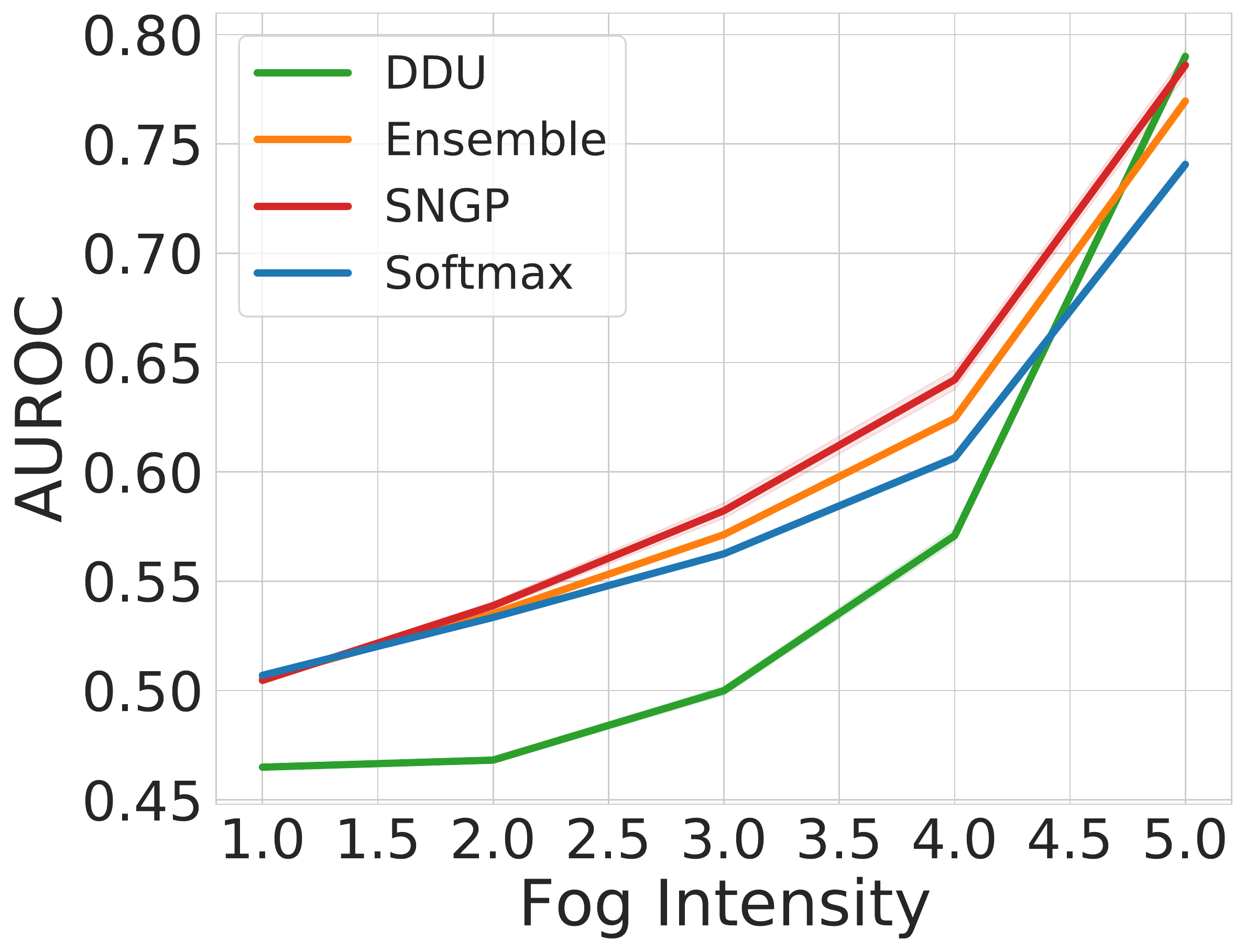}
    \end{subfigure}
    \begin{subfigure}{0.18\linewidth}
        \centering
        \includegraphics[width=\linewidth]{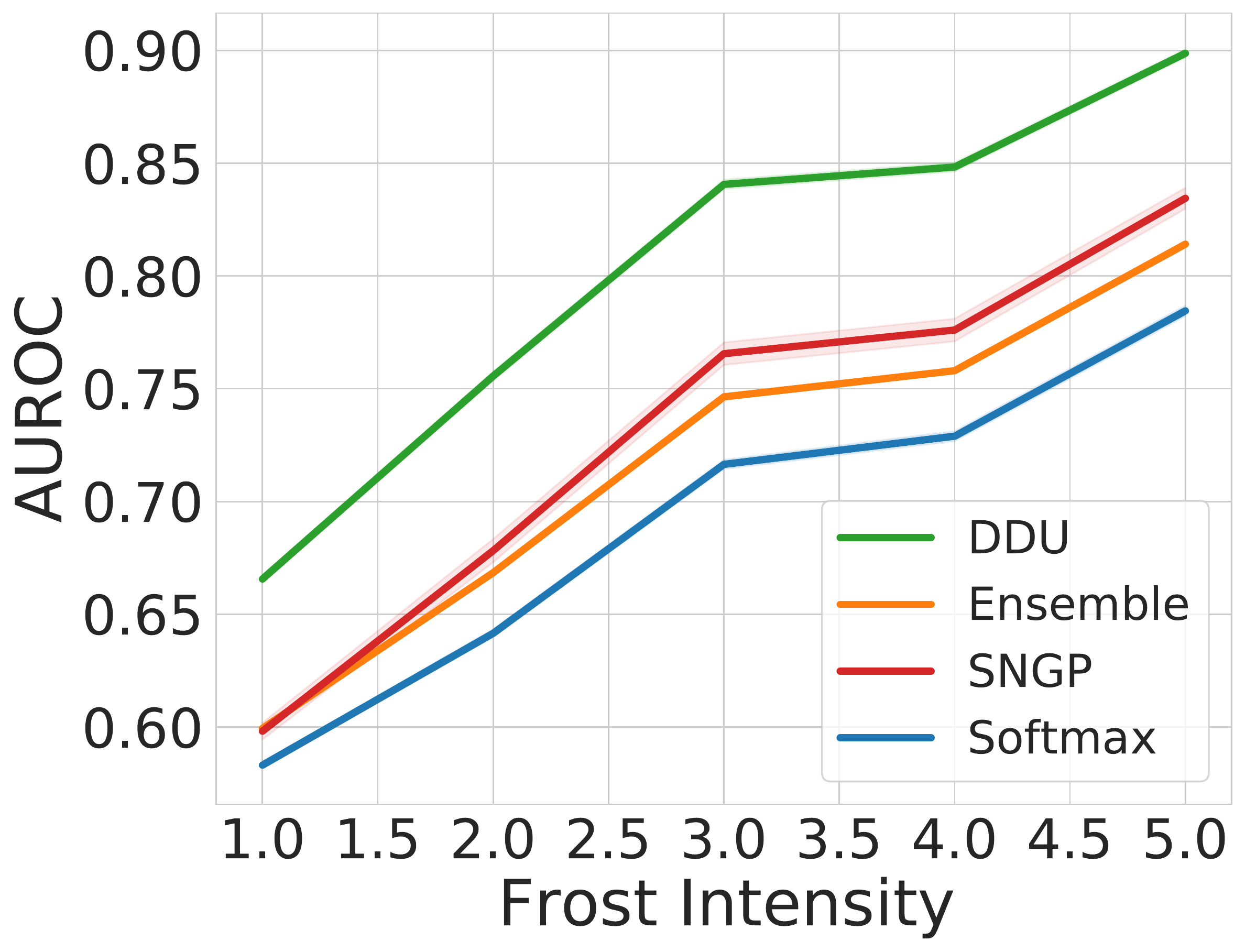}
    \end{subfigure}
    \begin{subfigure}{0.18\linewidth}
        \centering
        \includegraphics[width=\linewidth]{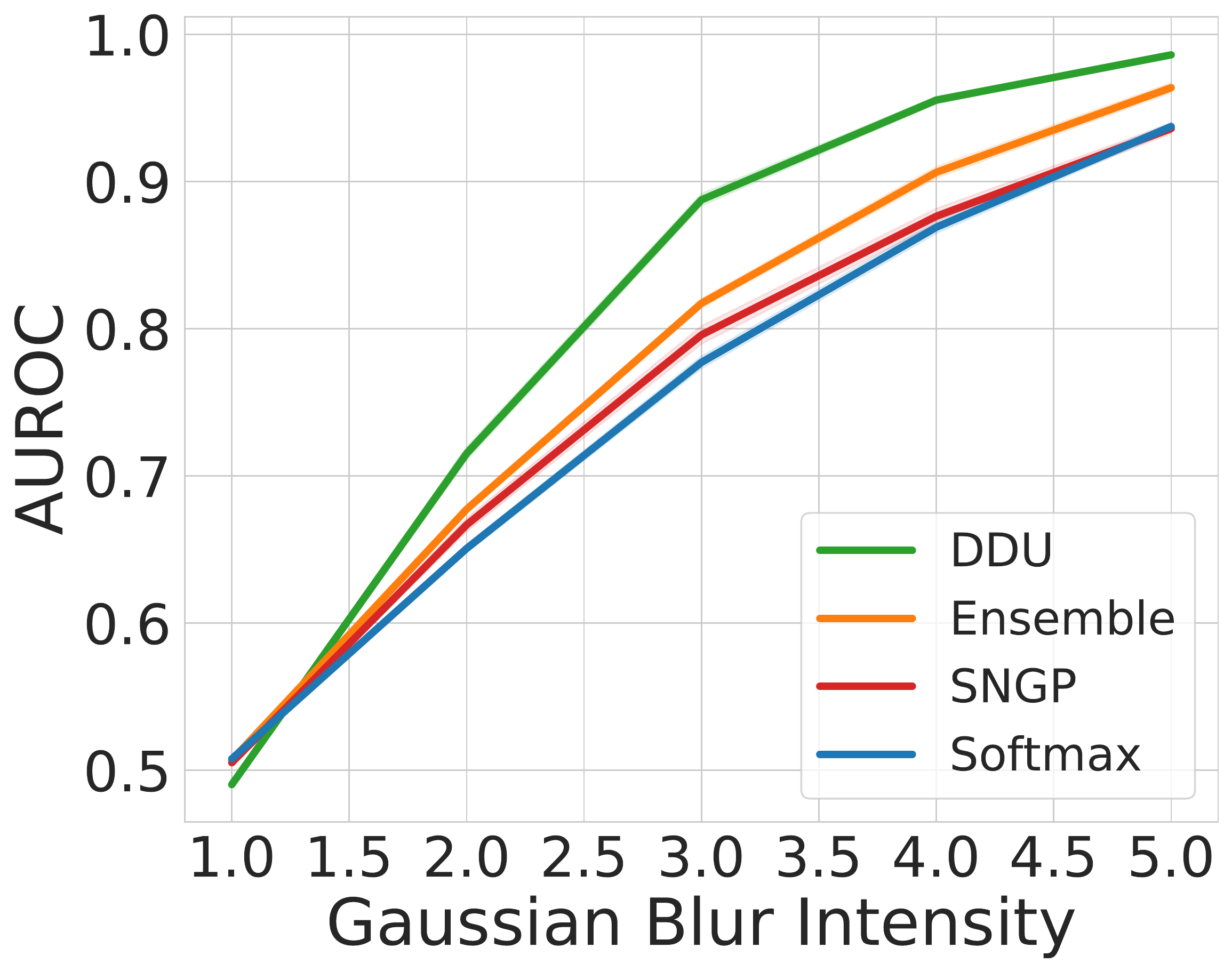}
    \end{subfigure}
    \begin{subfigure}{0.18\linewidth}
        \centering
        \includegraphics[width=\linewidth]{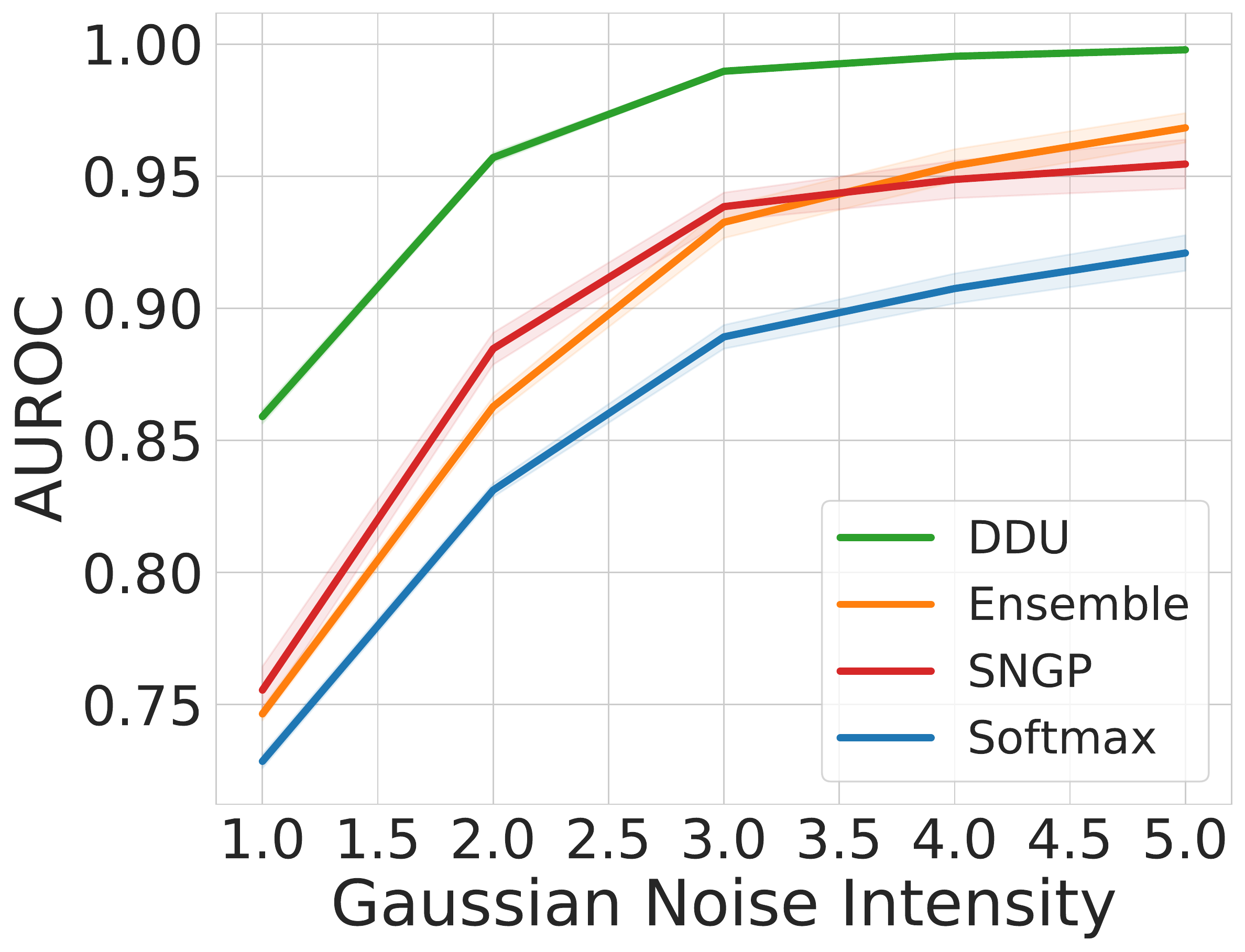}
    \end{subfigure}
    \begin{subfigure}{0.18\linewidth}
        \centering
        \includegraphics[width=\linewidth]{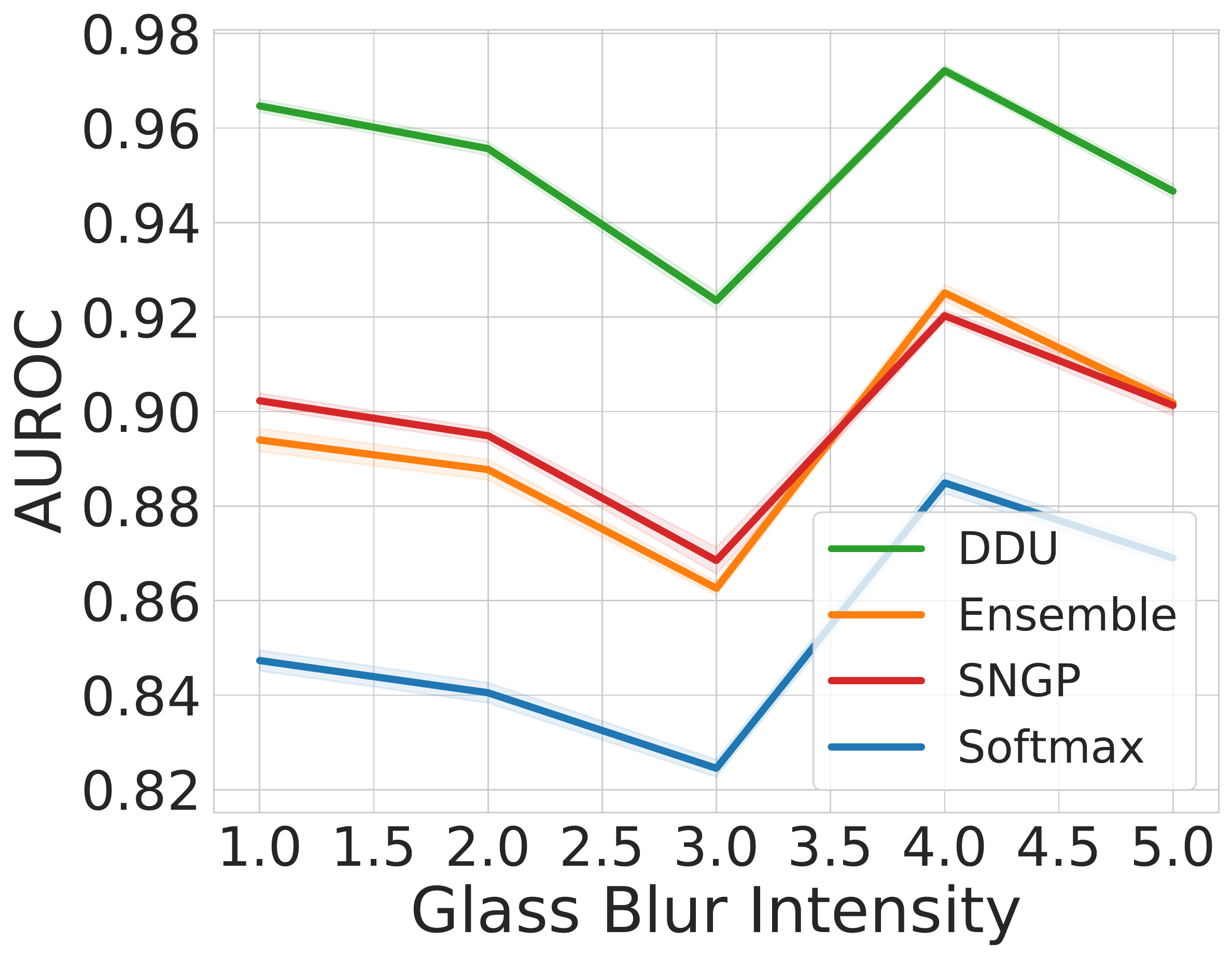}
    \end{subfigure}
    \begin{subfigure}{0.18\linewidth}
        \centering
        \includegraphics[width=\linewidth]{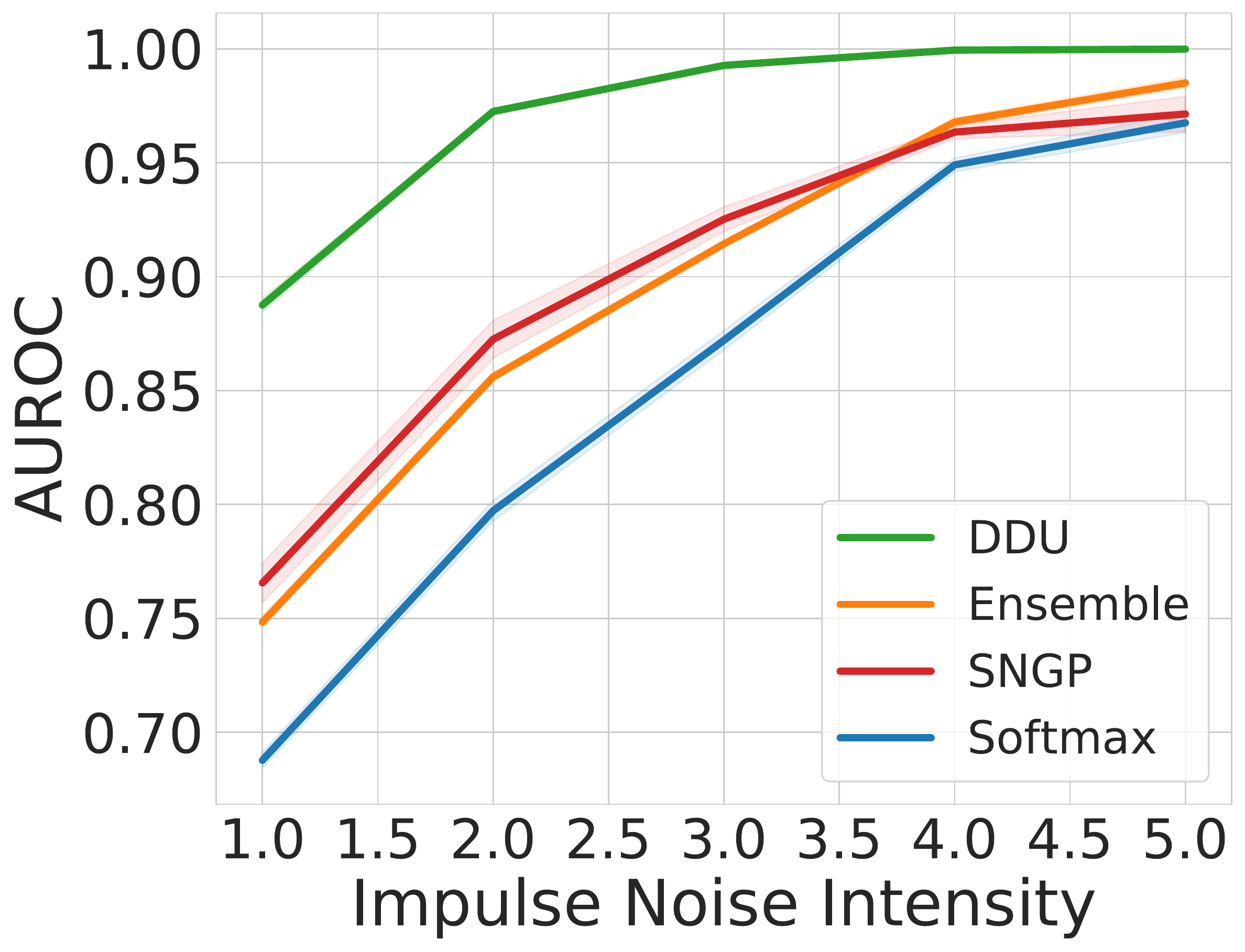}
    \end{subfigure}
    \begin{subfigure}{0.18\linewidth}
        \centering
        \includegraphics[width=\linewidth]{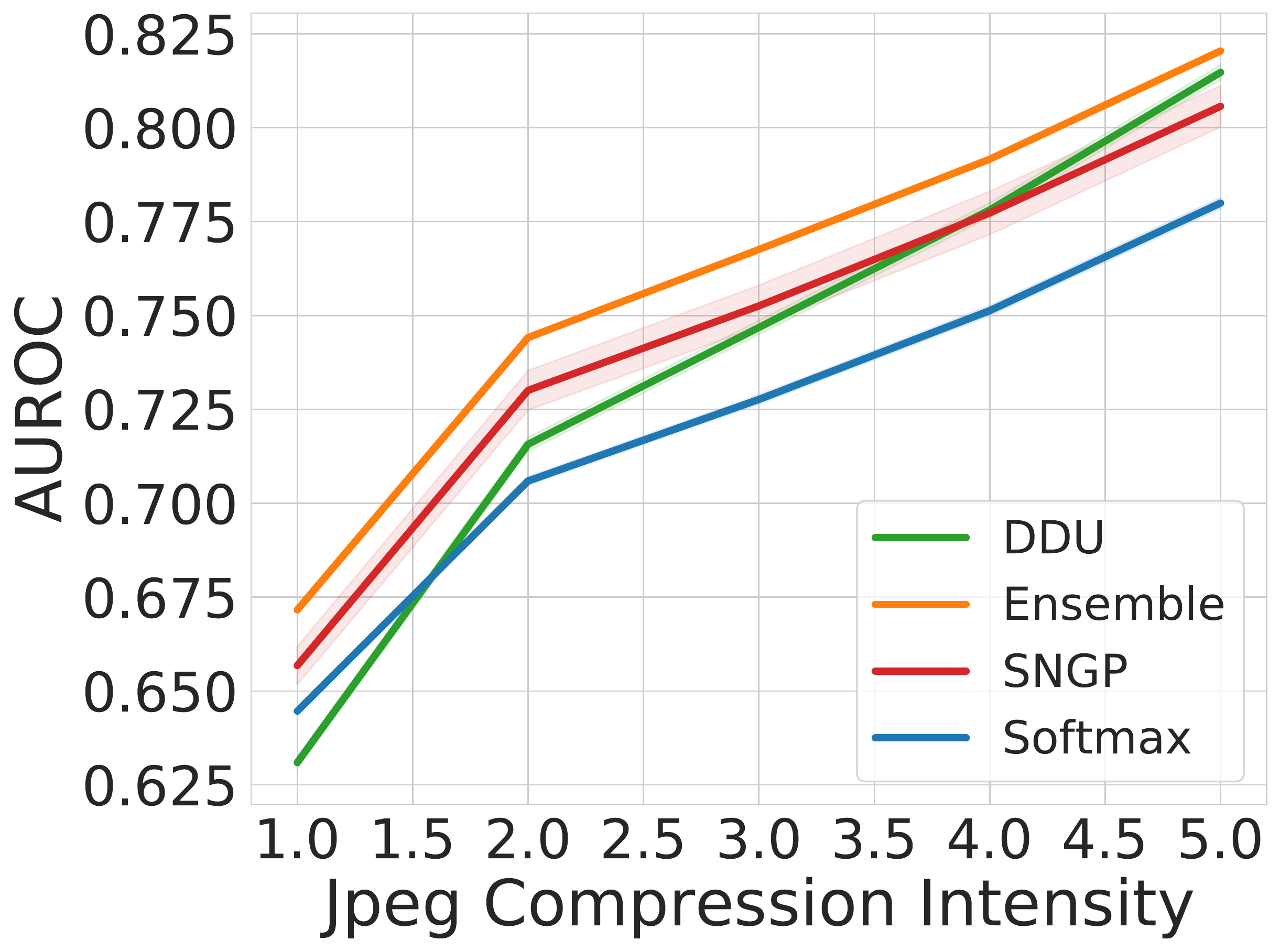}
    \end{subfigure}
    \begin{subfigure}{0.18\linewidth}
        \centering
        \includegraphics[width=\linewidth]{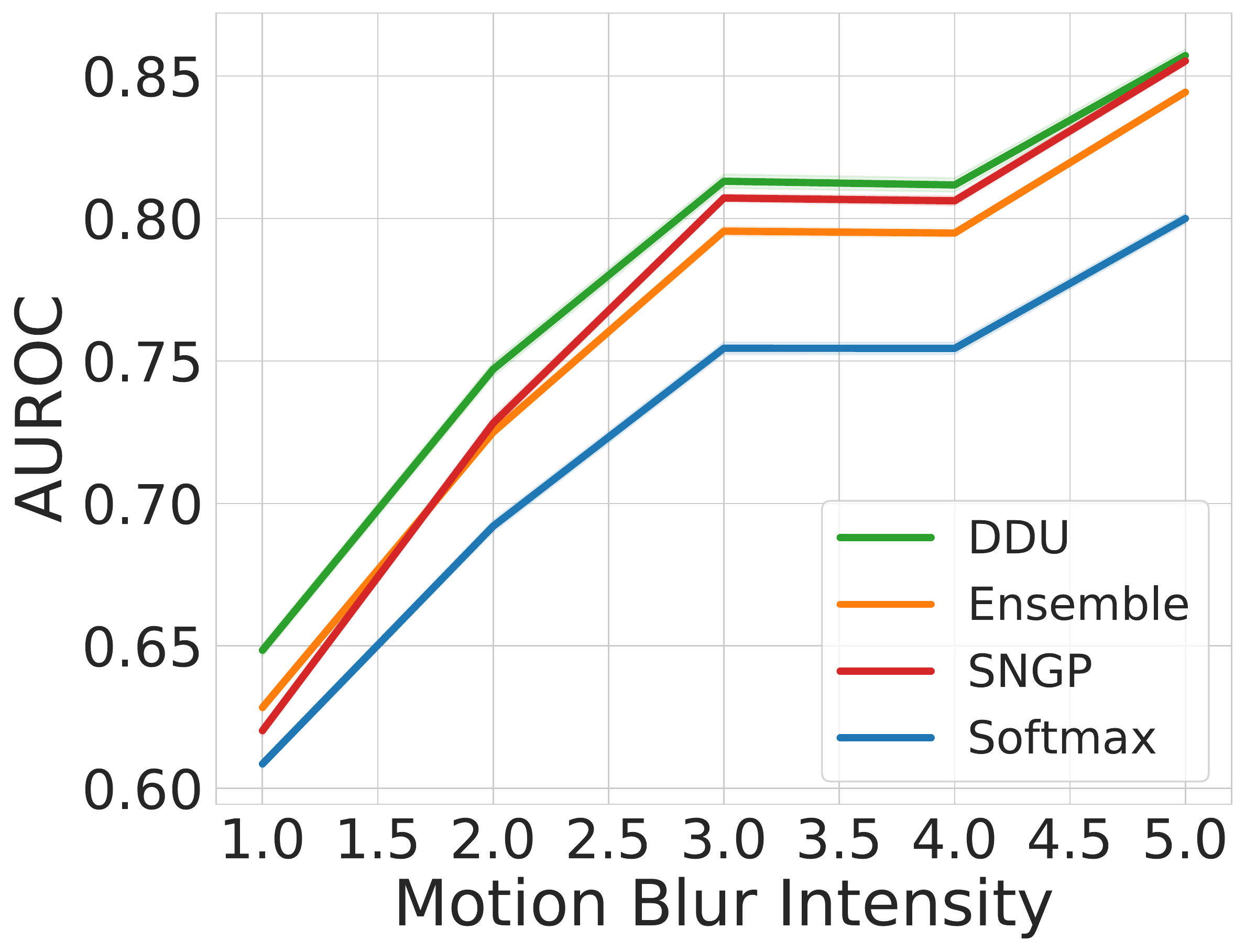}
    \end{subfigure}
    \begin{subfigure}{0.18\linewidth}
        \centering
        \includegraphics[width=\linewidth]{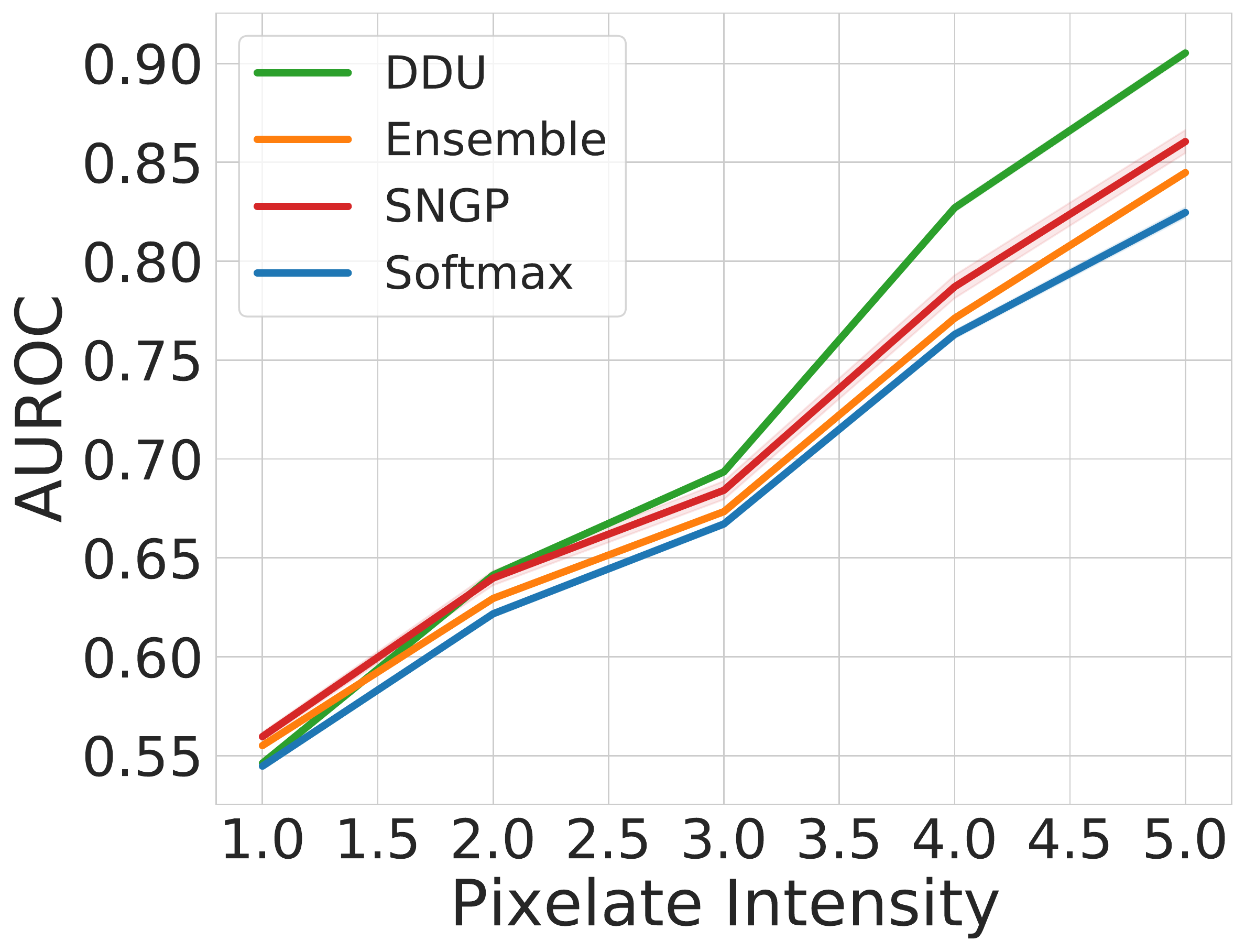}
    \end{subfigure}
    \begin{subfigure}{0.18\linewidth}
        \centering
        \includegraphics[width=\linewidth]{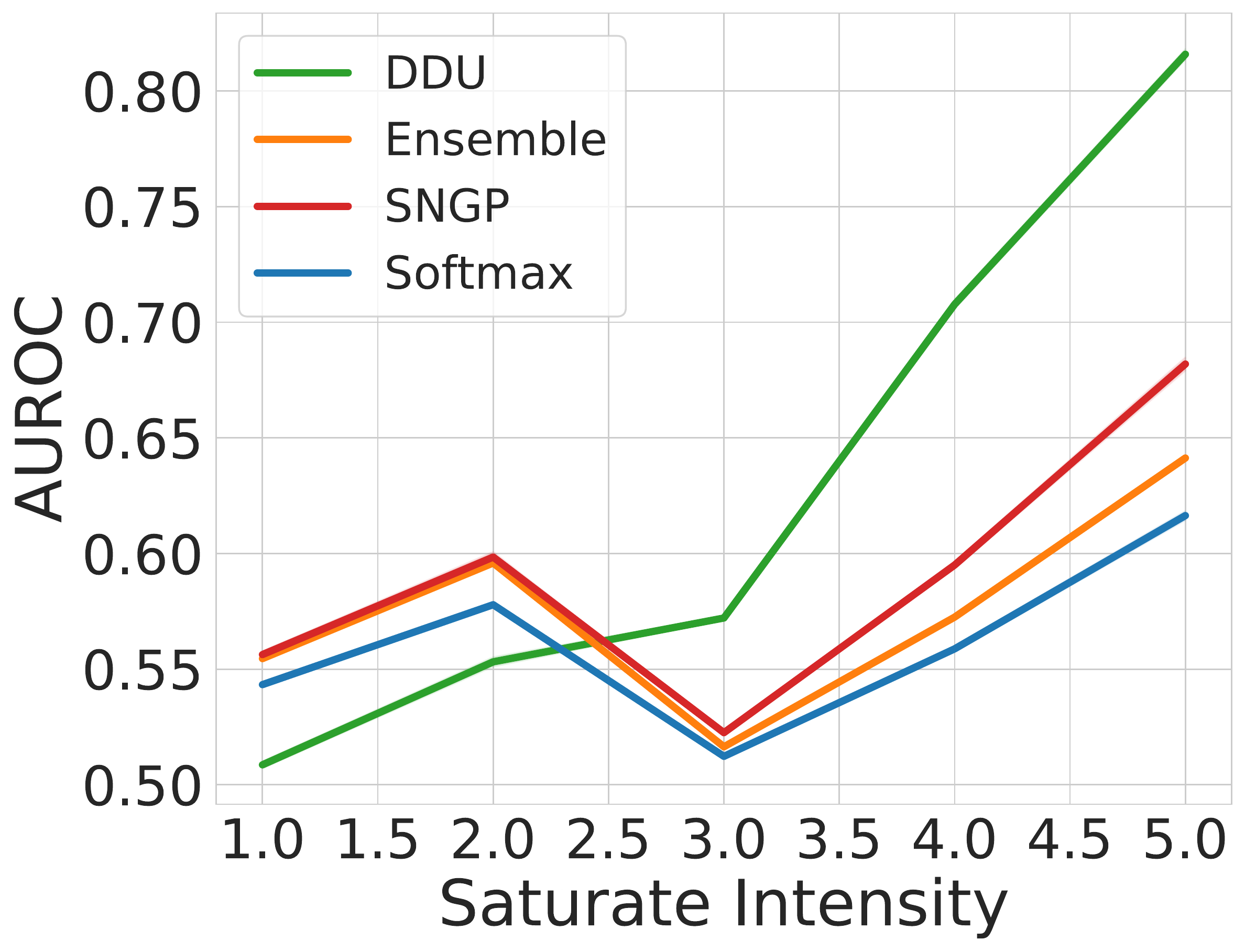}
    \end{subfigure}
    \begin{subfigure}{0.18\linewidth}
        \centering
        \includegraphics[width=\linewidth]{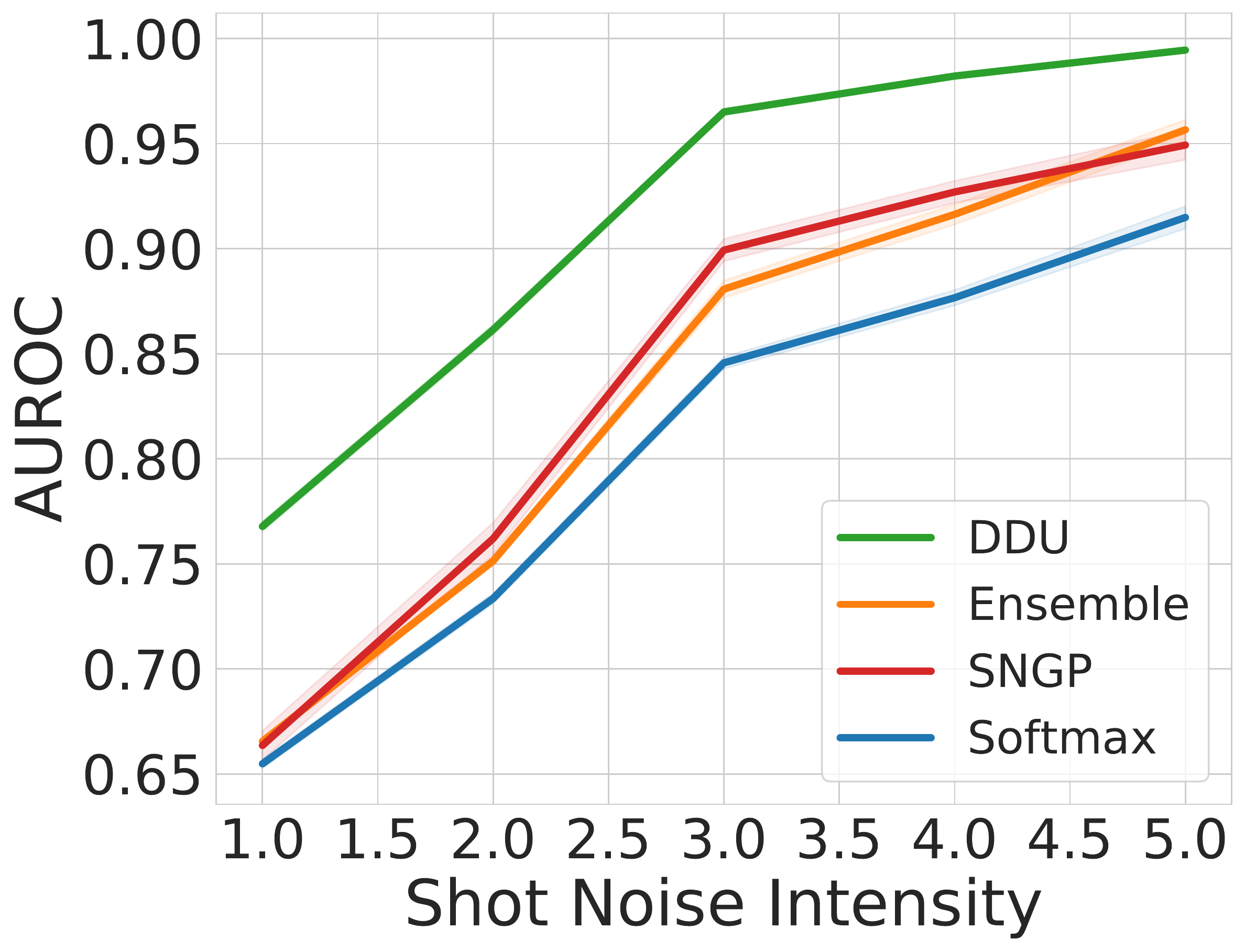}
    \end{subfigure}
    \begin{subfigure}{0.18\linewidth}
        \centering
        \includegraphics[width=\linewidth]{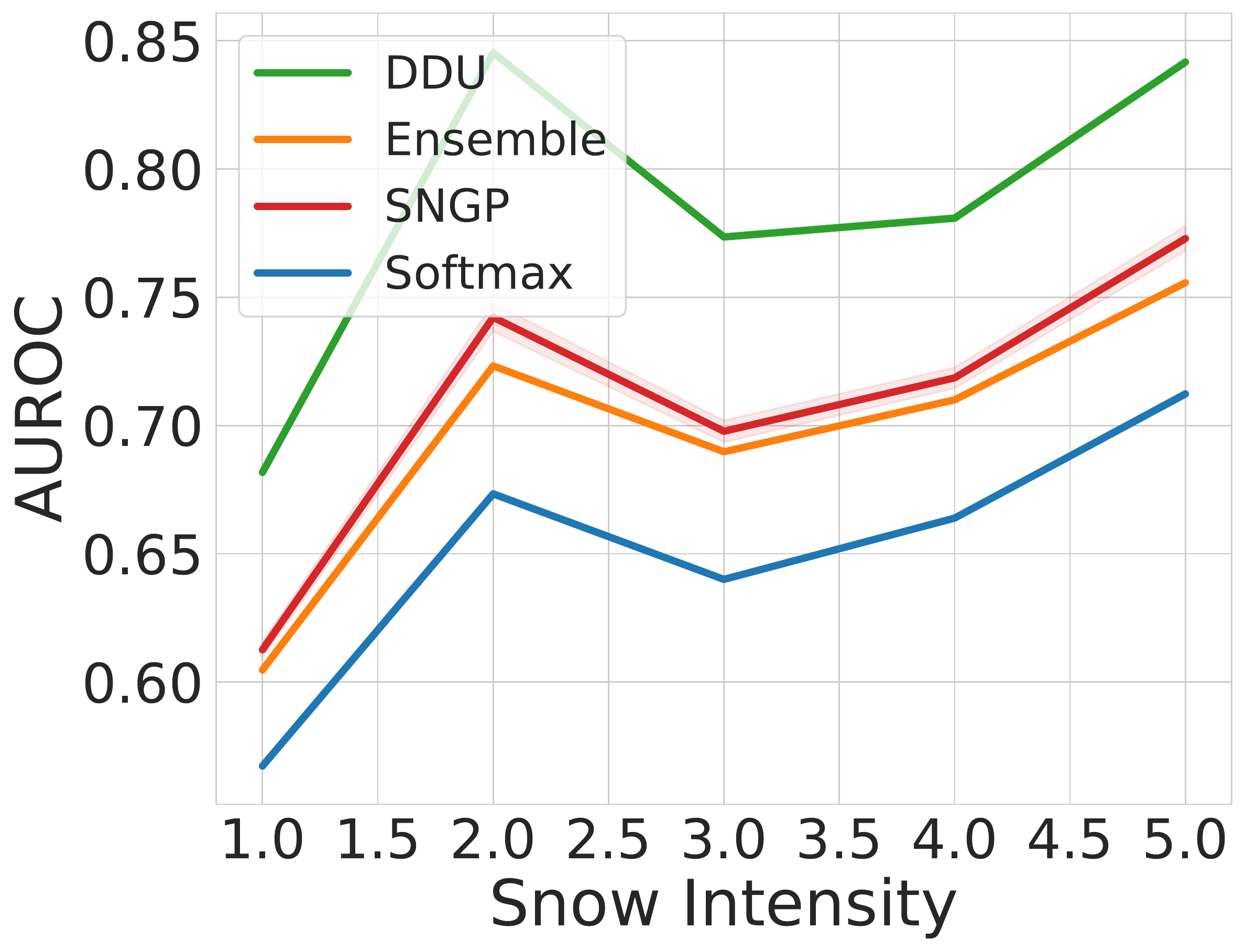}
    \end{subfigure}
    \begin{subfigure}{0.18\linewidth}
        \centering
        \includegraphics[width=\linewidth]{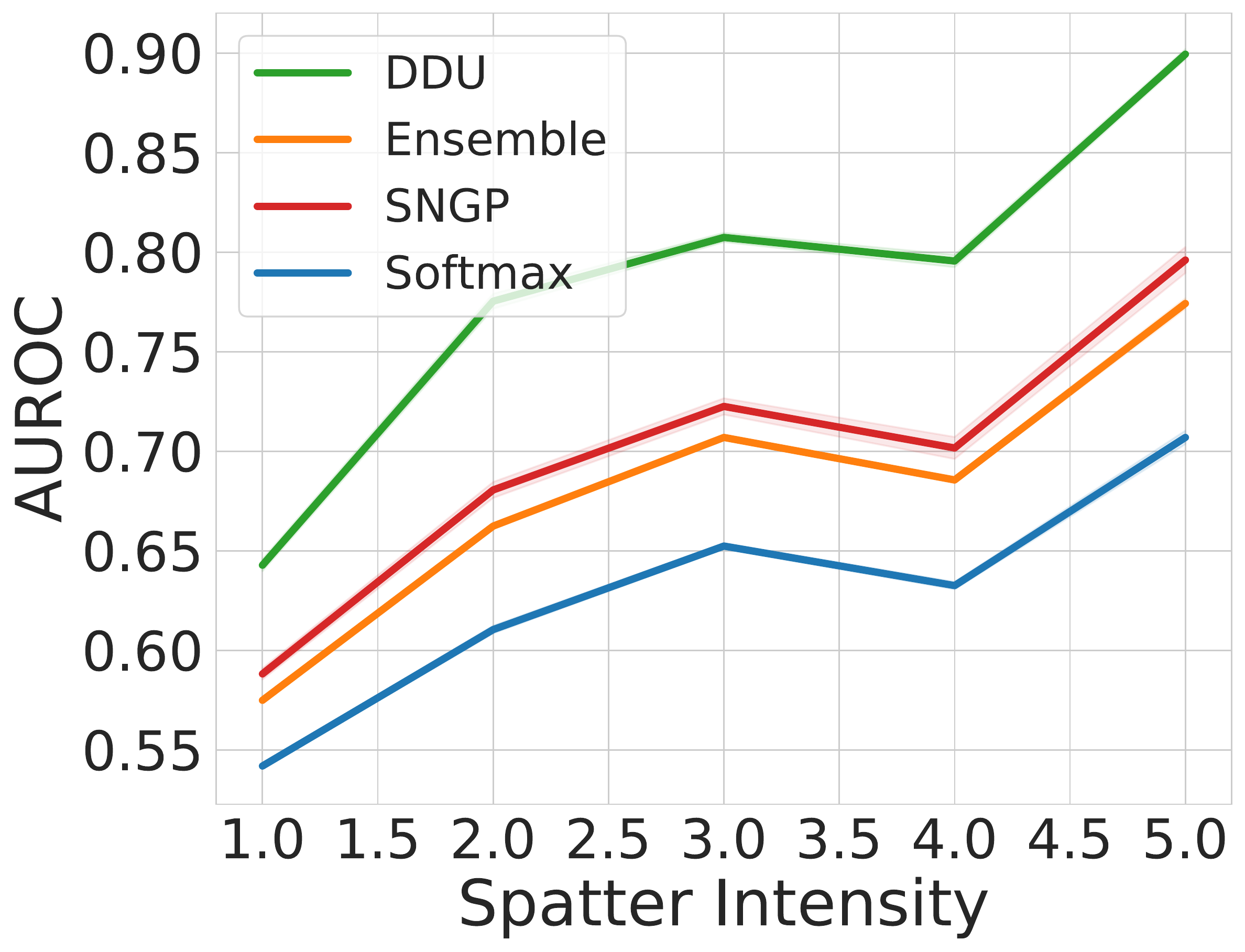}
    \end{subfigure}
    \begin{subfigure}{0.18\linewidth}
        \centering
        \includegraphics[width=\linewidth]{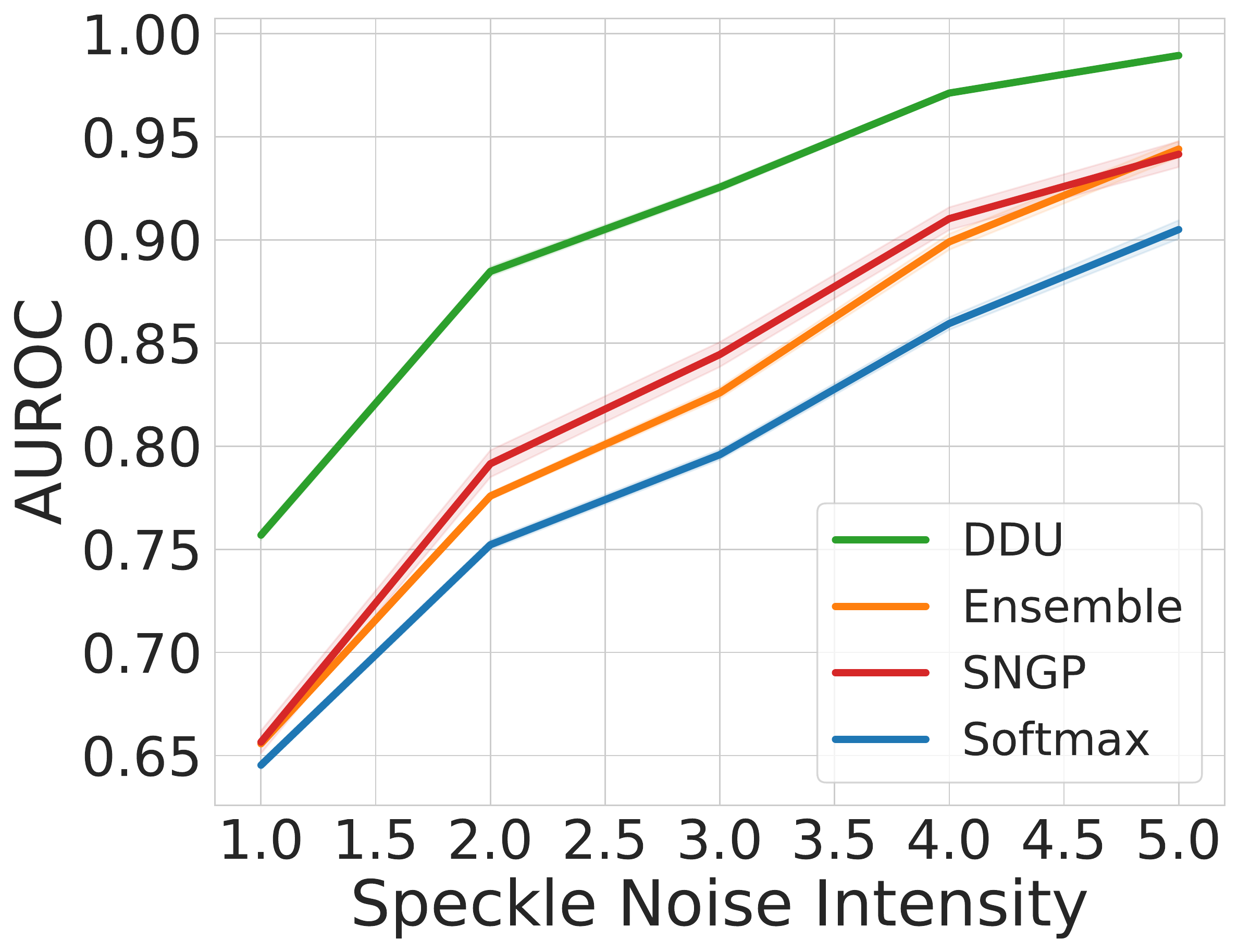}
    \end{subfigure}
    \begin{subfigure}{0.18\linewidth}
        \centering
        \includegraphics[width=\linewidth]{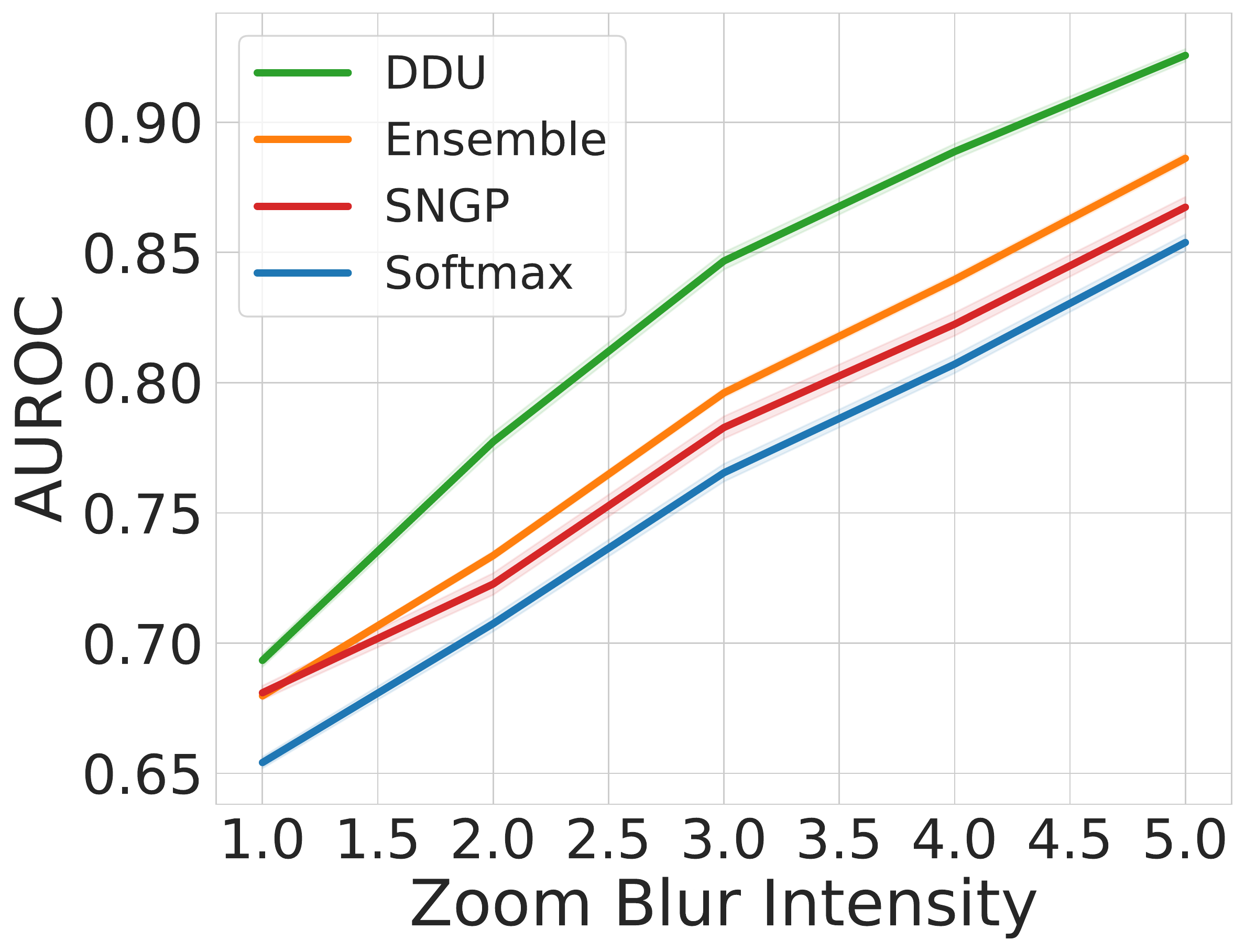}
    \end{subfigure}
    \caption{
    AUROC vs corruption intensity for all corruption types in CIFAR-10-C with Wide-ResNet-28-10 as the architecture and baselines: Softmax Entropy, Ensemble (using Predictive Entropy as uncertainty), SNGP and DDU feature density.
    }
    \label{fig:cifar10_c_results_wide_resnet}
\end{figure}

\begin{figure}[!t]
    \centering
    \begin{subfigure}{0.18\linewidth}
        \centering
        \includegraphics[width=\linewidth]{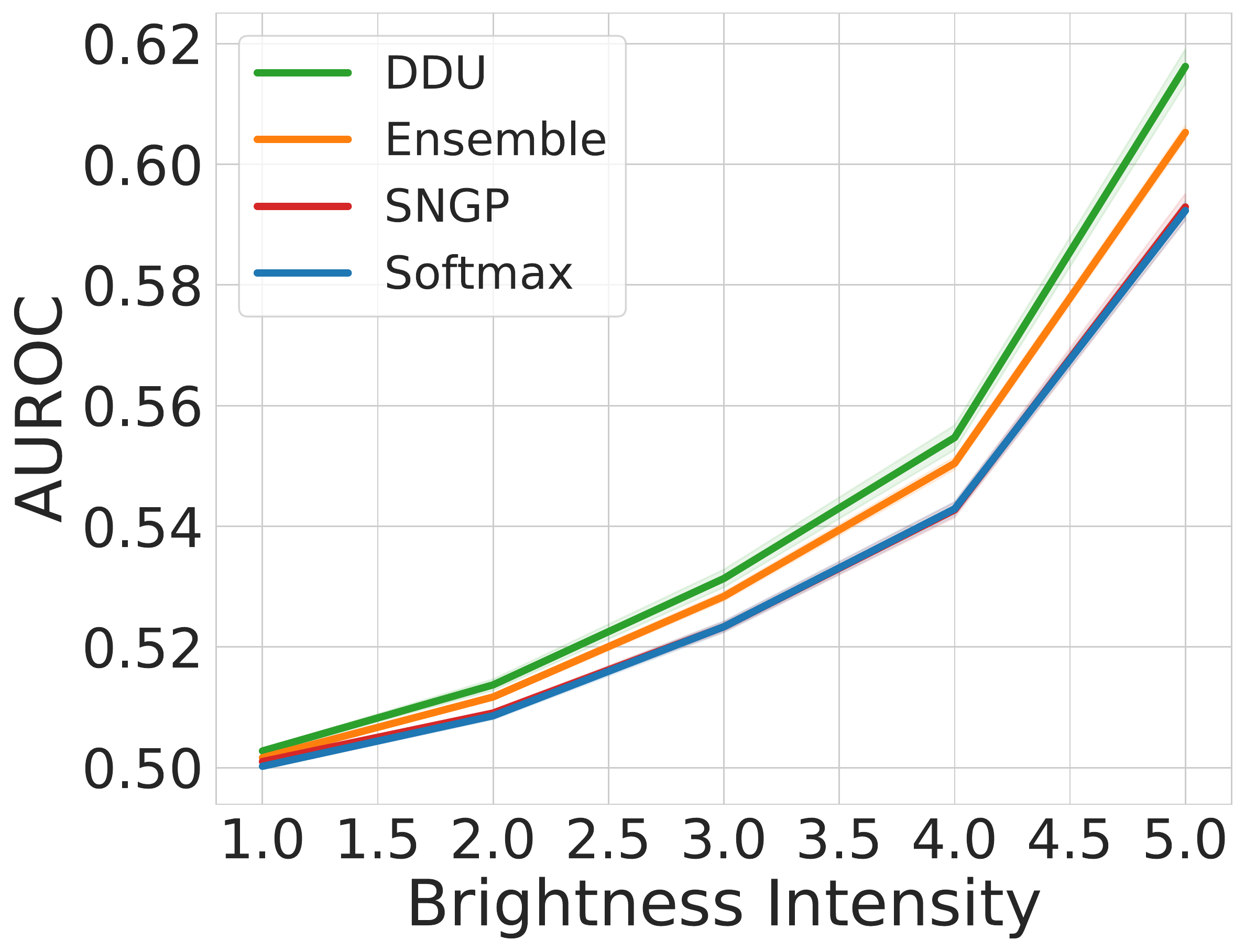}
    \end{subfigure}
    \begin{subfigure}{0.18\linewidth}
        \centering
        \includegraphics[width=\linewidth]{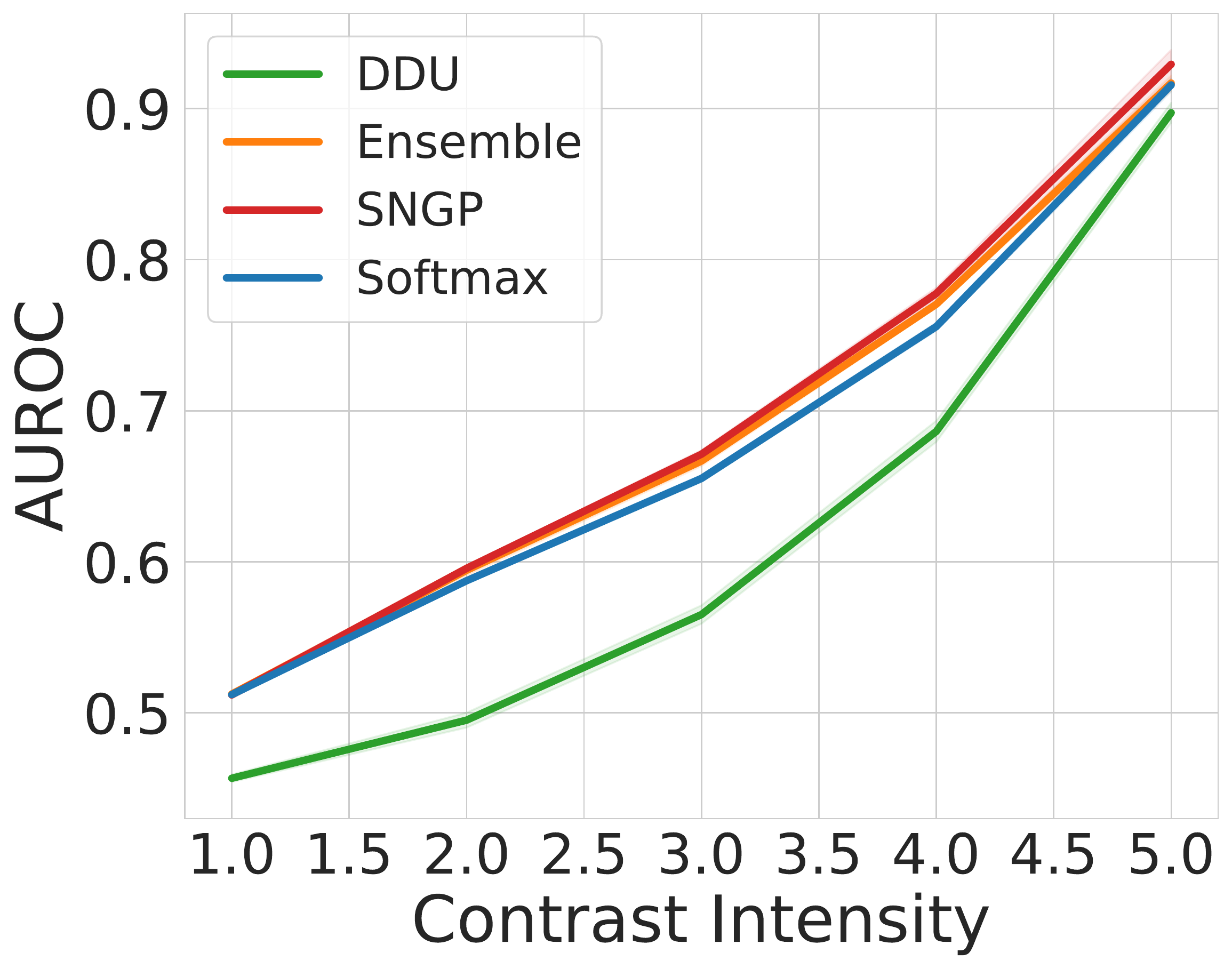}
    \end{subfigure} 
    \begin{subfigure}{0.18\linewidth}
        \centering
        \includegraphics[width=\linewidth]{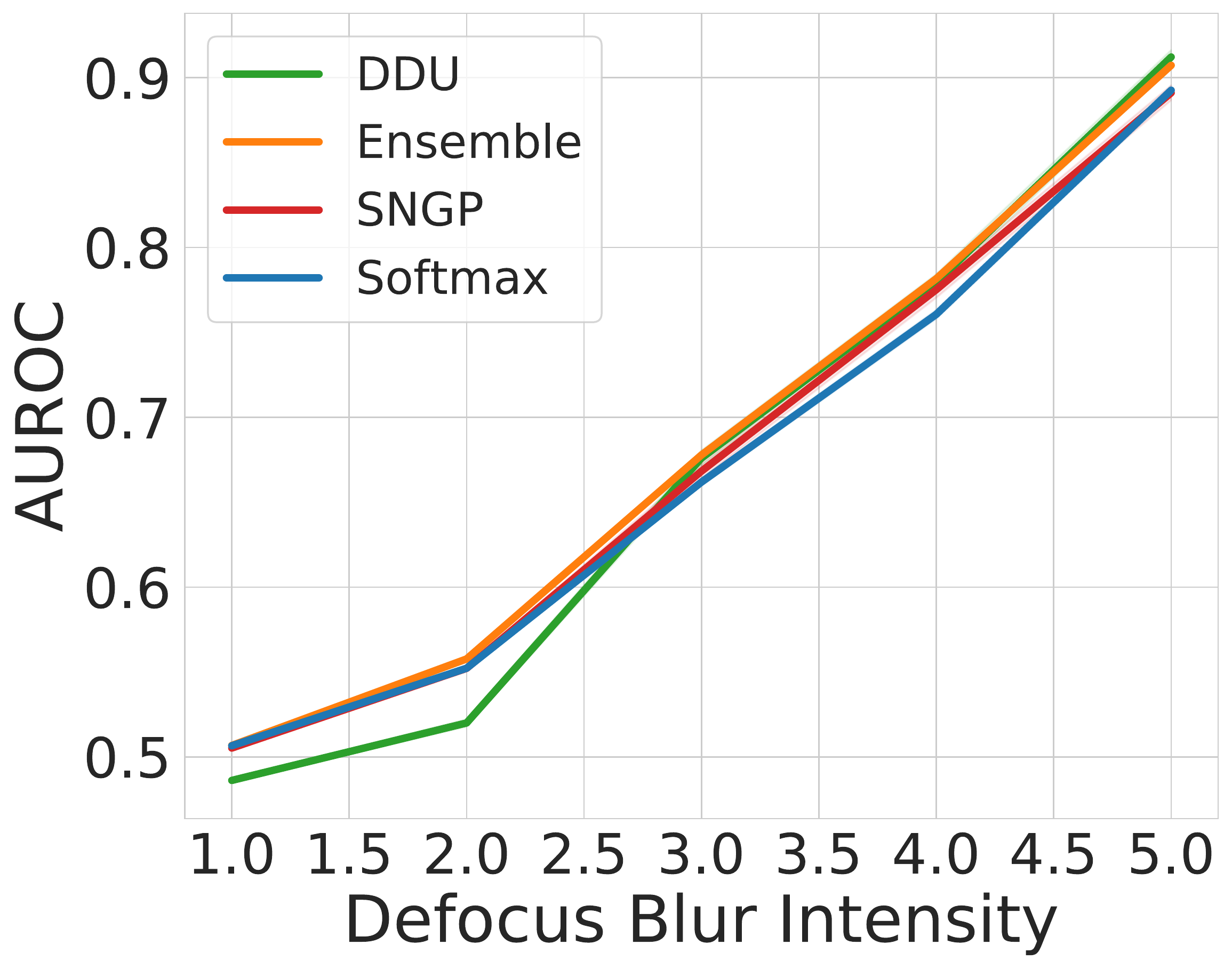}
    \end{subfigure} 
    \begin{subfigure}{0.18\linewidth}
        \centering
        \includegraphics[width=\linewidth]{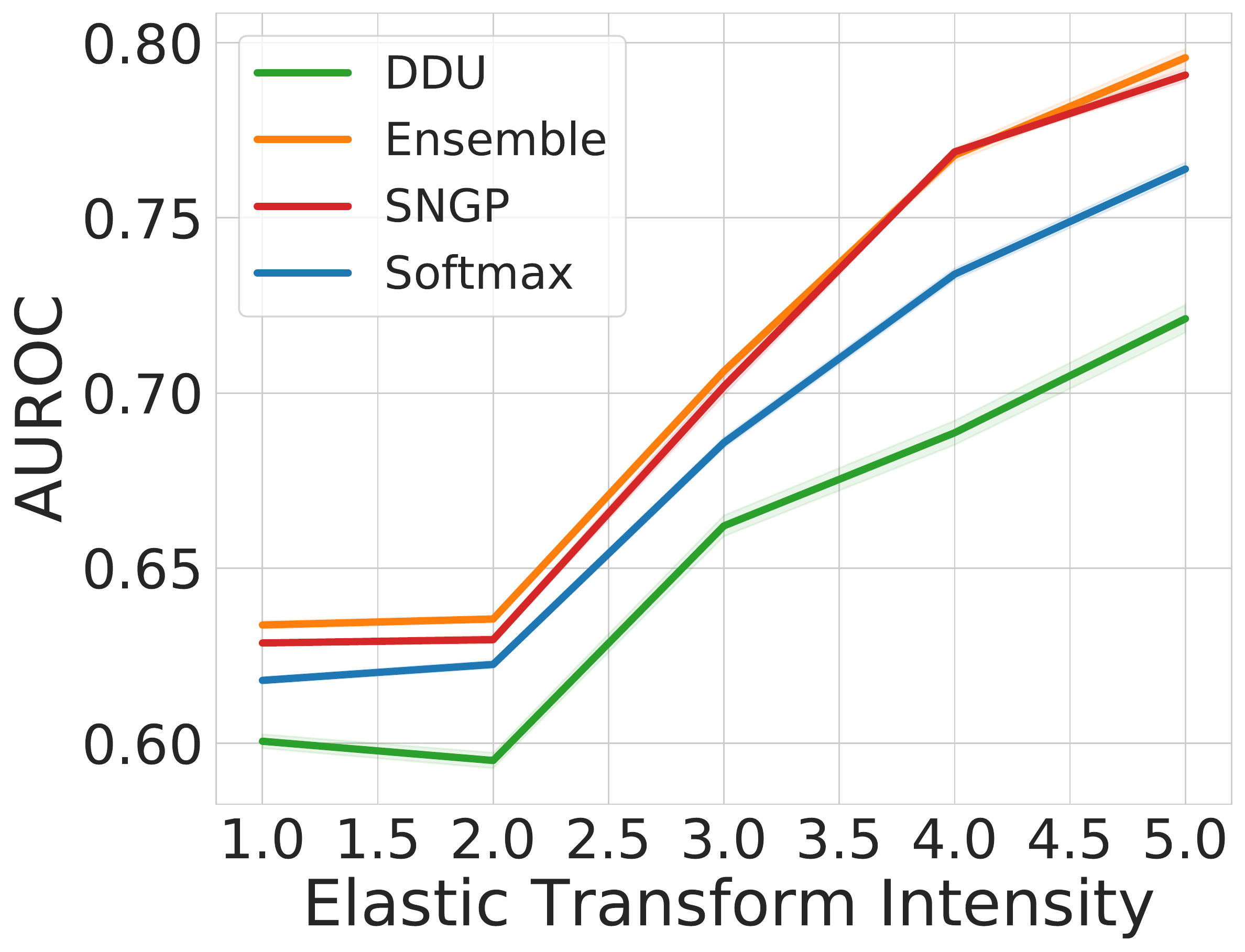}
    \end{subfigure}
    \begin{subfigure}{0.18\linewidth}
        \centering
        \includegraphics[width=\linewidth]{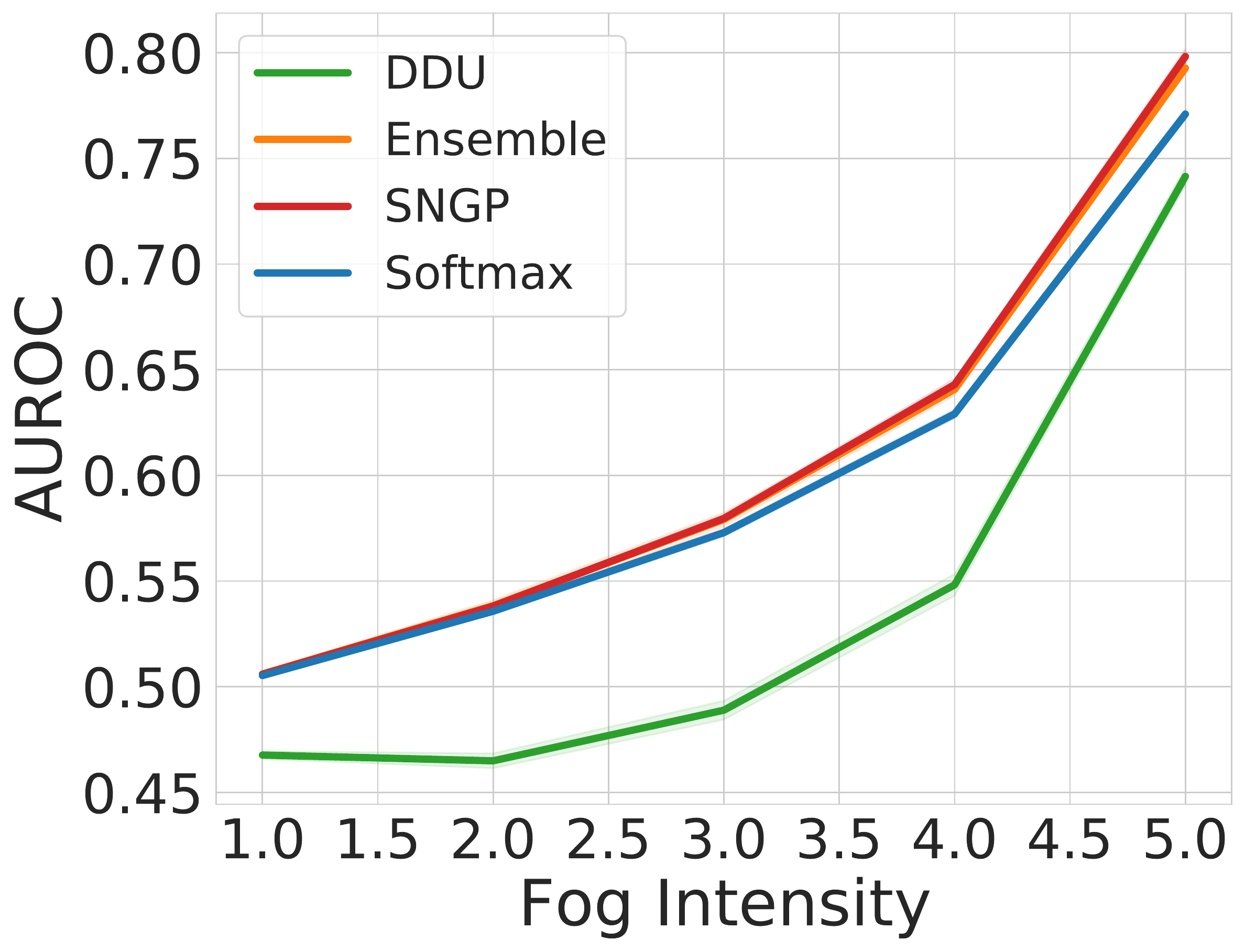}
    \end{subfigure}
    \begin{subfigure}{0.18\linewidth}
        \centering
        \includegraphics[width=\linewidth]{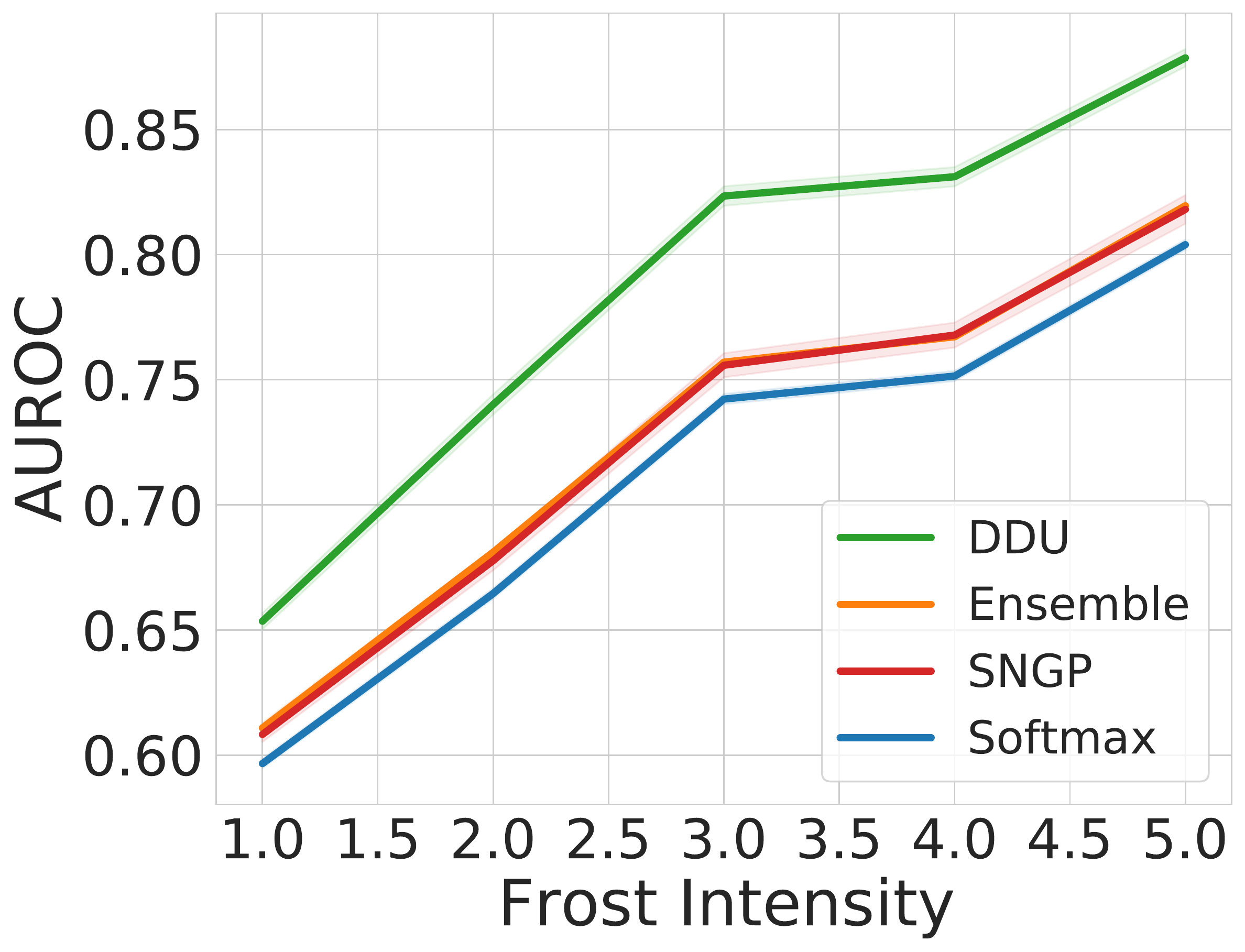}
    \end{subfigure}
    \begin{subfigure}{0.18\linewidth}
        \centering
        \includegraphics[width=\linewidth]{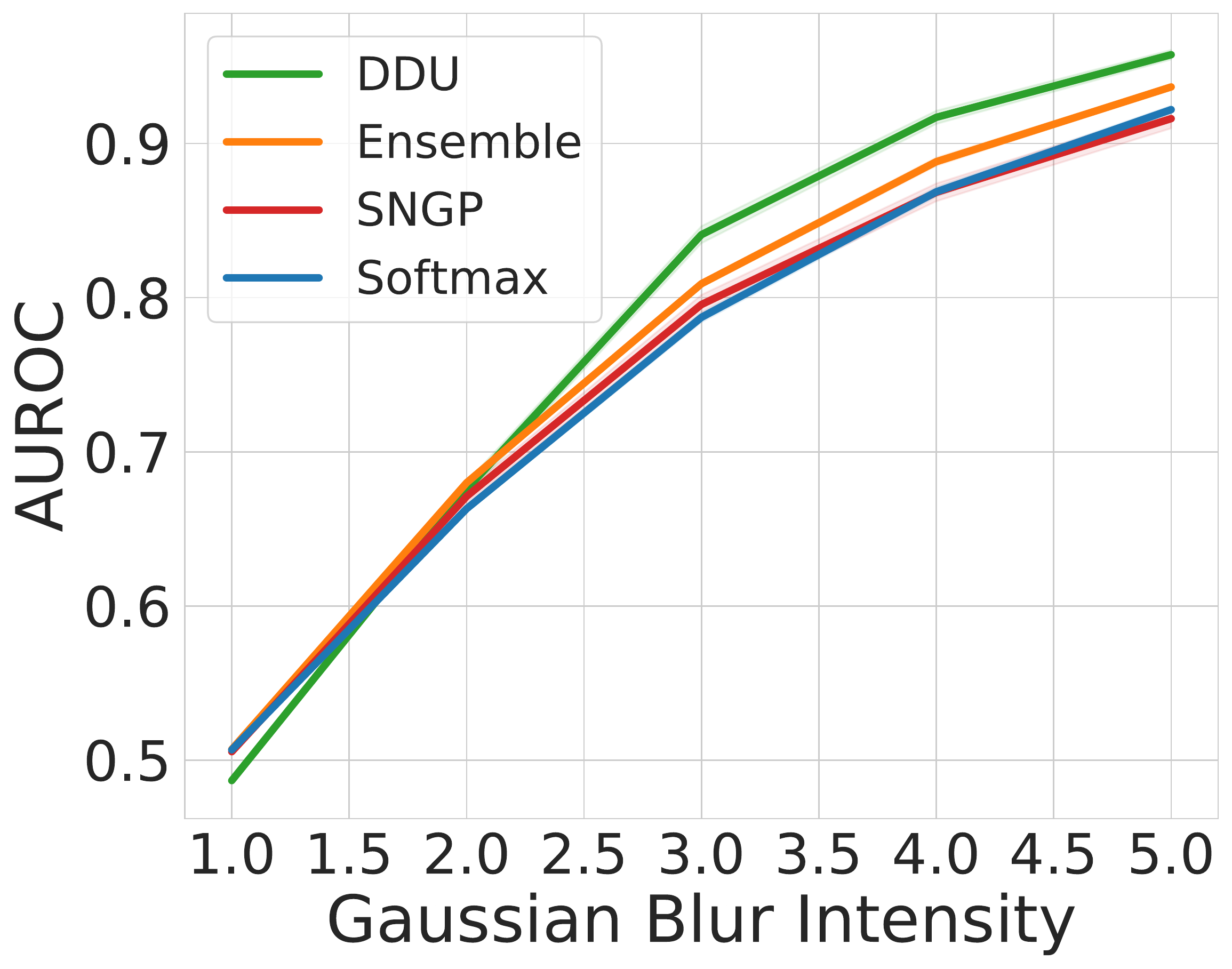}
    \end{subfigure}
    \begin{subfigure}{0.18\linewidth}
        \centering
        \includegraphics[width=\linewidth]{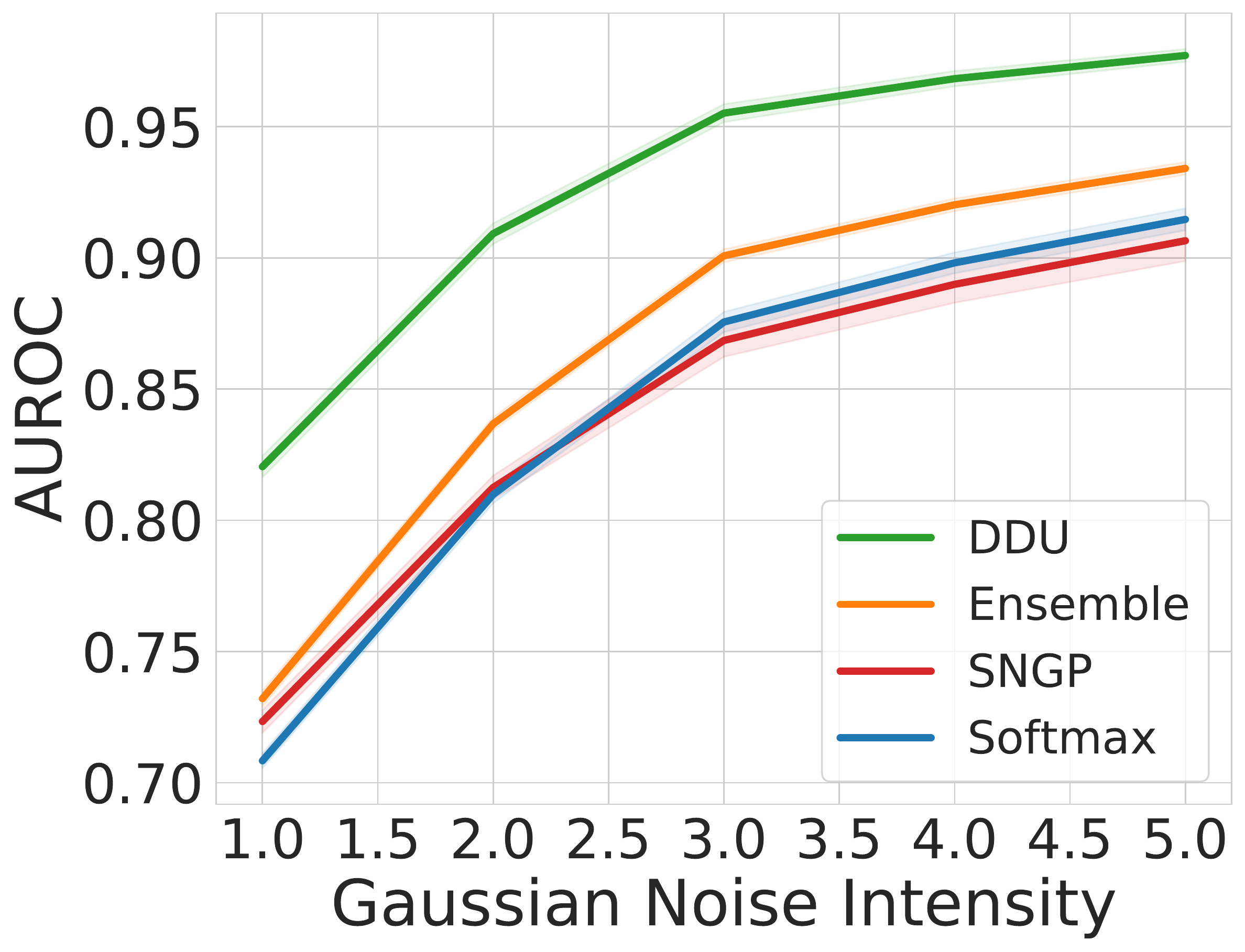}
    \end{subfigure}
    \begin{subfigure}{0.18\linewidth}
        \centering
        \includegraphics[width=\linewidth]{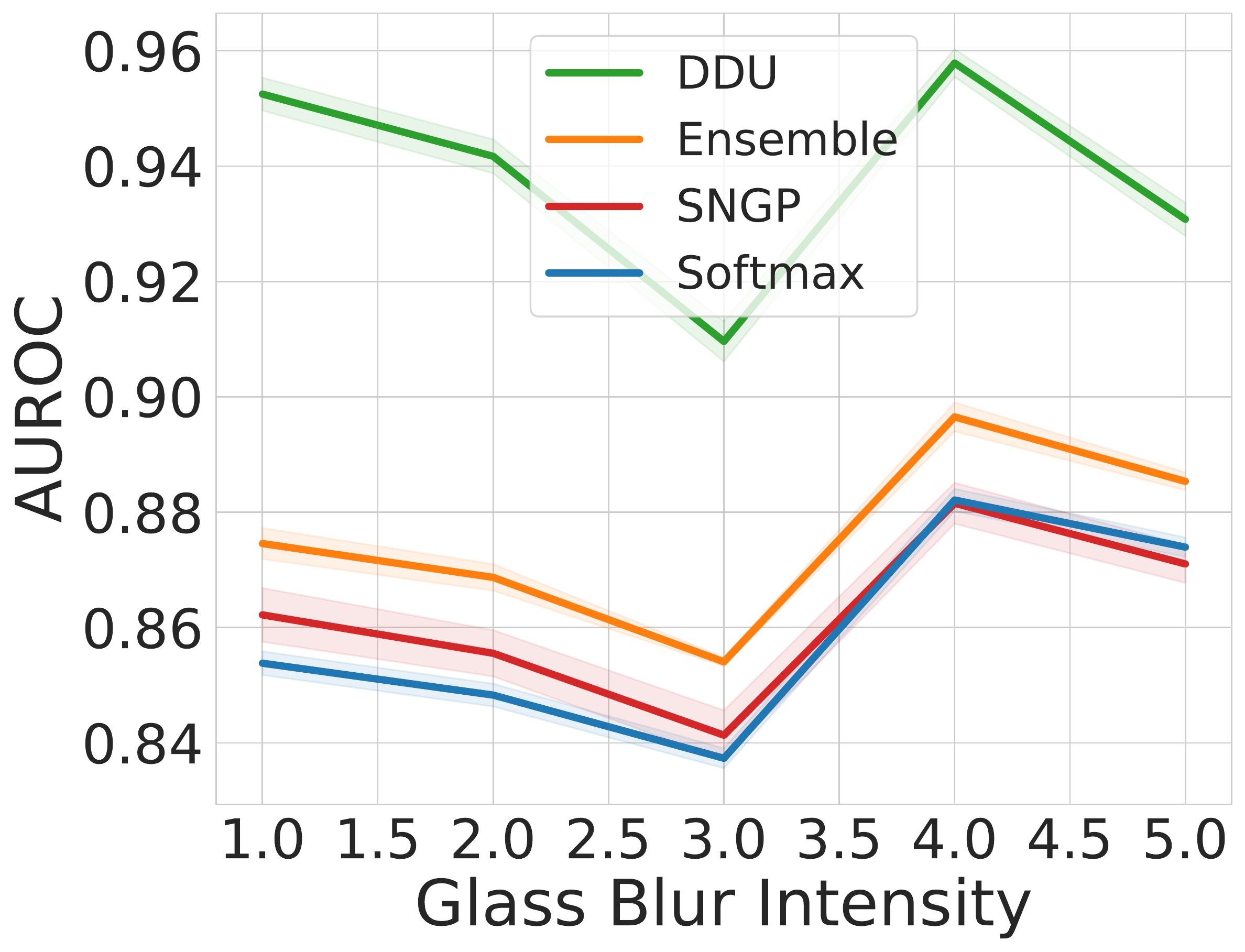}
    \end{subfigure}
    \begin{subfigure}{0.18\linewidth}
        \centering
        \includegraphics[width=\linewidth]{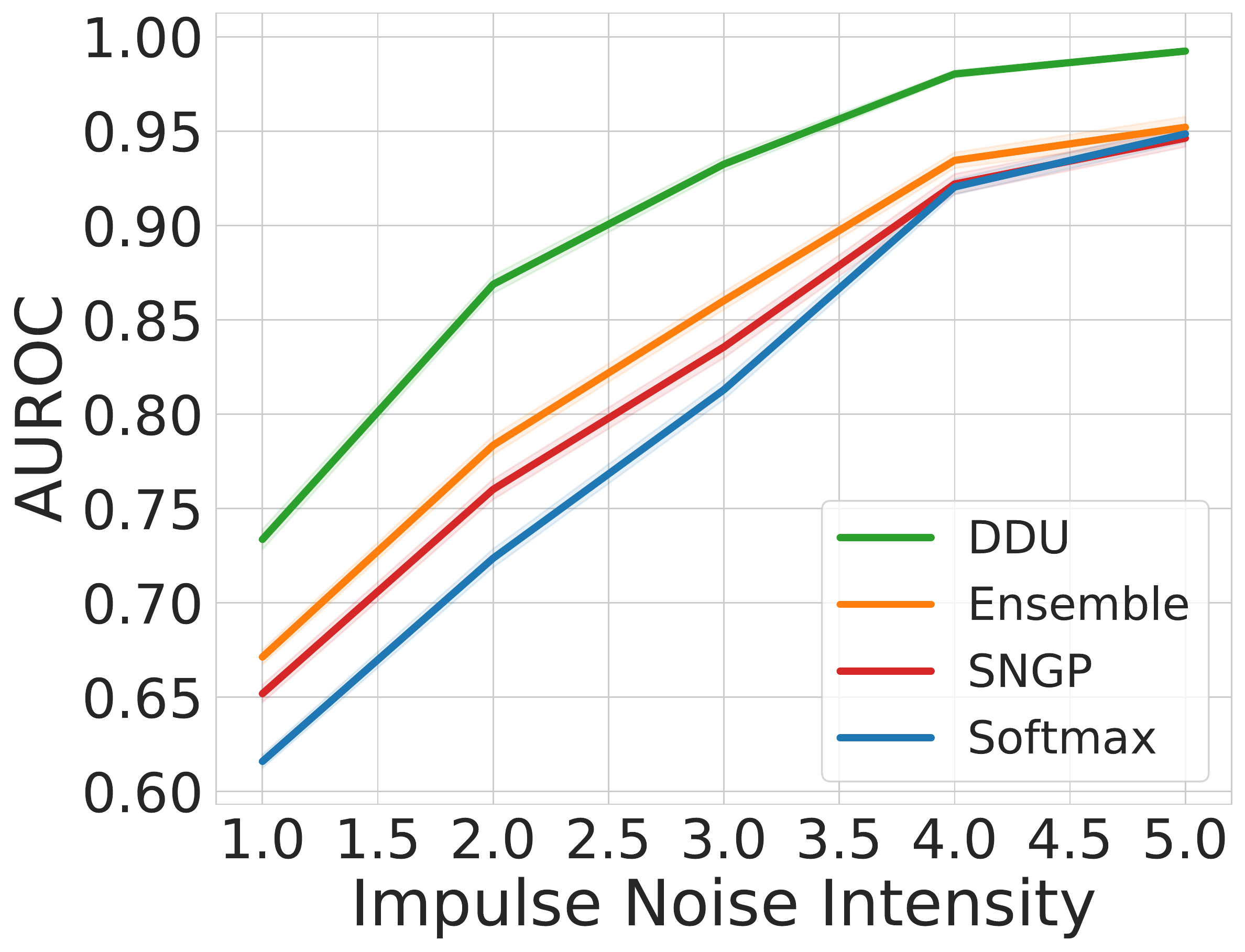}
    \end{subfigure}
    \begin{subfigure}{0.18\linewidth}
        \centering
        \includegraphics[width=\linewidth]{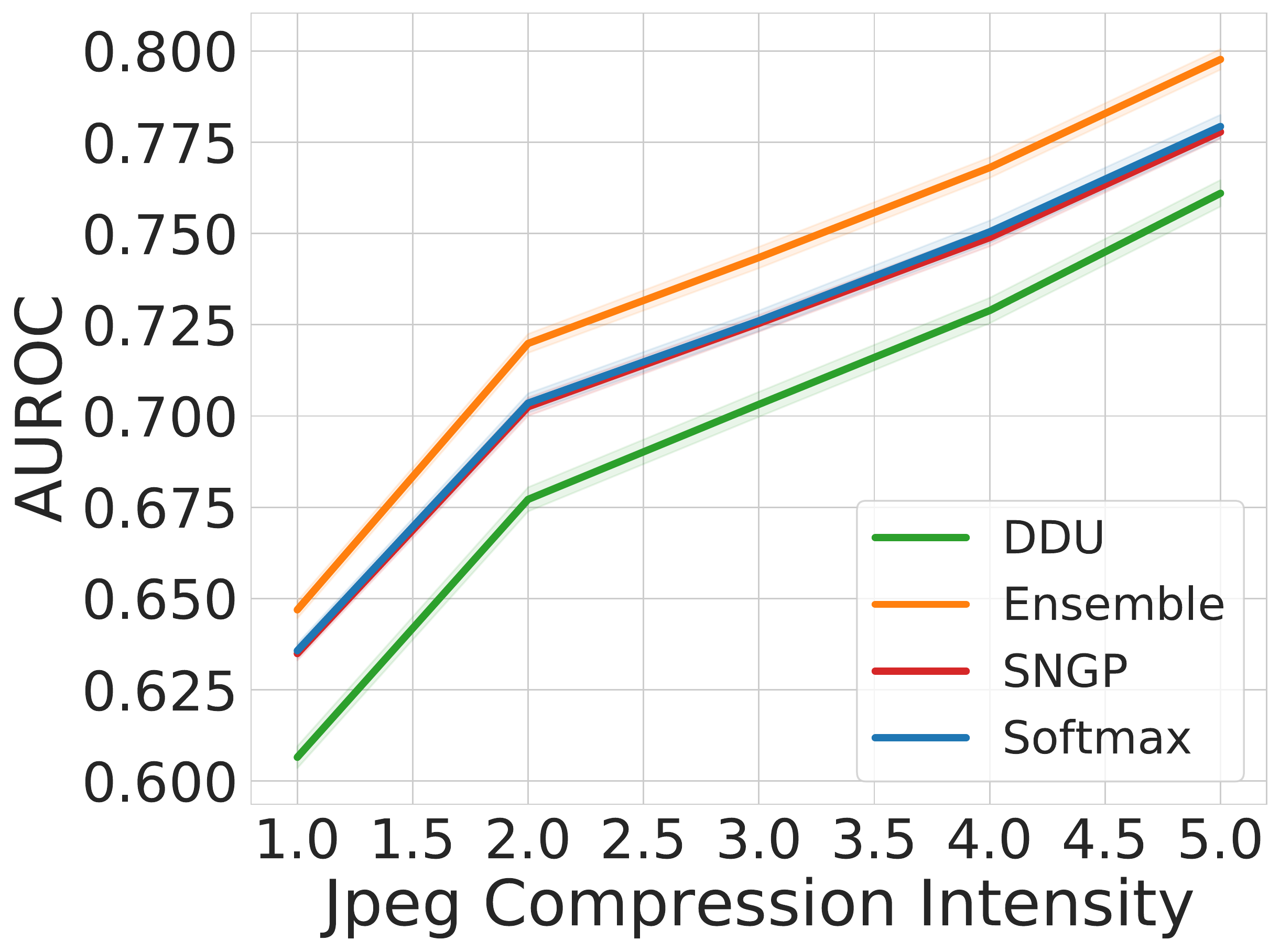}
    \end{subfigure}
    \begin{subfigure}{0.18\linewidth}
        \centering
        \includegraphics[width=\linewidth]{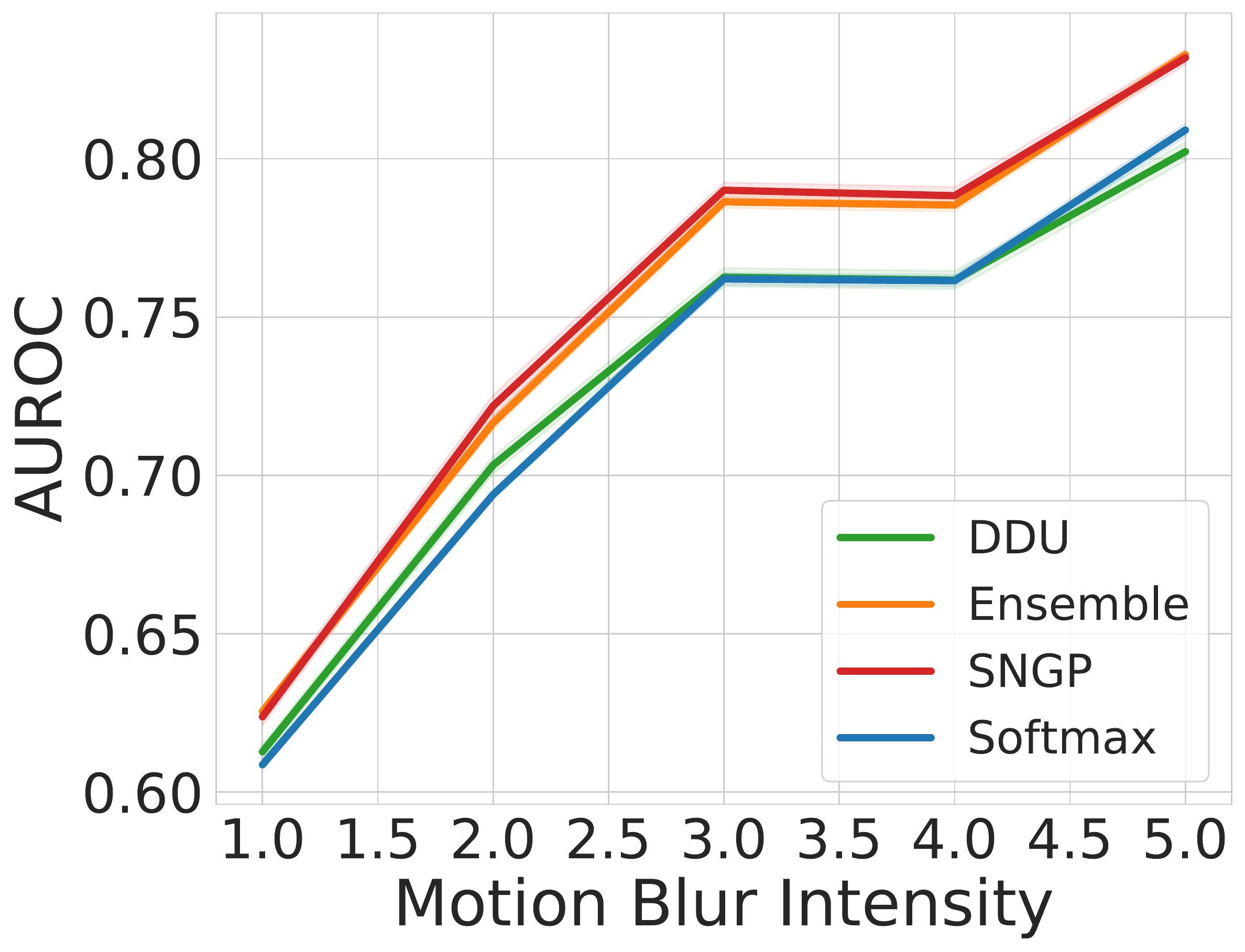}
    \end{subfigure}
    \begin{subfigure}{0.18\linewidth}
        \centering
        \includegraphics[width=\linewidth]{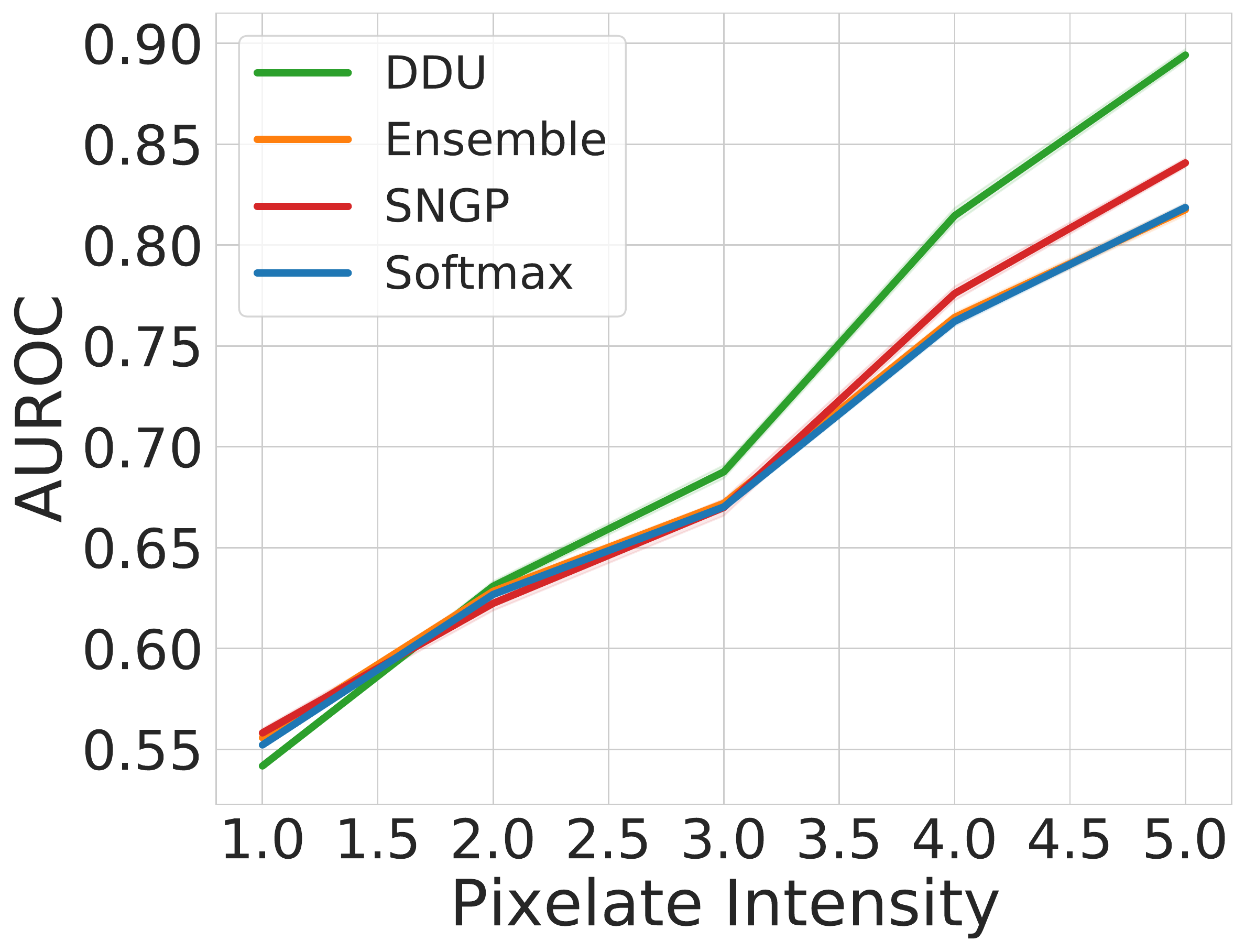}
    \end{subfigure}
    \begin{subfigure}{0.18\linewidth}
        \centering
        \includegraphics[width=\linewidth]{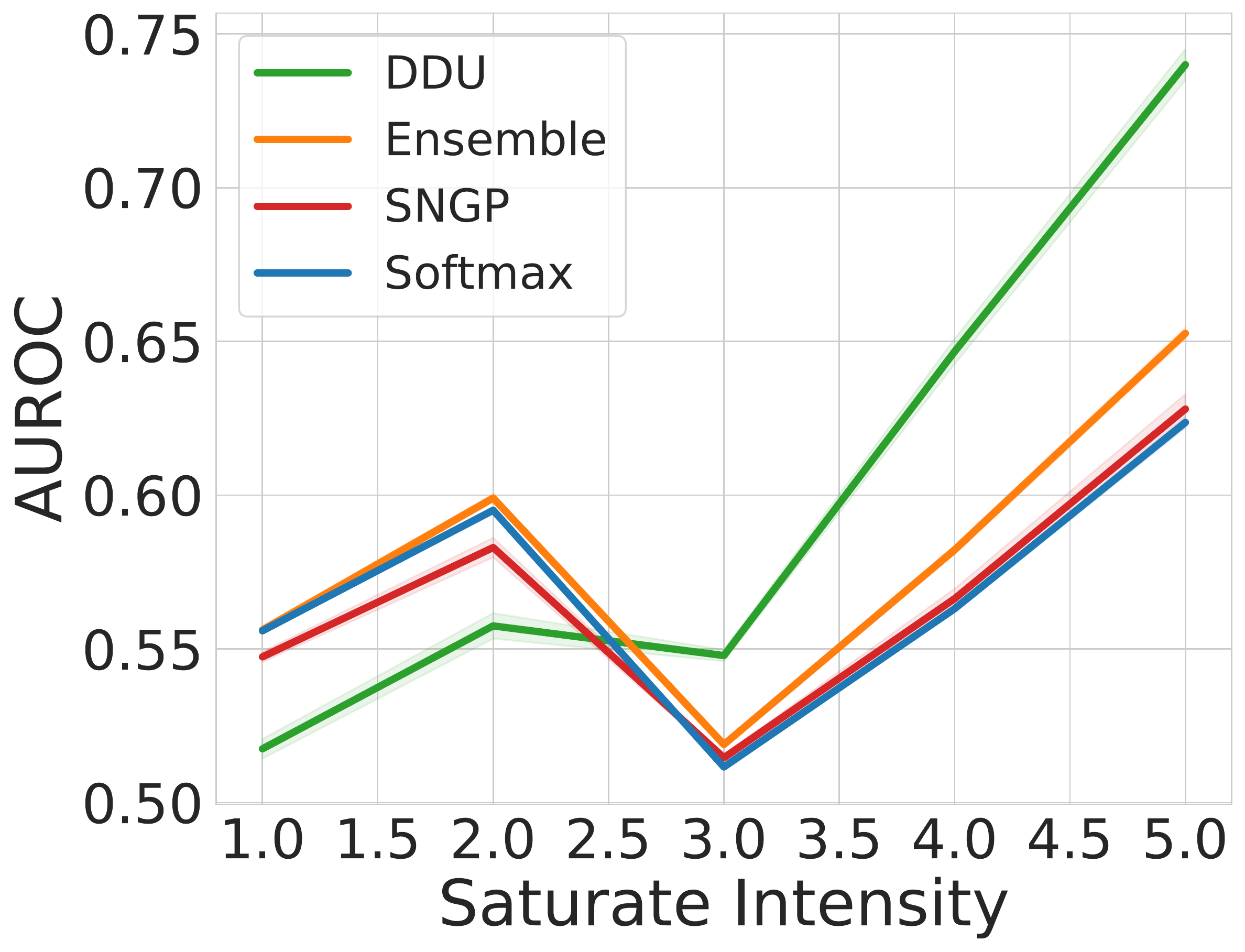}
    \end{subfigure}
    \begin{subfigure}{0.18\linewidth}
        \centering
        \includegraphics[width=\linewidth]{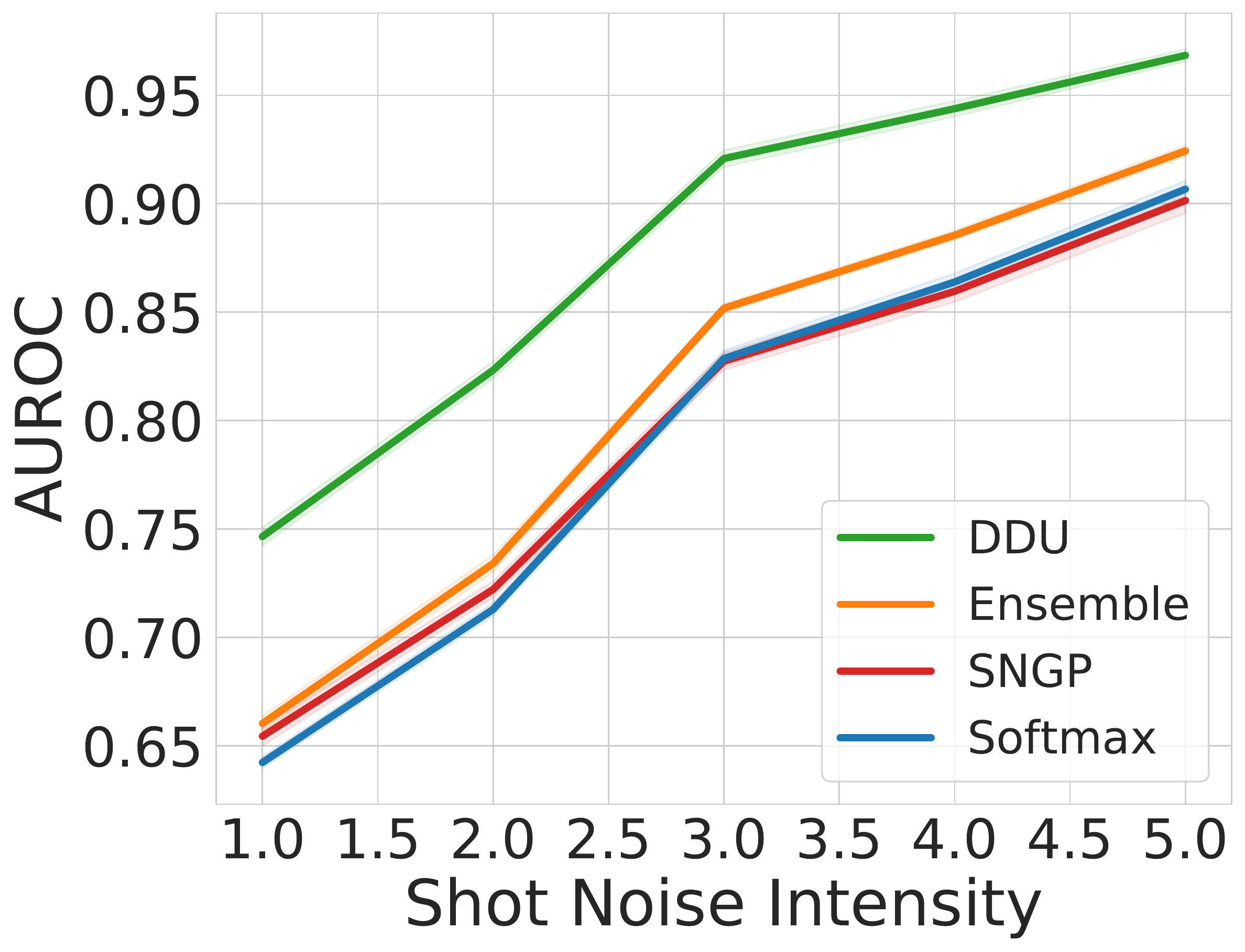}
    \end{subfigure}
    \begin{subfigure}{0.18\linewidth}
        \centering
        \includegraphics[width=\linewidth]{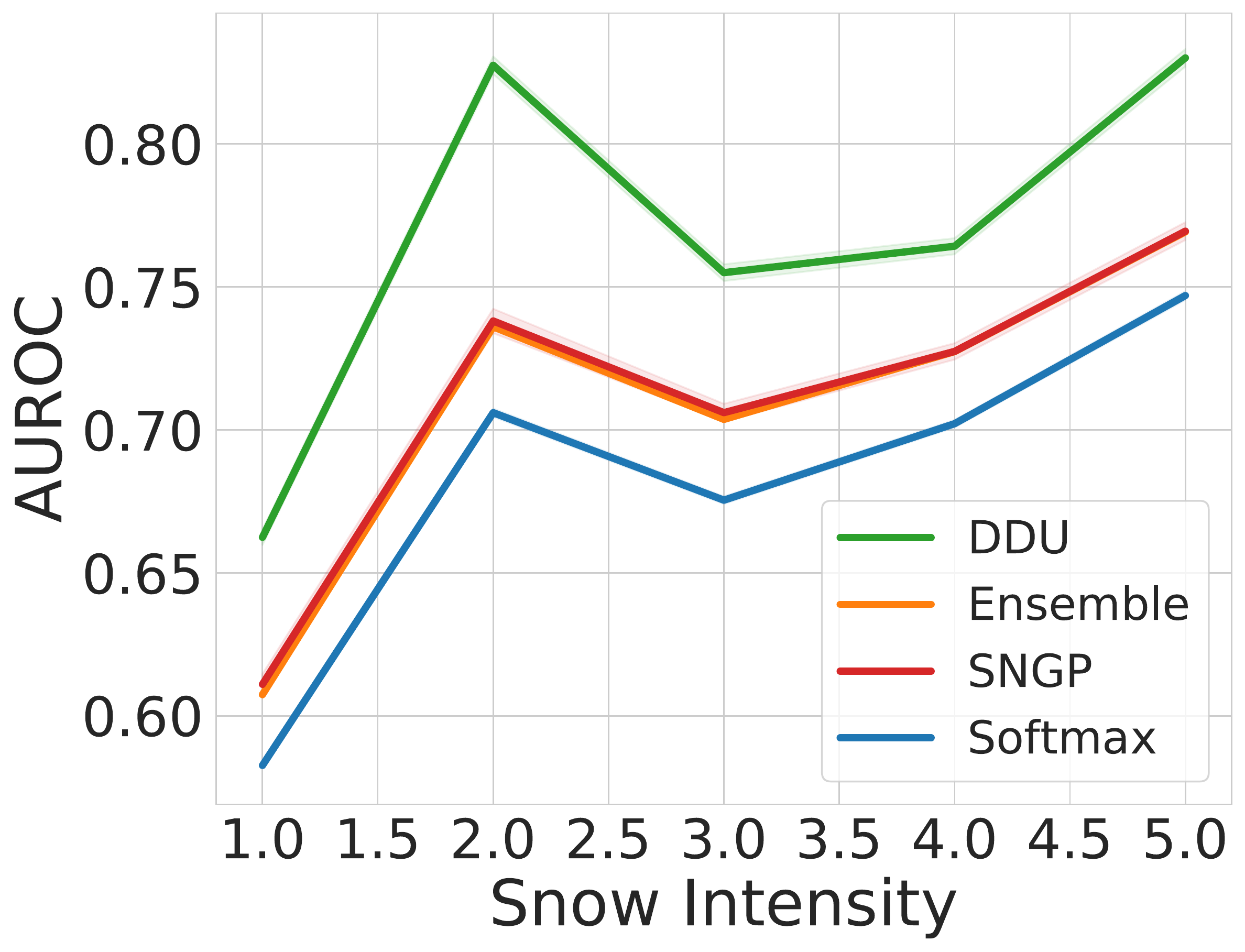}
    \end{subfigure}
    \begin{subfigure}{0.18\linewidth}
        \centering
        \includegraphics[width=\linewidth]{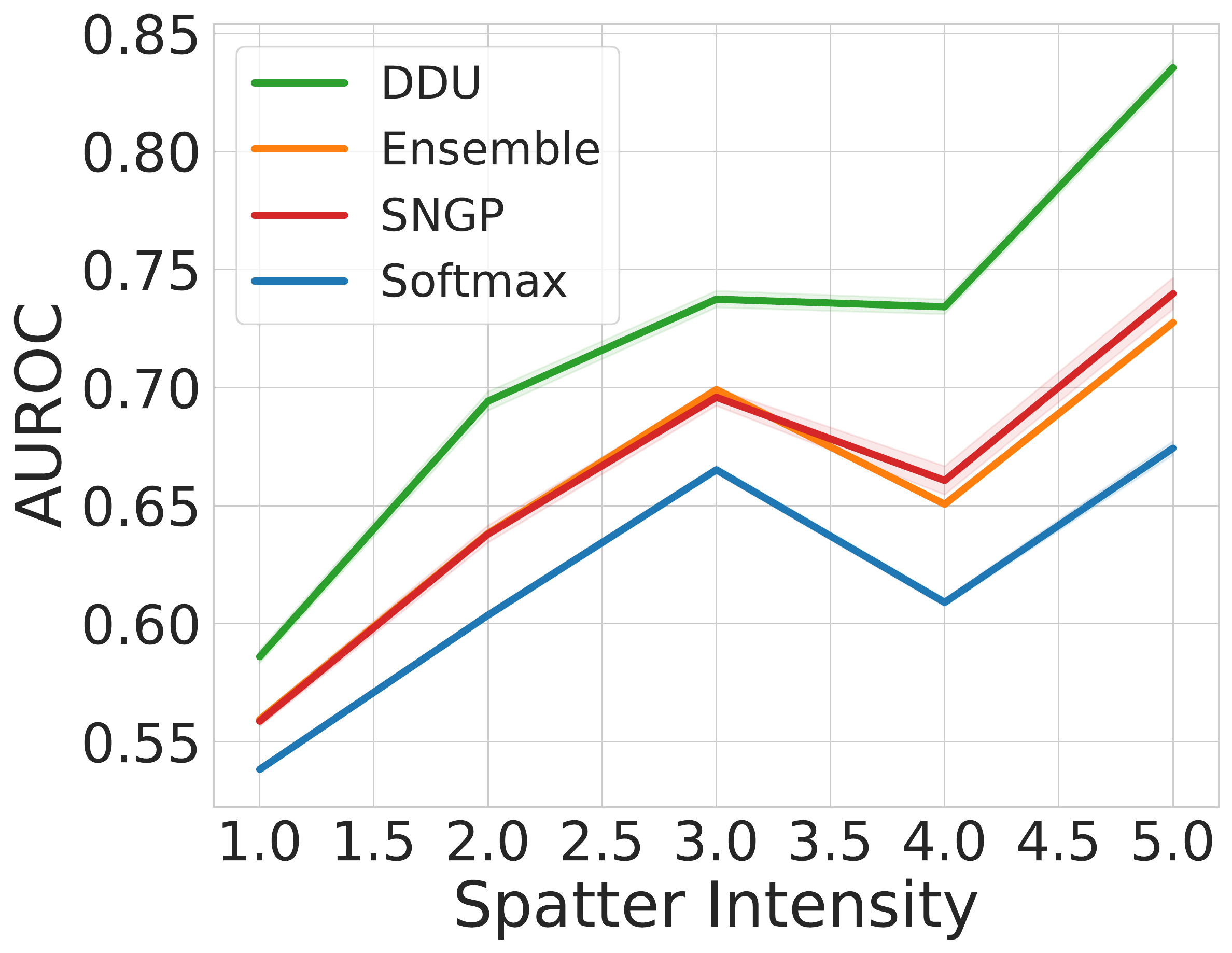}
    \end{subfigure}
    \begin{subfigure}{0.18\linewidth}
        \centering
        \includegraphics[width=\linewidth]{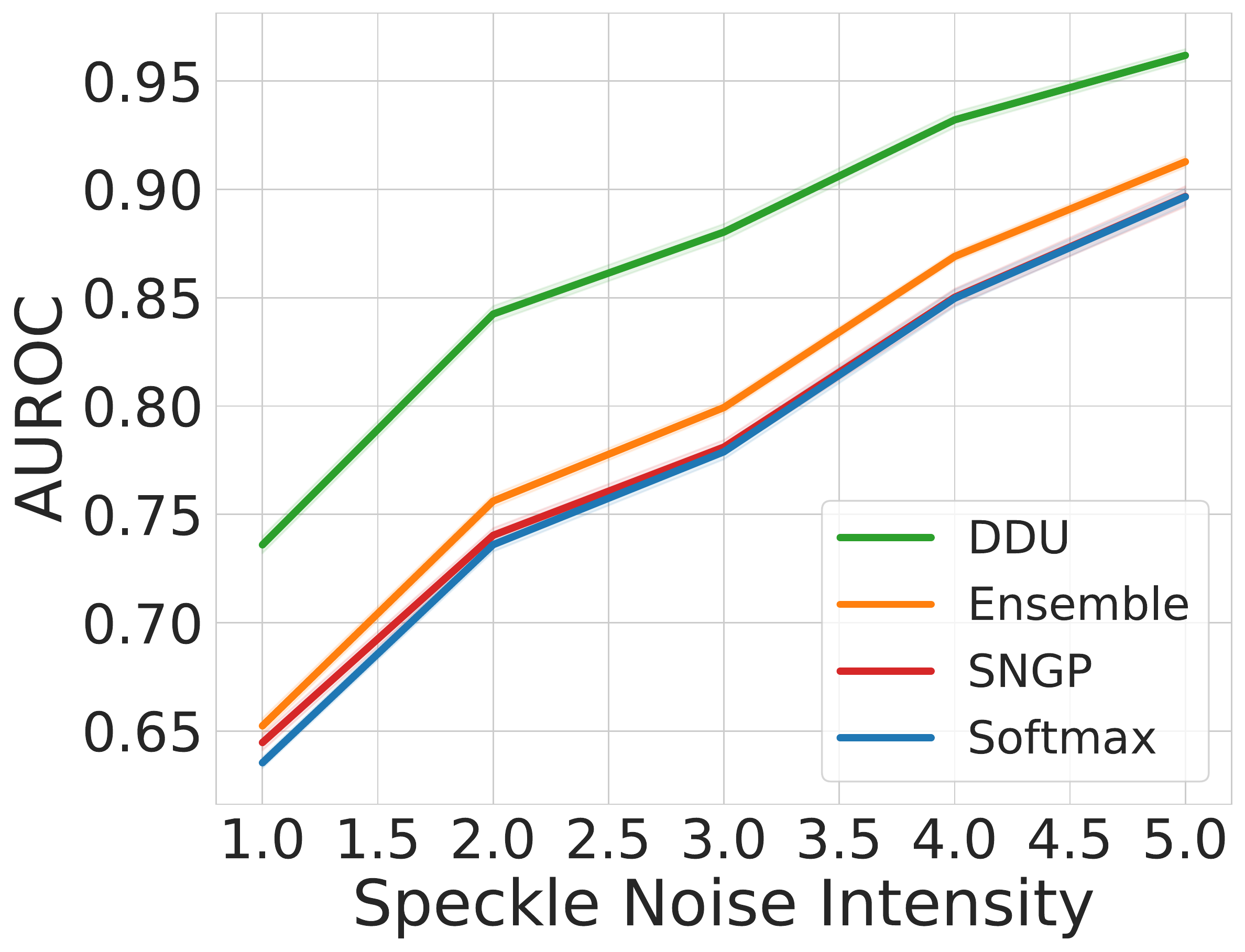}
    \end{subfigure}
    \begin{subfigure}{0.18\linewidth}
        \centering
        \includegraphics[width=\linewidth]{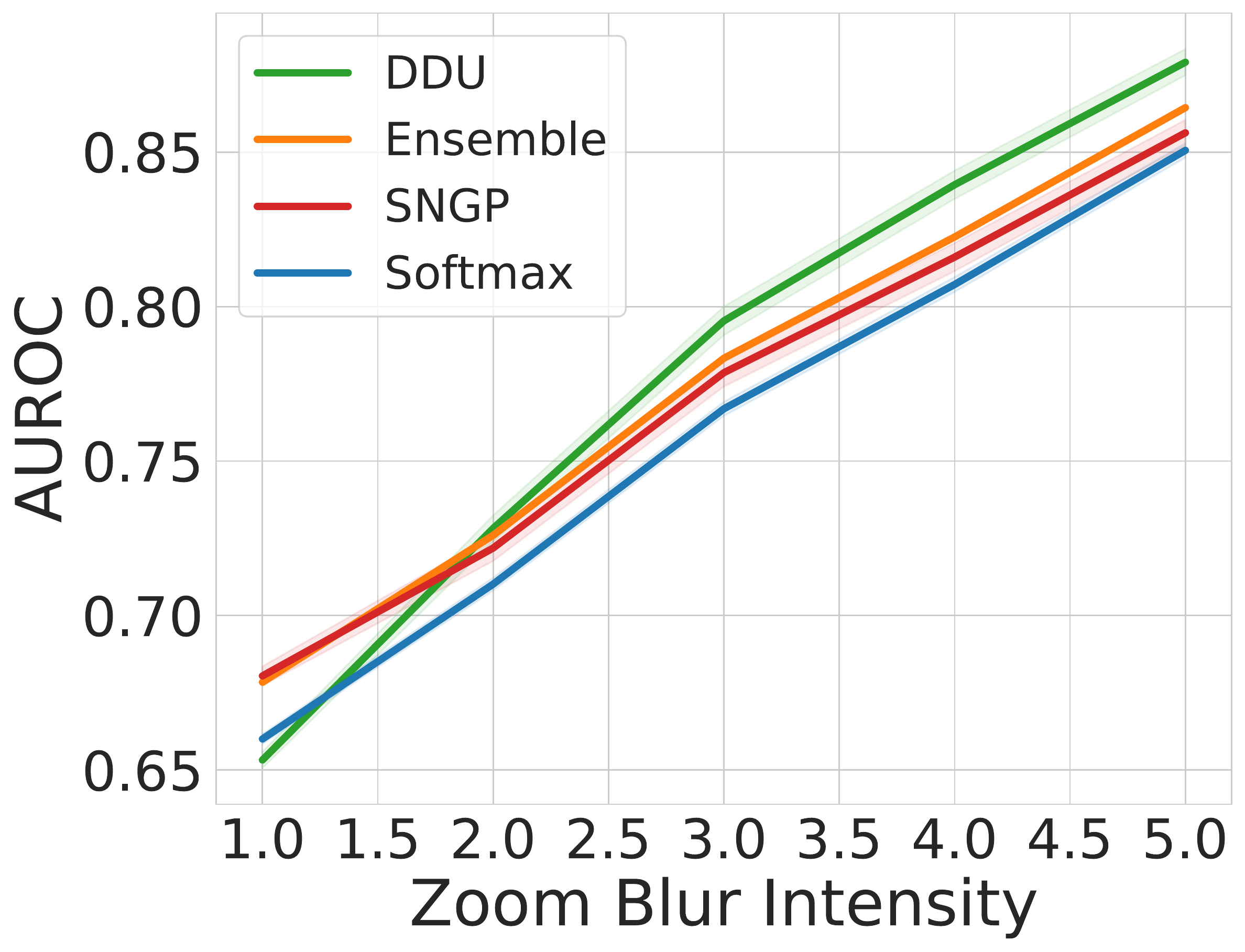}
    \end{subfigure}
    \caption{
    AUROC vs corruption intensity for all corruption types in CIFAR-10-C with ResNet-50 as the architecture and baselines: Softmax Entropy, Ensemble (using Predictive Entropy as uncertainty), SNGP and DDU feature density.
    }
    \label{fig:cifar10_c_results_resnet50}
\end{figure}

\begin{figure}[!t]
    \centering
    \begin{subfigure}{0.18\linewidth}
        \centering
        \includegraphics[width=\linewidth]{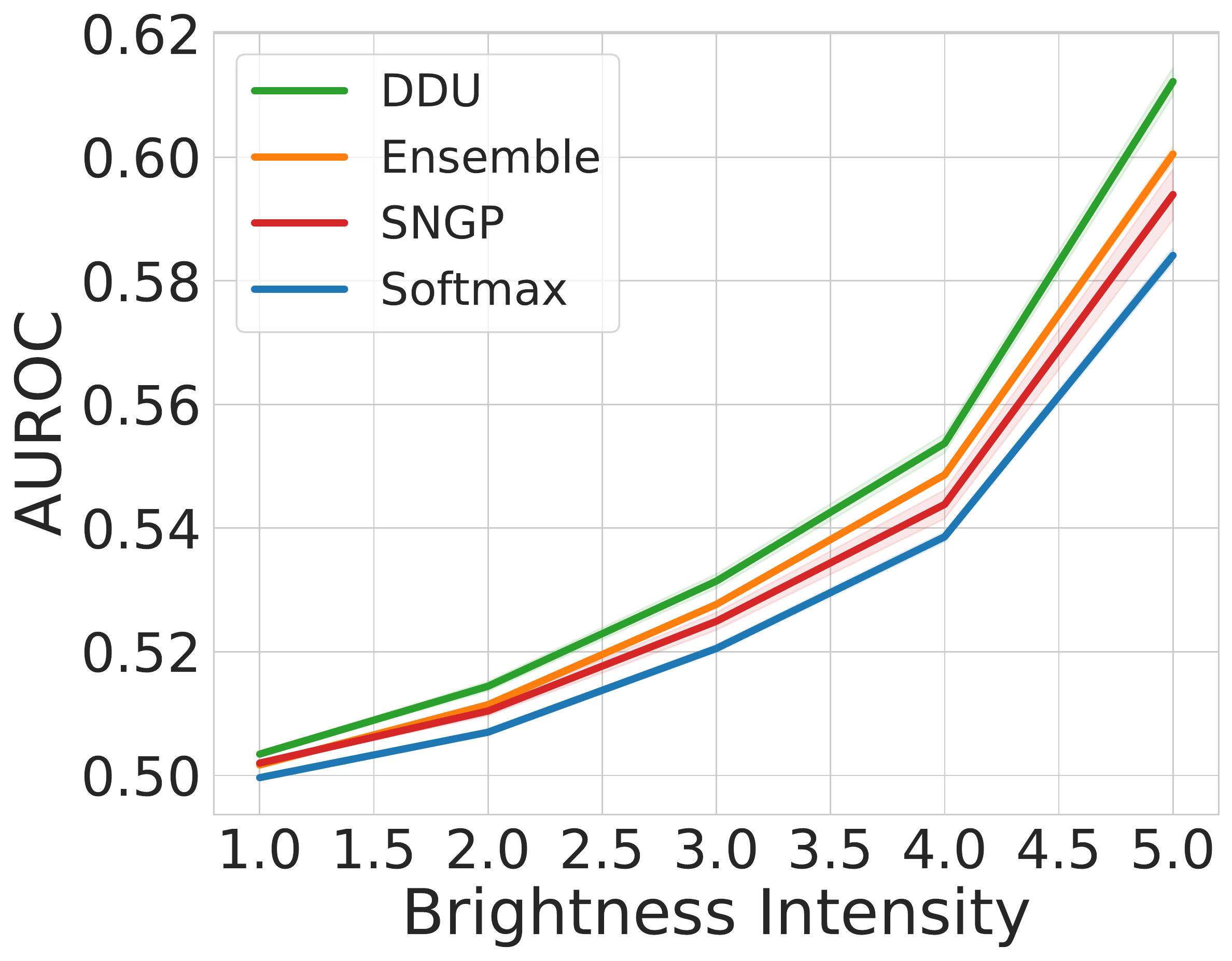}
    \end{subfigure}
    \begin{subfigure}{0.18\linewidth}
        \centering
        \includegraphics[width=\linewidth]{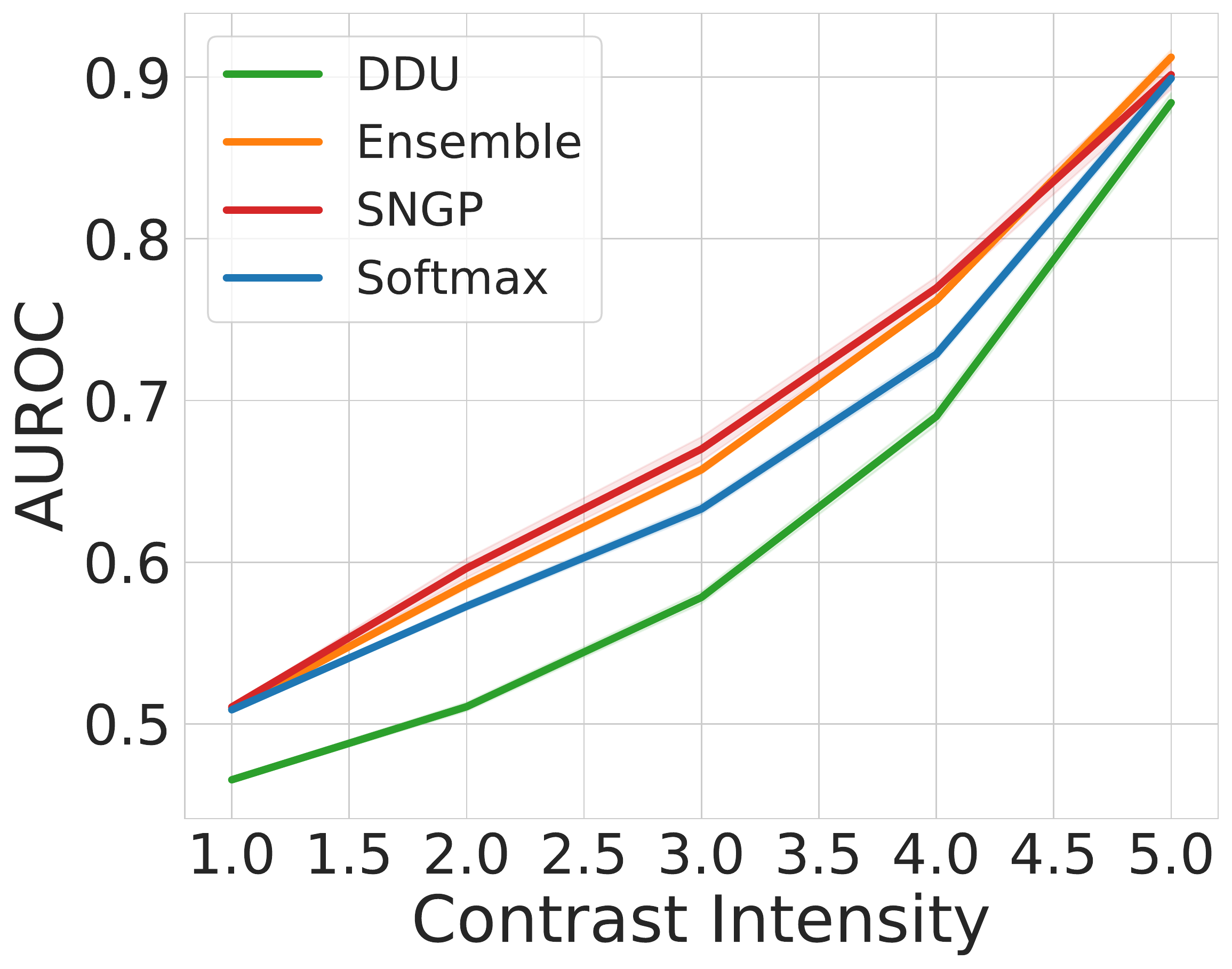}
    \end{subfigure} 
    \begin{subfigure}{0.18\linewidth}
        \centering
        \includegraphics[width=\linewidth]{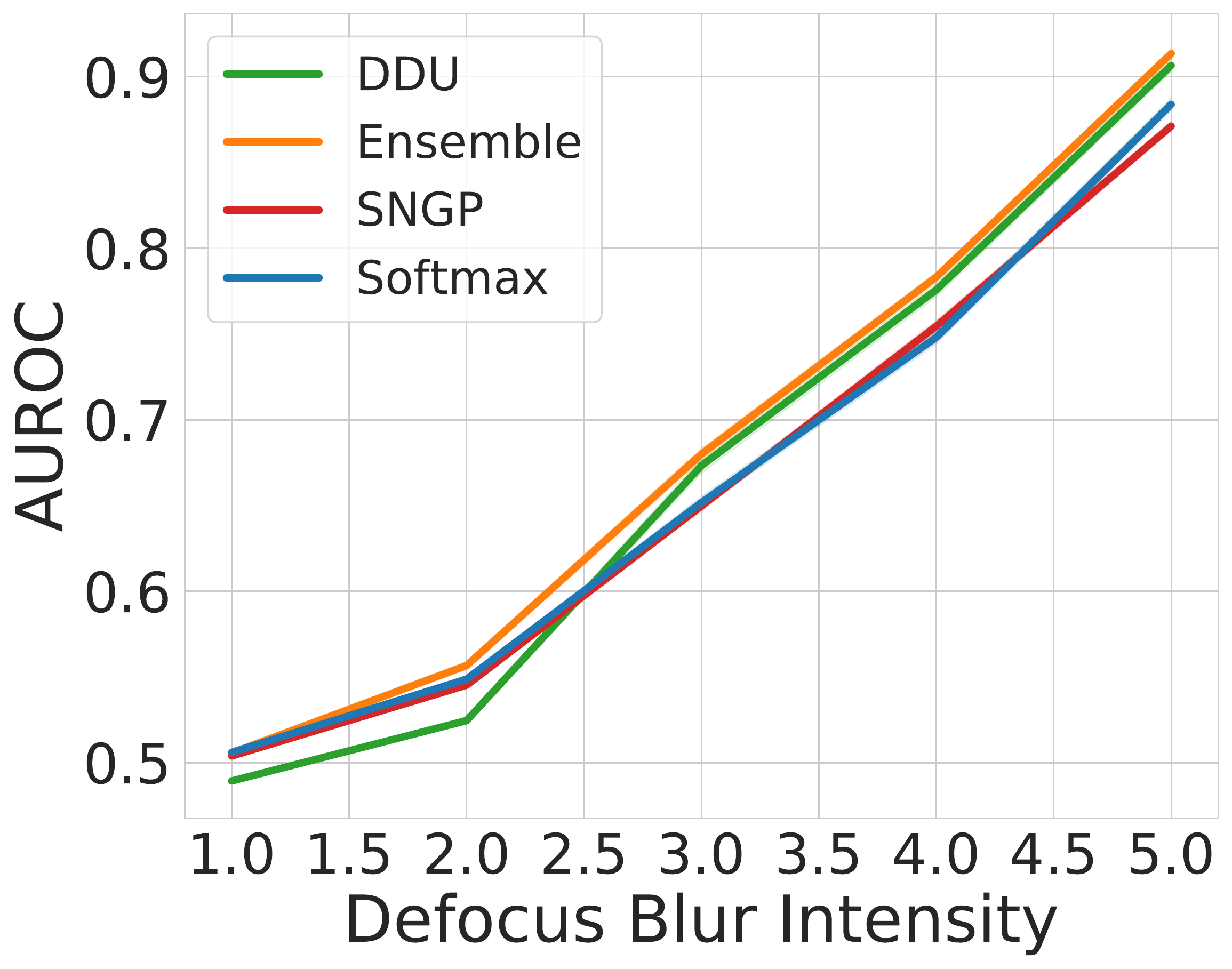}
    \end{subfigure} 
    \begin{subfigure}{0.18\linewidth}
        \centering
        \includegraphics[width=\linewidth]{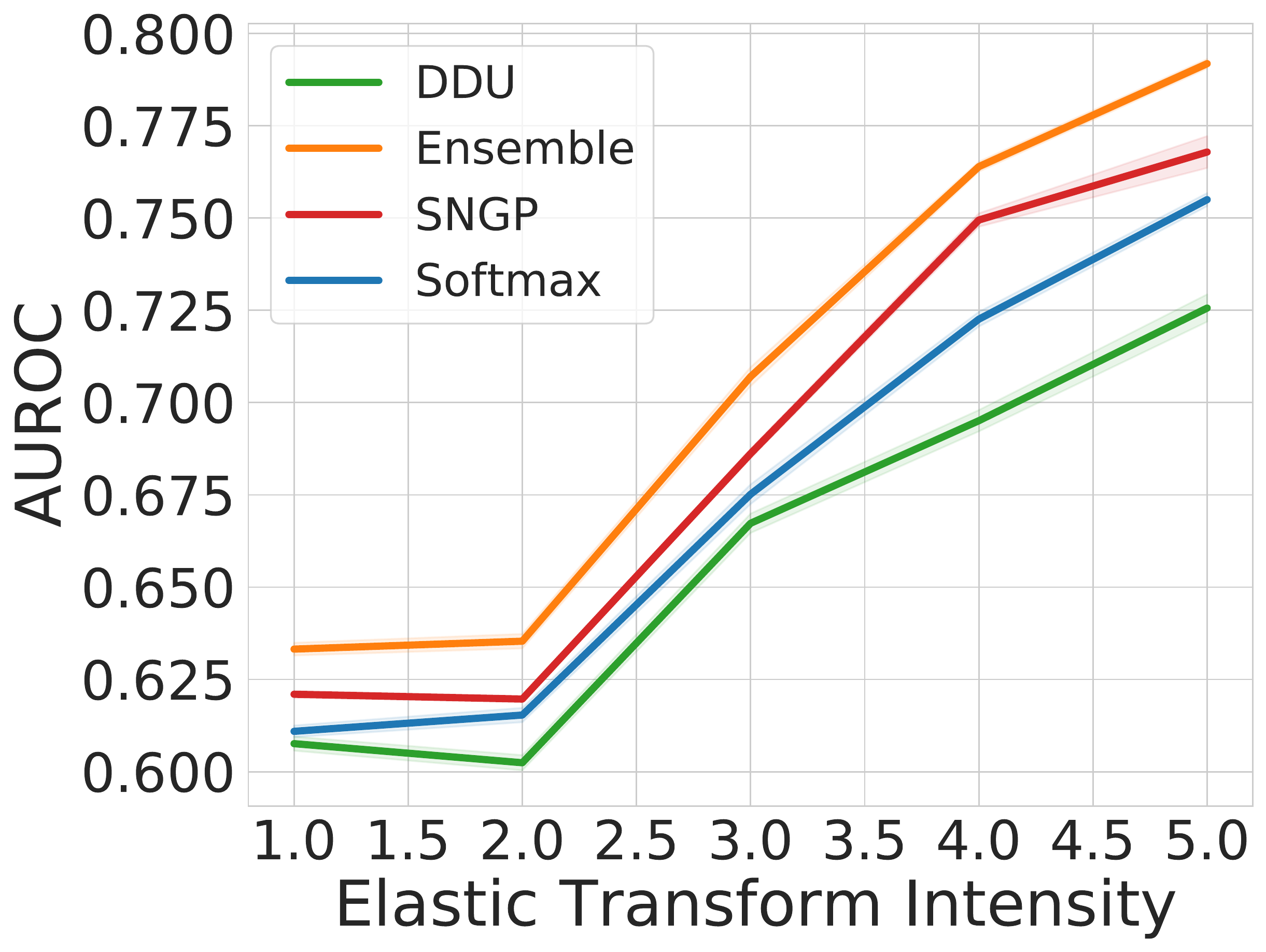}
    \end{subfigure}
    \begin{subfigure}{0.18\linewidth}
        \centering
        \includegraphics[width=\linewidth]{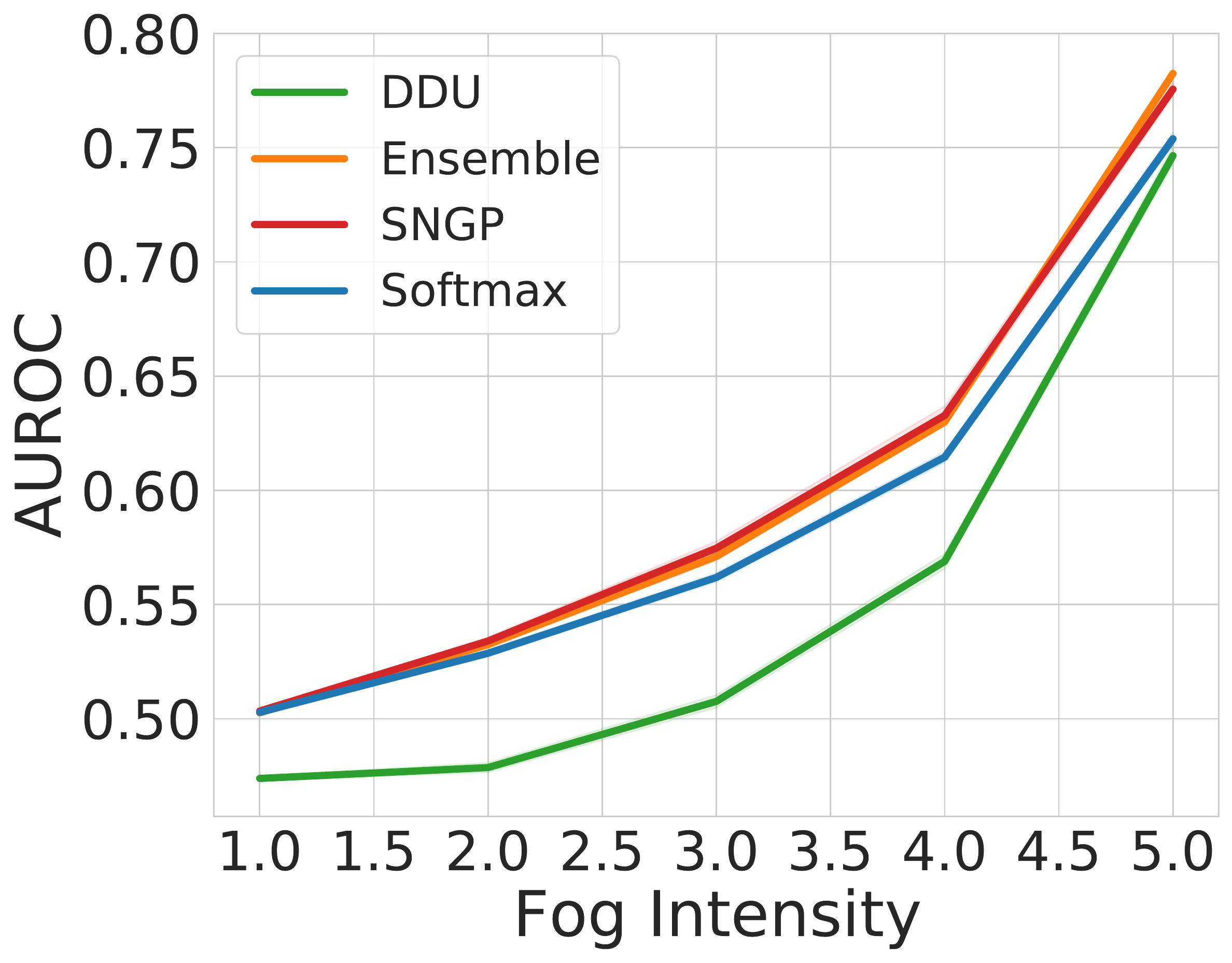}
    \end{subfigure}
    \begin{subfigure}{0.18\linewidth}
        \centering
        \includegraphics[width=\linewidth]{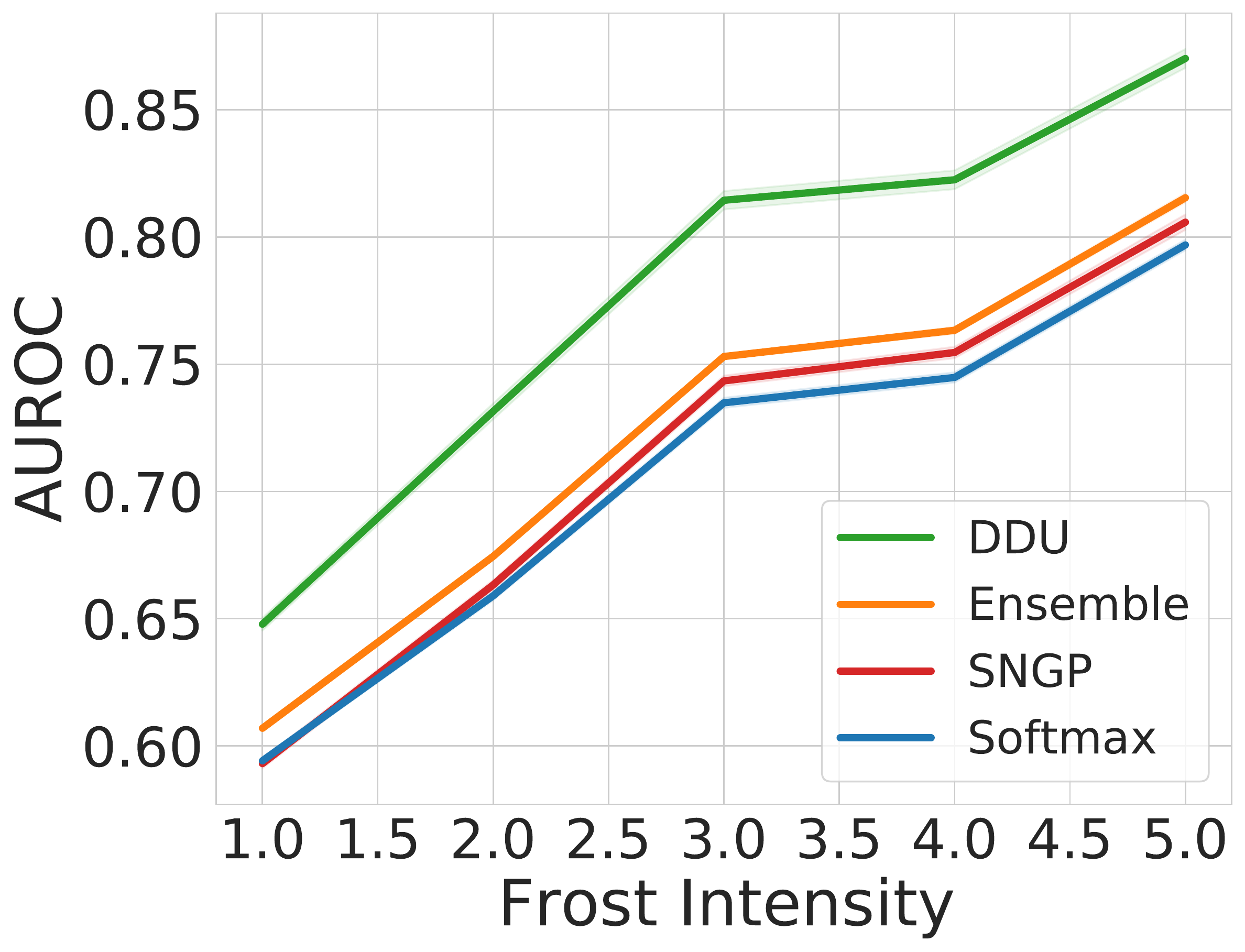}
    \end{subfigure}
    \begin{subfigure}{0.18\linewidth}
        \centering
        \includegraphics[width=\linewidth]{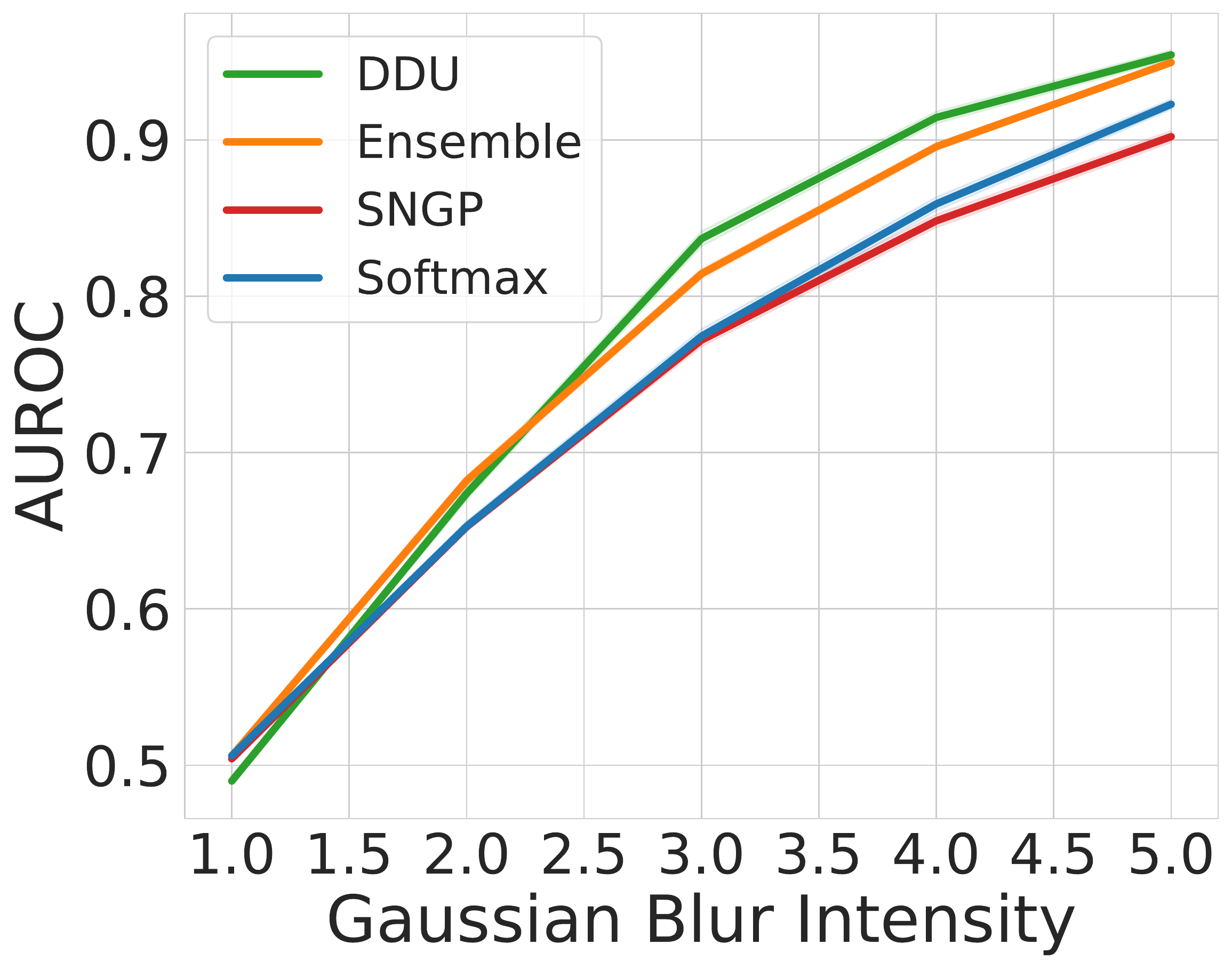}
    \end{subfigure}
    \begin{subfigure}{0.18\linewidth}
        \centering
        \includegraphics[width=\linewidth]{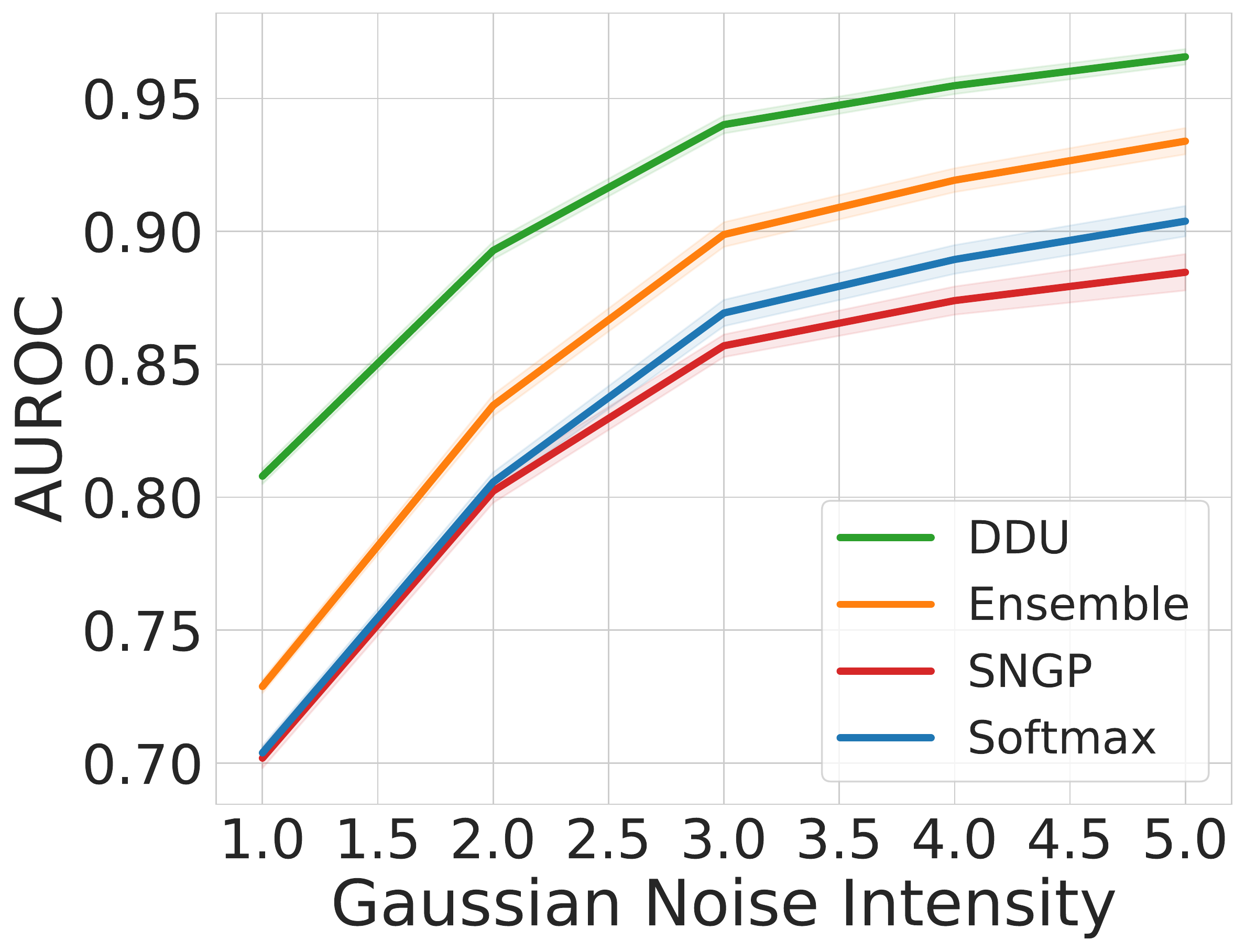}
    \end{subfigure}
    \begin{subfigure}{0.18\linewidth}
        \centering
        \includegraphics[width=\linewidth]{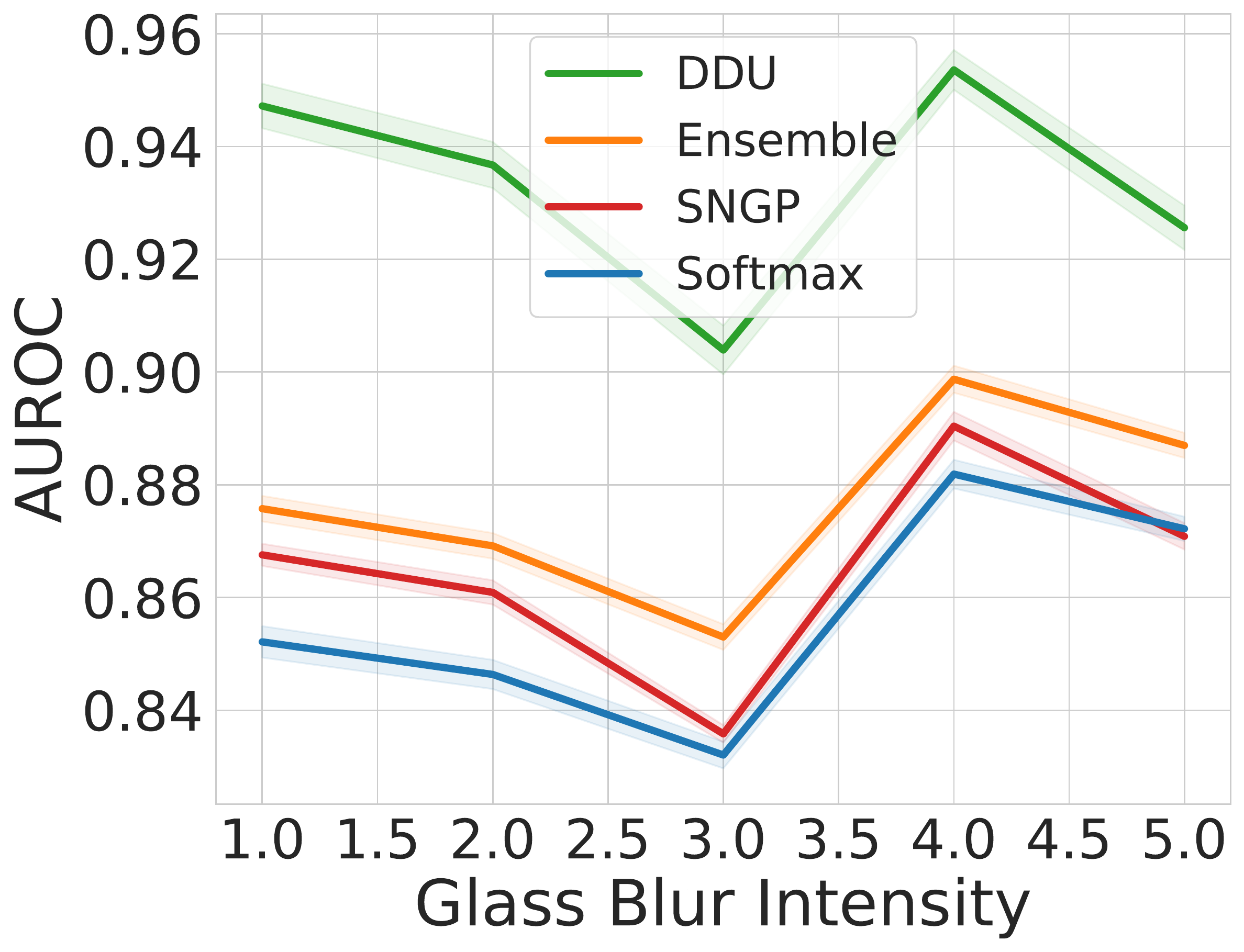}
    \end{subfigure}
    \begin{subfigure}{0.18\linewidth}
        \centering
        \includegraphics[width=\linewidth]{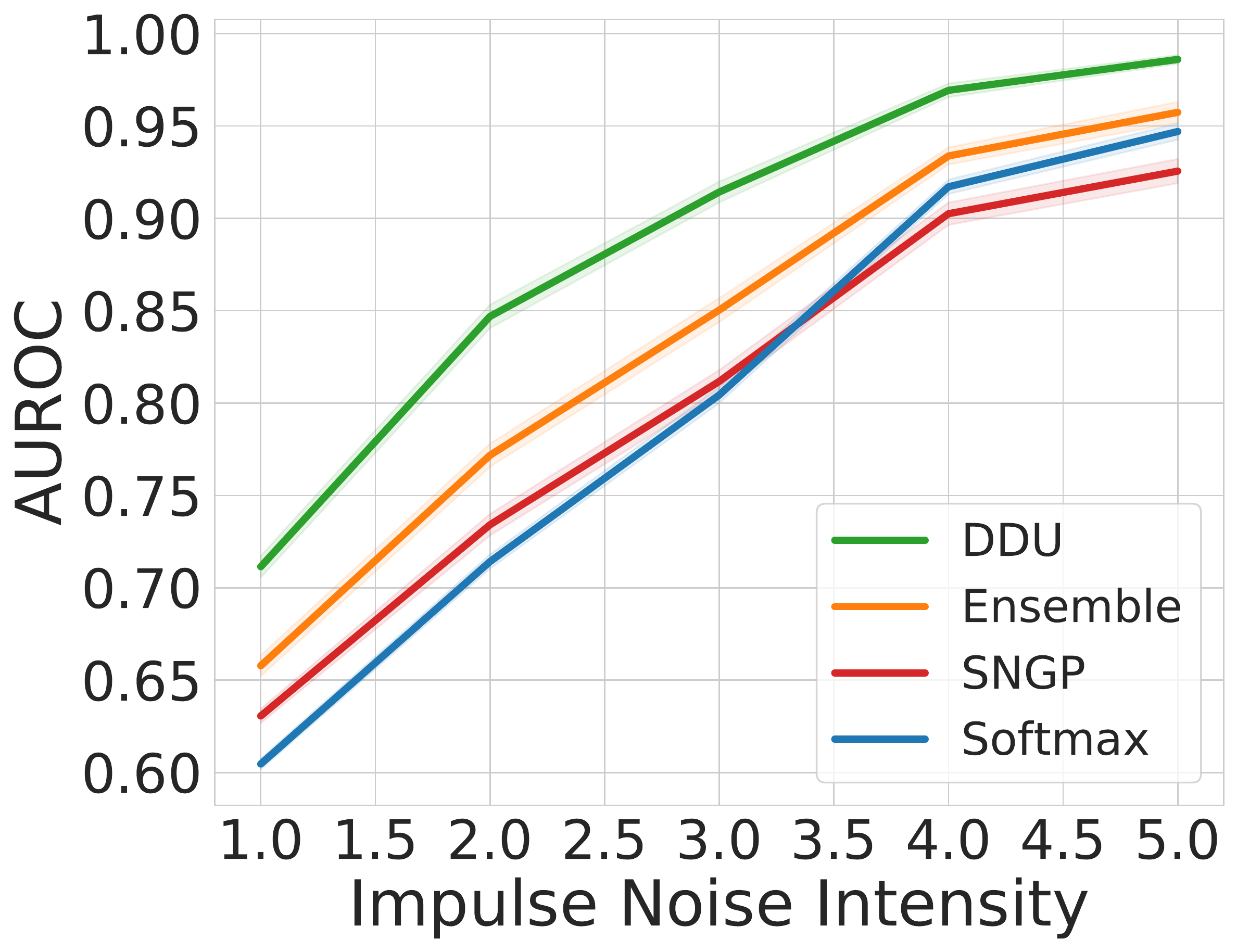}
    \end{subfigure}
    \begin{subfigure}{0.18\linewidth}
        \centering
        \includegraphics[width=\linewidth]{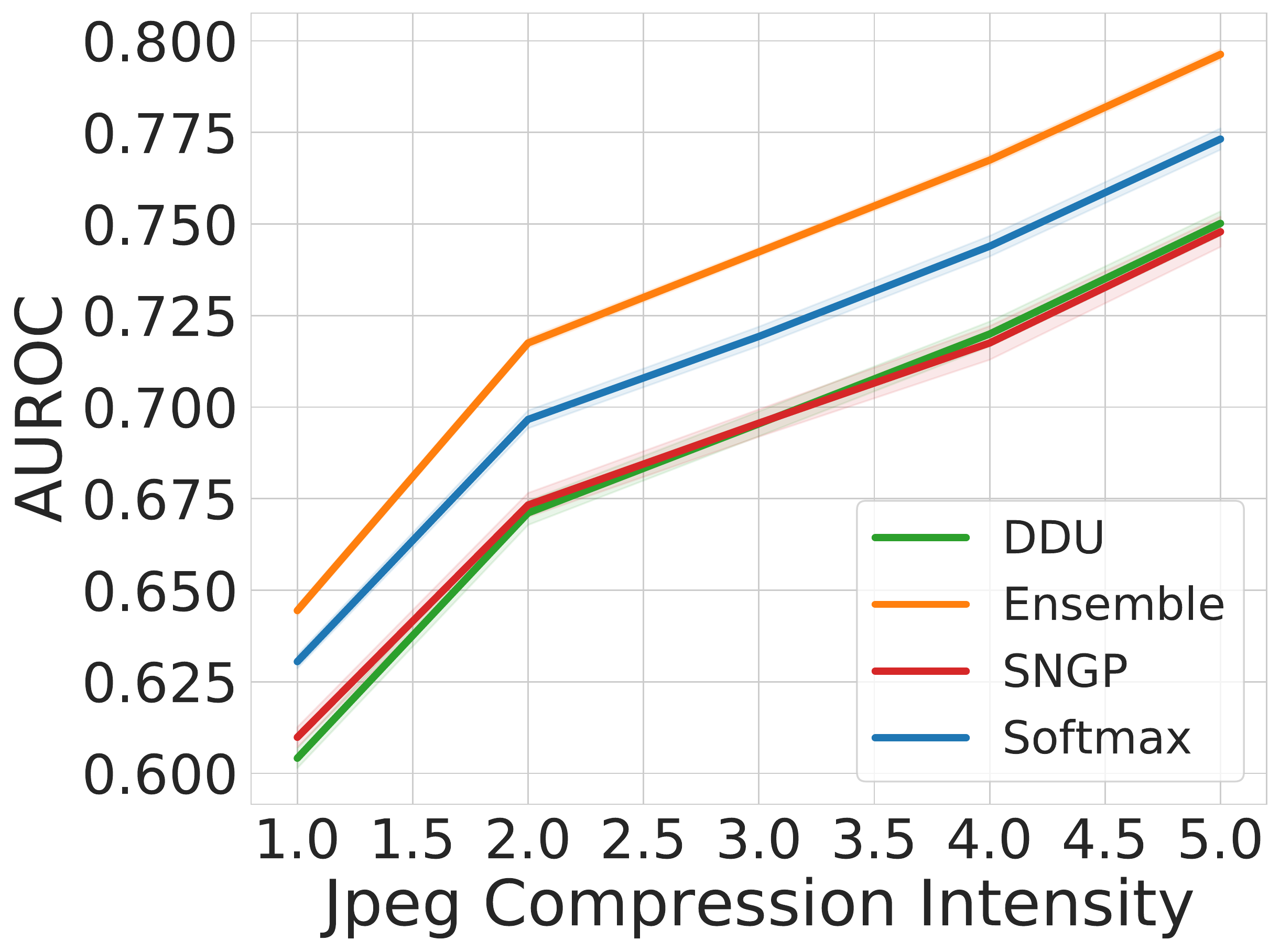}
    \end{subfigure}
    \begin{subfigure}{0.18\linewidth}
        \centering
        \includegraphics[width=\linewidth]{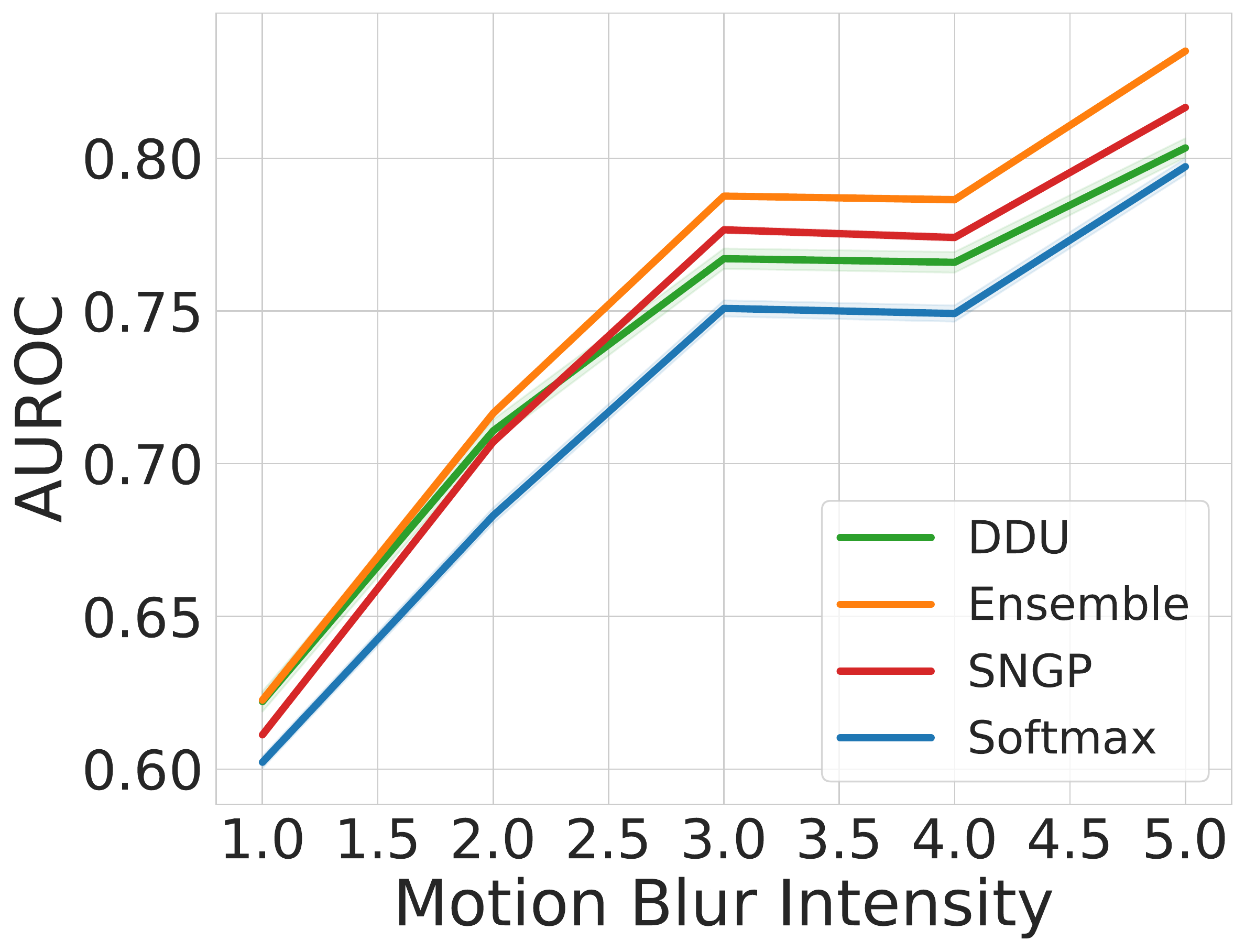}
    \end{subfigure}
    \begin{subfigure}{0.18\linewidth}
        \centering
        \includegraphics[width=\linewidth]{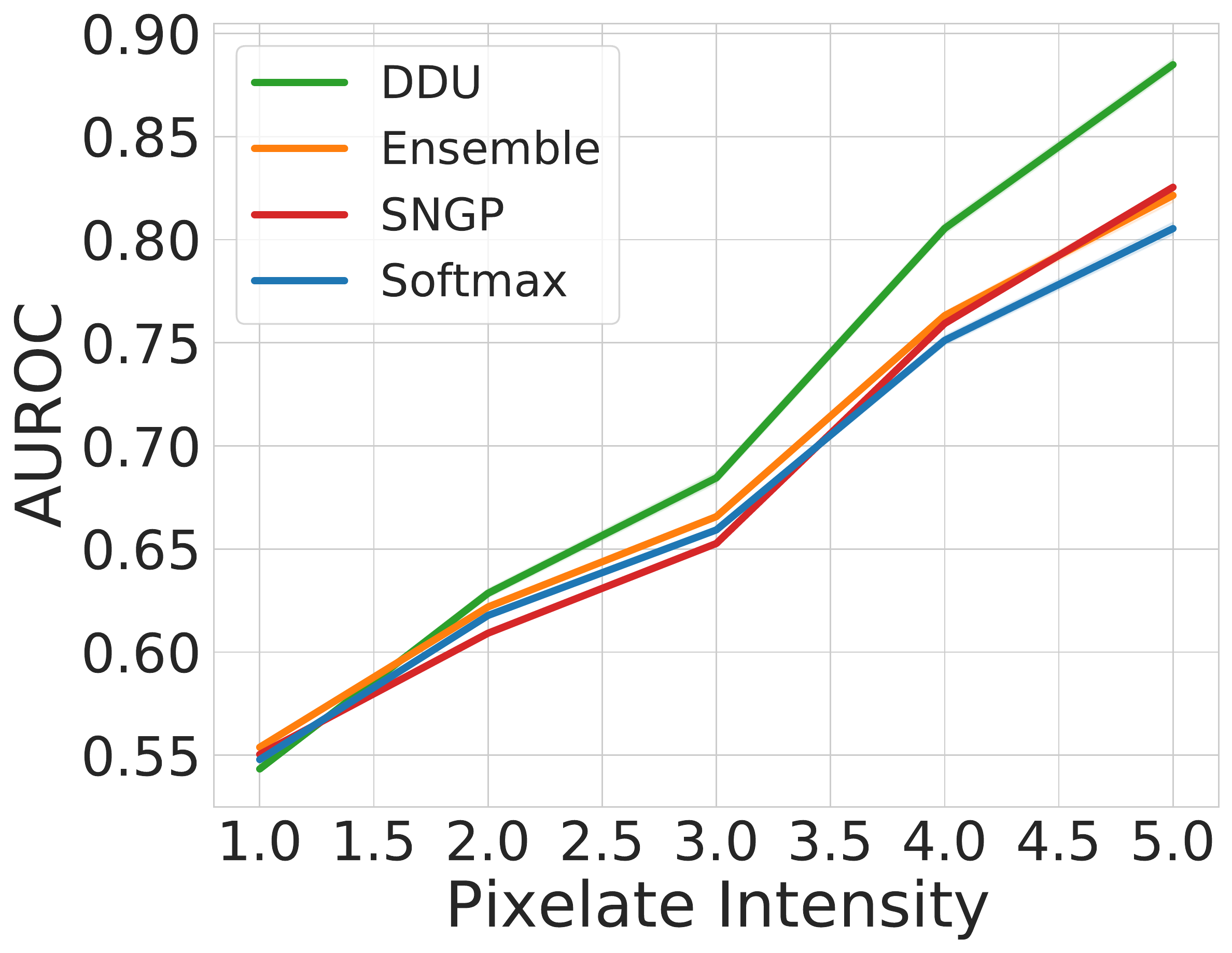}
    \end{subfigure}
    \begin{subfigure}{0.18\linewidth}
        \centering
        \includegraphics[width=\linewidth]{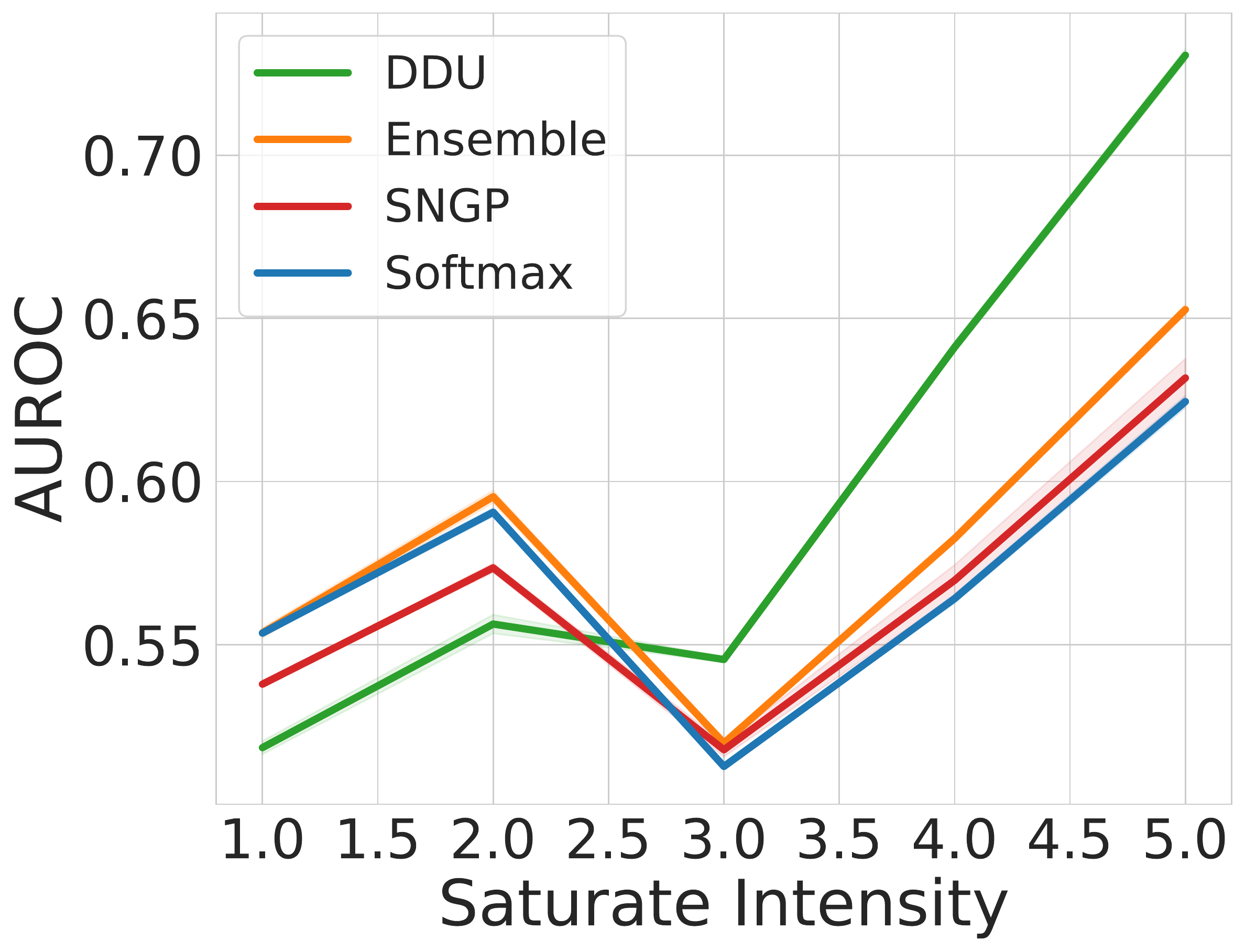}
    \end{subfigure}
    \begin{subfigure}{0.18\linewidth}
        \centering
        \includegraphics[width=\linewidth]{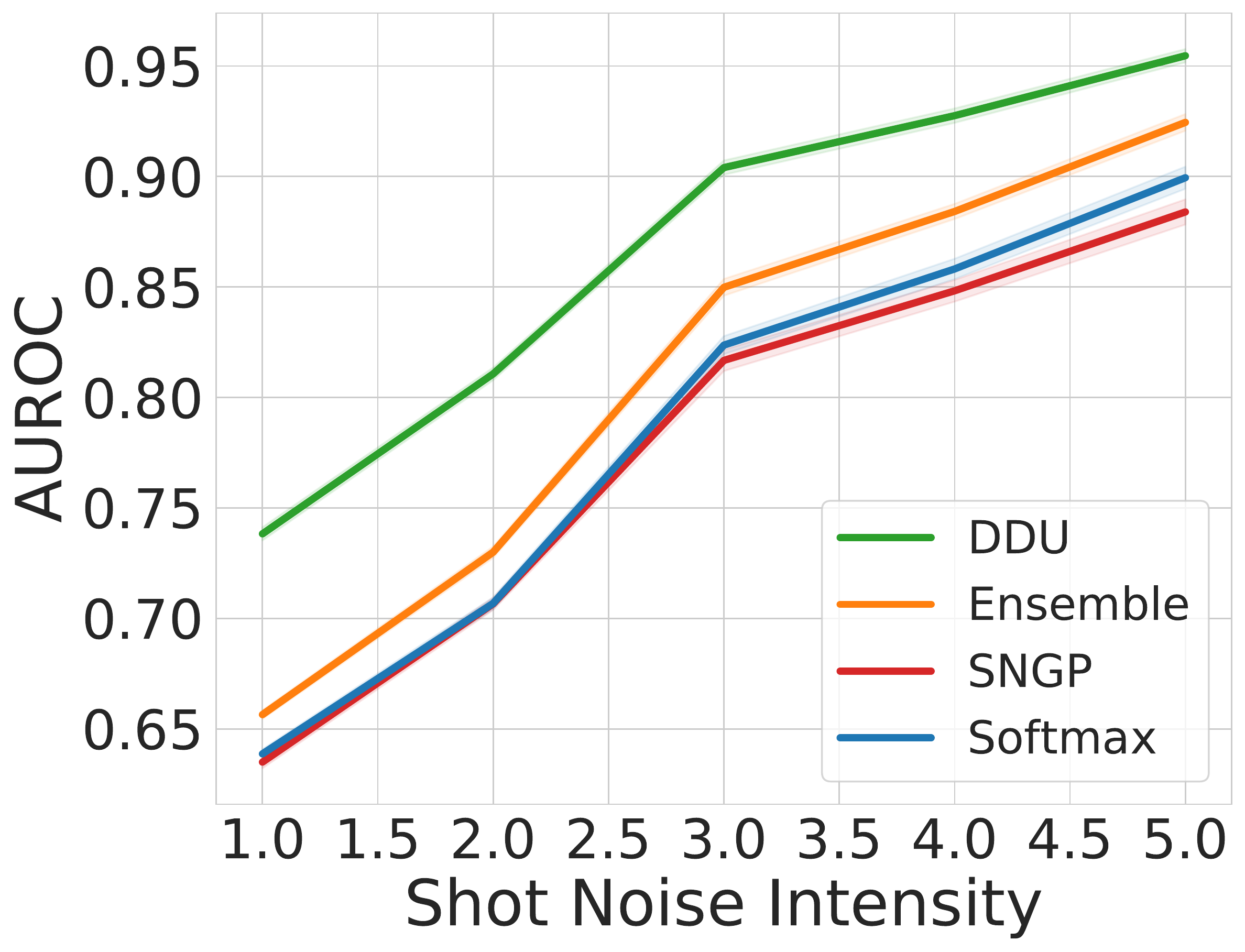}
    \end{subfigure}
    \begin{subfigure}{0.18\linewidth}
        \centering
        \includegraphics[width=\linewidth]{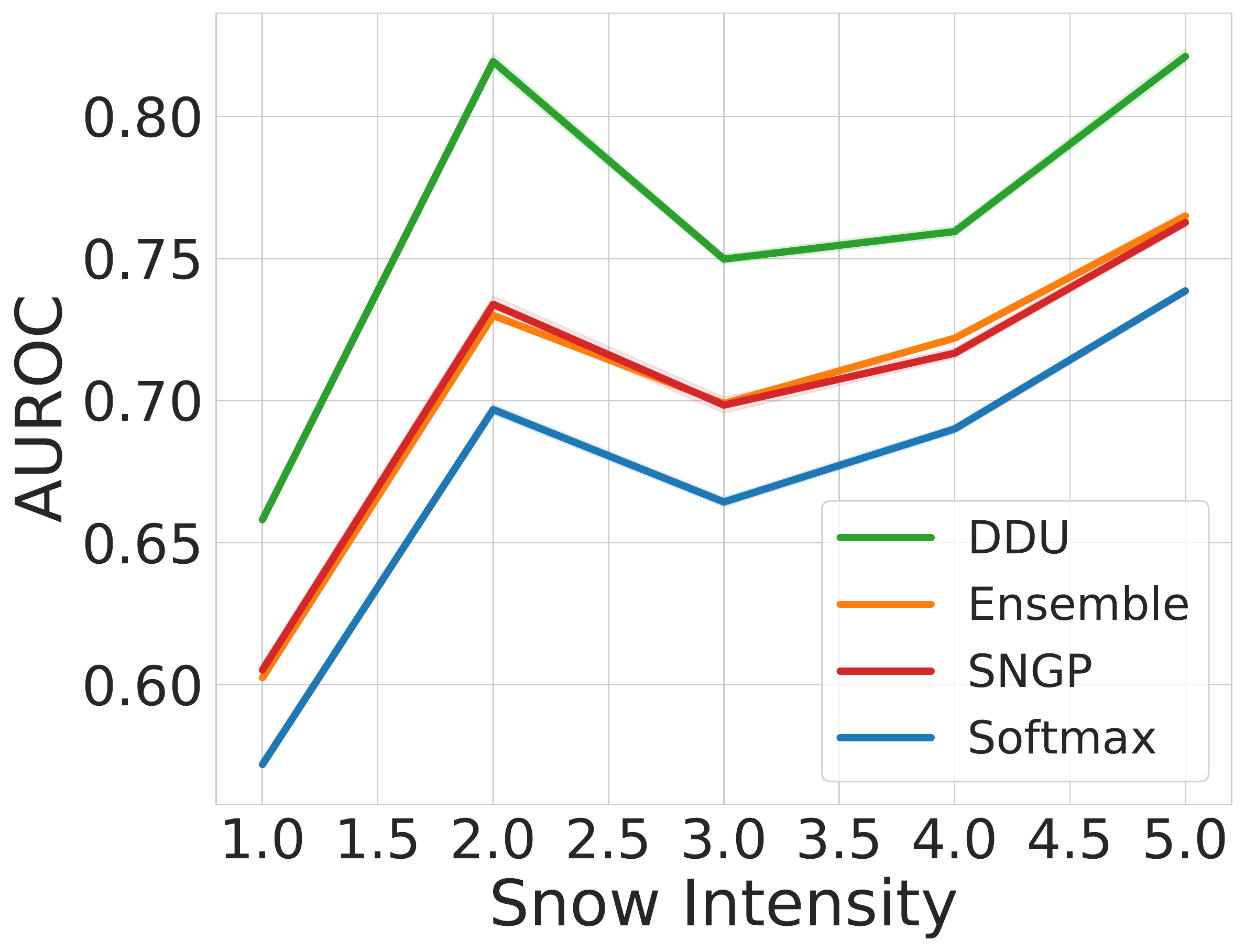}
    \end{subfigure}
    \begin{subfigure}{0.18\linewidth}
        \centering
        \includegraphics[width=\linewidth]{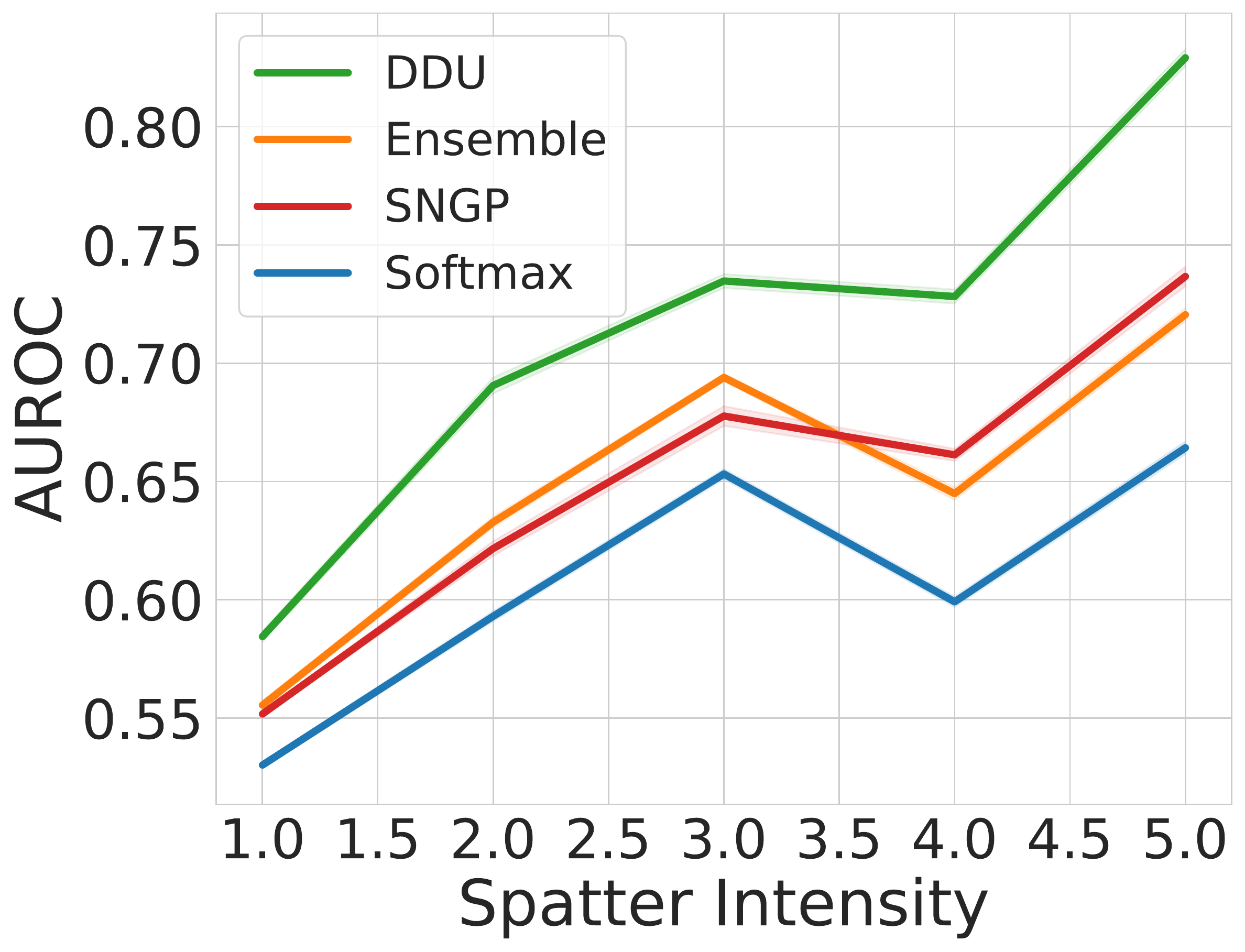}
    \end{subfigure}
    \begin{subfigure}{0.18\linewidth}
        \centering
        \includegraphics[width=\linewidth]{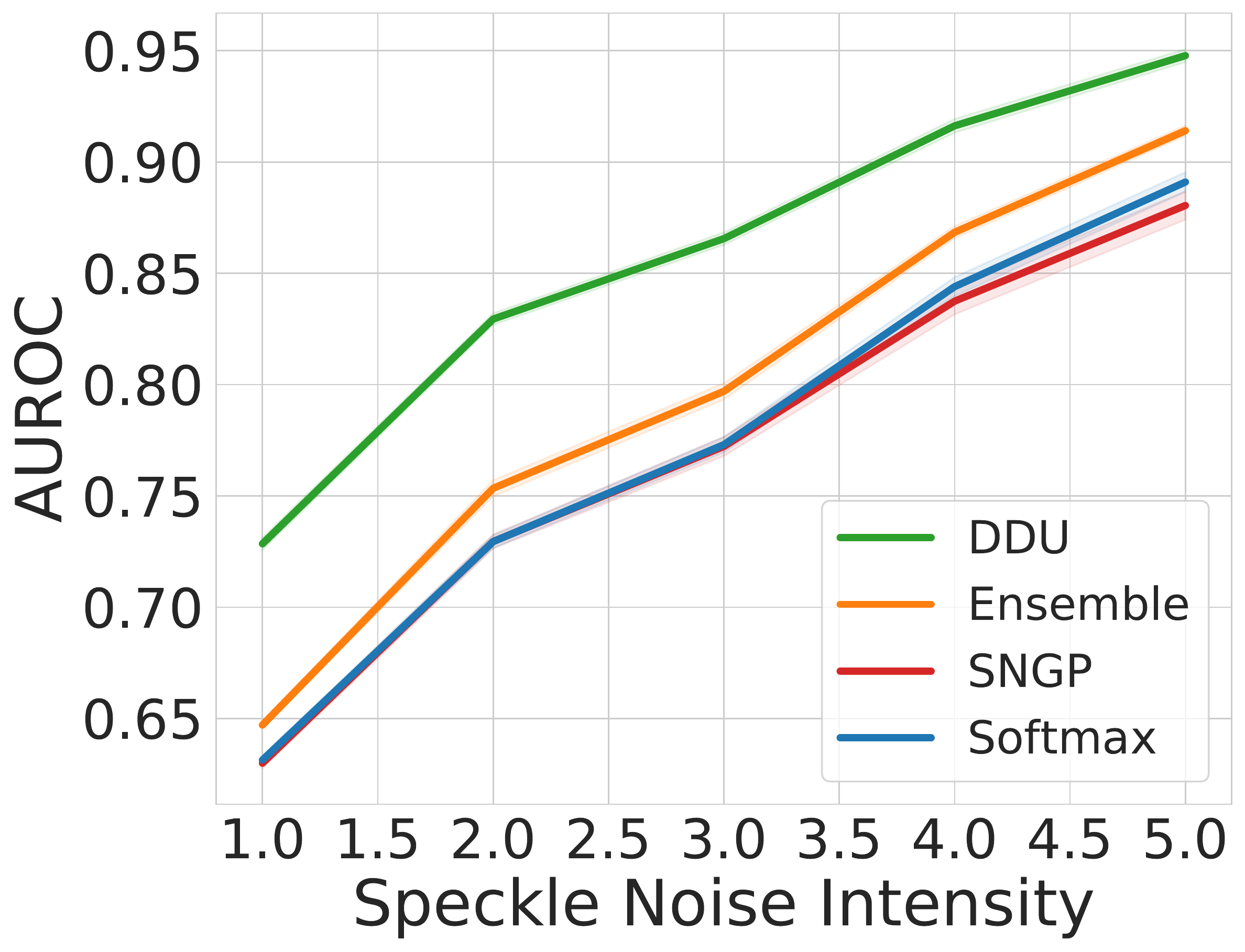}
    \end{subfigure}
    \begin{subfigure}{0.18\linewidth}
        \centering
        \includegraphics[width=\linewidth]{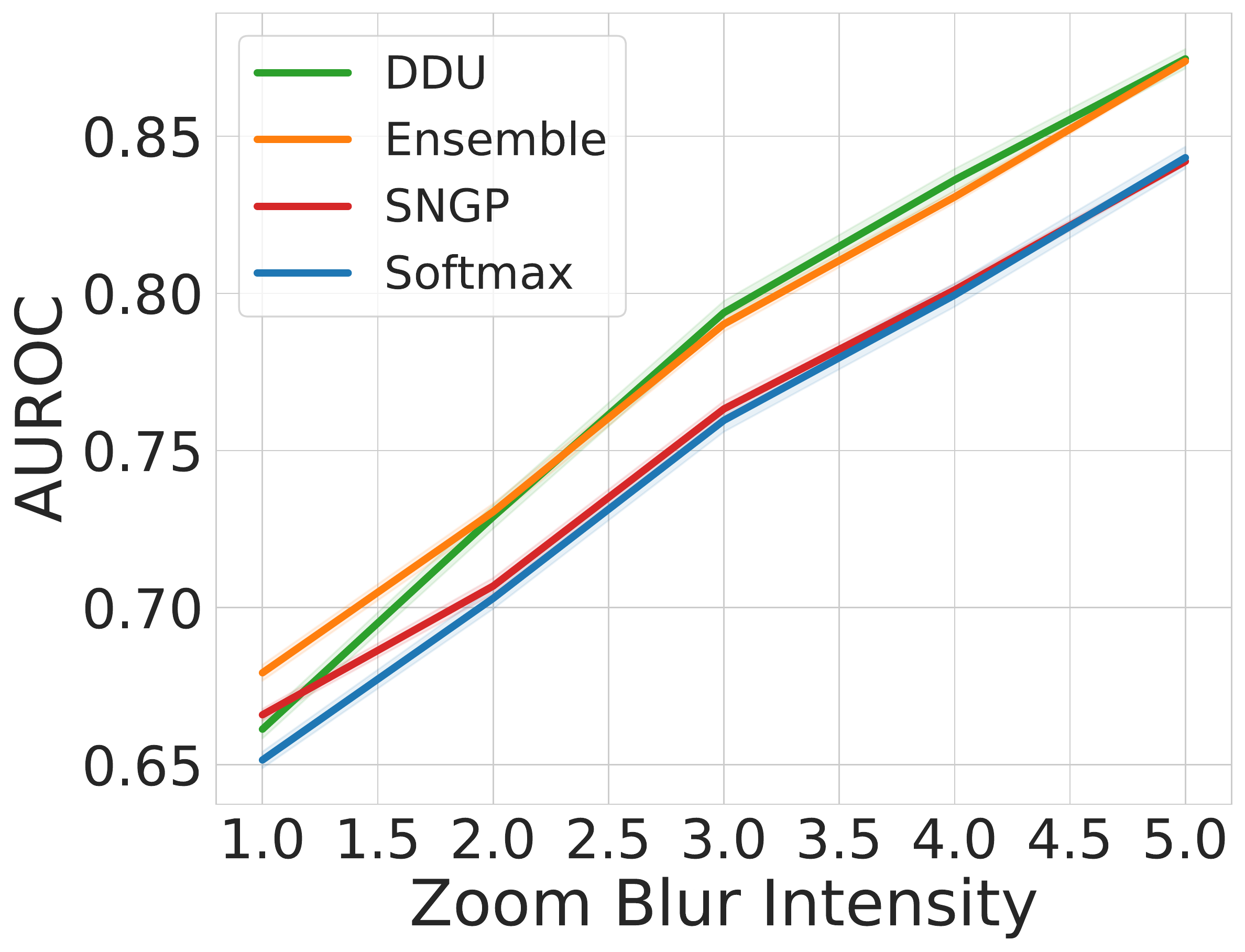}
    \end{subfigure}
    \caption{
    AUROC vs corruption intensity for all corruption types in CIFAR-10-C with ResNet-110 as the architecture and baselines: Softmax Entropy, Ensemble (using Predictive Entropy as uncertainty), SNGP and DDU feature density.
    }
    \label{fig:cifar10_c_results_resnet110}
\end{figure}

\begin{figure}[!t]
    \centering
    \begin{subfigure}{0.18\linewidth}
        \centering
        \includegraphics[width=\linewidth]{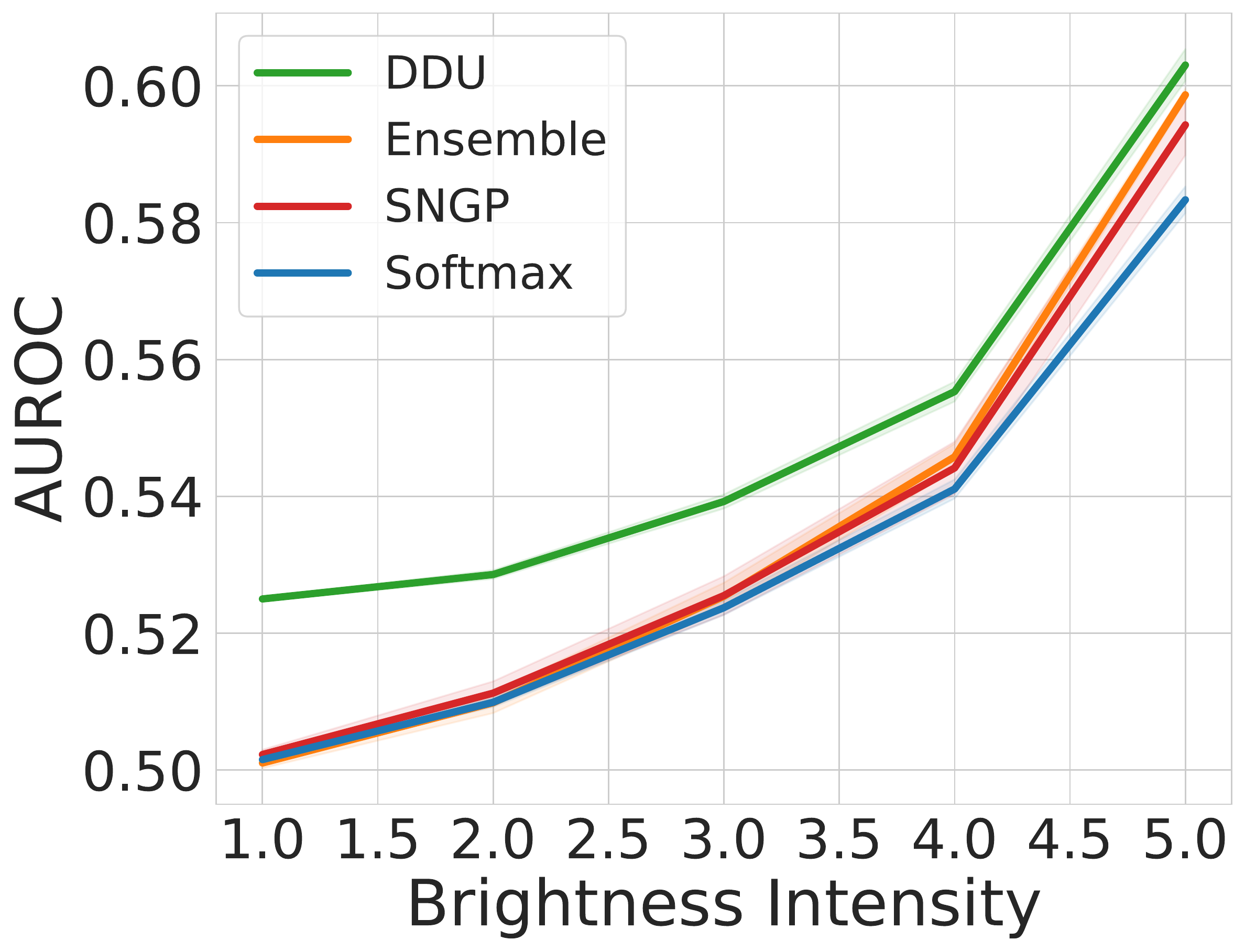}
    \end{subfigure}
    \begin{subfigure}{0.18\linewidth}
        \centering
        \includegraphics[width=\linewidth]{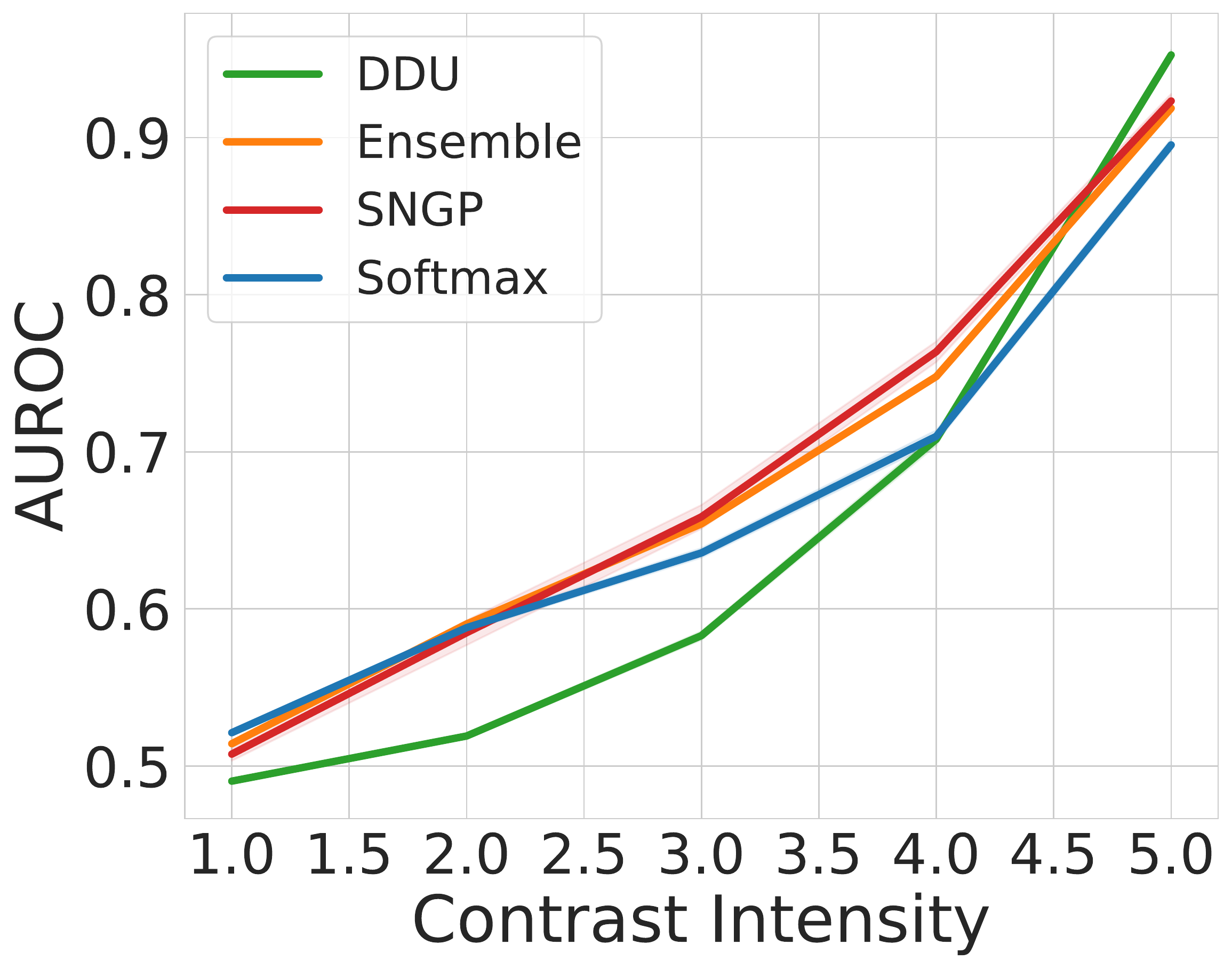}
    \end{subfigure} 
    \begin{subfigure}{0.18\linewidth}
        \centering
        \includegraphics[width=\linewidth]{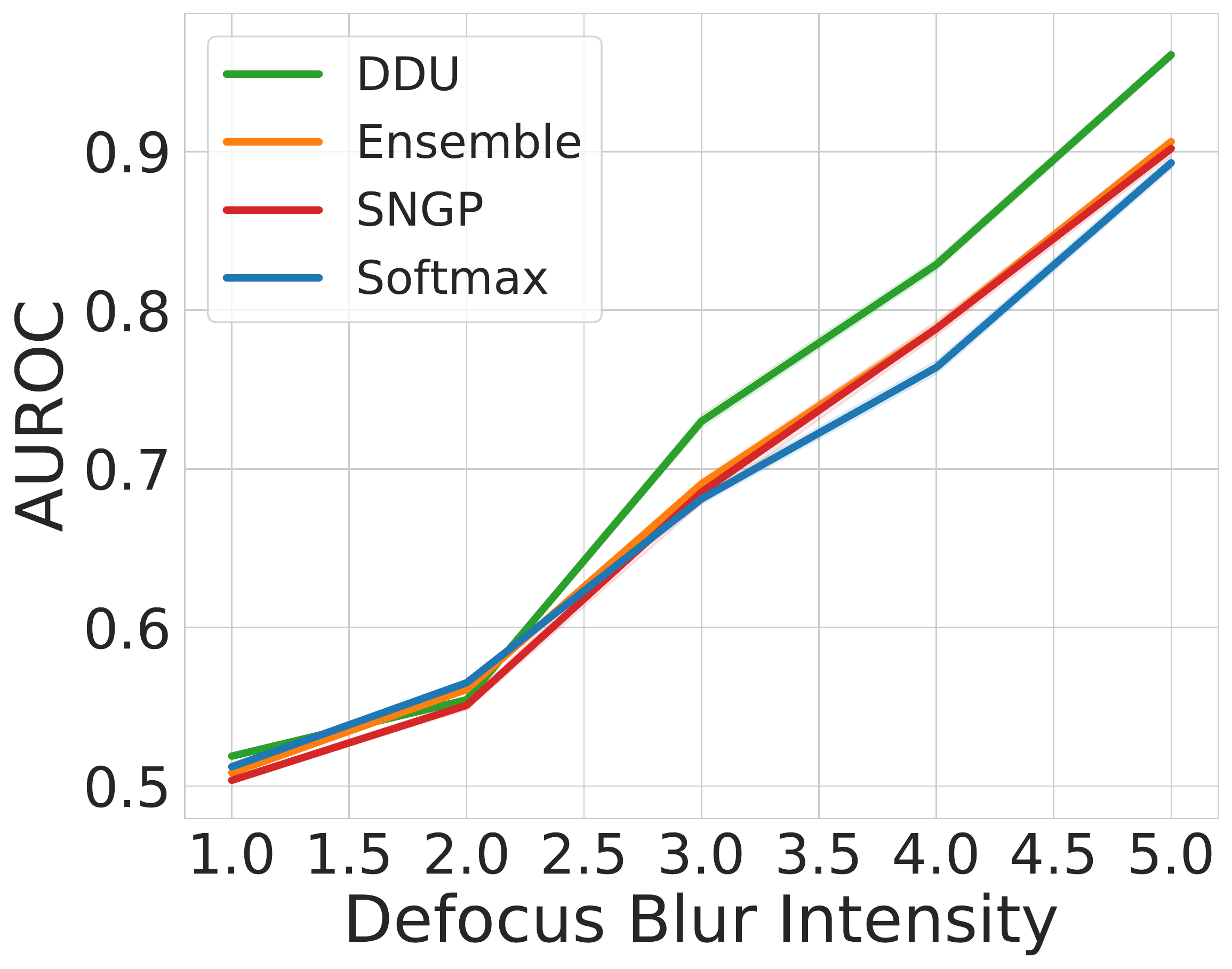}
    \end{subfigure} 
    \begin{subfigure}{0.18\linewidth}
        \centering
        \includegraphics[width=\linewidth]{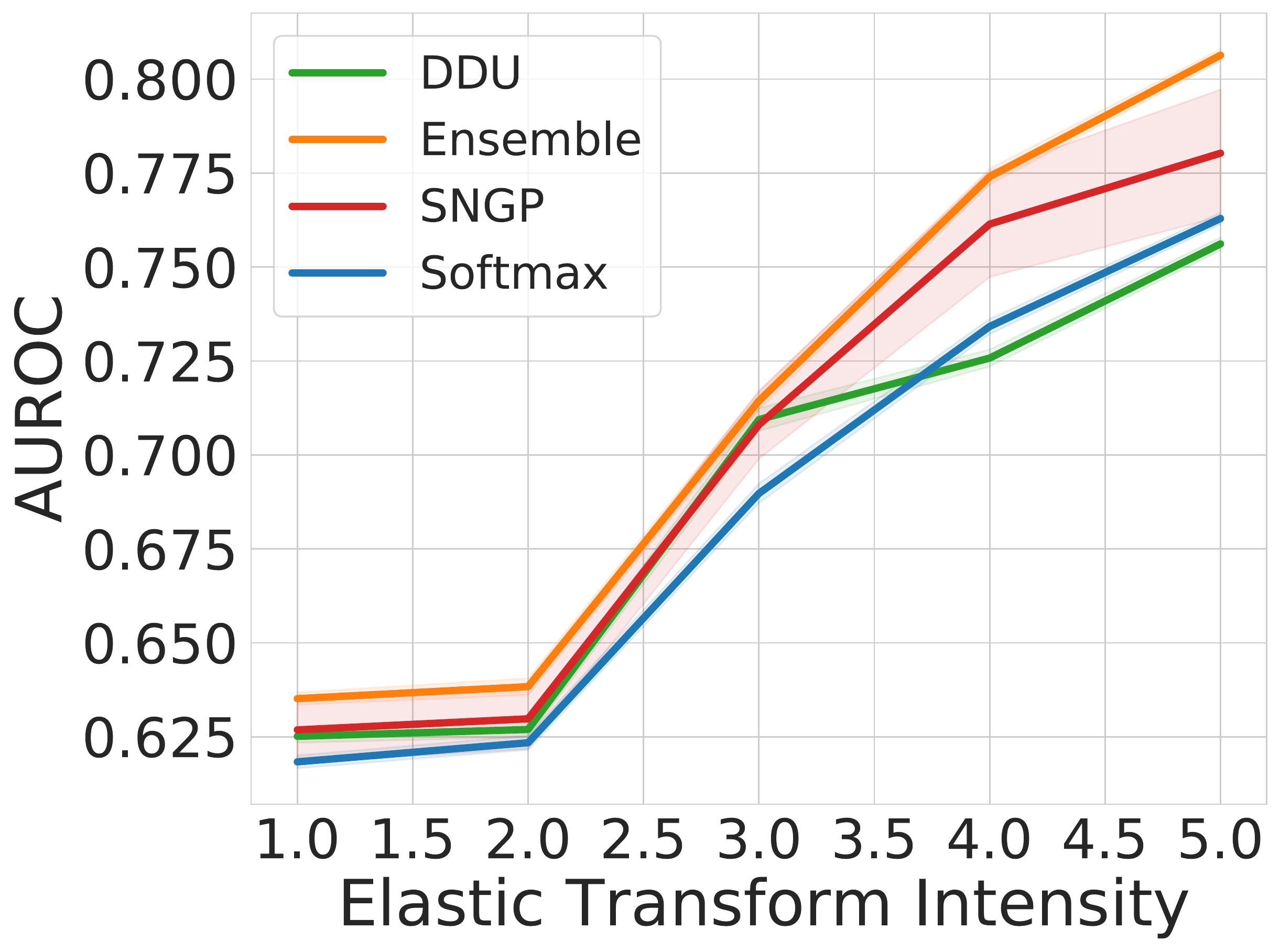}
    \end{subfigure}
    \begin{subfigure}{0.18\linewidth}
        \centering
        \includegraphics[width=\linewidth]{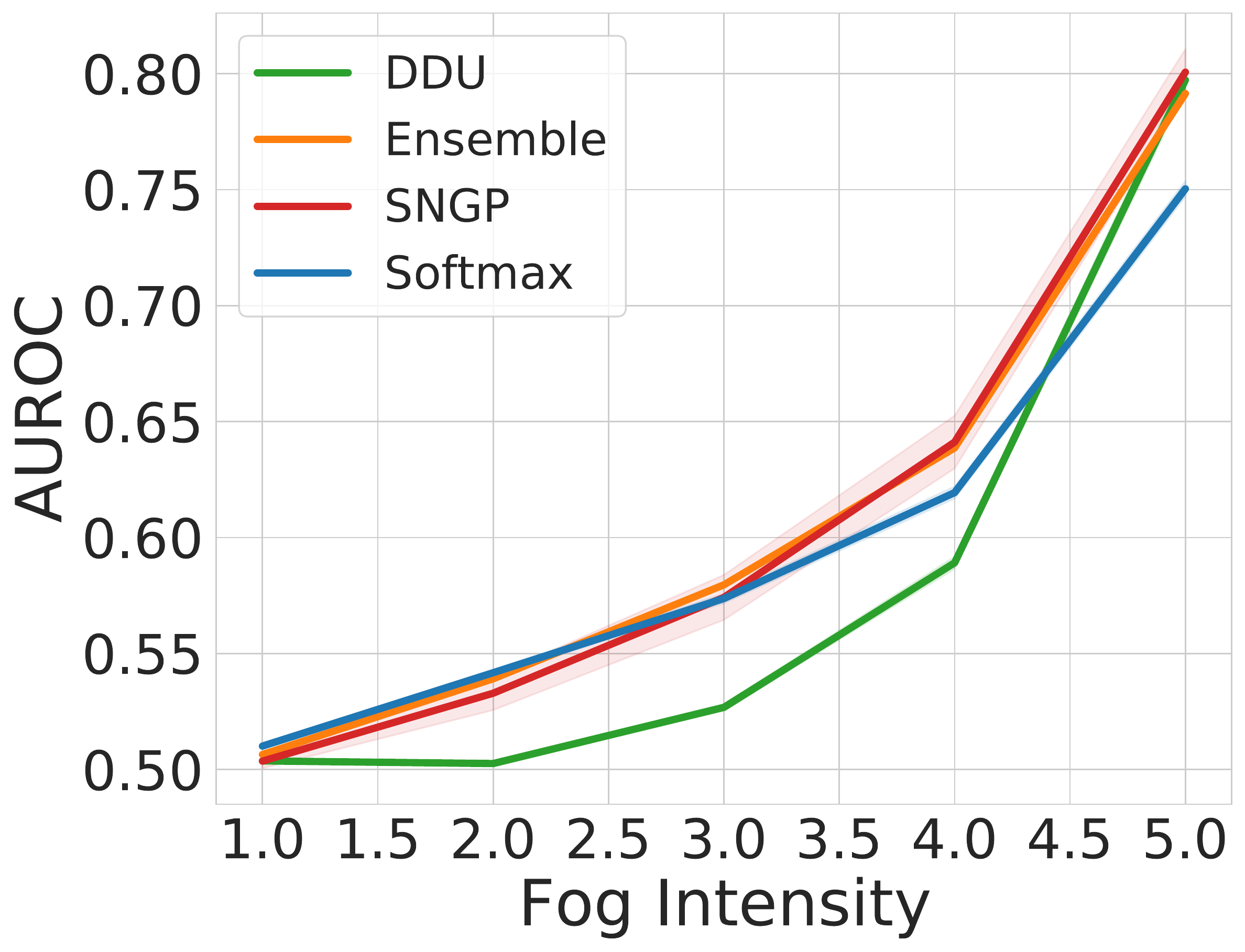}
    \end{subfigure}
    \begin{subfigure}{0.18\linewidth}
        \centering
        \includegraphics[width=\linewidth]{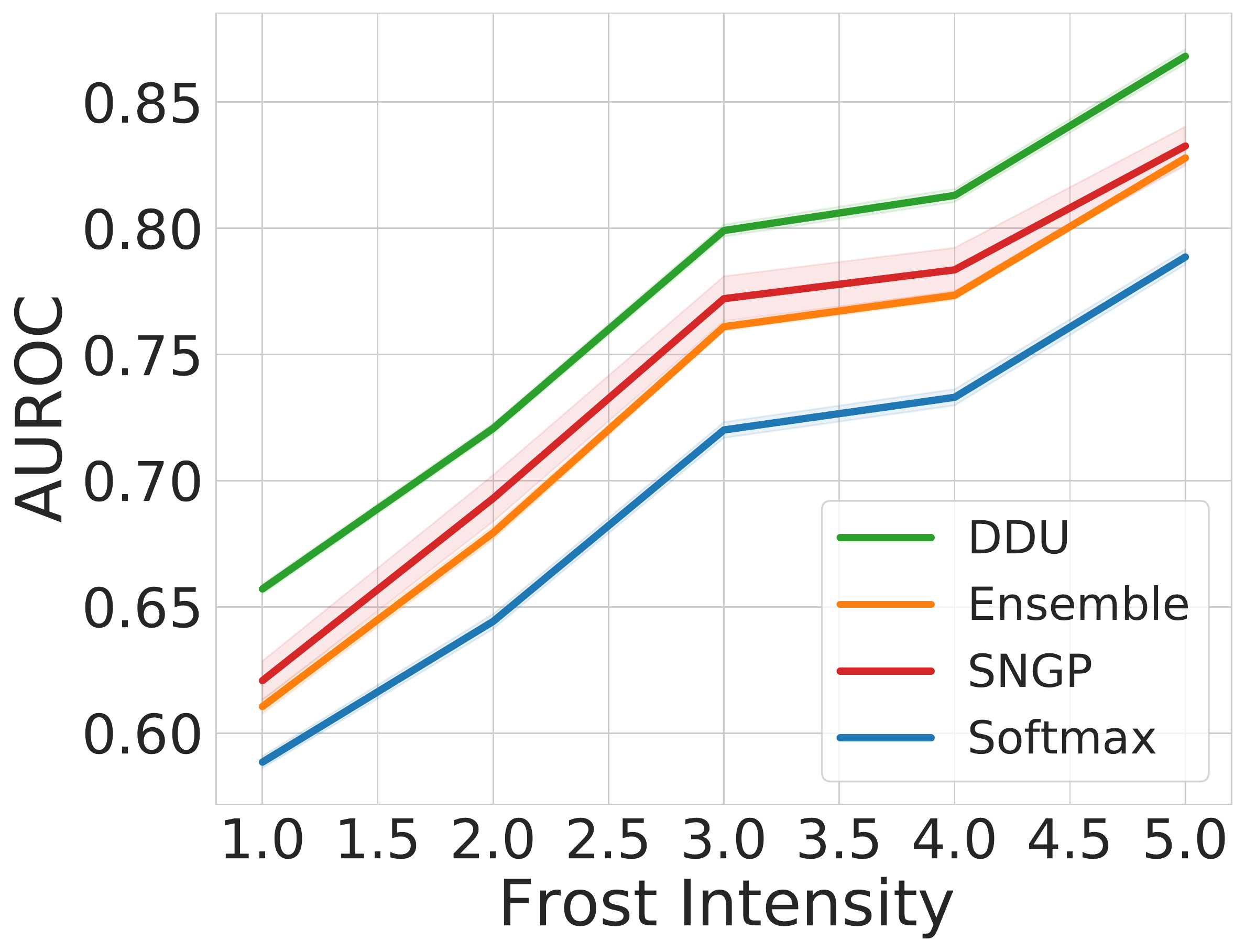}
    \end{subfigure}
    \begin{subfigure}{0.18\linewidth}
        \centering
        \includegraphics[width=\linewidth]{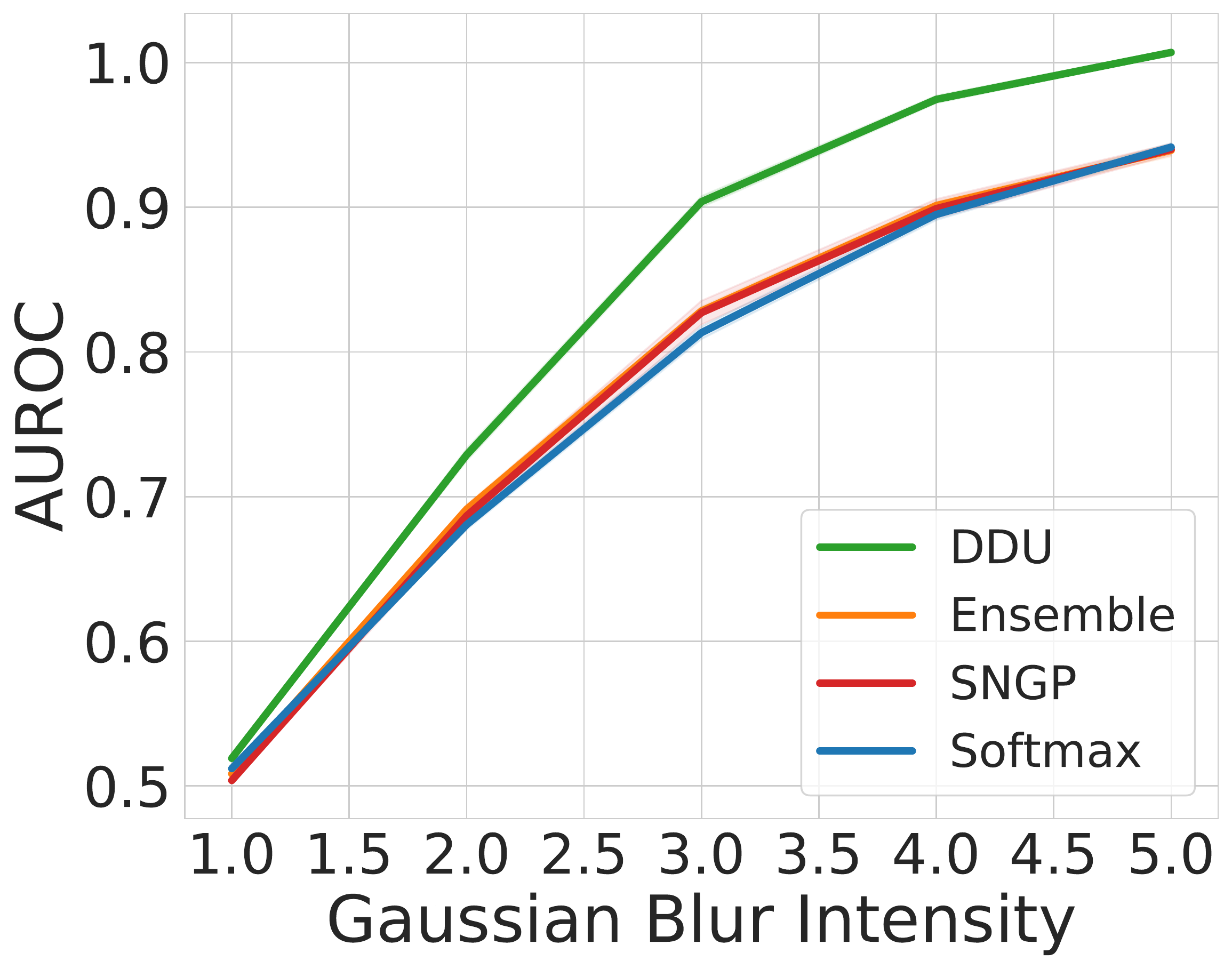}
    \end{subfigure}
    \begin{subfigure}{0.18\linewidth}
        \centering
        \includegraphics[width=\linewidth]{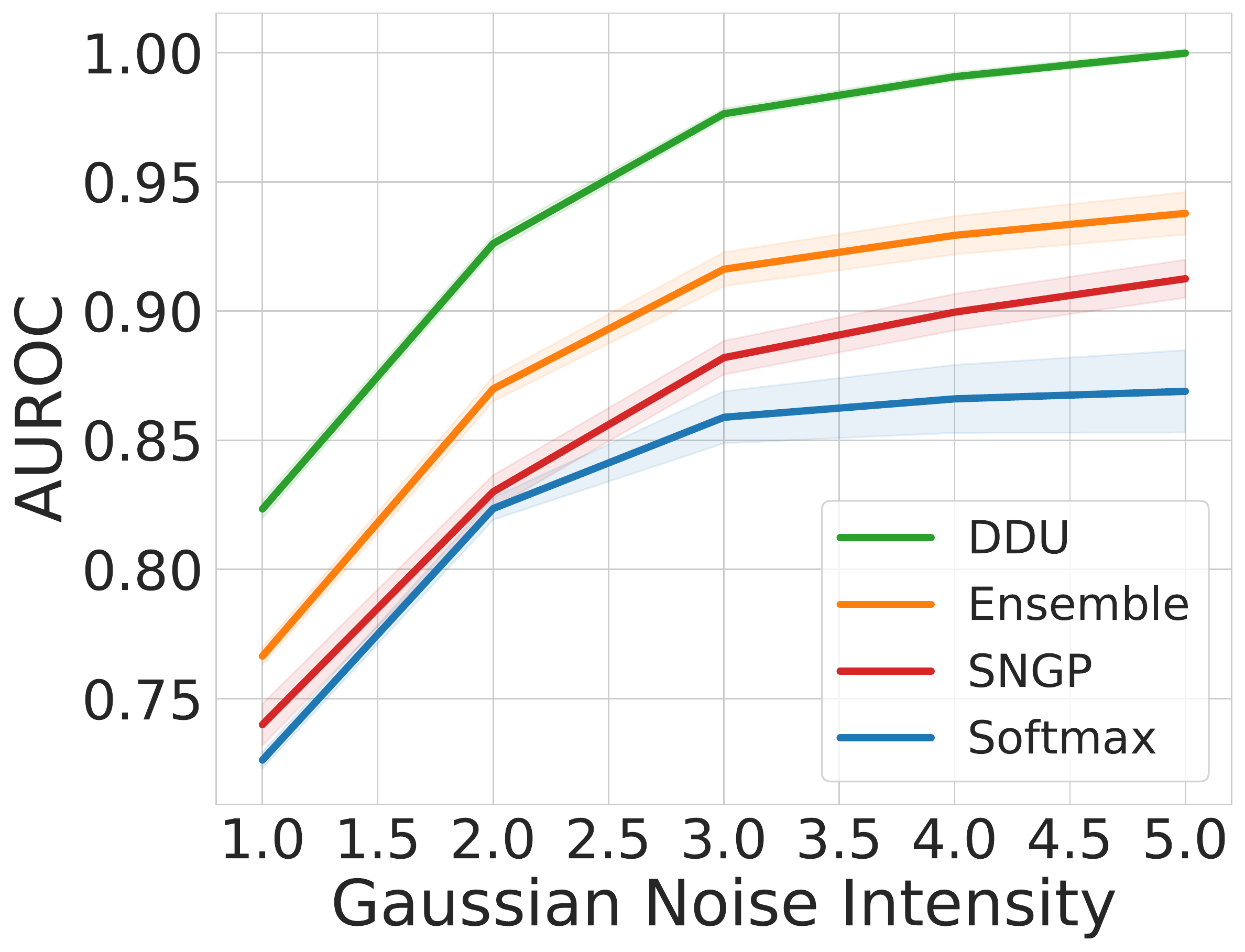}
    \end{subfigure}
    \begin{subfigure}{0.18\linewidth}
        \centering
        \includegraphics[width=\linewidth]{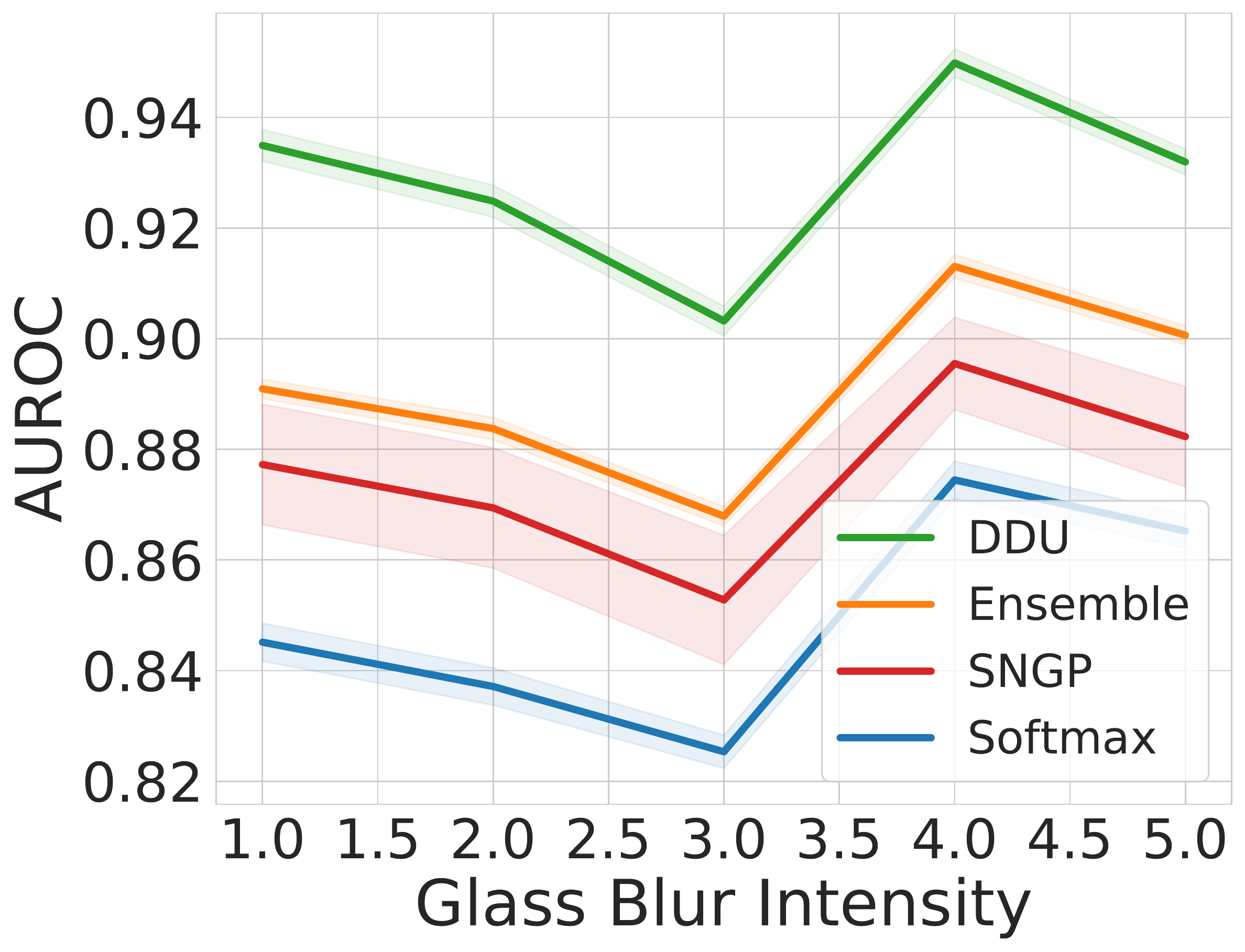}
    \end{subfigure}
    \begin{subfigure}{0.18\linewidth}
        \centering
        \includegraphics[width=\linewidth]{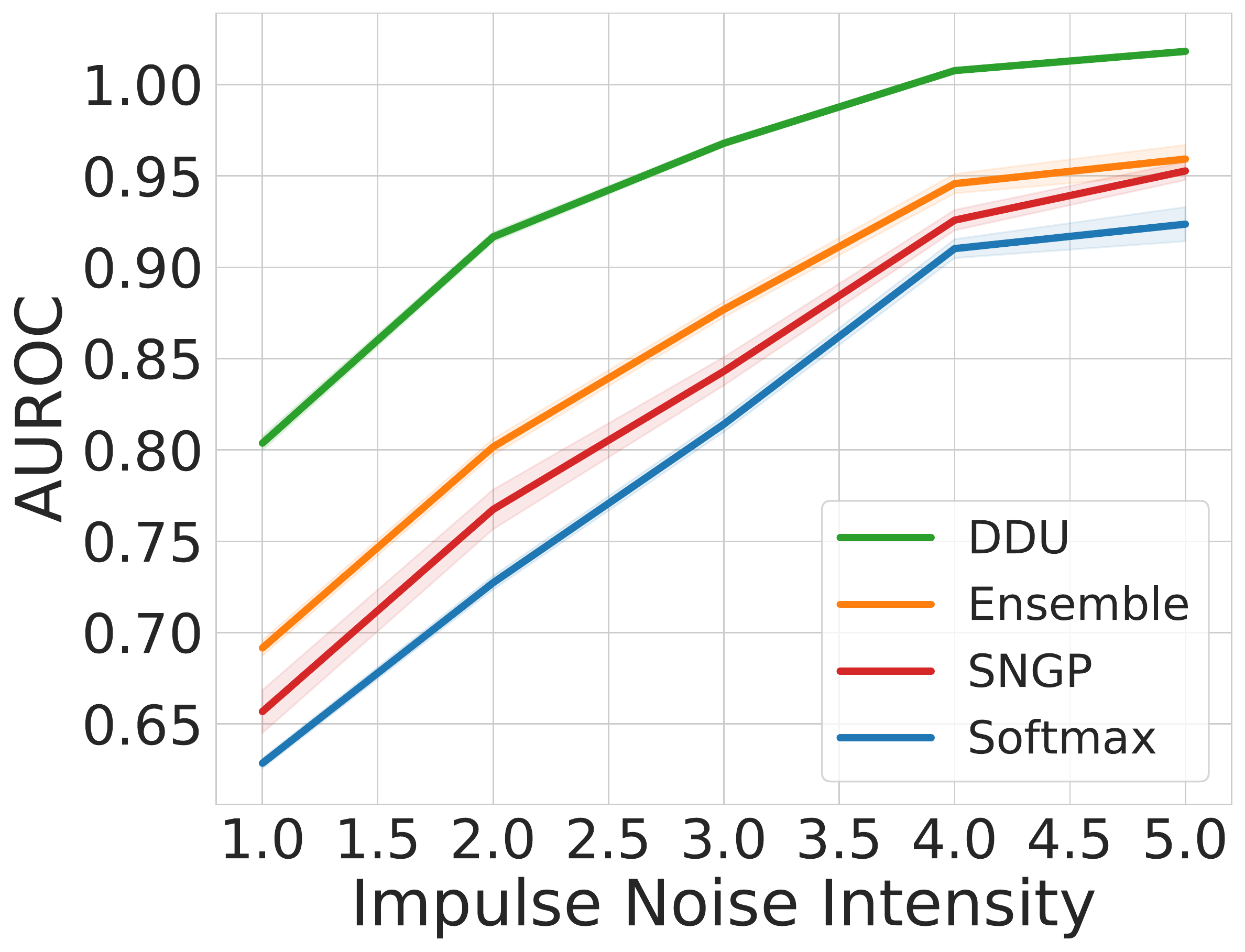}
    \end{subfigure}
    \begin{subfigure}{0.18\linewidth}
        \centering
        \includegraphics[width=\linewidth]{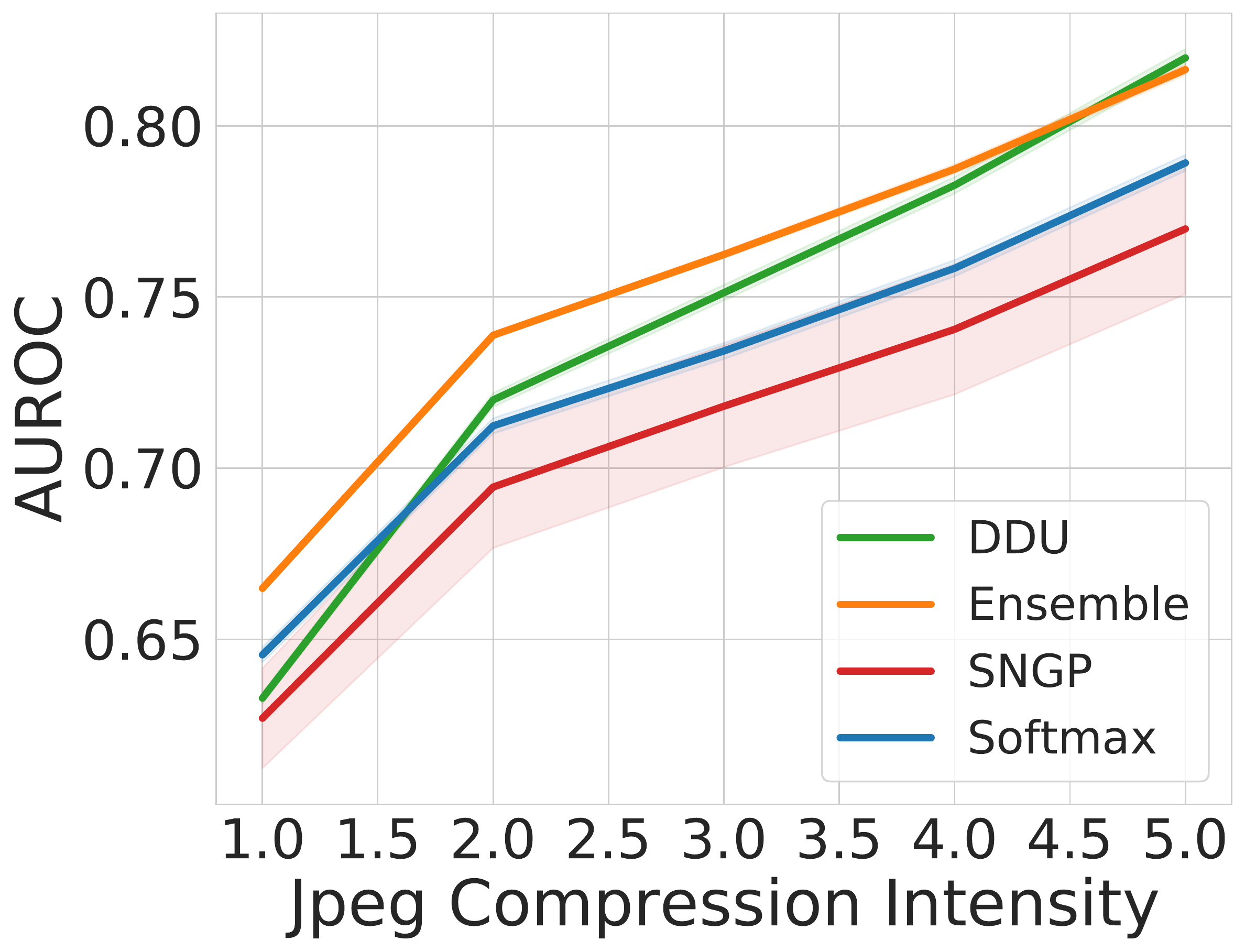}
    \end{subfigure}
    \begin{subfigure}{0.18\linewidth}
        \centering
        \includegraphics[width=\linewidth]{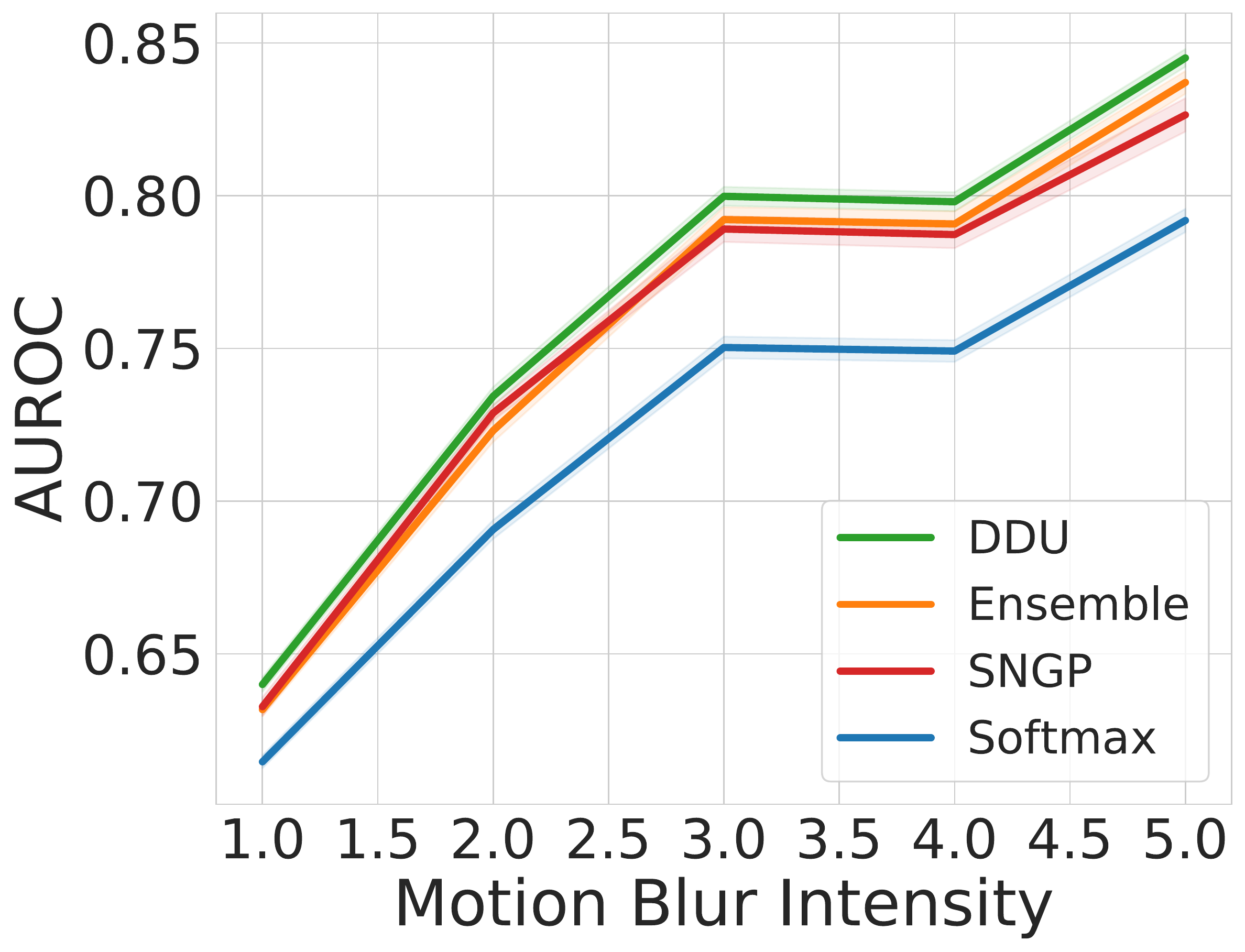}
    \end{subfigure}
    \begin{subfigure}{0.18\linewidth}
        \centering
        \includegraphics[width=\linewidth]{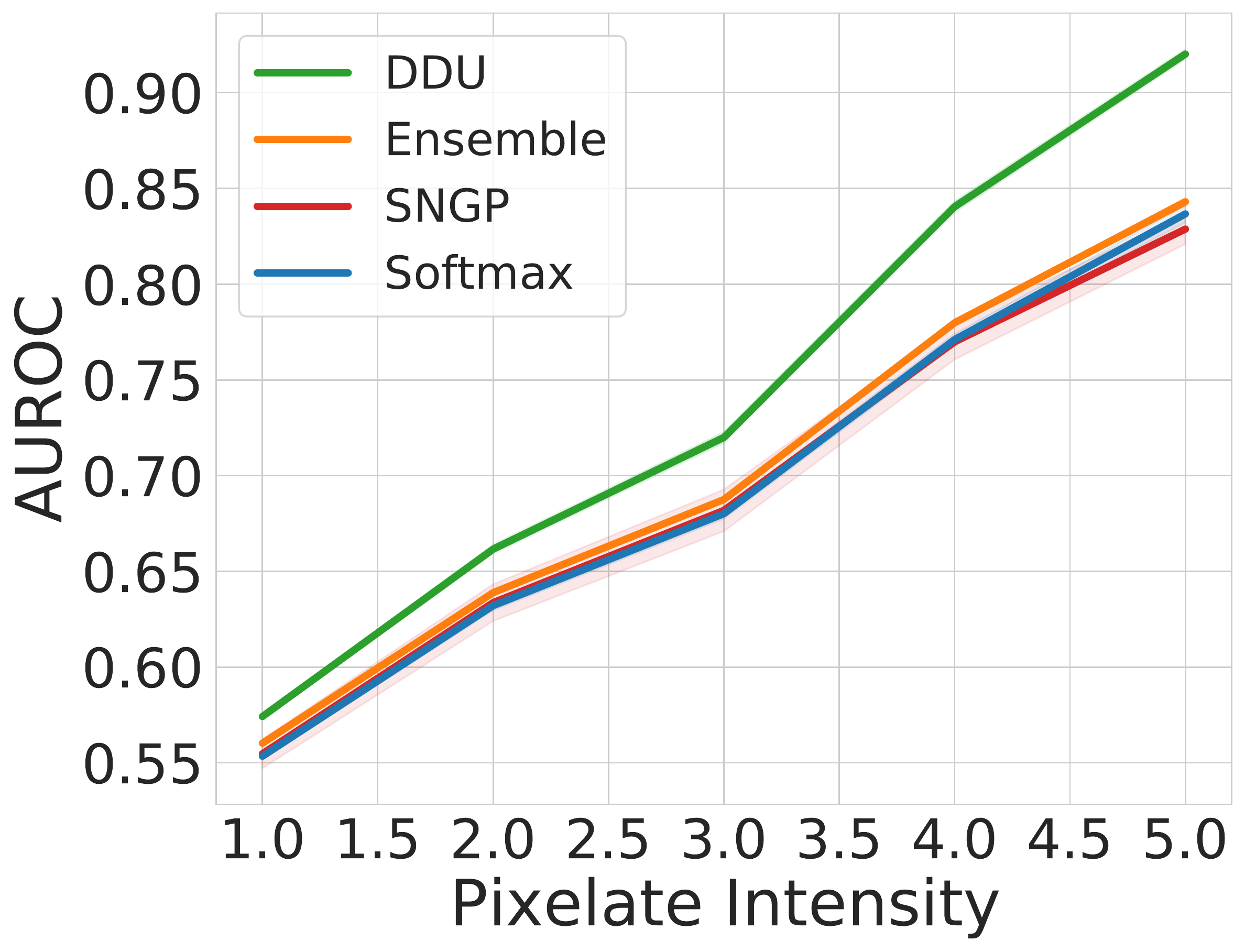}
    \end{subfigure}
    \begin{subfigure}{0.18\linewidth}
        \centering
        \includegraphics[width=\linewidth]{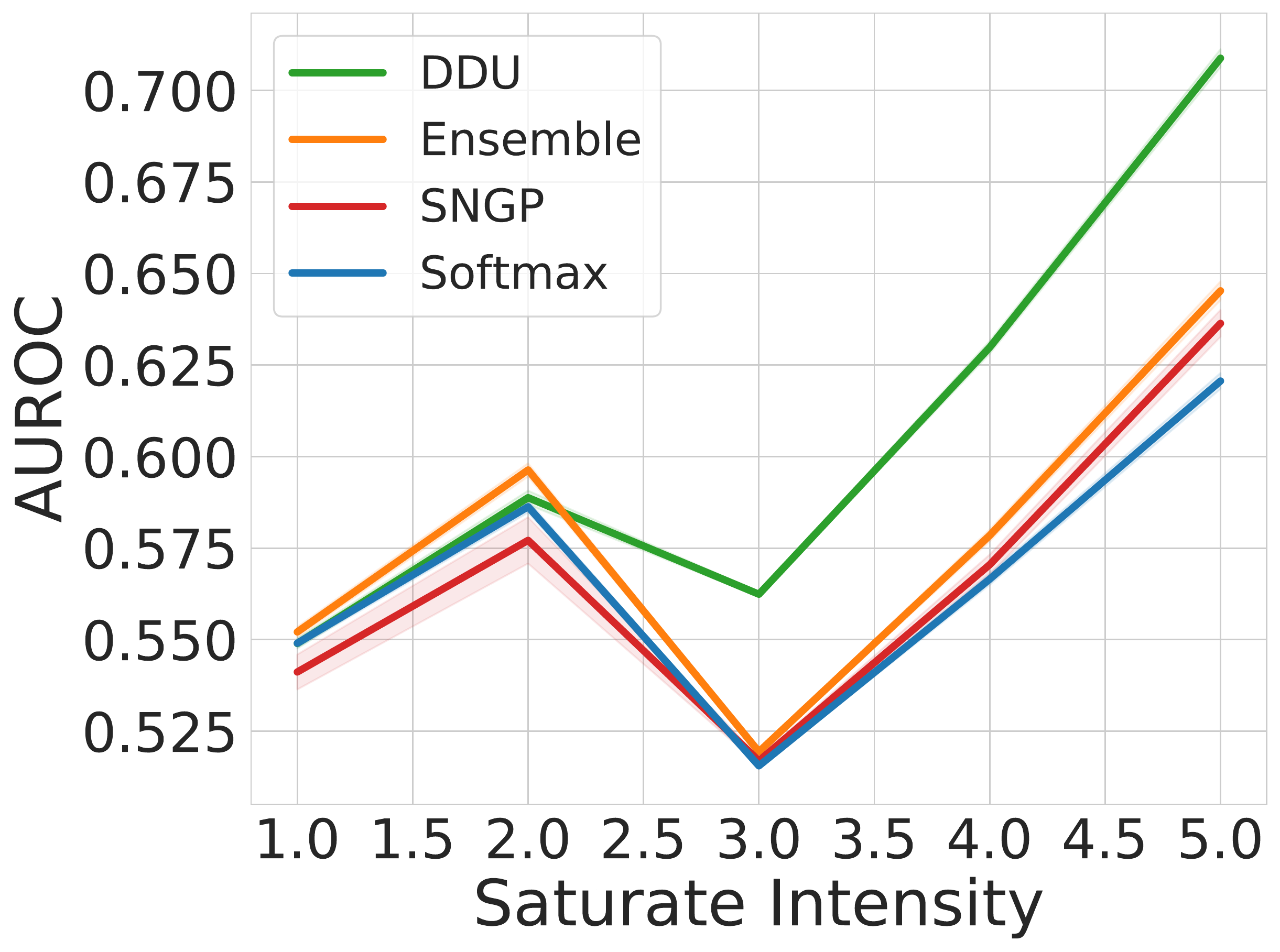}
    \end{subfigure}
    \begin{subfigure}{0.18\linewidth}
        \centering
        \includegraphics[width=\linewidth]{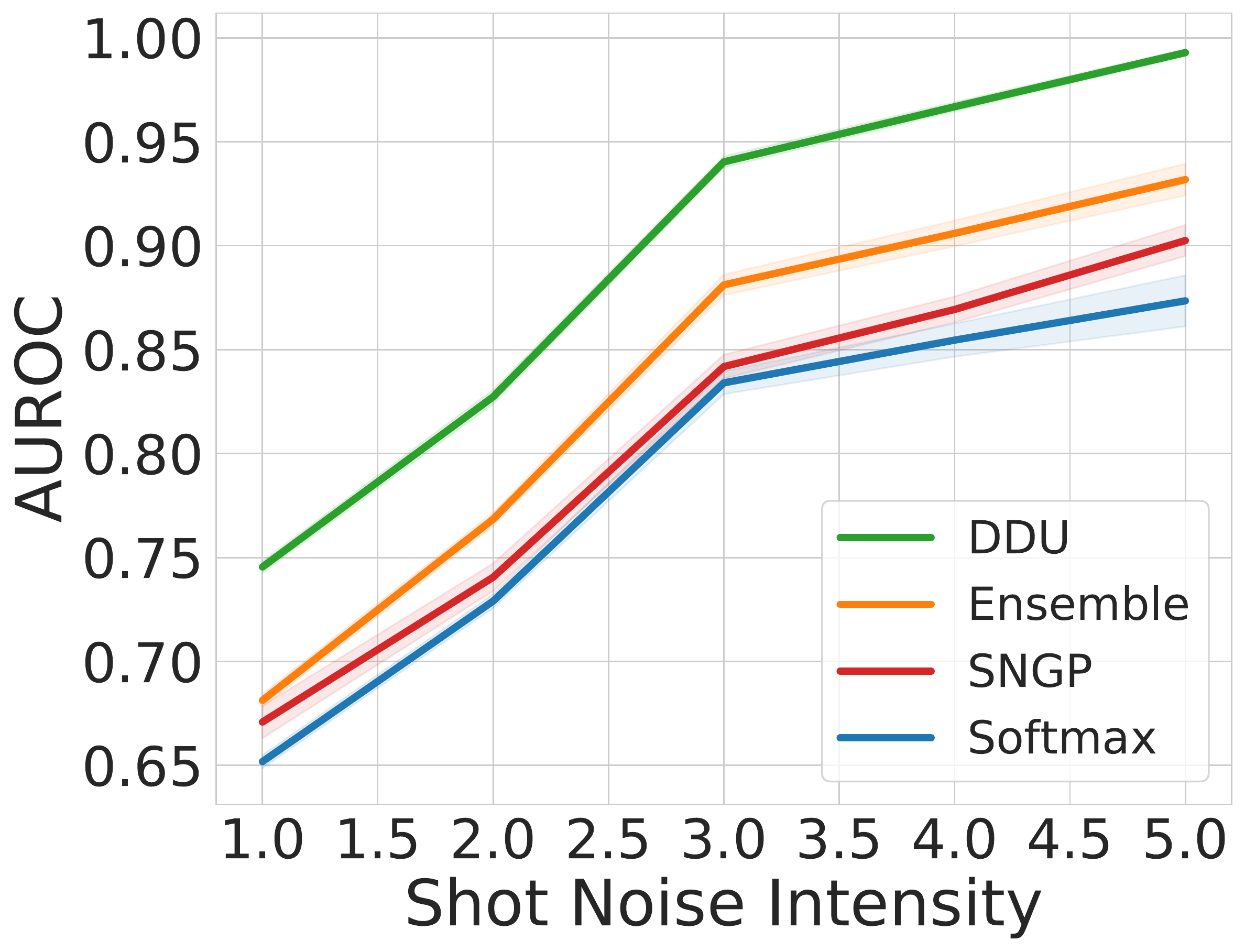}
    \end{subfigure}
    \begin{subfigure}{0.18\linewidth}
        \centering
        \includegraphics[width=\linewidth]{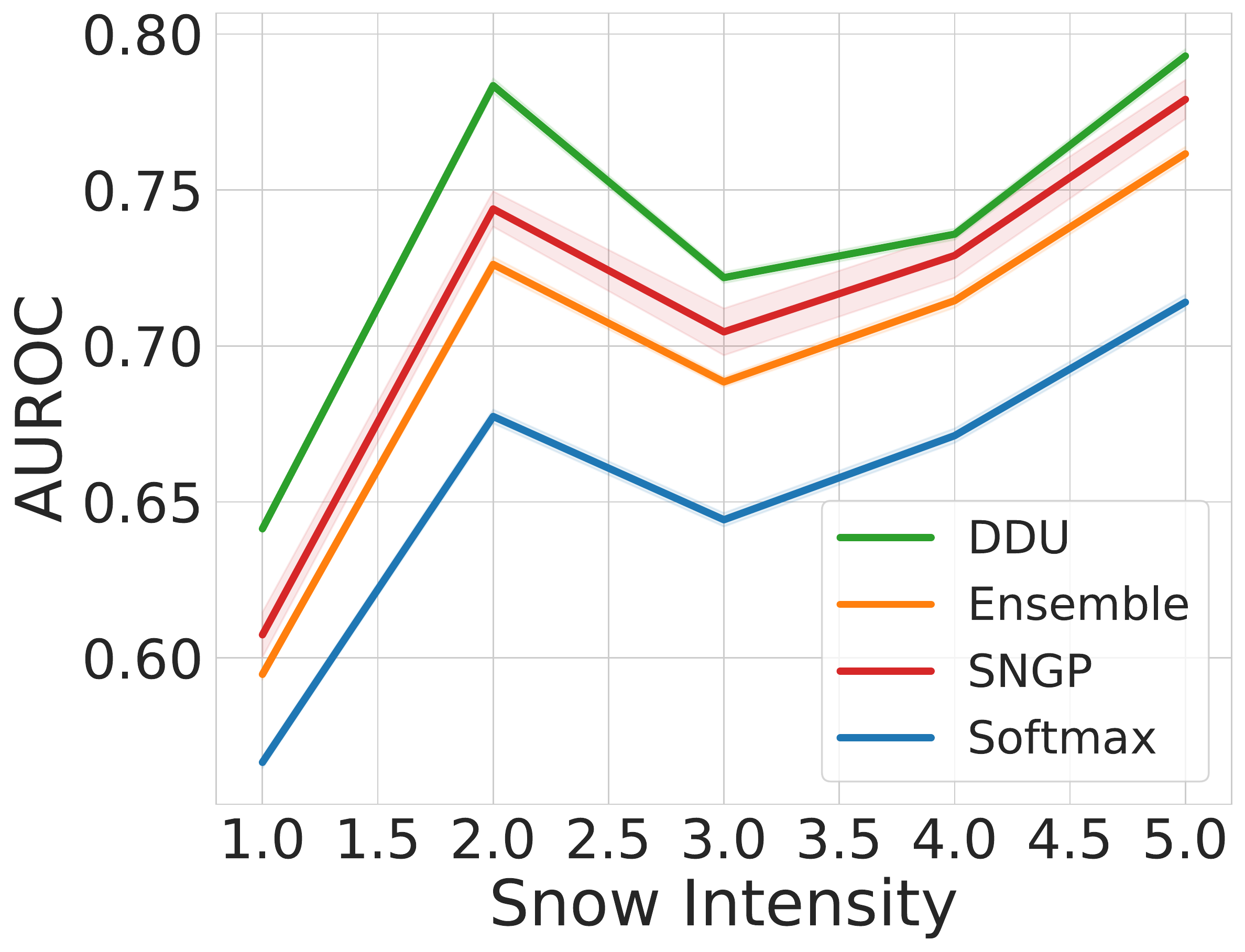}
    \end{subfigure}
    \begin{subfigure}{0.18\linewidth}
        \centering
        \includegraphics[width=\linewidth]{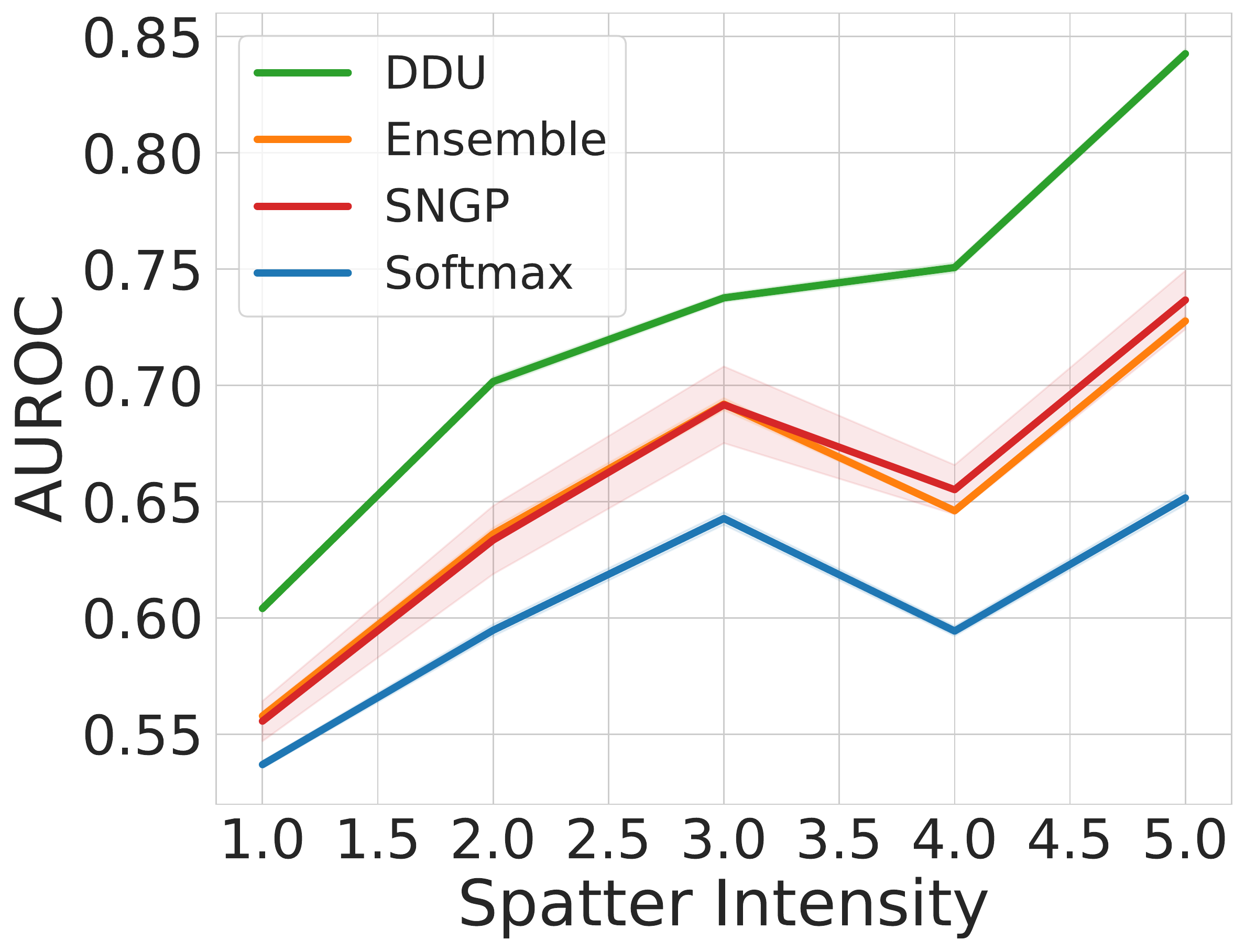}
    \end{subfigure}
    \begin{subfigure}{0.18\linewidth}
        \centering
        \includegraphics[width=\linewidth]{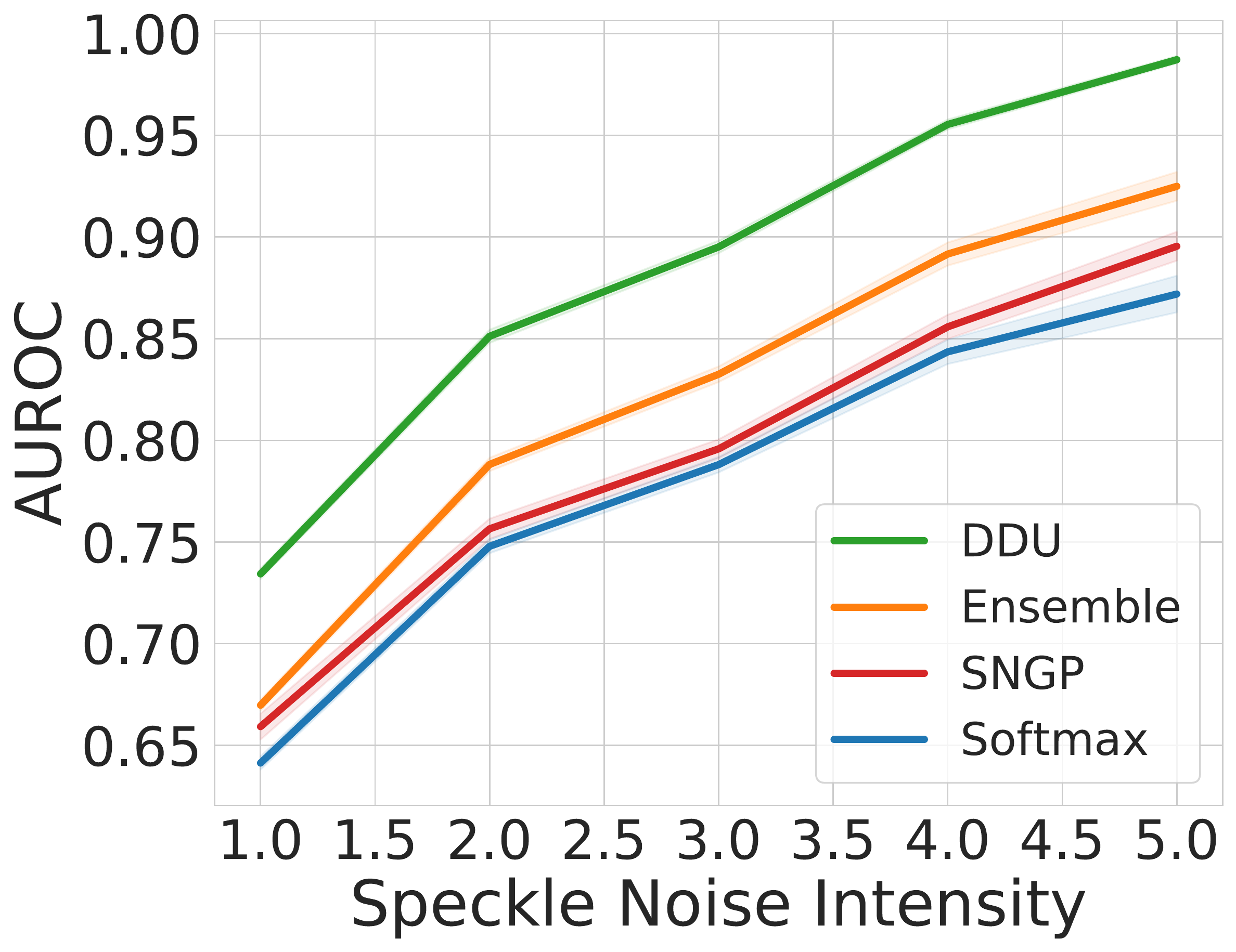}
    \end{subfigure}
    \begin{subfigure}{0.18\linewidth}
        \centering
        \includegraphics[width=\linewidth]{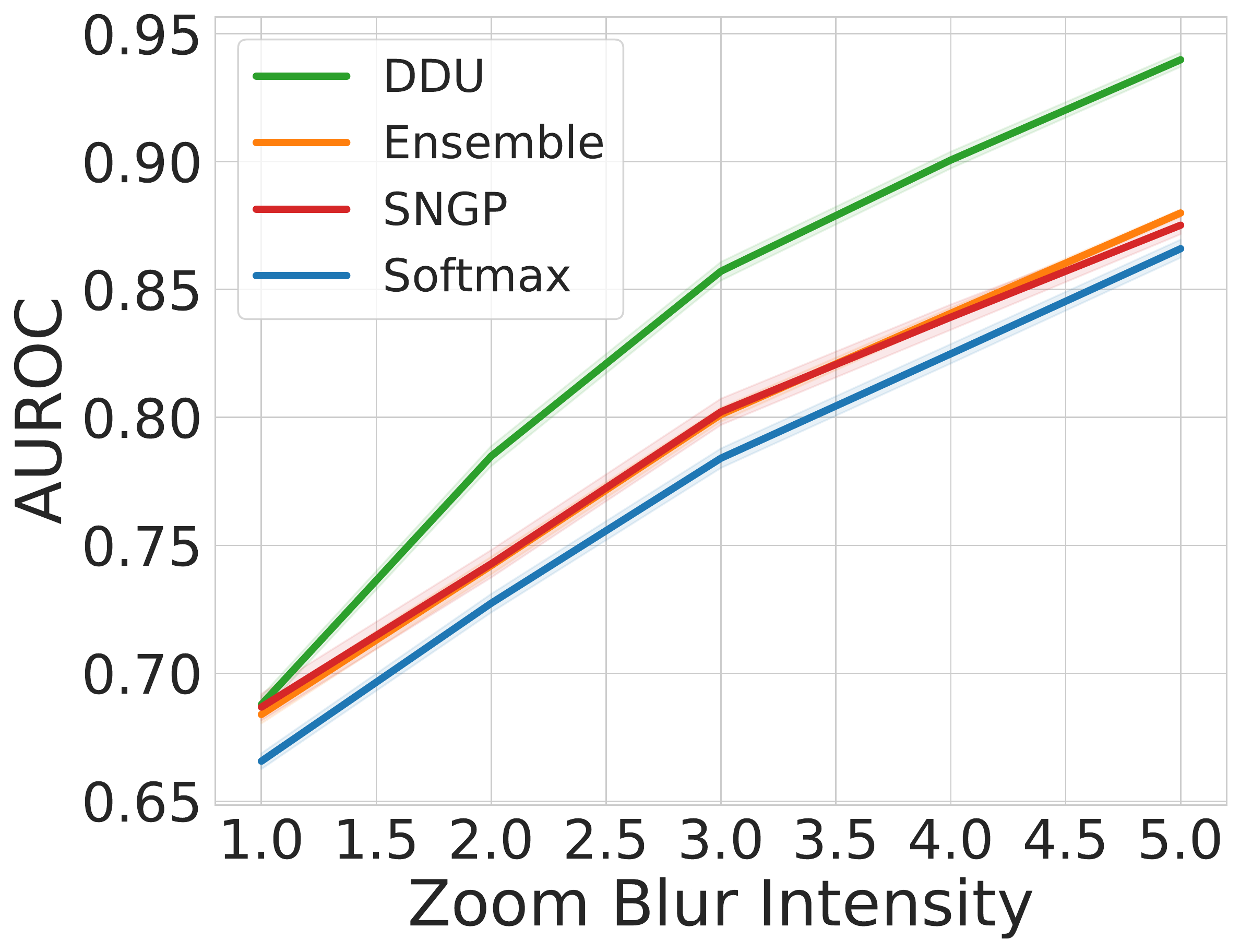}
    \end{subfigure}
    \caption{
    AUROC vs corruption intensity for all corruption types in CIFAR-10-C with DenseNet-121 as the architecture and baselines: Softmax Entropy, Ensemble (using Predictive Entropy as uncertainty), SNGP and DDU feature density.
    }
    \label{fig:cifar10_c_results_densenet121}
\end{figure}

\section{Additional Results}
\label{app:additional_exp_results}

In this section, we provide details of additional results on the OoD detection task using CIFAR-10 vs SVHN/CIFAR-100/Tiny-ImageNet/CIFAR-10-C and CIFAR-100 vs SVHN/Tiny-ImageNet for ResNet-50, ResNet-110 and DenseNet-121 architectures. We present results on ResNet-50, ResNet-110 and DenseNet-121 for CIFAR-10 vs SVHN/CIFAR-100/Tiny-ImageNet and CIFAR-100 vs SVHN/Tiny-ImageNet in \Cref{table:ood_resnet50}, \Cref{table:ood_resnet110} and \Cref{table:ood_densenet121} respectively. We also present results on individual corruption types for CIFAR-10-C for Wide-ResNet-28-10, ResNet-50, ResNet-110 and DenseNet-121 in \Cref{fig:cifar10_c_results_wide_resnet}, \Cref{fig:cifar10_c_results_resnet50}, \Cref{fig:cifar10_c_results_resnet110} and \Cref{fig:cifar10_c_results_densenet121} respectively.

Finally, we provide results for various ablations on DDU. As mentioned in \S\ref{section:methods}, DDU consists of a deterministic softmax model trained with appropriate inductive biases. It uses softmax entropy to quantify aleatoric uncertainty and feature-space density to quantify epistemic uncertainty. In the ablation, we try to experimentally evaluate the following scenarios:
\begin{enumerate}[leftmargin=*]
    \item \textbf{Effect of inductive biases (sensitivity + smoothness):} We want to see the effect of removing the proposed inductive biases (i.e.\ no sensitivity and smoothness constraints) on the OoD detection performance of a model. To do this, we train a VGG-16 with and without spectral normalisation. Note that VGG-16 does not have residual connections and hence, a VGG-16 does not follow the sensitivity and smoothness (bi-Lipschitz) constraints.
    
    \item \textbf{Effect of sensitivity alone:} Since residual connections make a model sensitive to changes in the input space by lower bounding its Lipschitz constant, we also want to see how a network performs with just the sensitivity constraint alone. To observe this, we train a Wide-ResNet-28-10 without spectral normalisation (i.e.\ no explicit upper bound on the Lipschitz constant of the model).
    
    \item \textbf{Metrics for aleatoric and epistemic uncertainty:} With the above combinations, we try to observe how different metrics for aleatoric and epistemic uncertainty perform. To quantify aleatoric uncertainty, we use the softmax entropy of the model.
    On the other hand, to quantify the epistemic uncertainty, we use \textbf{i)} the softmax entropy, \textbf{ii)} the softmax density \citep{liu2020energy} or \textbf{iii)} the GMM feature density (as described in \S\ref{section:methods}). 
\end{enumerate}

For the purposes of comparison, we also present scores obtained by a 5-Ensemble of the respective architectures (i.e.\ Wide-ResNet-28-10 and VGG-16) in \Cref{table:ood_2} for CIFAR-10 vs SVHN/CIFAR-100 and in \Cref{table:ood_3} for CIFAR-100 vs SVHN. Based on these results, we can make the following observations (in addition to the ones we make in \S\ref{sec:experiments_ood_detection}):

\textbf{Inductive biases are important for feature density.} From the AUROC scores in \Cref{table:ood_2}, we can see that using the feature density of a GMM in VGG-16 without the proposed inductive biases yields significantly lower AUROC scores as compared to Wide-ResNet-28-10 with inductive biases. In fact, in none of the datasets is the feature density of a VGG able to outperform its corresponding ensemble. This provides yet more evidence (in addition to \Cref{fig:intro_histograms}) to show that the GMM feature density alone cannot estimate epistemic uncertainty in a model that suffers from feature collapse. We need sensitivity and smoothness conditions (see \S\ref{sec:motivation}) on the feature space of the model to obtain feature densities that capture epistemic uncertainty.

\textbf{Sensitivity creates a bigger difference than smoothness.} We note that the difference between AUROC obtained from feature density between Wide-ResNet-28-10 models with and without spectral normalisation is minimal. Although Wide-ResNet-28-10 with spectral normalisation (i.e.\ smoothness constraints) still outperforms its counterpart without spectral normalisation, the small difference between the AUROC scores indicates that it might be the residual connections (i.e.\ sensitivity constraints) that make the model detect OoD samples better. This observation is also intuitive as a sensitive feature extractor should map OoD samples farther from iD ones.

\textbf{DDU as a simple baseline.} In DDU, we use the softmax output of a model to get aleatoric uncertainty. We use the GMM's feature-density to estimate the epistemic uncertainty. Hence, DDU does not suffer from miscalibration as the softmax outputs can be calibrated using post-hoc methods like temperature scaling. At the same time, the feature-densities of the model are not affected by temperature scaling and capture epistemic uncertainty well.

\section{Toy Experiments \& Additional Ablations}
Here, we provide details for toy experiments from the main paper which are visualized in \Cref{fig:intro_histograms}, \Cref{fig:lewis_vis} and \Cref{fig:kendall_viz}.

\subsection{Motivational Example in \Cref{fig:intro_histograms}}

\begin{table}[t]
    \centering
    \tiny
    \caption{\emph{ECE for Dirty-MNIST test set and AUROC for Dirty-MNIST vs Fashion-MNIST as proxies for aleatoric and epistemic uncertainty quality respectively.}
    }
    \label{table:auroc_tab1}
    \renewcommand{\arraystretch}{1.2} 
    \begin{tabular}{cccc}
    \toprule
    \textbf{Model} & \textbf{ECE} & \textbf{AUROC for Softmax Entropy (\textuparrow)} &
    \textbf{AUROC for Feature Density (\textuparrow)} \\
    
    \midrule
    LeNet & $2.22$ & $84.23$ & $71.41$\\
    VGG-16 & $\mathbf{2.11}$ & $84.04$ & $89.01$\\
    \textbf{ResNet-18+SN (DDU)} & $2.34$ & $83.01$ & $\mathbf{99.91}$\\
    \bottomrule
    \end{tabular}%
\end{table}

In \Cref{fig:intro_histograms} we train a LeNet \citep{lecun1998gradient}, a VGG-16 \citep{simonyan2014very} and a ResNet-18 with spectral normalisation \citep{he2016deep,miyato2018spectral} (ResNet-18+SN) on \emph{Dirty-MNIST}, a modified version of MNIST \citep{lecun1998gradient} with additional ambiguous digits (Ambiguous-MNIST).
\emph{Ambiguous-MNIST} contains samples with multiple plausible labels and thus higher aleatoric uncertainty (see \Cref{fig:intro_sample_viz}).
We refer to \S\ref{app:dirty_mnist} for details on how this dataset was generated. 
With ambiguous data having various levels of aleatoric uncertainty, Dirty-MNIST is more representative of real-world datasets compared to well-cleaned curated datasets, like MNIST and CIFAR-10, commonly used for benchmarking in ML \citep{filos2019systematic, krizhevsky2009learning}. Moreover, Dirty-MNIST also poses a challenge for recent uncertainty estimation methods, which often confound aleatoric and epistemic uncertainty \citep{van2020simple}. 
\Cref{fig:intro_softmax_ent} shows that the softmax entropy of a deterministic model is unable to distinguish between iD (Dirty-MNIST) and OoD (Fashion-MNIST \citep{xiao2017fashion}) samples as the entropy for the latter heavily overlaps with the entropy for Ambiguous-MNIST samples.
However, the feature-space density of the model with our inductive biases in \Cref{fig:intro_gmm} captures epistemic uncertainty reliably and is able to distinguish iD from OoD samples.
The same cannot be said for LeNet or VGG in \Cref{fig:intro_gmm}, whose densities are unable to separate OoD from iD samples.
This demonstrates the importance of the inductive bias to ensure the sensitivity and smoothness of the feature space as we further argue below.
Finally, \Cref{fig:intro_softmax_ent} and \Cref{fig:intro_gmm} demonstrate that our method is able to separate aleatoric from epistemic uncertainty: 
samples with low feature density have high epistemic uncertainty, whereas those with both high feature density and high softmax entropy have high aleatoric uncertainty---note the high softmax entropy for the most ambiguous Ambiguous-MNIST samples in \Cref{fig:intro_softmax_ent}.

\subsubsection{Disentangling Epistemic and Aleatoric Uncertainty}
\label{app:experiments_disentangling}

We used a simple example in \S\ref{sec:intro} to demonstrate that a single softmax model with a proper inductive bias can reliably capture epistemic uncertainty via its feature-space density and aleatoric uncertainty via its softmax entropy.
To recreate the natural characteristics of uncurated real-world datasets, which contain ambiguous samples, we use MNIST \citep{lecun1998gradient} as an iD dataset of unambiguous samples, Ambiguous-MNIST as an iD dataset of ambiguous samples and Fashion-MNIST \citep{xiao2017fashion} as an OoD dataset (see \Cref{fig:intro_sample_viz}), with more details
in \S\ref{app:dirty_mnist}.
We train a LeNet \citep{lecun1998gradient}, a VGG-16 \citep{simonyan2014very} and a ResNet-18 \citep{he2016deep} with spectral normalisation (SN) on Dirty-MNIST (a mix of Ambiguous- and standard MNIST) with the training setup detailed in \S\ref{app:exp_details_dirty_mnist}.

\Cref{table:auroc_tab1} gives a quantitative evaluation of the qualitative results in \S\ref{sec:intro}.
The AUROC metric reflects the quality of the epistemic uncertainty as it measures the probability that iD and OoD samples can be distinguished, and OoD samples are never seen during training while iD samples are semantically similar to training samples. The ECE metric measures the quality of the aleatoric uncertainty.
The softmax outputs capture aleatoric uncertainty well, as expected, and all 3 models obtain similar ECE scores on the Dirty-MNIST test set.
However, with an AUROC of around $84\%$ for all the 3 models, on Dirty-MNIST vs Fashion-MNIST, we conclude that softmax
entropy is unable to capture epistemic uncertainty well.
This is reinforced in \Cref{fig:intro_softmax_ent}, which shows a strong overlap between the softmax entropy of OoD and ambiguous iD samples.
At the same time, the feature-space densities of LeNet and VGG-16, with AUROC scores around $71\%$ and $89\%$ respectively, are unable to distinguish OoD from iD samples, indicating that simply using feature-space density without appropriate inductive biases (as seen in \citep{lee2018simple}) is not sufficient.

\emph{Only by fitting a GMM on top of a feature extractor with appropriate inductive biases (DDU) and using its feature density are we able to obtain performance far better (with AUROC of $99.9\%$) than the alternatives in the ablation study (see \Cref{table:auroc_tab1}, also noticeable in \Cref{fig:intro_gmm})}.
The entropy of a softmax model can capture aleatoric uncertainty, even without additional inductive biases, but it \emph{cannot} be used to estimate epistemic uncertainty (see \S\ref{sec:motivation}). On the other hand, feature-space density can \emph{only} be used to estimate epistemic uncertainty \emph{when the feature extractor is sensitive and smooth}, as achieved by using a ResNet and spectral normalisation in DDU.

\subsection{Effects of a well-regularized feature space on feature collapse}
\label{app:feature_collapse_fig2}
The effects of a well-regularized feature space on feature collapse are visible in \Cref{fig:intro_histograms}. In the case of feature collapse, we must have \textit{some} OoD inputs for which the features are mapped on top of the features of iD inputs. The distances of these OoD features to each class centroid must be equal to the distances of the corresponding iD inputs to class centroids, and hence the density for these OoD inputs must be equal to the density of the iD inputs. If the density histograms do not overlap, no feature collapse can be present\footnote{Note though that the opposite (`if the density histograms overlap then there must be feature collapse') needs not hold: the histograms can also overlap due to other reasons.}. We see no overlapping densities in \Cref{fig:intro_histograms}(c, right), therefore we indeed have no feature collapse.
First, the effects of having a well-regularized feature space on feature collapse can be analysed from \Cref{fig:intro_histograms}. In case of feature collapse we must have \textit{some} OoD features mapped ontop of iD features, therefore the distances of at  least some OoD features to the class centroids must be equal to iD's distances, hence OoD density must overlap with iD density. We see this in \Cref{fig:intro_histograms}(c) (left) in the case without a regularized feature space. On the other hand, when we regularize the feature space, we see the densities do not overlap, i.e.\ the distances of the features of OoD examples to the centroids are larger than the distances of iD examples (\Cref{fig:intro_histograms}(c) right), hence feature collapse is not present.

\subsection{Two Moons}
\label{app:2_moons}

\begin{figure*}[!t]
    \centering
    \begin{subfigure}{0.24\linewidth}
        \centering
        \includegraphics[width=\linewidth]{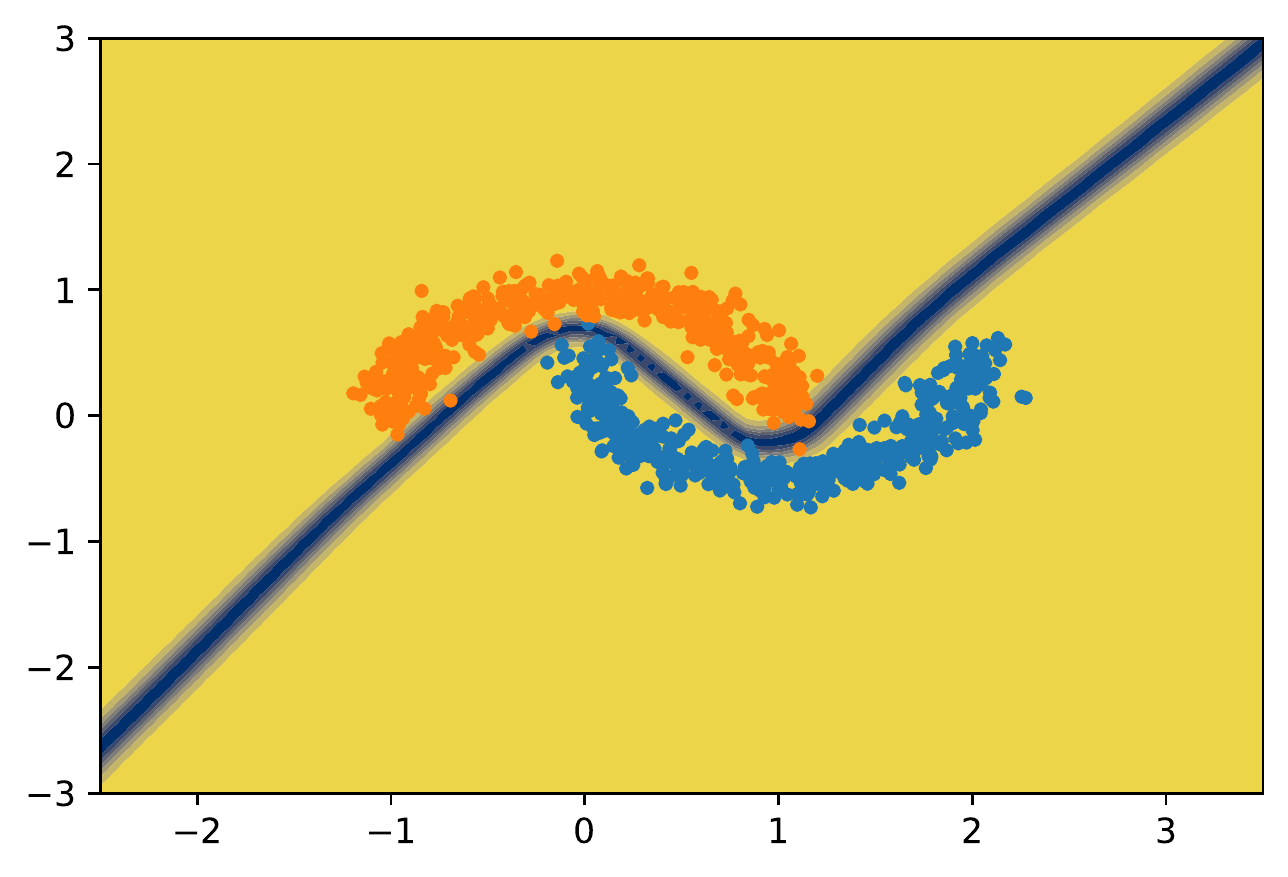}
        \caption{Softmax Entropy}
        \label{subfig:two_moons_softmax_entropy_no_sn}
    \end{subfigure} \hfill
    \begin{subfigure}{0.24\linewidth}
        \centering
        \includegraphics[width=\linewidth]{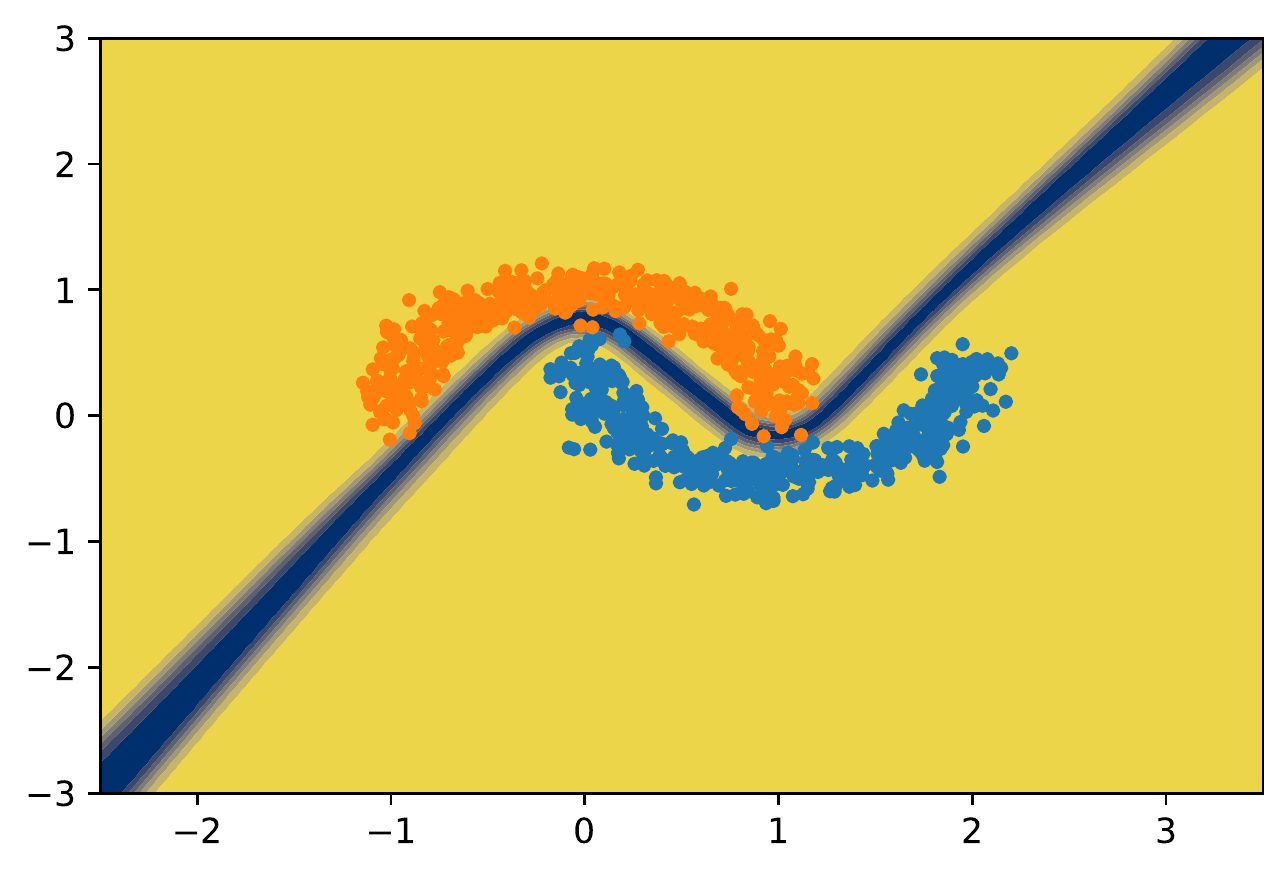}
        \caption{Ensemble Predictive Entropy}
        \label{subfig:two_moons_ensemble_predictive_entropy_no_sn}
    \end{subfigure} \hfill
    \begin{subfigure}{0.24\linewidth}
        \centering
        \includegraphics[width=\linewidth]{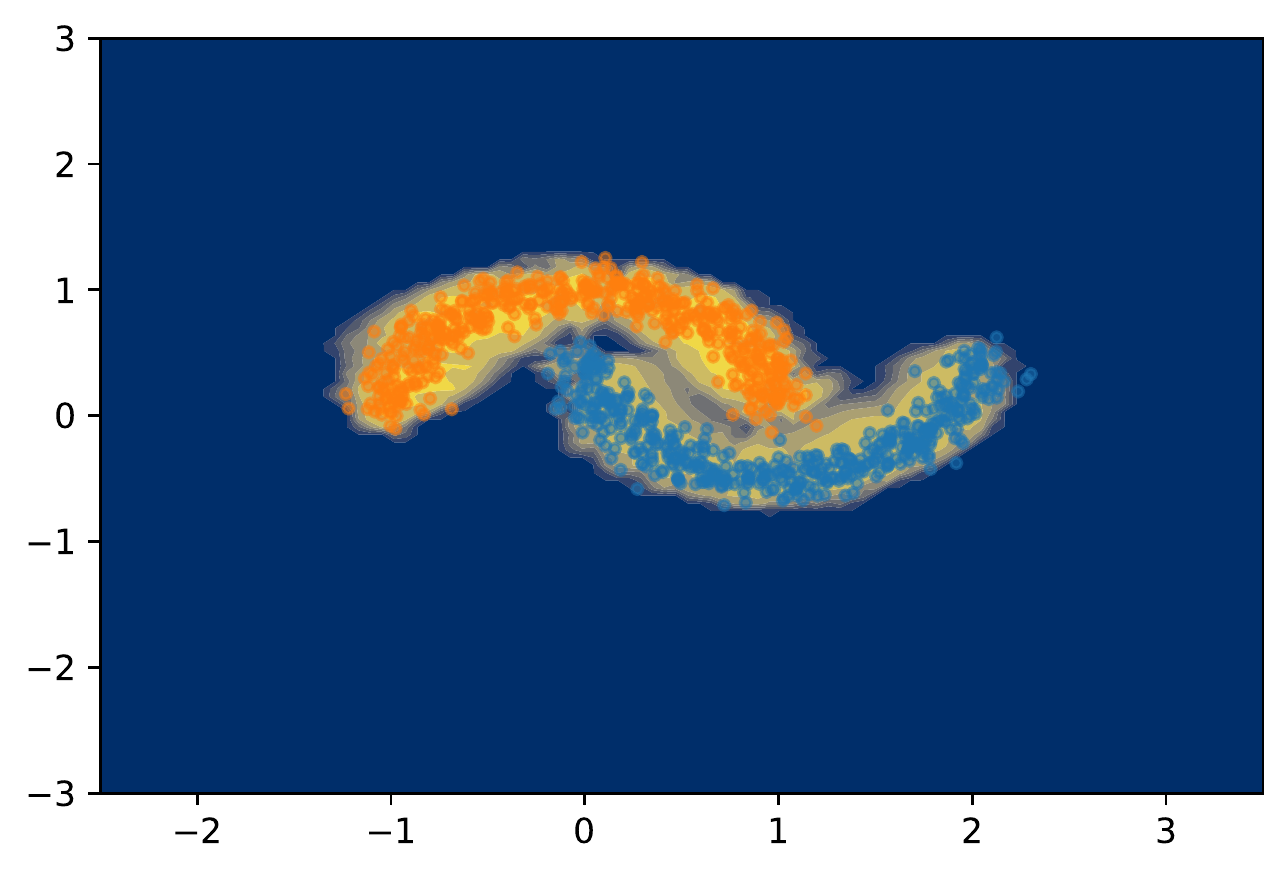}
        \caption{DDU Feature-space density}
        \label{subfig:two_moons_ddu_sn}
    \end{subfigure}
    \begin{subfigure}{0.24\linewidth}
        \centering
        \includegraphics[width=\linewidth]{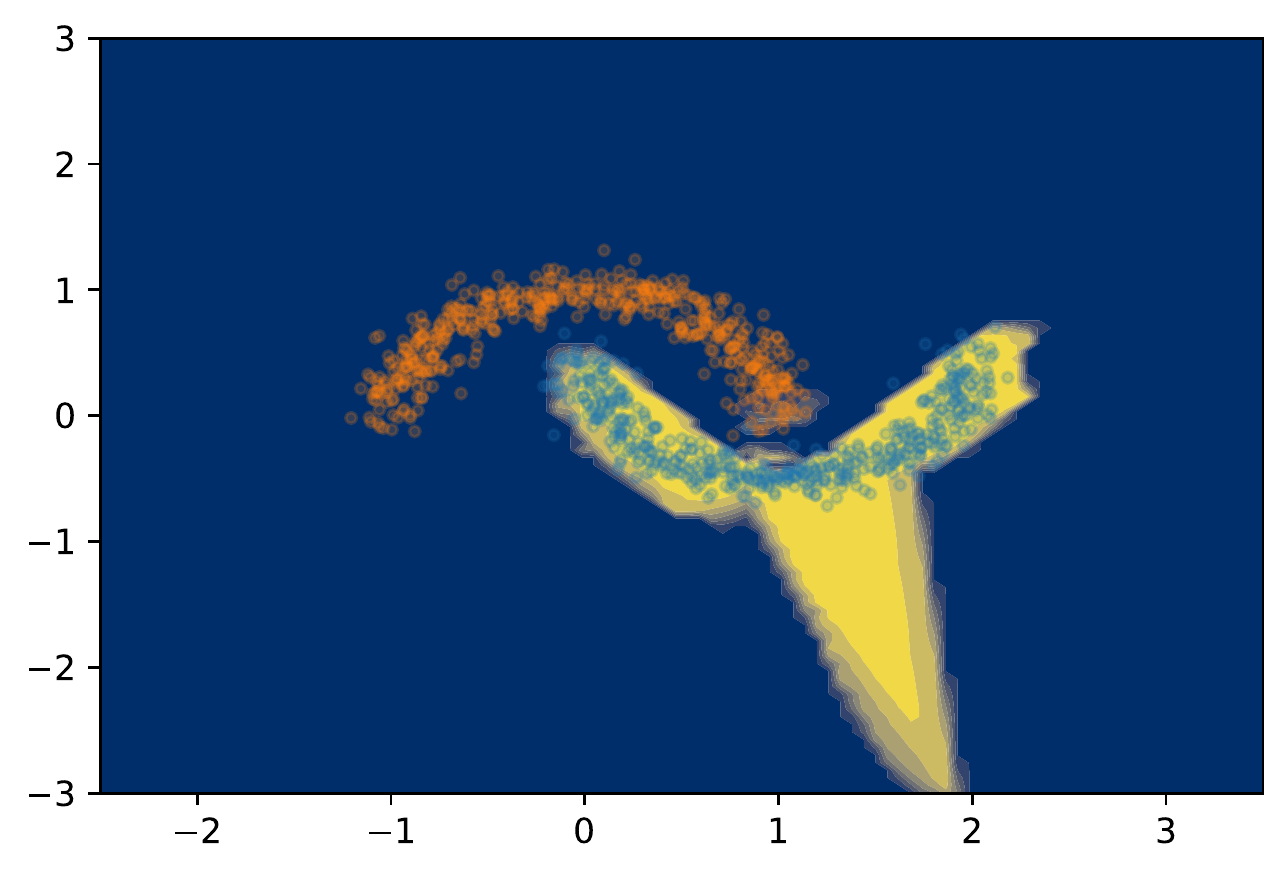}
        \caption{FC-Net Feature-space density}
        \label{subfig:two_moons_fcnet}
    \end{subfigure}
    \caption{
    \emph{Uncertainty on Two Moons dataset}. Blue indicates high uncertainty and yellow indicates low uncertainty. Both the softmax entropy of a single model as well as the predictive entropy of a deep ensemble are uncertain only along the decision boundary whereas the feature-space density of DDU is uncertain everywhere except on the data distribution (the ideal behaviour). However, the feature density of a normal fully connected network (FC-Net) without any inductive biases can't capture uncertainty properly.
    }
    \label{fig:2-moons}
\end{figure*}

In this section, we evaluate DDU's performance on a well-known toy setup: the Two Moons dataset. We use scikit-learn's \emph{datasets} package to generate 2000 samples with a noise rate of 0.1. We use a 4-layer fully connected architecture, ResFFN-4-128 with 128 neurons in each layer and a residual connection, following \citep{liu2020simple}. As an ablation, we also train using a 4-layer fully connected architecture with 128 neurons in each layer, but \emph{without the residual connection}. We name this architecture FC-Net. The input is 2-dimensional and is projected into the 128 dimensional space using a fully connected layer. Using the ResFFN-4-128 architecture we train 3 baselines:
\begin{enumerate}[leftmargin=*]
    \item \textbf{Softmax: } We train a single softmax model and use the softmax entropy as the uncertainty metric.
    \item \textbf{3-ensemble: } We train an ensemble of 3 softmax models and use the predictive entropy of the ensemble as the measure of uncertainty.
    \item \textbf{DDU: } We train a single softmax model applying spectral normalization on the fully connected layers and using the feature density as the measure of model confidence.
\end{enumerate}

Each model is trained using the Adam optimiser for 150 epochs. In \Cref{fig:2-moons}, we show the uncertainty results for all the above 3 baselines. It is clear that both the softmax entropy as well as the predictive entropy of the ensemble is uncertain only along the decision boundary between the two classes whereas DDU is confident only on the data distribution and is not confident anywhere else. It is worth mentioning that even DUQ and SNGP perform well in this setup and deep ensembles have been known to underperform in the Two-Moons setup primarily due to the simplicity of the dataset causing all the ensemble components to generalise in the same way. Finally, also note that the feature space density of FC-Net without residual connections is not able to capture uncertainty well (see \Cref{subfig:two_moons_fcnet}), thereby reaffirming our claim that proper inductive biases are indeed a necessary component to ensure that feature space density captures uncertainty reliably. Thus, in addition to its excellent performance in active learning, CIFAR-10 vs SVHN/CIFAR-100/Tiny-ImageNet/CIFAR-10-C, CIFAR-100 vs SVHN/Tiny-ImageNet, and ImageNet vs ImageNet-O, we note that DDU captures uncertainty reliably even in a small 2D setup like Two Moons.

\subsection{5-Ensemble Visualisation}
\label{app:5_ensemble}

In \Cref{fig:lewis_vis}, we provide a visualisation of a 5-ensemble (with five deterministic softmax networks) to see how softmax entropy fails to capture epistemic uncertainty precisely because the mutual information (MI) of an ensemble does not (see \S\ref{sec:motivation}). We train the networks on 1-dimensional data with binary labels 0 and 1. The data is shown in \Cref{subfig:softmax_output}. From \Cref{subfig:softmax_output} and \Cref{subfig:softmax_entropy}, we find that the softmax entropy is high in regions of ambiguity where the label can be both 0 and 1 (i.e.\ x between -4.5 and -3.5, and between 3.5 and 4.5). This indicates that softmax entropy can capture aleatoric uncertainty. Furthermore, in the x interval $(-2, 2)$, we find that the deterministic softmax networks disagree in their predictions (see \Cref{subfig:softmax_output}) and have softmax entropies which can be high, low or anywhere in between (see \Cref{subfig:softmax_entropy}) following our claim in \S\ref{sec:motivation}. In fact, this disagreement is the very reason why the MI of the ensemble is high in the interval $(-2, 2)$, thereby reliably capturing epistemic uncertainty. Finally, note that the predictive entropy (PE) of the ensemble is high both in the OoD interval $(-2, 2)$ as well as at points of ambiguity (i.e.\ at -4 and 4). This indicates that the PE of a Deep Ensemble captures both epistemic and aleatoric uncertainty well. From these visualisations, we draw the conclusion that the softmax entropy of a deterministic softmax model cannot capture epistemic uncertainty precisely because the MI of a Deep Ensemble can.

\subsection{Feature-Space Density \& Epistemic Uncertainty vs Softmax Entropy \& Aleatoric Uncertainty}
\label{app:kendall_viz}

\begin{figure*}[!t]
    \centering
    \begin{subfigure}{0.16\linewidth}
        \centering
        \includegraphics[width=\linewidth]{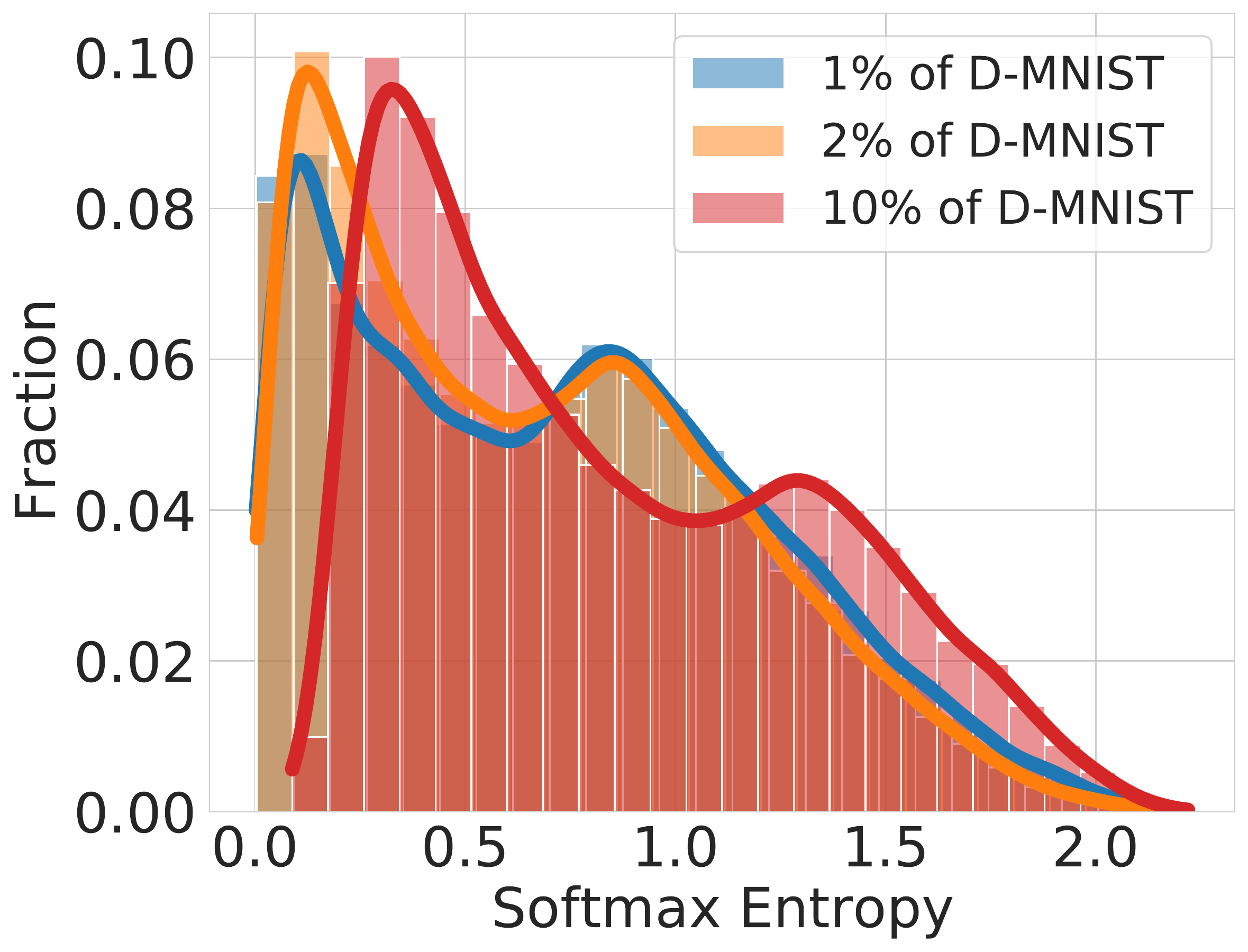}
        \caption{SE D-MNIST}
        \label{subfig:softmax_entropy_dmnist}
    \end{subfigure}
    \begin{subfigure}{0.16\linewidth}
        \centering
        \includegraphics[width=\linewidth]{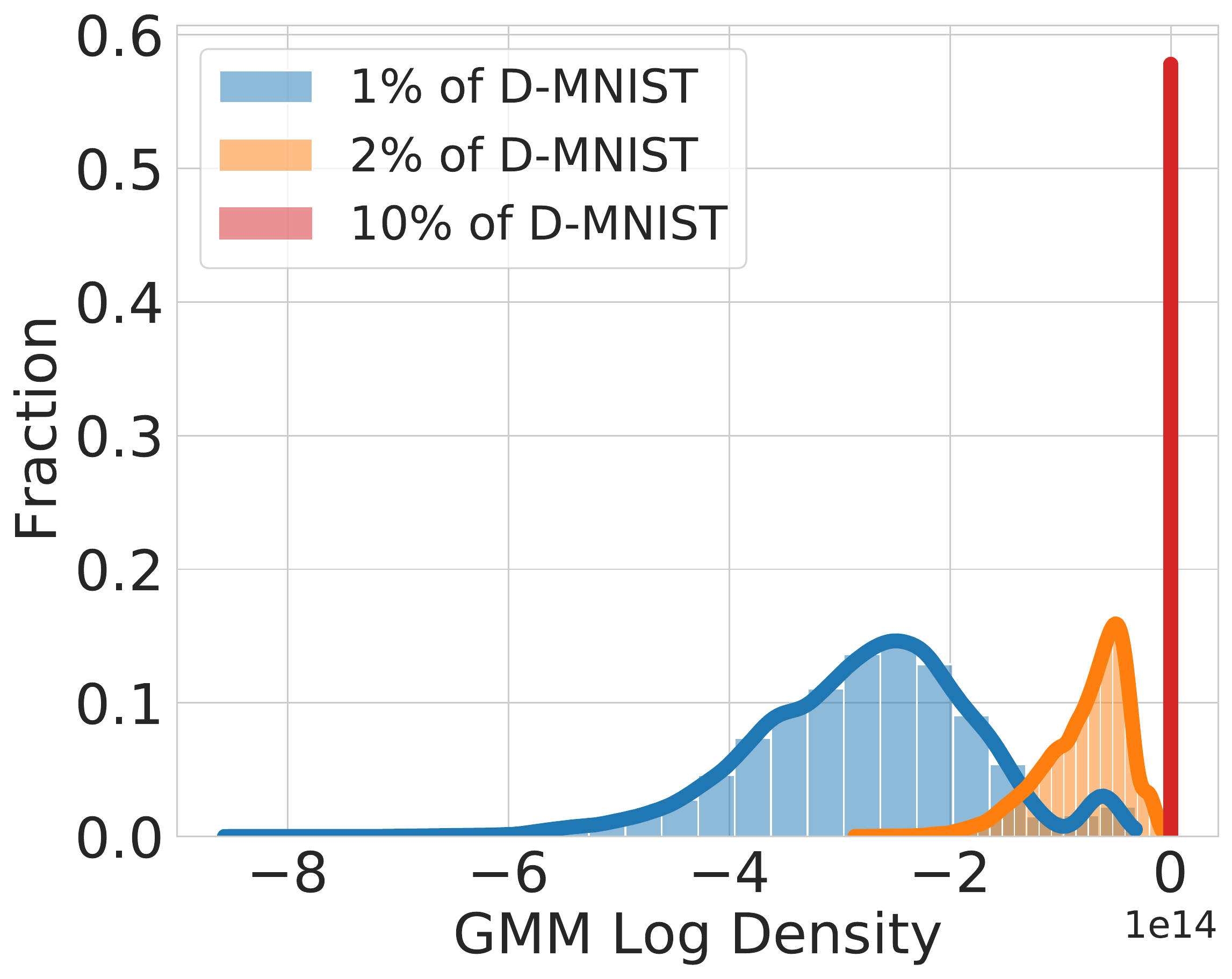}
        \caption{Density D-MNIST}
        \label{subfig:ddu_density_dmnist}
    \end{subfigure}
    \begin{subfigure}{0.16\linewidth}
        \centering
        \includegraphics[width=\linewidth]{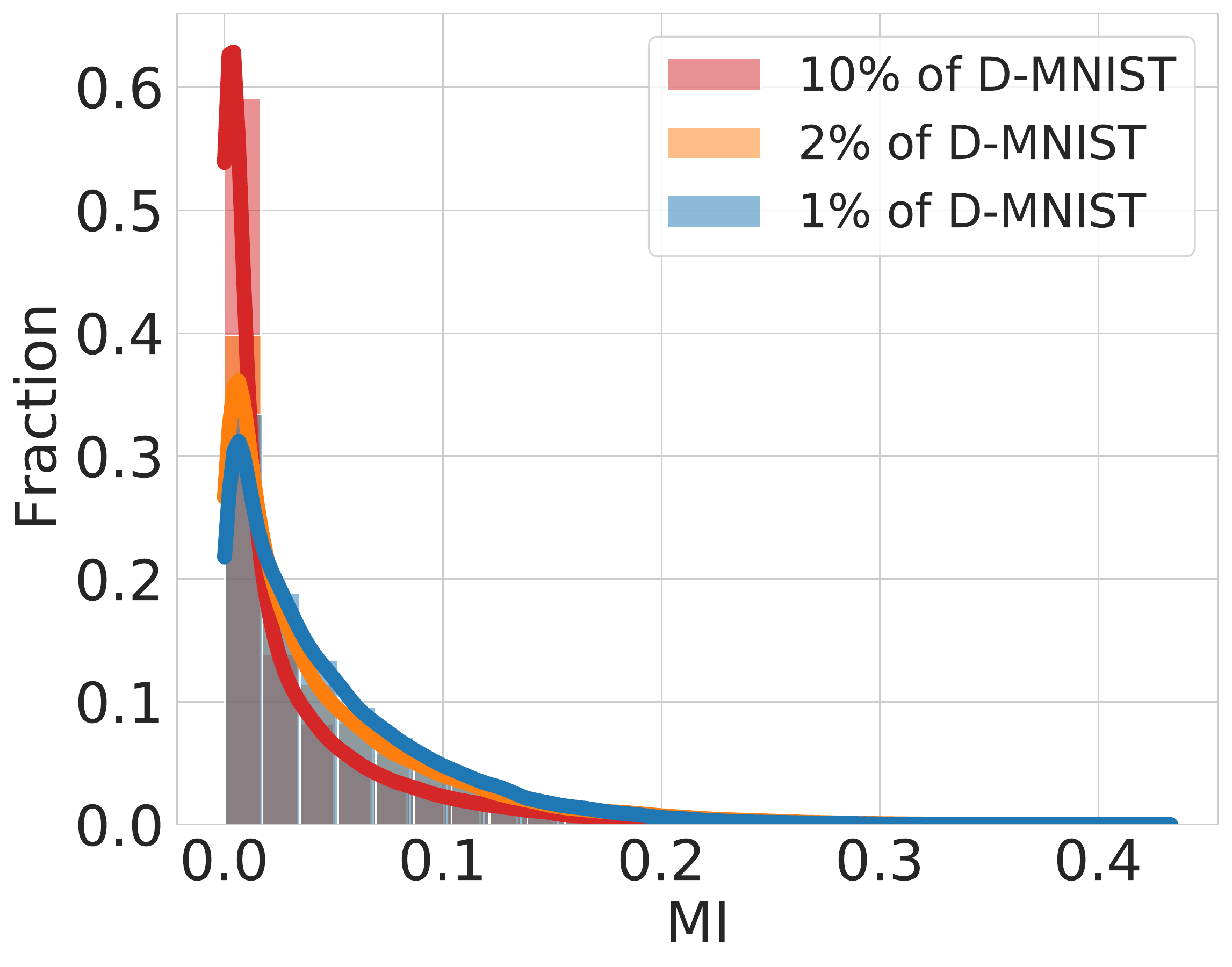}
        \caption{MI D-MNIST}
        \label{subfig:ensemble_mi_dmnist}
    \end{subfigure}
    \begin{subfigure}{0.16\linewidth}
        \centering
        \includegraphics[width=\linewidth]{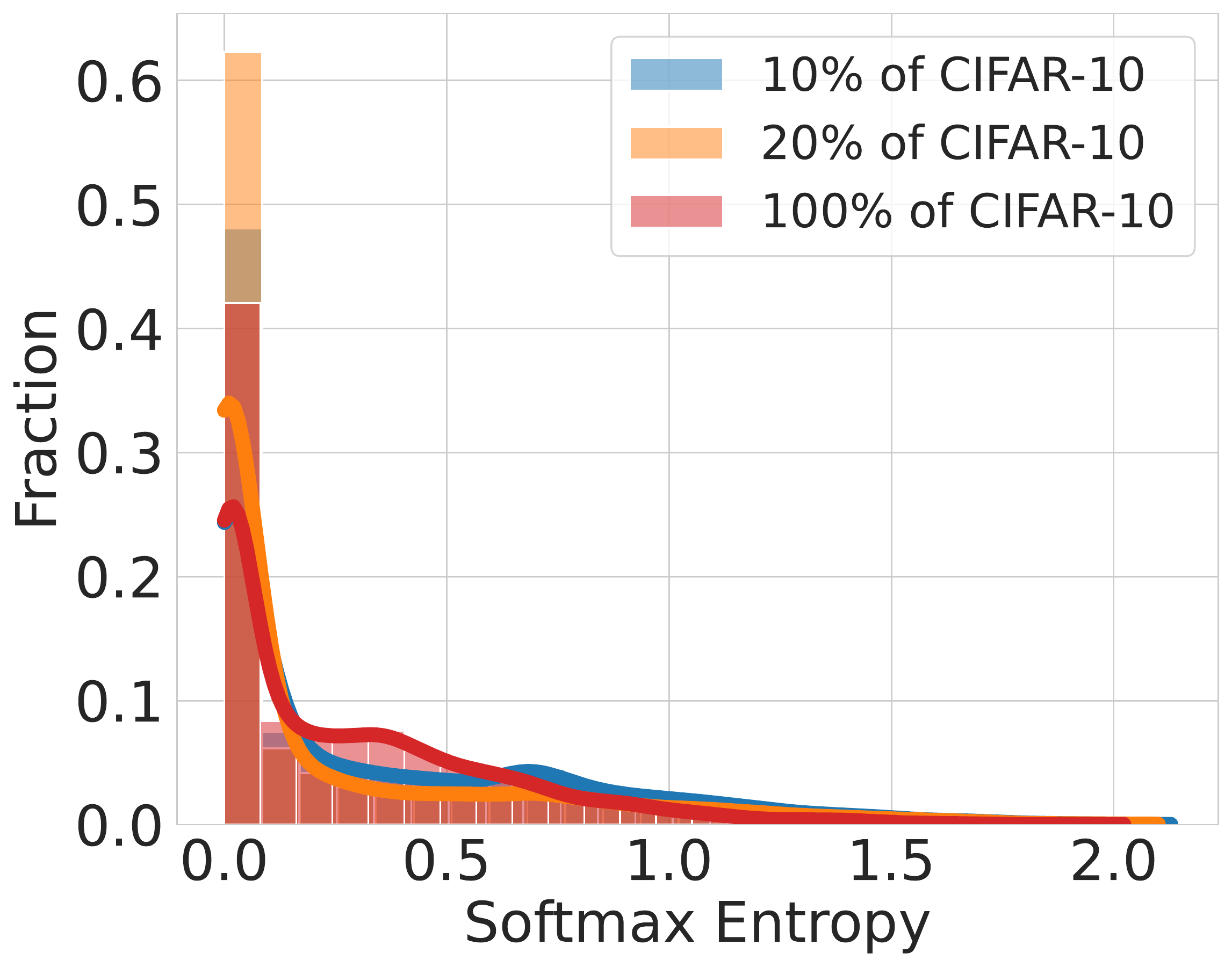}
        \caption{SE CIFAR-10}
        \label{subfig:softmax_entropy_cifar10}
    \end{subfigure}
    \begin{subfigure}{0.16\linewidth}
        \centering
        \includegraphics[width=\linewidth]{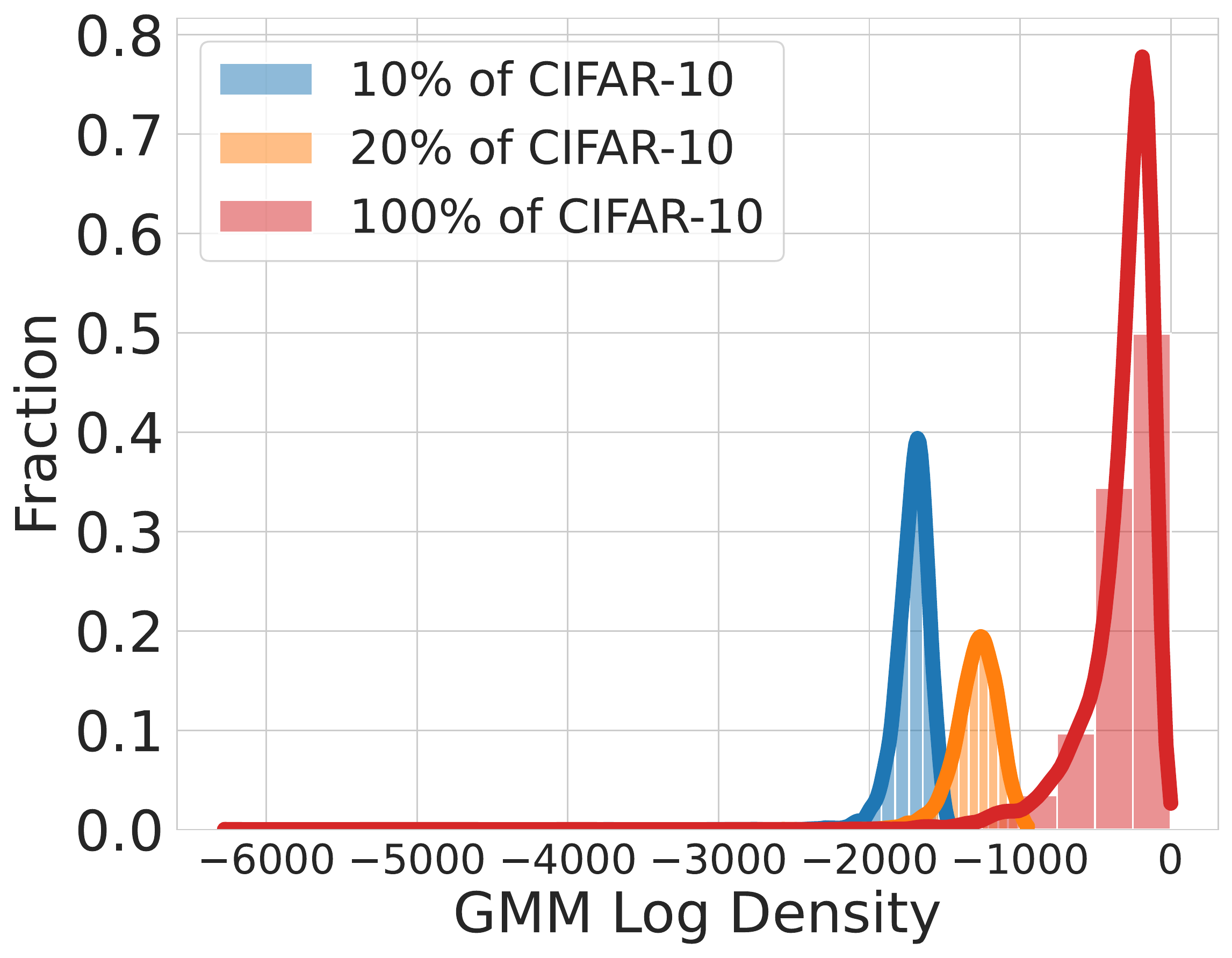}
        \caption{Density CIFAR-10}
        \label{subfig:ddu_density_cifar10}
    \end{subfigure}
    \begin{subfigure}{0.16\linewidth}
        \centering
        \includegraphics[width=\linewidth]{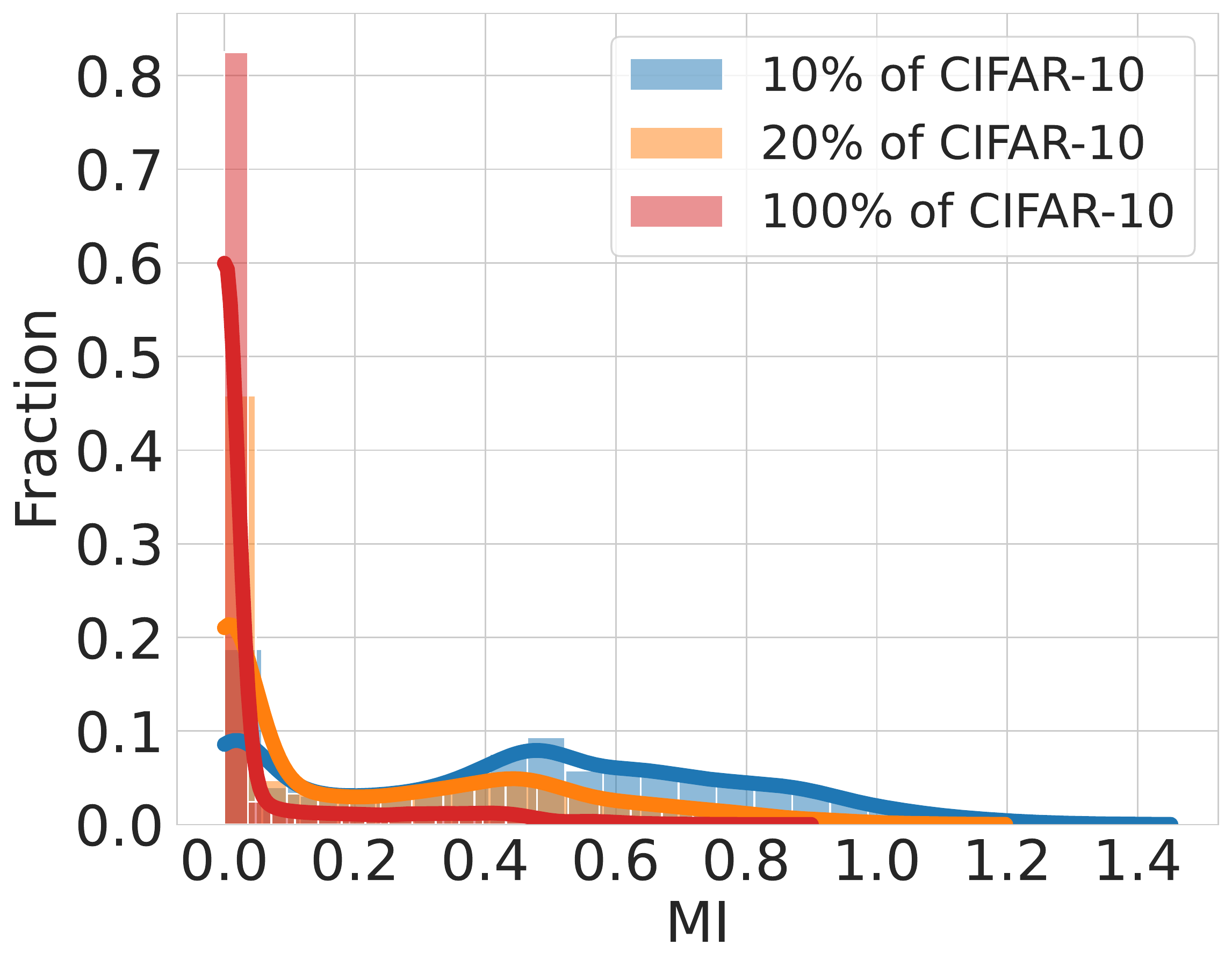}
        \caption{MI CIFAR-10}
        \label{subfig:ensemble_mi_cifar10}
    \end{subfigure}
    \caption{
    \emph{Comparison of epistemic and aleatoric uncertainty captured b(ResNet-18+SN) on increasingly large subsets of Dirty-MNIST and CIFAR-10. Clearly, feature density captures epistemic uncertainty which reduces when the model is trained on increasingly large subsets of training data, whereas softmax entropy (SE) does not. For comparison, we also plot a deep-ensemble's epistemic uncertainty, through mutual information (MI) for the same settings. For more details, see \Cref{table:kendall_viz}.}}
    \label{fig:kendall_viz_2}
\end{figure*}

\begin{table}[t!]
    \centering
    \caption{
    \emph{
        Average softmax entropy (SE) and feature-space density of the test set for models trained on different amounts of the training set (Dirty-MNIST and CIFAR-10) behave consistently with aleatoric and epistemic uncertainty.
    }
    Aleatoric uncertainty for individual samples does not change much as more data is added to the training set while epistemic uncertainty decreases as more data is added. This is also consistent with Table 3 in \citet{kendall2017uncertainties}. Finally, we observe a consistent strong positive correlation between the negative log feature space density and the mutual information (MI) of a deep ensemble trained on the same subsets of data for both Dirty-MNIST and CIFAR-10. However, the correlation between softmax entropy and MI is not consistent.}
    \label{table:kendall_viz}
    \resizebox{\linewidth}{!}{%
    \renewcommand{\arraystretch}{1.2} 
    \begin{tabular}{cccccc}
    \toprule
    \textbf{Training Set} & \textbf{Avg Test SE ($\boldmath\approx$)} & \textbf{Avg Test Log GMM Density (\textuparrow)} & \textbf{Avg Test MI} & \textbf{Correlation(SE || MI)} & \textbf{Correlation(-Log GMM Density || MI)} \\
    \midrule
    1\% of D-MNIST & $0.7407$ & $-2.7268e+14$ & $0.0476$ & \multirow{3}{*}{$-0.79897$} & \multirow{3}{*}{$0.8132$} \\ 
    2\% of D-MNIST & $0.6580$ & $-7.8633e+13$ & $0.0447$ && \\ 
    10\% of D-MNIST & $0.8295$ & $-1279.1753$ & $0.0286$ && \\
    \midrule
    10\% of CIFAR-10 & $0.3189$ & $-1715.3516$ & $0.4573$ & \multirow{3}{*}{$0.5663$} & \multirow{3}{*}{$0.9556$} \\
    20\% of CIFAR-10 & $0.2305$ & $-1290.1726$ & $0.2247$ && \\
    100\% of CIFAR-10 & $0.2747$ & $-324.8040$ & $0.0479$ && \\
    \bottomrule
    \end{tabular}%
  }
\end{table}

To empirically verify the connection between feature-space density and epistemic uncertainty on the one hand and the connection between softmax entropy and aleatoric uncertainty on the other hand, we train ResNet-18+SN models on increasingly large subsets of DirtyMNIST and CIFAR-10 and evaluate the epistemic and aleatoric uncertainty on the corresponding test sets using the feature-space density and softmax entropy, respectively. Moreover, we also train a 5-ensemble on the same subsets of data and use the ensemble's mutual information as a baseline measure of epistemic uncertainty.

In \Cref{fig:kendall_viz}, \Cref{fig:kendall_viz_2} and \Cref{table:kendall_viz}, we see that with larger training sets, the average feature-space density increases which is consistent with the epistemic uncertainty decreasing as more data is available as reducible uncertainty. This is also evident from the consistent strong positive correlation between the negative log density and mutual information of the ensemble.

On the other hand, the softmax entropy stays roughly the same which is consistent with aleatoric uncertainty as irreducible uncertainty, which is independent of the training data. Importantly, all of this is also consistent with the experiments comparing epistemic and aleatoric uncertainty on increasing training set sizes in Table 3 of \citet{kendall2017uncertainties}.

\subsection{Objective Mismatch Ablation with Wide-ResNet-28-10 on CIFAR-10}

\begin{table}[t!]
    \centering
    \caption{\emph{Objective Mismatch Ablation with WideResNet-28-10 models with and without spectral normalisation on CIFAR-10.} While GMMs perform much better than Softmax Entropy for feature-space density/epistemic uncertainty estimation, they underperform for aleatoric uncertainty estimation: both accuracy and in particular ECE are much worse than a regular softmax layer. Averaged over 25 runs.}
    \renewcommand{\arraystretch}{1.2} 
    \begin{tabular}{llrr}
    \toprule
        \textbf{Model} & \textbf{Prediction Source} & \textbf{Accuracy in \% (\textuparrow)} & \textbf{ECE (\textdownarrow)} \\
        \midrule
        WideResNet-28-10 & Softmax & $95.98\pm0.02$ & $2.29\pm0.02$ \\
         & GMM & $95.86\pm0.02$ & $4.13\pm0.02$ \\
        WideResNet-28-10+SN & Softmax & $95.97\pm0.03$ & $2.23\pm0.03$ \\
         & GMM & $95.88\pm0.02$ & $4.12\pm0.02$ \\
        \bottomrule
    \end{tabular}
    \label{table:objective_mismatch_real}
\end{table}

We further validate \Cref{pro:objectivemismatch} by running an ablation on Wide-ResNet-28-10 on CIFAR-10. \Cref{table:objective_mismatch_real} shows that the feature-space density estimator indeed performs worse than the softmax layer for aleatoric uncertainty (accuracy and ECE).

\section{Theoretical Results}
\label{app:theory}
\subsection{Softmax entropy ``cannot'' capture epistemic uncertainty because Deep Ensembles ``can''}

\subsubsection{Qualitative Statement}

We start with a proof of \Cref{pro:ensemble_softmax}, which quantitatively examines the qualitative statemets that given the same predictive entropy, higher epistemic uncertainty for one point than another will cause some ensemble members to have lower softmax entropy.

\ensemblesoftmax*

\begin{proof}
From \Cref{eq:BALD}, we obtain
\begin{align}
\begin{split}
    &|\MIc{Y_1}{\omega}{x_1, \data} + \E{\probc{\omega}{\data}}{\Hc{Y_1}{x_1, \omega}} \\
    &-\MIc{Y_2}{\omega}{x_2, \data} - \E{\probc{\omega}{\data}}{\Hc{Y_2}{x_2, \omega}}| \le \epsilon.
\end{split}
\end{align}
and hence we have
\begin{align}
& \begin{split}
    & \E{\probc{\omega}{\data}}{\Hc{Y_1}{x_1, \omega}} - \E{\probc{\omega}{\data}}{\Hc{Y_2}{x_2, \omega}} \\
    & + \underbrace{(\MIc{Y_1}{\omega}{x_1, \data} - \MIc{Y_2}{\omega}{x_2, \data})}_{> \delta} \le \epsilon.
\end{split}
\end{align}
We rearrange the terms:
\begin{equation}
    \E{\probc{\omega}{\data}}{\Hc{Y_1}{x_1, \omega}} <
    \E{\probc{\omega}{\data}}{\Hc{Y_2}{x_2, \omega}} - (\delta - \epsilon).
\end{equation}
Now, the statement follows by contraposition: if ${\Hc{Y_1}{x_1, \omega}} \ge \E{\probc{\omega}{\data}}{\Hc{Y_2}{x_2, \omega}} - (\delta - \epsilon)$ for all $\omega$, the monotonicity of the expectation would yield $\E{\probc{\omega}{\data}}{\Hc{Y_1}{x_1, \omega}} \ge \E{\probc{\omega}{\data}}{\Hc{Y_2}{x_2, \omega}} - (\delta - \epsilon).$
Thus, there is a non-null-set $\Omega'$ with $\prob{\Omega'} > 0$, such that
\begin{equation}
    {\Hc{Y_1}{x_1, \omega}} < {\Hc{Y_2}{x_2, \omega}} - (\delta - \epsilon),
\end{equation}
for all $\omega \in \Omega'$.
\end{proof}

While this statement provides us with an intuition for why ensemble members and thus deterministic models cannot provide epistemic uncertainty reliably through their softmax entropies, we can examine this further by establishing some upper bounds.

\subsubsection{Infinite Deep Ensemble}
There are two interpretations of the ensemble parameter distribution $\probc{\omega}{\data}$: we can view it as an empirical distribution given a specific ensemble with members $\omega_{i\in\{1,\ldots,K\}}$, or we can view it as a distribution over all possible trained models, given: random weight initializations, the dataset, stochasticity in the minibatches and the optimization process. In that case, any Deep Ensemble with $K$ members can be seen as finite Monte-Carlo sample of this posterior distribution. The predictions of an ensemble then are an unbiased estimate of the predictive distribution $\E{\probc{\omega}{\data}}{\probc{y|x,\omega}}$, and similarly the expected information gain computed using the members of the Deep Ensemble is just a (biased) estimator of $\MIc{Y}{\omega}{x, \data}$.

\newcommand{\infogain}{\MIc{Y}{\omega}{x}}
\newcommand{\expectedsment}{\chainedE{\omega}{\Hc{Y}{x,\omega}}}
\newcommand{\preddis}{\probc{y}{x}}
\newcommand{\predent}{\Hc{Y}{x}}
\newcommand{\smdis}{\probc{y}{x,\omega}}
\newcommand{\sment}{\Hc{Y}{x,\omega}}
\newcommand{\Dir}{\operatorname{Dir}}
\newcommand{\Cat}{\operatorname{Cat}}

\subsubsection{Analysis of Softmax Entropy of a Single Deterministic Model on OoD Data using Properties of Deep Ensembles}
\label{app:app_softmax_theory}
Based on the interpretation of Deep Ensembles as a distribution over model parameters, we can walk backwards and, given \emph{some value} for the predictive distribution and epistemic uncertainty of a Deep Ensemble, estimate what the softmax entropies from each ensemble component must have been. I.e.\ if we observe Deep Ensembles to have high epistemic uncertainty on OoD data, we can deduce from that what the softmax entropy of deterministic neural nets (the ensemble components) must look like.
More specifically, given a predictive distribution $\preddis$ and epistemic uncertainty, that is expected information gain $\infogain$, of the infinite Deep Ensemble, we estimate the expected softmax entropy from a single deterministic model, considered as a sample $\omega \sim \probc{\omega}{\data}$ and model the variance. Empirically, we find the real variance to be higher by a large amount for OoD samples, showing that softmax entropies do not capture epistemic uncertainty well for samples with high epistemic uncertainty. 

We will need to make several strong assumptions that limit the generality of our estimation, but we can show that our analysis models the resulting softmax entropy distributions appropriately. This will show that deterministic softmax models can have widely different entropies and confidence values.

Given the predictive distribution $\preddis$ and epistemic uncertainty $\infogain$, we can approximate the distribution over softmax probability vectors $p(y|x,\omega)$ for different $\omega$ using its  maximum-entropy estimate: a Dirichlet distribution $\left(Y_{1}, \ldots, Y_{K}\right) \sim \Dir(\alpha)$ with non-negative concentration parameters $\alpha = (\alpha_1, \ldots, \alpha_K)$ and $\alpha_0 := \sum \alpha_i$. Note that the Dirichlet distribution is used \emph{only as an analysis tool}, and at no point do we need to actually fit Dirichlet distributions to our data.

\paragraph{Preliminaries} ~\\

\newcommand{\dirp}{\mathbf{p}}

Before we can establish our main result, we need to look more closely at Dirichlet-Multinomial distributions. Given a Dirichlet distribution $\Dir(\alpha)$ and a random variable $\dirp \sim \Dir(\alpha)$, we want to quantify the expected entropy $\chainedE{\dirp \sim \Dir(\alpha)} {\opH_{Y\sim \Cat(\dirp)}[Y]}$ and its variance $\opVar_{\dirp \sim \Dir(\alpha)} {\opH_{Y\sim \Cat(\dirp)}[Y]}$. For this, we need to develop more theory. In the following, $\Gamma$ denotes the Gamma function, $\psi$ denotes the Digamma function, $\psi'$ denotes the Trigamma function.
\begin{lemma}
\label{lemma:dirichlet_basics}
\begin{enumerate}[wide, labelwidth=!, labelindent=0pt]
Given a Dirichlet distribution and random variable $\dirp\sim\Dir(\alpha)$, the following hold:
\item The expectation $\E{}{\log \dirp_i}$ is given by:
\begin{align}
    \label{eq:psi_log_expectation}
    \E{}{\log \dirp_i} = \psi(\alpha_i) - \psi(\alpha_0).
\end{align}
\item 
The covariance $\Cov{\log \dirp_i, \log \dirp_j}$ is given by
\begin{align}
    \label{eq:psip_log_expectation}
    \Cov{\log \dirp_i,\log \dirp_j} = \psi'(\alpha_i) \, \delta_{ij} - \psi'(\alpha_0).
\end{align}
\item The expectation $\E{}{\dirp_i^n \dirp_j^m \log \dirp_i}$ is given by:
\begin{align}
\begin{split}
\label{eq:nm_log_expectation}
\MoveEqLeft[3] \E{}{\dirp_i^n \dirp_j^m \log \dirp_i} \\
={}& \frac{\alpha_i^{\overline{n}} \, \alpha_j^{\overline{m}}}{\alpha_0^{\overline{n+m}}} \left ( \psi(\alpha_i + n) - \psi(\alpha_0 + n + m) \right ),
\end{split}
\end{align}
where $i \not =j$, and $n^{\overline{k}}=n \, (n+1) \, \ldots \, (n+k-1)$ denotes the rising factorial.
\end{enumerate}
\end{lemma}
\begin{proof}
\begin{enumerate}[wide, labelwidth=!, labelindent=0pt]
\item
The Dirichlet distribution is members of the exponential family. Therefore the moments of the sufficient statistics are given by the derivatives of the partition function with respect to the natural parameters.
The natural parameters of the Dirichlet distribution are just its concentration parameters $\alpha_i$. The partition function is
\begin{align}
    A(\alpha) = \sum_{i=1}^{k} \log \Gamma\left(\alpha_{i}\right)-\log \Gamma\left(\alpha_{0}\right),
\end{align}
the sufficient statistics is $T(x) = \log x$, and the expectation $\E{}{T}$ is given by
\begin{align}
    \E{}{T_i} = \frac{\partial A(\alpha)}{\partial \alpha_{i}}
\end{align}
as the Dirichlet distribution is a member of the exponential family.
Substituting the definitions and evaluating the partial derivative yields
\begin{align}
    \E{}{\log \dirp_i} &= \frac{\partial}{\partial \alpha_{i}} \left [ \sum_{i=1}^{k} \log \Gamma\left(\alpha_{i}\right)-\log \Gamma\left(\sum_{i=1}^{k} \alpha_i\right) \right ] \\
    & = \psi\left(\alpha_{i}\right) - \psi\left(\alpha_{0}\right) \frac{\partial}{\partial \alpha_{i}} \alpha_0,
\end{align}
where we have used that the Digamma function $\psi$ is the log derivative of the Gamma function $\psi(x) = \frac{d}{dx} \ln \Gamma(x)$.
This proves \eqref{eq:psi_log_expectation} as $\frac{\partial}{\partial \alpha_{i}} \alpha_0 = 1$.
\item Similarly, the covariance is obtained using a second-order partial derivative:
\begin{align}
    \Cov{T_i, T_j} = \frac{\partial^2 A(\alpha)}{\partial \alpha_{i} \, \partial \alpha_{i}}.
\end{align}
Again, substituting yields
\begin{align}
    \Cov{\log \dirp_i, \log \dirp_j} &= \frac{\partial}{\partial \alpha_{j}} \left [ \psi\left(\alpha_{i}\right) - \psi\left(\alpha_{0}\right) \right ] \\
    &= \psi'\left(\alpha_{i}\right) \delta_{ij} - \psi'\left(\alpha_{0}\right).
\end{align}
\item
We will make use of a simple reparameterization to prove the statement using \Cref{eq:psi_log_expectation}.
Expanding the expectation and substituting the density $\Dir(\dirp; \alpha)$, we obtain
\begin{align}
    & \E{}{\dirp_i^n \dirp_j^m \log \dirp_i} = \int \Dir(\dirp; \alpha) \, \dirp_i^n \, \dirp_j^m \, \log \dirp_i \, d\dirp \\
    & \quad = \int \frac{\Gamma\left(\alpha_0\right)}{\prod_{i=1}^{K} \Gamma\left(\alpha_{i}\right)} \prod_{k=1}^{K} \dirp_{k}^{\alpha_{k}-1}  \, \dirp_i^n \, \dirp_j^m \, \log \dirp_i \, d\dirp \\
    \begin{split}
    & \quad = \frac{\Gamma(\alpha_i + n) \Gamma(\alpha_j + m) \Gamma(\alpha_0 + n + m)}
    {\Gamma(\alpha_i) \Gamma(\alpha_j) \Gamma(\alpha_0)} \\
    &\quad \quad  \int \Dir(\hat\dirp; \hat\alpha) \, \hat\dirp_i^n \, \hat\dirp_j^m \, \log \hat\dirp_i \, d\hat\dirp
    \end{split} \\
    & \quad = \frac{\alpha_i^{\overline{n}} \, \alpha_j^{\overline{m}}}{\alpha_0^{\overline{n+m}}} \E{}{\log \hat\dirp_i},
\end{align}
where $\hat\dirp \sim \Dir(\hat\alpha)$ with $\hat\alpha = (\alpha_0, \dots, \alpha_i + n, \ldots, \alpha_j + m, \ldots, \alpha_K)$ and we made use of the fact that $\frac{\Gamma(z+n)}{\Gamma(z)}=z^{\overline{n}}$. Finally, we can apply \Cref{eq:psi_log_expectation} on $\hat\dirp \sim \Dir(\hat\alpha)$ to show
\begin{align}
    \quad = \frac{\alpha_i^{\overline{n}} \, \alpha_j^{\overline{m}}}{\alpha_0^{\overline{n+m}}} \left ( \psi(\alpha_i + n) - \psi(\alpha_0 + n + m) \right ).
\end{align}
\end{enumerate}
\end{proof}
With this, we can already quantify the expected entropy $\chainedE{\dirp \sim \Dir(\alpha)} {\opH_{Y\sim \Cat(\dirp)}[Y]}$:
\begin{lemma}
\label{lemma:dirichlet_expected_categorical}
Given a Dirichlet distribution and a random variable $\dirp \sim Dir(\alpha$), the expected entropy $\chainedE{\dirp \sim \Dir(\alpha)} {\opH_{Y\sim \Cat(\dirp)}[Y]}$ of the categorical distribution $Y \sim \Cat(\dirp)$ is given by
\begin{align}
     \chainedE{\probc{\dirp}{\alpha}}{\Hc{Y}{\dirp}} = \psi(\alpha_0+1) - \sum_{y=1}^K \frac{\alpha_i}{\alpha_0} \psi(\alpha_i +1).  
\end{align}
\end{lemma}
\begin{proof}
Applying the sum rule of expectations and \Cref{eq:nm_log_expectation} from \Cref{lemma:dirichlet_basics}, we can write
\begin{align}
    & \chainedE{}{\Hc{Y}{\dirp}} = \E{}{-\sum_{i=1}^K \dirp_i \log \dirp_i} = -\sum_i \E{}{\dirp_i \log \dirp_i} \\
    & \quad \quad = - \sum_i \frac{\alpha_i}{\alpha_0} \left ( \psi(\alpha_i+1) - \psi(\alpha_0+1)\right).
\end{align}
The result follows after rearranging and making use of $\sum_i \frac{\alpha_i}{\alpha_0} = 1$.
\end{proof}
With these statements, we can answer a slightly more complex problem:
\begin{lemma}
\label{lemma:log_covariance}
Given a Dirichlet distribution and a random variable $\dirp\sim\Dir(\alpha)$,
the covariance $\Cov{\dirp_i^n \log \dirp_i,  \dirp_j^m \log \dirp_j}$ is given by
\begin{align}
    \MoveEqLeft[3] \Cov{\dirp_i^n \log \dirp_i,  \dirp_j^m \log \dirp_j} \\
    \begin{split}
        ={}& \frac{\alpha_i^{\overline{n}}\,\alpha_j^{\overline{m}}}{\alpha_0^{\overline{n+m}}}
            \left ((\psi(\alpha_i+n)-\psi(\alpha_0+n+m)) \right. \\
         &  (\psi(\alpha_j+m)-\psi(\alpha_0+n+m)) \\
         & \left. - \psi'(\alpha_0+n+m) \right)\\
         & + \frac{\alpha_i^{\overline{n}}\,\alpha_j^{\overline{m}}}{\alpha_0^{\overline{n}}\,\alpha_0^{\overline{m}}}
         (\psi(\alpha_i+n)-\psi(\alpha_0+n)) \\
         &  (\psi(\alpha_j+m)-\psi(\alpha_0+n)),
    \end{split}
\end{align}
for $i\not=j$, where $\psi$ is the Digamma function and $\psi'$ is the Trigamma function.
Similarly, the covariance $\Cov{\dirp_i^n \log \dirp_i,  \dirp_i^m \log \dirp_i}$ is given by
\begin{align}
    \MoveEqLeft[3] \Cov{\dirp_i^n \log \dirp_i,  \dirp_i^m \log \dirp_i} \\
    \begin{split}
        ={}& \frac{\alpha_i^{\overline{n+m}}}{\alpha_0^{\overline{n+m}}}
            \left((\psi(\alpha_i+n+m)-\psi(\alpha_0+n+m))^2 \right.\\
         & + \left. \psi'(\alpha_i+n+m) - \psi'(\alpha_0+n+m)\right) \\
         & + \frac{\alpha_i^{\overline{n}}\,\alpha_i^{\overline{m}}}{\alpha_0^{\overline{n}}\,\alpha_0^{\overline{m}}}
         (\psi(\alpha_i+n)-\psi(\alpha_0+n)) \\
         & \quad \quad (\psi(\alpha_i+m)-\psi(\alpha_0+n)).
    \end{split}
\end{align}
\end{lemma}
Regrettably, the equations are getting large. By abuse of notation, we introduce a convenient shorthand before proving the lemma.
\newcommand{\psishort}[2]{\overline{\E{}{\log \hat{\dirp}_{#1}^{#2}}}}
\newcommand{\psipshort}[4]{\overline{\Cov{\log \hat{\dirp}_{#1}^{#2}, \log \hat{\dirp}_{#3}^{#4}}}}
\begin{definition}
We will denote by 
\begin{align}
    \psishort{i}{n,m} = \psi(\alpha_i+n)-\psi(\alpha_0+n+m),
\end{align} and use $\psishort{i}{n}$ for $\psishort{i}{n,0}$.
Likewise,
\begin{align}
    \psipshort{i}{n,m}{j}{n,m} = \psi'(\alpha_i + n) \delta_{ij} - \psi'(\alpha_0 + n + m).
\end{align}
\end{definition}
This notation agrees with the proof of \Cref{eq:psi_log_expectation} and \eqref{eq:psip_log_expectation} in \Cref{lemma:dirichlet_basics}. With this, we can significantly simplify the previous statements:
\begin{corollary}
Given a Dirichlet distribution and random variable $\dirp\sim\Dir(\alpha)$,
\begin{align}
    \E{}{\dirp_i^n \dirp_j^m \log \dirp_i} &= \frac{\alpha_i^{\overline{n}} \, \alpha_j^{\overline{m}}}{\alpha_0^{\overline{n+m}}} \psishort{i}{n,m},
\end{align}
\begin{align}
\MoveEqLeft[3] \Cov{\dirp_i^n \log \dirp_i,  \dirp_j^m \log \dirp_j} \\
\begin{split}
={}& \frac{\alpha_i^{\overline{n}}\,\alpha_j^{\overline{m}}}{\alpha_0^{\overline{n+m}}}
\left( \psishort{i}{n,m} \psishort{j}{m,n} \right. \\
& \quad \quad \left. \psipshort{i}{n,m}{j}{n,m} \right) \\
& + \frac{\alpha_i^{\overline{n}}\,\alpha_j^{\overline{m}}}{\alpha_0^{\overline{n}}\,\alpha_0^{\overline{m}}}
\psishort{i}{n} \psishort{j}{m} \quad \text{for $i \not= j$, and}
\end{split} \\
\MoveEqLeft[3] \Cov{\dirp_i^n \log \dirp_i,  \dirp_i^m \log \dirp_i} \\
    \begin{split}
        ={}& \frac{\alpha_i^{\overline{n+m}}}{\alpha_0^{\overline{n+m}}}
        \left(\psishort{i}{n+m}^2 \right.\\ 
        & \left. + \psipshort{i}{n+m}{i}{n+m} \right)\\
        & + \frac{\alpha_i^{\overline{n}}\,\alpha_i^{\overline{m}}}{\alpha_0^{\overline{n}}\,\alpha_0^{\overline{m}}} \psishort{i}{n}\psishort{j}{m}.
    \end{split}
\end{align}
\end{corollary}
\begin{proof}[Proof of \Cref{lemma:log_covariance}]
This proof applies the well-know formula \textbf{(cov)} $\Cov{X,Y} = \E{}{X \, Y} - \E{}{X} \E{}{Y}$ once forward and once backward \textbf{(rcov)} $\E{}{X \, Y} = \Cov{X,Y} + \E{}{X}\E{}{Y}$ while applying \Cref{eq:nm_log_expectation} several times:
\begin{align}
    & \Cov{\dirp_i^n \log \dirp_i,  \dirp_j^m \log \dirp_j} \\
    \begin{split}
    & \quad \overset{\textbf{cov}}{=} \E{}{\dirp_i^n \log(\dirp_i) \, \dirp_j^m \log(\dirp_j)} \\
    & \quad \quad - \E{}{\dirp_i^n \log \dirp_i}\E{}{\dirp_j^m \log \dirp_j}   
    \end{split} \\
    \begin{split}
    & \quad = \frac{\alpha_i^{\overline{n}}\,\alpha_j^{\overline{m}}}{\alpha_0^{\overline{n+m}}} \E{}{\log (\hat\dirp_i^{i,j}) \log(\hat\dirp_j^{i,j})} \\
    & \quad \quad - \E{}{\log \hat\dirp_i^{i}}\E{}{\log \dirp_j^{j}}
    \end{split} \\ 
    \begin{split}
    & \quad \overset{\textbf{(rcov)}}{=} \frac{\alpha_i^{\overline{n}}\,\alpha_j^{\overline{m}}}{\alpha_0^{\overline{n+m}}} 
    \left ( \Cov{\log \hat\dirp_i^{i,j}, \log \hat\dirp_j^{i,j} } \right. \\
    & \left. \quad \quad \quad \quad + \E{}{\log \hat\dirp_i^{i,j}}\E{}{\log \hat\dirp_j^{i,j}} \right ) \\
    & \quad \quad - \frac{\alpha_i^{\overline{n}}\,\alpha_j^{\overline{m}}}{\alpha_0^{\overline{n}}\,\alpha_0^{\overline{m}}} \E{}{\log \hat\dirp_i^{i}}\E{}{\log \dirp_j^{j}},
    \end{split}
\end{align}
where $\dirp^{i,j} \sim \Dir(\alpha^{i,j})$ with $\alpha^{i,j} = (\ldots, \alpha_i + n, \ldots, \alpha_j +m,\ldots)$. $\dirp^{i/j}$ and $\alpha^{i/j}$ are defined analogously. Applying \Cref{eq:psip_log_expectation} and \Cref{eq:psi_log_expectation} from \Cref{lemma:dirichlet_basics} yields the statement.
For $i=j$, the proof follows the same pattern.
\end{proof}
Now, we can prove the theorem that quantifies the variance of the entropy of $Y$:
\begin{theorem}
Given a Dirichlet distribution and a random variable $\dirp \sim Dir(\alpha)$, the variance of the entropy $\opVar_{\dirp \sim \Dir(\alpha)} {\opH_{Y\sim \Cat(\dirp)}[Y]}$ of the categorical distribution $Y \sim \Cat(\dirp)$ is given by
\begin{align}
    \MoveEqLeft[2] \Var{\Hc{Y}{\dirp}} & \\
    \begin{split}
    ={} & \sum_i \frac{\alpha_i^{\overline{2}}}{\alpha_0^{\overline{2}}} \left(\psipshort{i}{2}{i}{2} + \psishort{i}{2}^2 \right ) \\
    +{} & \sum_{i\not=j} \frac{\alpha_i \, \alpha_j}{\alpha_0^{\overline{2}}} \left( \psipshort{i}{1}{j}{1} + \psishort{i}{1,1} \, \psishort{j}{1,1} \right) \\
    -& \sum_{i,j} \frac{\alpha_i \, \alpha_j}{\alpha_0^{2}} \psishort{i}{1} \psishort{j}{1}.
    \end{split}\label{thm:E6}
\end{align}
\end{theorem}
\begin{proof}
We start by applying the well-known formula $\Var{\sum_i X_i}=\sum_{i,j} \Cov{X_i, X_j}$ and then apply  \Cref{lemma:log_covariance} repeatedly.
\end{proof}

\paragraph{Main Result} ~\\

Given that we can view an ensemble member as a single deterministic model and vice versa, this provides an intuitive explanation for why single deterministic models report inconsistent and widely varying predictive entropies and confidence scores for OoD samples for which a Deep Ensemble would report high epistemic uncertainty (expected information gain) and high predictive entropy.

Assuming that $p(y|x, \omega)$ only depends on $\preddis$ and $\infogain$, we model the distribution of $p(y|x,\omega)$ (as a function of $\omega$) using a Dirichlet distribution $\Dir(\alpha)$ which satisfies:
\begin{align}
    \preddis &= \frac{\alpha_i}{\alpha_0} \\
    \predent - \infogain &= \psi(\alpha_0+1)\\
    &\quad - \sum_{y=1}^K \probc{y}{x} \psi(\alpha_0 \,\probc{y}{x} +1)..
\end{align}
Then, we can model the softmax distribution using a random variable $\dirp \sim \Dir(\alpha)$ as:
\begin{align}
    \smdis \overset{\approx}{\sim} \Cat(\dirp).
\end{align}
The variance $\Var{\sment}$ of the softmax entropy for different samples $x$ given $\preddis$ and $\infogain$ is then approximated by $\Var{\Hc{Y}{\dirp}}$:
\begin{align}
    \opVar_\omega[\sment] \approx \opVar_\dirp[\Hc{Y}{\dirp}]
\end{align}
with the latter term given in eq.\ \eqref{thm:E6}. We empirically find this to provide a lower bound on the true variance $\opVar_\omega[\sment]$. 

\paragraph{Empirical Results} ~\\

\begin{figure*}[t]
    \centering
    \includegraphics[width=\linewidth]{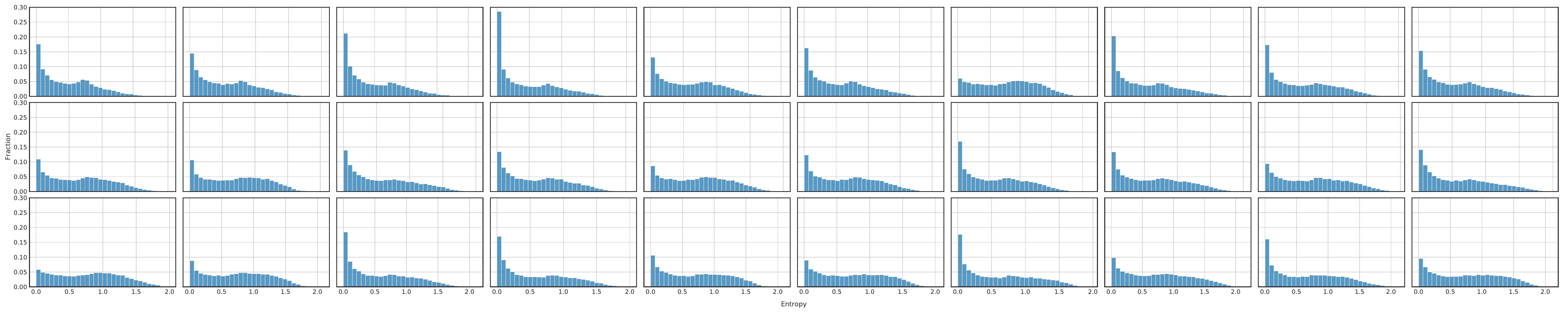}
    \caption{
    \emph{
    Softmax entropy histograms of 30 \emph{Wide-ResNet-28-10+SN} models trained on CIFAR-10, evaluated on SVHN (OoD).}   
    The softmax entropy distribution of the different models varies considerably.
    }
    \label{fig:wide_resnet_mod_sn_svhn_test_entropy_histograms}
\end{figure*}

\begin{figure*}[t]
    \centering
    \begin{subfigure}{0.32\linewidth}
        \centering
        \includegraphics[width=\linewidth]{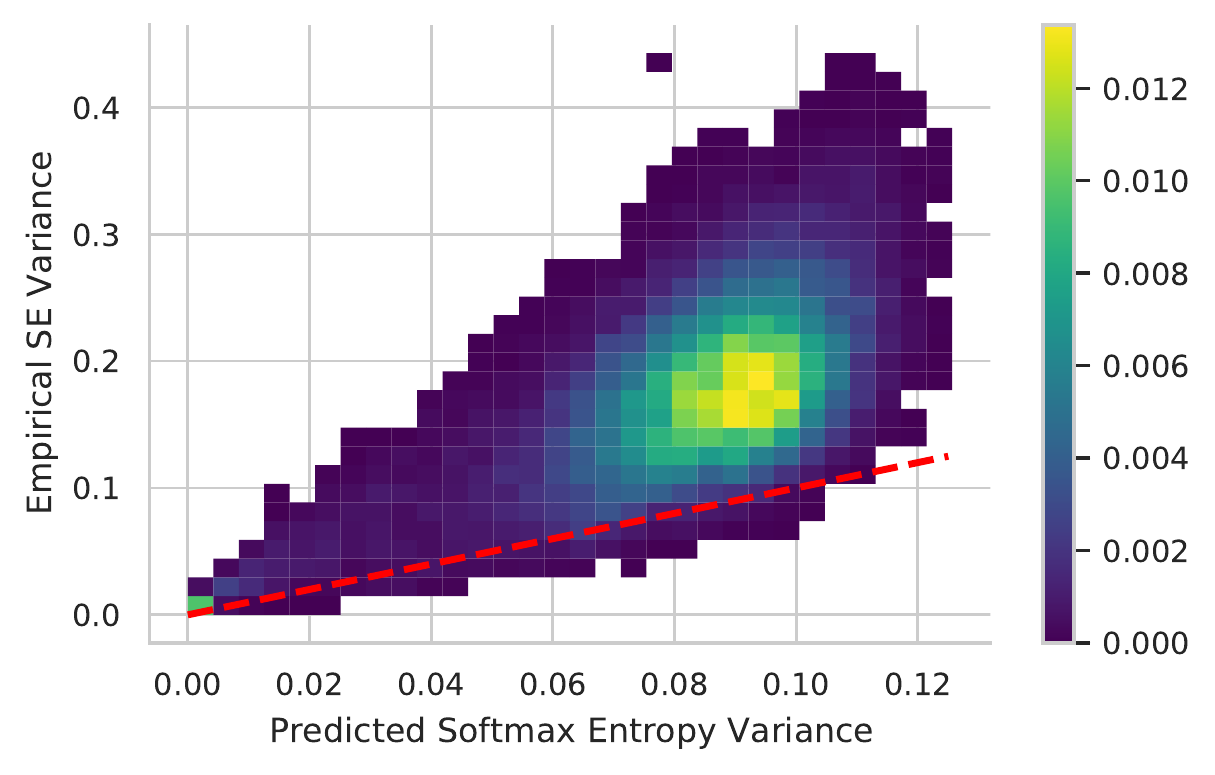}
        \caption{Simulated vs Ground-Truth}
        \label{fig:var_sm_lower_bound:verification}
    \end{subfigure}
    \begin{subfigure}{0.32\linewidth}
        \centering
        \includegraphics[width=\linewidth]{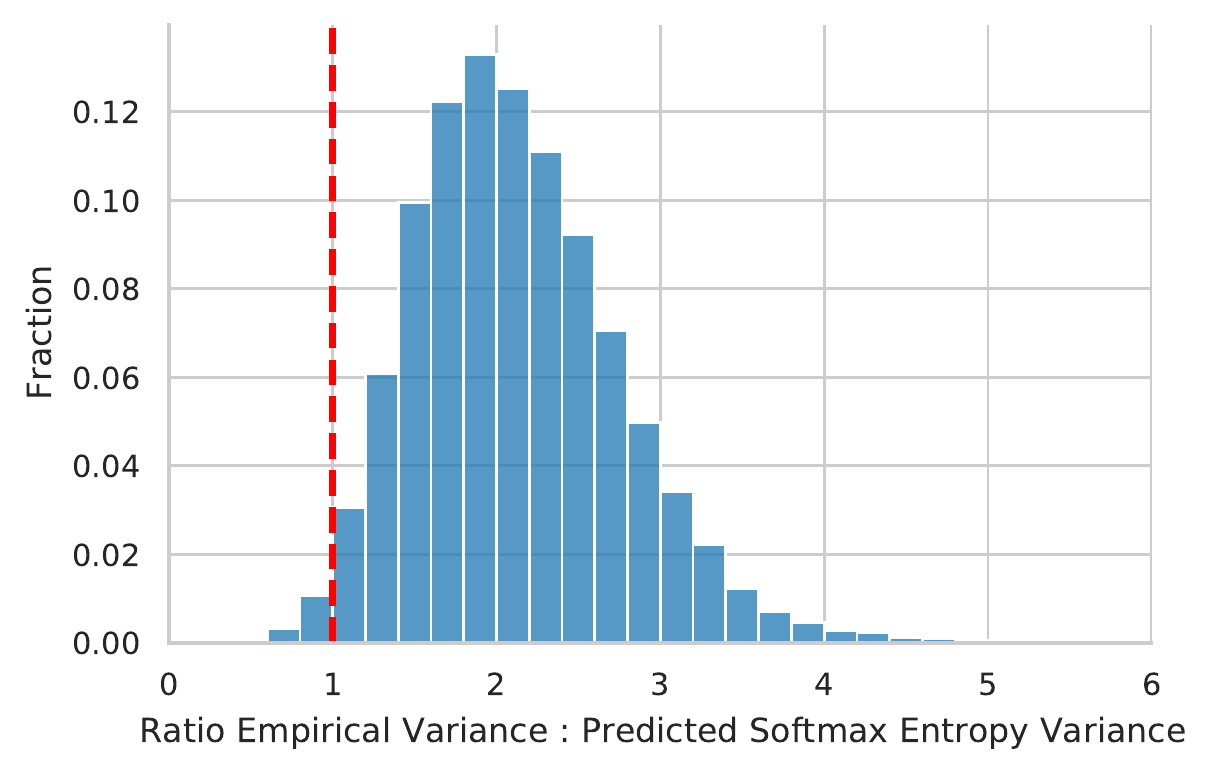}
        \caption{Ratio Histogram}
        \label{fig:var_sm_lower_bound:ratio}
    \end{subfigure}
    \begin{subfigure}{0.32\linewidth}
        \centering
        \includegraphics[width=\linewidth]{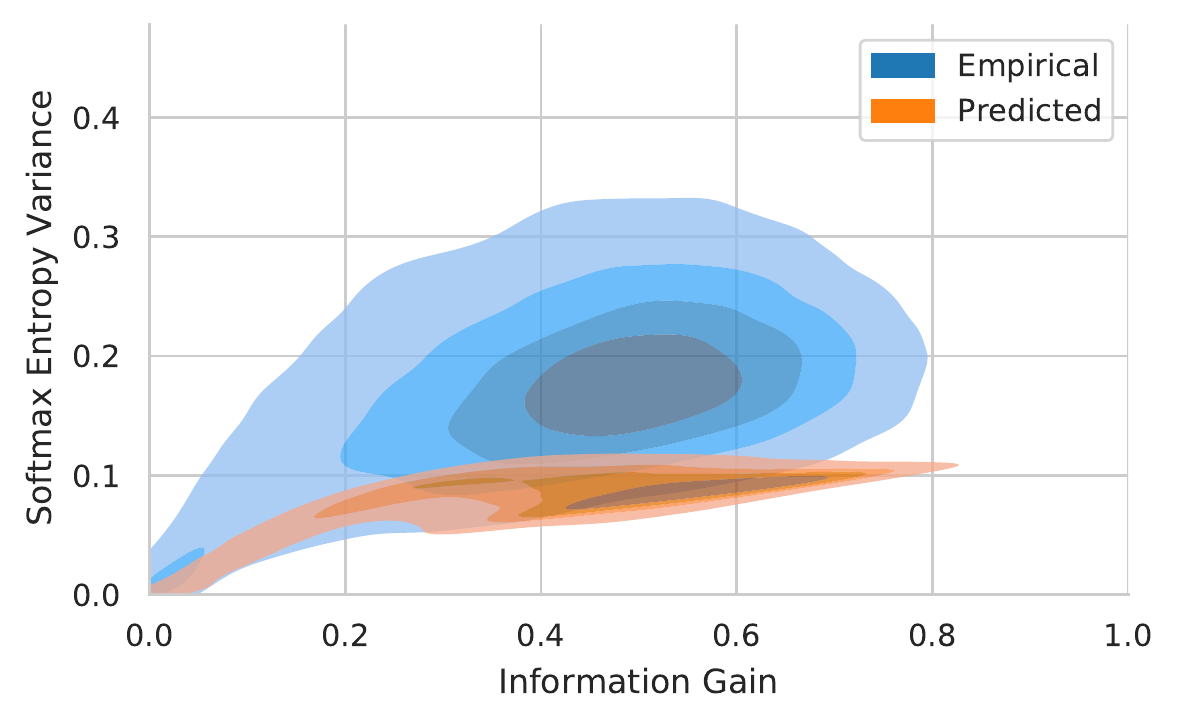}
        \caption{Density Plot}
        \label{fig:var_sm_lower_bound:density_infogain}
    \end{subfigure}
    \caption{
    \emph{The variance of softmax entropies can be lower-bounded by fitting Dirichlet distributions on the samples $\smdis$.}
    \textbf{\subref{fig:var_sm_lower_bound:verification}} The empirical variance of softmax entropies is lower-bounded by $\Var{\Hc{Y}{\dirp}}$. The red dashed line depicts equality.
    \textbf{\subref{fig:var_sm_lower_bound:ratio}} The ratio histogram shows that there are only few violations due to precision issues ($<2\%$).%
    \textbf{\subref{fig:var_sm_lower_bound:density_infogain}} The variance of the softmax entropy is not linearly correlated to the epistemic uncertainty. For both high and low epistemic uncertainty, the variance decreases.
    }
    \label{fig:var_sm_lower_bound}
\end{figure*}

\begin{figure}[t]
    \centering
    \begin{subfigure}{0.495\linewidth}
        \centering
        \includegraphics[width=\linewidth]{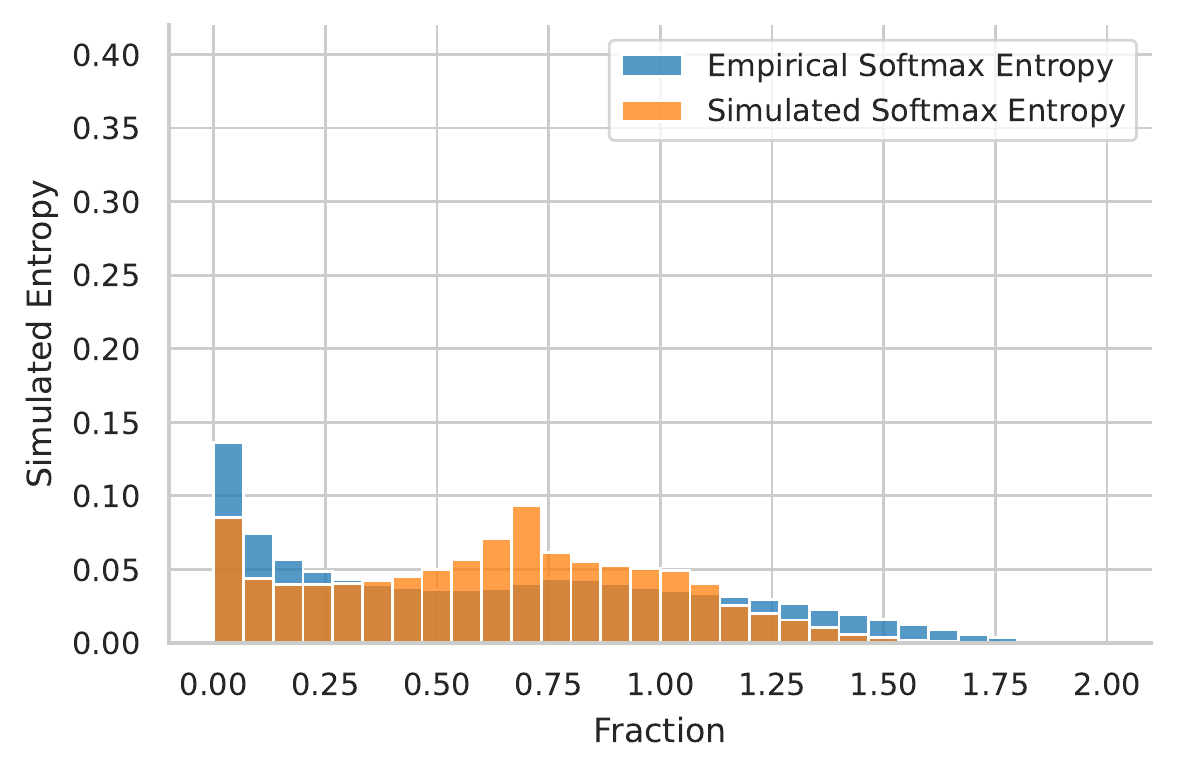}
        \caption{WideResNet-28-10+SN}
        \label{fig:vgg16_vs_wideresnet_sm:wrn_simulation}
    \end{subfigure} 
    \begin{subfigure}{0.495\linewidth}
        \centering
        \includegraphics[width=\linewidth]{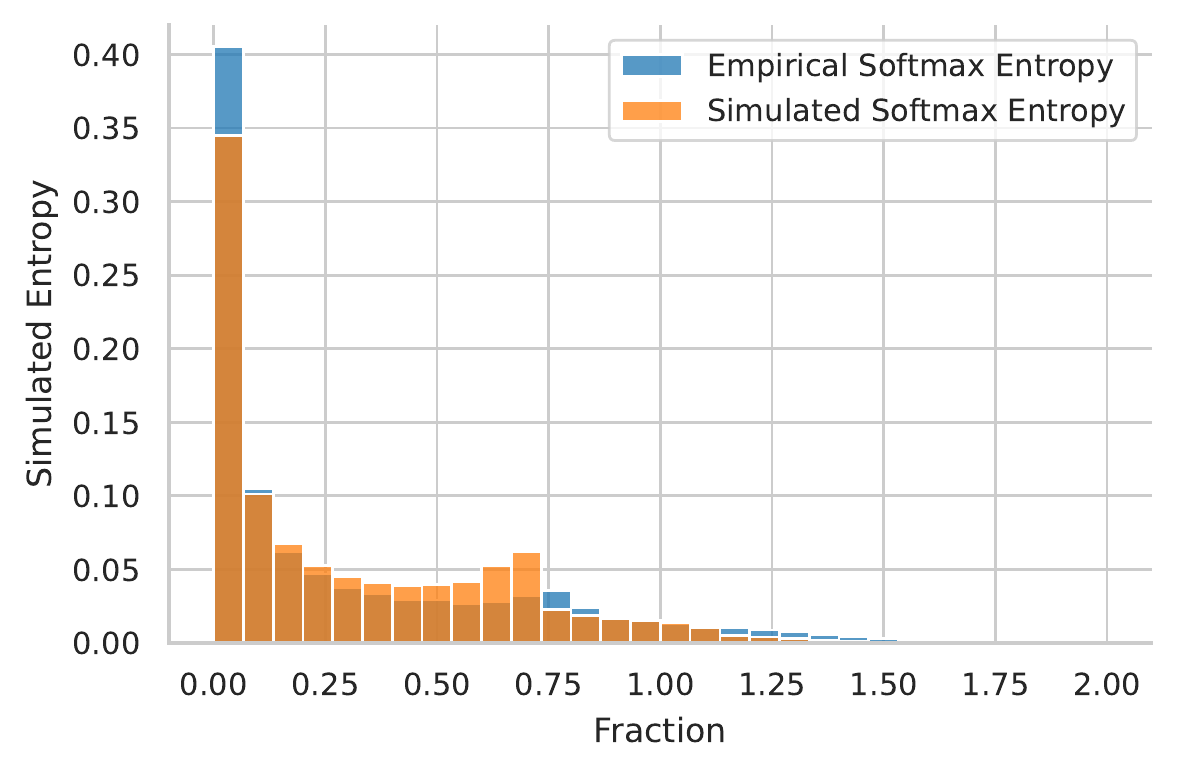}
        \caption{VGG16}
        \label{fig:vgg16_vs_wideresnet_sm:vgg16_simulation}
    \end{subfigure}
    \caption{
    \emph{Simulated vs empirical softmax entropy on WideResNet-28-10+SN and VGG16.} Even though the Dirichlet variance approximation lower-bounds the empirical softmax entropy variance, sampling from the fitted Dirichlet distributions does approximate the empirical entropy distribution quite well. 
    }
    \label{fig:vgg16_vs_wideresnet_sm}
\end{figure}

We empirically verify that softmax entropies vary considerably in \Cref{fig:wide_resnet_mod_sn_svhn_test_entropy_histograms}. In \Cref{fig:var_sm_lower_bound}, we verify that the predicted softmax entropy variance indeed lower-bounds the empirical softmax entropy variance. Morever, \Cref{fig:var_sm_lower_bound:density_infogain} shows both \textbf{i)} the non-linear relationship between epistemic uncertainty and variance in the softmax entropies and \textbf{ii)} that Dirichlet distributions cannot capture it and can only provide a lower bound. Nonetheless, this simple approximation seems to be able to capture the empirical entropy distribution quite well as shown in \Cref{fig:vgg16_vs_wideresnet_sm}.

\subsection{Objective Mismatch}
In \S\ref{sec:motivation}, we noted that the objectives that lead to optimal estimators for aleatoric and epistemic uncertainty via softmax entropy and feature-space density do not match, and DDU therefore uses the softmax layer as a discriminative classifier (implicit LDA) to estimate the predictive entropy, while it is using a GMM as generative classifier to estimate the feature-space density. Here we prove this.

\subsubsection{Preliminaries}
Before we prove \Cref{pro:objectivemismatch}, we will introduce some additional notation following \citet{kirsch2020unpacking}.
\begin{definition}
\label{def:it_notation}
\begin{enumerate}[wide, labelwidth=!, labelindent=0pt]
    \item $\hpprob{y, z}$ is the data distribution of the $\data$ in feature space with class labels $y$ and feature representation $z$.
    \item $\ptprob{\cdot}$ is a probability distribution parameterized by $\theta$.
    \item Entropies and conditional entropies are over the empirical data distribution $\hpprob{\cdot}$:
    \begin{align}
        \Entropy{\cdot} = \opH \left ( \hpprob{\cdot} \right )= \E{\hpprob{\cdot}}{-\log\hpprob{\cdot}}.
    \end{align}
    \item $\Hc{Y}{z}$ is the entropy of $\hpprobc{y}{z}$ for a given $z$, whereas $\Hc{Y}{Z}$ is the conditional entropy:
    \begin{align}
        \Hc{Y}{Z} = \chainedE{\hpprob{z}}{\Hc{Y}{z}}.
    \end{align}
    \item $\CrossEntropy{\prob{y, z}}{\qprobc{y|z}}$ is the cross-entropy of $\qprobc{y}{z}$ under $\probc{y}{z}$ in expectation over $\prob{z}$:
    \begin{align*}
        \CrossEntropy{\prob{y, z}}{\qprobc{y}{z}} &= \chainedE{\prob{z}}{\CrossEntropy{\probc{y}{z}}{\qprobc{y}{z}}} \\
        &= \E{\prob{y, z}}{-\log \qprobc{y}{z}}.
    \end{align*}
    \item Similarly, $\Kale{\prob{y, z}}{\qprobc{y|z}}$ is the Kullback-Leibler divergence of $\qprobc{y}{z}$ under $\probc{y}{z}$ in expectation over $\prob{z}$:
    \begin{align*}
        \Kale{\prob{y, z}}{\qprobc{y|z}} &=
        \chainedE{\prob{z}}{\Kale{\probc{y}{z}}{\qprobc{y|z}}} \\
        &=
        \CrossEntropy{\prob{y, z}}{\qprobc{y|z}} - \Hc{Y}{Z}
    \end{align*}
    \item For cross-entropies of $\ptprob{\cdot}$ under $\hpprob{z, y}$, we use the convenient short-hand $\Ht{\cdot} = \CrossEntropy{\hpprob{z, y}}{\ptprob{\cdot}}$. 
\end{enumerate}
\end{definition}
Then we can observe the following connection between $\Ht{\cdot}$ and $\Entropy{\cdot}$:
\begin{lemma}
Cross-entropies upper-bound the respective entropy with equality when $\ptprob{\cdot} = \hpprob{\cdot}$, which is important for variational arguments:
\begin{enumerate}
    \item $\Ht{Y, Z} \ge \Entropy{Y, Z}$,
    \item $\Ht{Z} \ge \Entropy{Z}$, and
    \item $\Htc{Y}{Z} \ge \Hc{Y}{Z}$.
\end{enumerate}
\end{lemma}
\begin{proof}
\begin{enumerate}[wide, labelwidth=!, labelindent=0pt]
    \item \label{enum:joint_variational}
\begin{math}
\Ht{Y, Z} - \Entropy{Y, Z} = \Kale{\hpprob{y,z}}{\ptprob{y,z}} \ge 0.
\end{math}
\item follows from \Cref{enum:joint_variational}.
\item We expand the expectations and note that inequality commutes with expectations:
\begin{align*}
    \Htc{Y}{Z} - \Hc{Y}{Z} = \E{\hpprob{z}}{\Htc{Y}{z} - \Hc{Y}{z}} \ge 0,    
\end{align*}
because $\Htc{Y}{z} - \Hc{Y}{z} \ge 0$ for all $z$.
The equality conditions follows from the properties of the Kullback-Leibler divergence as well.
\end{enumerate}

\end{proof}
We also have:
\begin{lemma}
\begin{align}
\Ht{Y, Z} &= \Htc{Y}{Z} + \Ht{Z} \\
&= \Htc{Z}{Y} + \Ht{Y}.
\end{align}
\end{lemma}
\begin{proof}
We substitute the definitions and obtain:
\begin{align}
\Ht{Y, Z} &= \E{\prob{y, z}}{-\log \qprob{y ,z}} \\
&= \E{\prob{y, z}}{-\log \qprobc{y}{z}} + \E{\prob{y, z}}{-\log \qprob{z}} \\
&= \Htc{Y}{Z} + \Ht{Z}.
\end{align}
\end{proof}
The same holds for entropies: 
\begin{math}
    \Entropy{Y, Z} = \Hc{Y}{Z} + \Entropy{Z} 
    = \Hc{Y}{Z} + \Entropy{Y}
\end{math}
\citep{cover1999elements}.

\subsubsection{Proof}
We can now prove the observation.
\newcommand{\trivialminimizer}{\sqprob{y, z}{^*}}
\objectivemismatch*
\begin{proof}
\begin{enumerate}[wide, labelwidth=!, labelindent=0pt]
    \item 
    The conditional log-likelihood is a strictly proper scoring rule \citep{gneiting2007strictly}. The optimization objective can be rewritten as
    \begin{align}
        \label{eq:discriminative_classifier}
        \max_\theta \chainedE{\log \ptprobc{y}{z}} = \min_\theta \Htc{Y}{Z} \ge \Hc{Y}{Z}.
    \end{align}
    An optimal discriminative classifier $\ptprobc{y}{z}$ would thus capture the true (empirical) distribution everywhere: $\ptprobc{y}{z} = \hpprobc{y}{z}$. This means the negative conditional log-likelihood will be equal $\Hc{Y}{Z}$ and $\Htc{Y}{z} = \Hc{Y}{z}$ for all $z$. 
    \item
    For density estimation $\qprob{z}$, the maximum likelihood $\E{}{\log \qprob{z}}$ using the empirical data distribution is maximized. We can rewrite this as
    \begin{align}
        \label{eq:density_estimator}
        \max_\theta \chainedE{\hpprob{y,z}}{\log \ptprob{z}}= \min_\theta \Ht{Z} \ge \Entropy{Z}.
    \end{align}
    We see that the negative marginalized likelihood of the density estimator upper-bounds the entropy of the feature representations $\Entropy{Z}$. We have equality and $\ptprob{z} = \hpprob{z}$ in the optimum case.
    \item
    Using $\Ht{Y, Z} = \Htc{Y}{Z} + \Ht{Z}$, we can relate the objectives from \Cref{eq:discriminative_classifier} and \eqref{eq:density_estimator} to each other. 
    First, we characterize a shared optimum, and then we show that both objectives are generally not minimized at the same time.
    For both objectives to be minimized, we have $\nabla \Htc{Y}{Z} = 0$ and $\nabla \Ht{Z} = 0$, and we obtain
    \begin{align}
        \nabla \Ht{Y, Z} = \nabla \Htc{Y}{Z} + \nabla \Ht{Z} = 0.
    \end{align}
    From this we conclude that minimizing both objectives also minimizes $\Ht{Y, Z}$, and that generally the objectives trade-off with each other at stationary points $\theta$ of $\Ht{Y, Z}$:
    \begin{align}
        \label{eq:classifiers_trade_off}
        \nabla \Htc{Y}{Z} = -\nabla \Ht{Z} \quad \text{when $\nabla \Ht{Y, Z} = 0$}.
    \end{align}
    This tells us that to construct a case where the optima do not coincide, discriminative classification needs to be opposed better density estimation. 

    Specifially, when we have a GMM with one component per class, minimizing $\Ht{Y, Z}$ on an empirical data distribution is equivalent to Gaussian Discriminant Analysis, as is easy to check, and minimizing $\Ht{Z}$ is equivalent to fitting a density estimator, following \Cref{eq:density_estimator}.
    The difference is that using a GMM as a density estimator does not constrain the component assignment, unlike in GDA.
    
    Consequently, we see that \emph{both objectives can be minimized at the same time exactly when the feature representations of different classes are perfectly separated}, such that a GMM fit as density estimator would assign each class's feature representations to a single component.

    By the above, we can construct a simple case: if the samples of different classes are not separated in feature-space, optimas for the objectives will not coincide, so for example if samples were drawn from the same Gaussian and labeled randomly. On the other hand, if we have classes whose features lie in well-separated clusters, GDA will minimize all objectives.
\end{enumerate}
\end{proof}

Given that perfect separation is impossible with ambiguous data for a GMM, a shared optimum will be rare with noisy real-world data, but only then would GDA be optimal.
In all other cases, GDA does not optimize both objectives, and neither can any other GMM with one component per class.
Moreover, \Cref{eq:classifiers_trade_off} shows that a GMM fit using EM is a better density estimator than GDA, and a softmax layer is a better classifier, as optimizing the softmax objective $\Htc{Y}{Z}$ or density objective $\Entropy{Z}$ using gradient descent will move away from the GDA optimum.

As can easily be verified, a trivial optimal minimizer $\trivialminimizer$ for $\Ht{Y, Z}$ given an empirical data distribution $\hpprob{y,z}$ is an adapted Parzen estimator:
\begin{align}
    \trivialminimizer = \sum_y \hpprob{y} \chainedE{\hat{z} \sim \hpprobc{z}{y}}{\normaldistpdf{z}{\hat{z}}{\sigma^2\mathbf{I}}},
\end{align}
for small enough $\sigma$.
This shows that above proposition is not general.

\subsubsection{Intuitions \& Validation with a Toy Example} %
\label{app:objective_mismatch_toy_example}

\begin{table}[b!]
    \centering
    \caption{\emph{Realized objective scores (columns) for different optimization objectives (rows) for the synthetic 2D toy example depicted in \Cref{fig:objective_mismatch}.} Smaller is better. We see that each objectives minimizes its own score while being suboptimal in regards to the other two objectives (when it is possible to compute the scores). This empirically further validates \Cref{pro:objectivemismatch}.
    }
    \label{table:objective_mismatch}
    \begin{tabular}{lrrr}
    \toprule
        \textbf{Objective} & {\boldmath$\Htc{Y}{Z}$} (\textdownarrow) & {\boldmath$\Ht{Y, Z}$} (\textdownarrow) & {\boldmath$\Ht{Z}$} (\textdownarrow) \\ 
        \midrule
        $\min \Htc{Y}{Z}$ & 0.1794 & 5.4924 & 5.2995 \\ 
        $\min \Ht{Y, Z}$ & 0.2165 & 4.9744 & 4.7580 \\ 
        $\min \Ht{Z}$ & n/a & n/a & 4.7073 \\
        \bottomrule
    \end{tabular}%
\end{table}

\Cref{fig:objective_mismatch} visualises this on a synthetic 2D dataset with three classes and 4\% label noise, which causes the optima to diverge as described in the proof. Label noise is a common issue in real-world datasets. Non-separability even more so. To explain \Cref{pro:objectivemismatch} in an intuitive way, we focus on on a simple 2D toy case and fit a GMM using the different objectives. We sample "latents" z from 3 Gaussians (each representing a different class y) with 4\% label noise. Following the construction in the proof, this will lead the objectives to have different optima.

We know discuss the different objectives in \Cref{fig:objective_mismatch} and the resulting scores in more detail:

{\boldmath${\min \Htc{Y}{Z}}$}. A softmax linear layer is equivalent to an LDA (Linear Discriminant Analysis) with conditional likelihood as detailed in \citet{murphy2012machine}, for example. We optimize an LDA with the usual objective "$\min -1/N \sum \log \probc{y}{z}$", i.e. the cross-entropy of $\probc{y}{z}$ or (average) negative log-likelihood (NLL). Following \Cref{def:it_notation}, we use the short-hand "$\min \Htc{Y}{Z}$" for this cross-entropy.

Because we optimize only $\probc{y}{z}$, $\prob{z}$ does not affect the objective and is thus not optimized. Indeed, the components do not actually cover the latents well, as can be seen in the first density plot of \Cref{fig:objective_mismatch_density}. However, it does provide the lowest NLL.

{\boldmath$\min \Ht{Y, Z}$}. We optimize a GDA for the combined objective "$\min -1/N \sum \log \qprob{y, z}$", i.e. the cross-entropy of $\qprob{y, z}$. We use the short-hand "$\min \Htc{Y}{Z}$" for this.

{\boldmath$\min \Ht{Z}$}. We optimize a GMM for the objective "$\min -1/N \sum \log \qprob{z}$", i.e. the cross-entropy of $\qprob{z}$. We use the short-hand "$\min \Ht{Z}$" for this.

We do not provide scores for $\Htc{Y}{Z}$ and $\Ht{Y,Z}$ for the third objective $\min \Ht{Z}$ in \Cref{table:objective_mismatch} as it does not depend on $Y$, and hence the different components do not actually model the different classes necessarily. Hence, we also use a single color to visualize the components for this objective in \Cref{fig:objective_mismatch_density}.

In \Cref{table:objective_mismatch} and \Cref{fig:objective_mismatch_density}, we see that each solution minimizes its own objective best. The GMM provides the best density model (best fit according to the entropy), while the LDA (like a softmax linear layer) provides the best NLL for the labels. The GDA provides a density model that is almost as good.

\textbf{Entropy.} Looking at the entropy plots in \Cref{fig:objective_mismatch_entropy}, we first notice that the LDA solution optimized for $\min \Htc{Y}{Z}$ has a wide decision boundary. This is due to the overlap of the Gaussian components, which is necessary to provide the right aleatoric uncertainty.

Optimizing the negative log-likehood $-\log \probc{y}{z}$ is a proper scoring rule, and hence is optimized for calibrated predictions.

Compared to this, the GDA solution (optimized for $\min \Ht{Y, Z}$ has a much narrower decision boundary and cannot capture aleatoric uncertainty as well. This is reflected in the higher NLL. Moreover, unlike for LDA, GDA decision boundaries behave differently than one would naively expect due to the untied covariance matrices. They can be curved and the decisions change far away from the data \citep{murphy2012machine}.

To show the difference between the two objectives we have marked an ambiguous point near $(0, -5)$ with  a yellow star . Under the first objective $\min \Ht{Y, Z}$, the point has high aleatoric uncertainty (high entropy), as seen in the left entropy plot while under the second objective ($\min \Ht{Y, Z}$) the point is only assigned very low entropy. The GDA optimized for the second objective thus is overconfident.

As above explained above, we do not show an entropy plot of $Y \mid Z$ for the third objective $\min \Ht{Z}$ in \Cref{fig:objective_mismatch_entropy} because the objective does not depend on $Y$, and there are thus no class predictions.

Intuitively, for aleatoric uncertainty, the Gaussian components need to overlap to express high aleatoric uncertainty (uncertain labelling). At the same time, this necessarily provides looser density estimates. On the other hand, the GDA density is much tighter, but this comes at the cost of NLL for classification because it cannot express aleatoric uncertainty that well. \Cref{fig:objective_mismatch} visualizes how the objectives trade-off between each other, and why we use the softmax layer trained for $\probc{y}{z}$ for classification and aleatoric uncertainty, and GDA as density model for $\qprob{z}$.

\begin{sidewaysfigure}[t!]
    \centering
    \begin{subfigure}{\linewidth}
        \includegraphics[width=\linewidth]{figs/dirty_ambiguous_fashion_mnist.pdf}%
        \caption{{\color{sns-blue} Dirty-MNIST (iD)} and {\color{sns-orange} Fashion-MNIST (OoD)}}
        \label{app_fig:intro_sample_viz}
    \end{subfigure}
    \begin{subfigure}{\linewidth}
        \begin{subfigure}{0.33\linewidth}
            \centering
            \includegraphics[width=\linewidth]{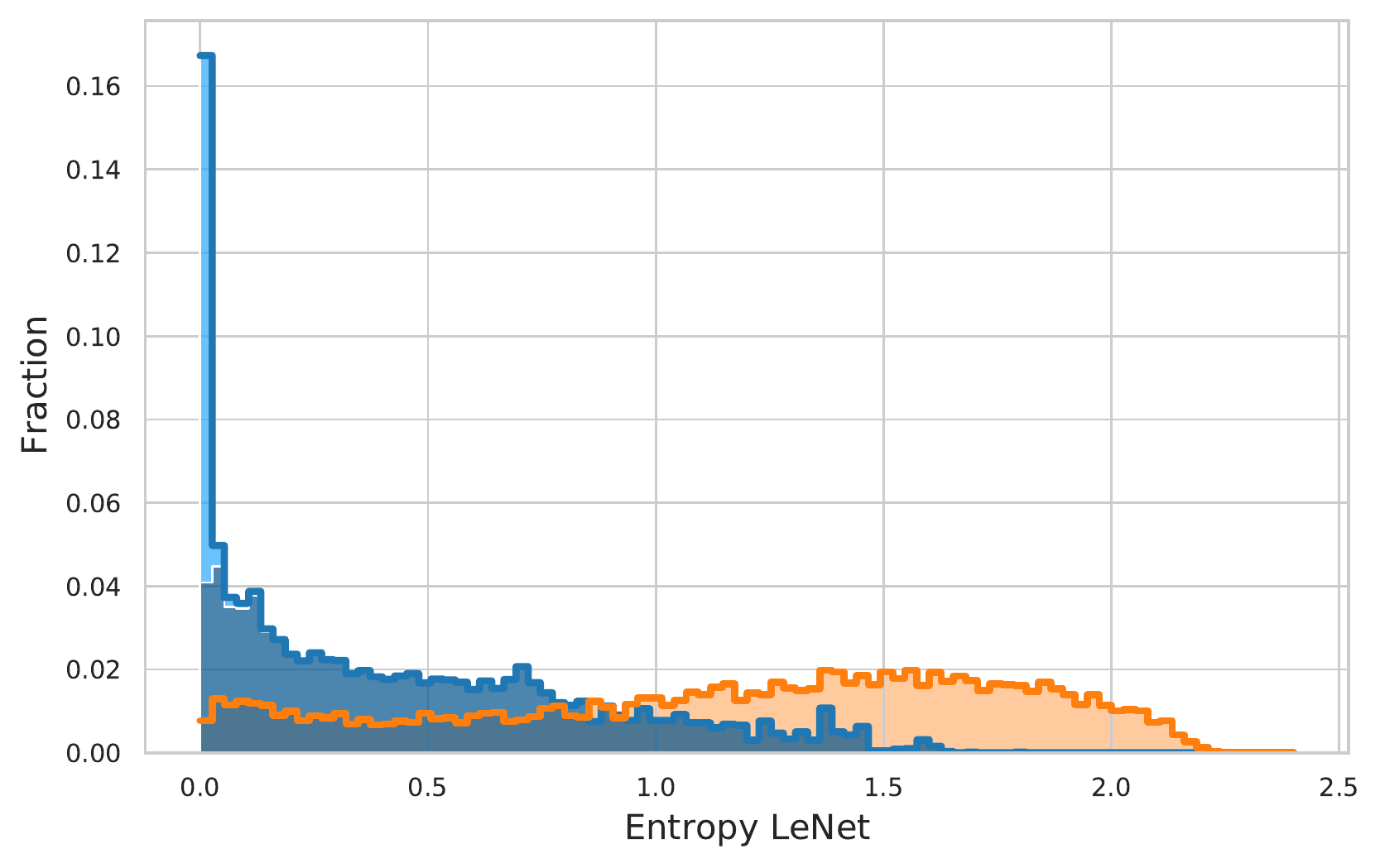}%
        \end{subfigure}%
        \begin{subfigure}{0.33\linewidth}
            \centering
            \includegraphics[width=\linewidth]{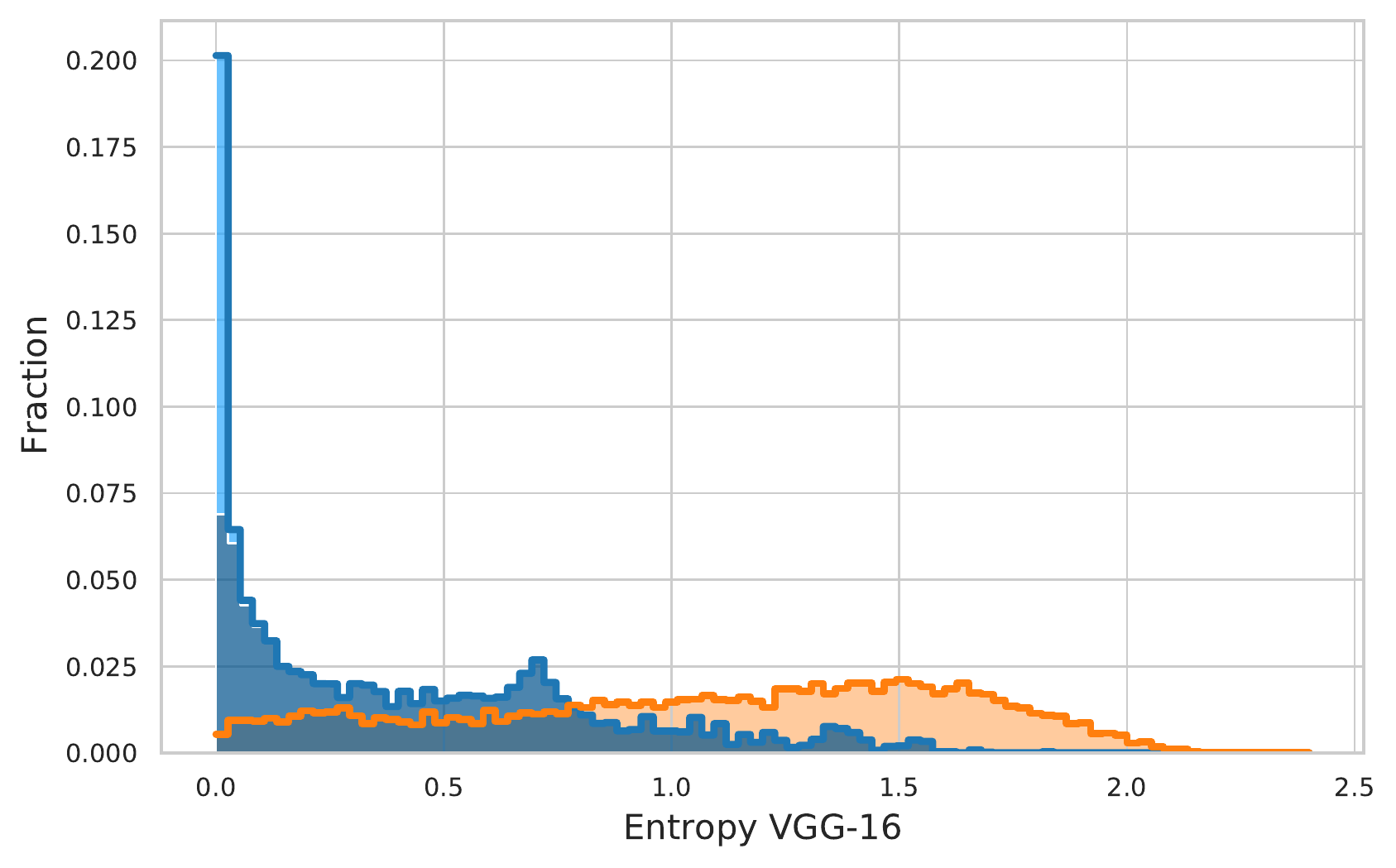}%
        \end{subfigure}%
        \begin{subfigure}{0.33\linewidth}
            \centering
            \includegraphics[width=\linewidth]{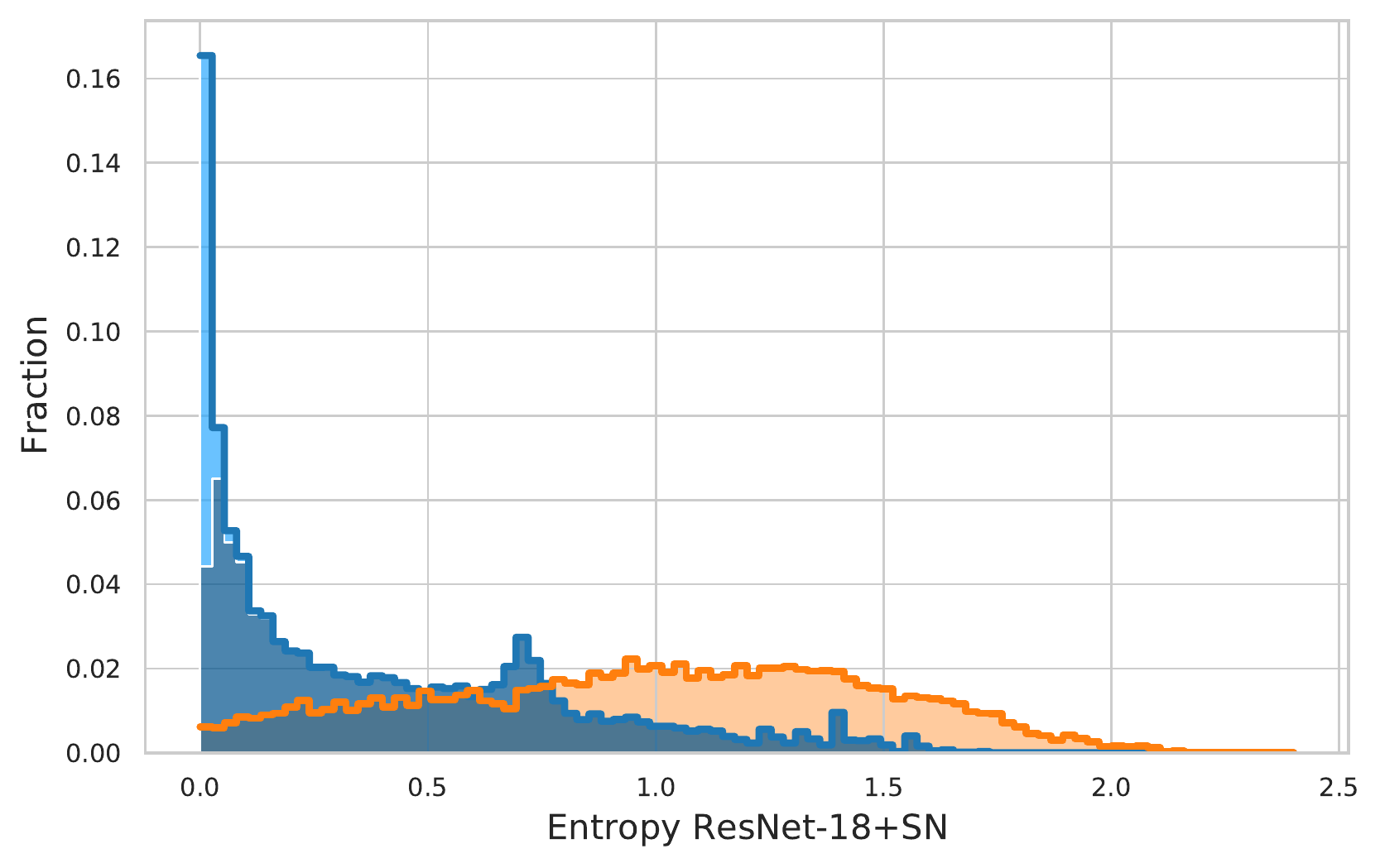}%
        \end{subfigure}%
        \subcaption{{Softmax entropy}}
        \label{app_fig:intro_softmax_ent}
    \end{subfigure}
    \begin{subfigure}{\linewidth}
        \begin{subfigure}{0.33\linewidth}
            \centering
            \includegraphics[width=\linewidth]{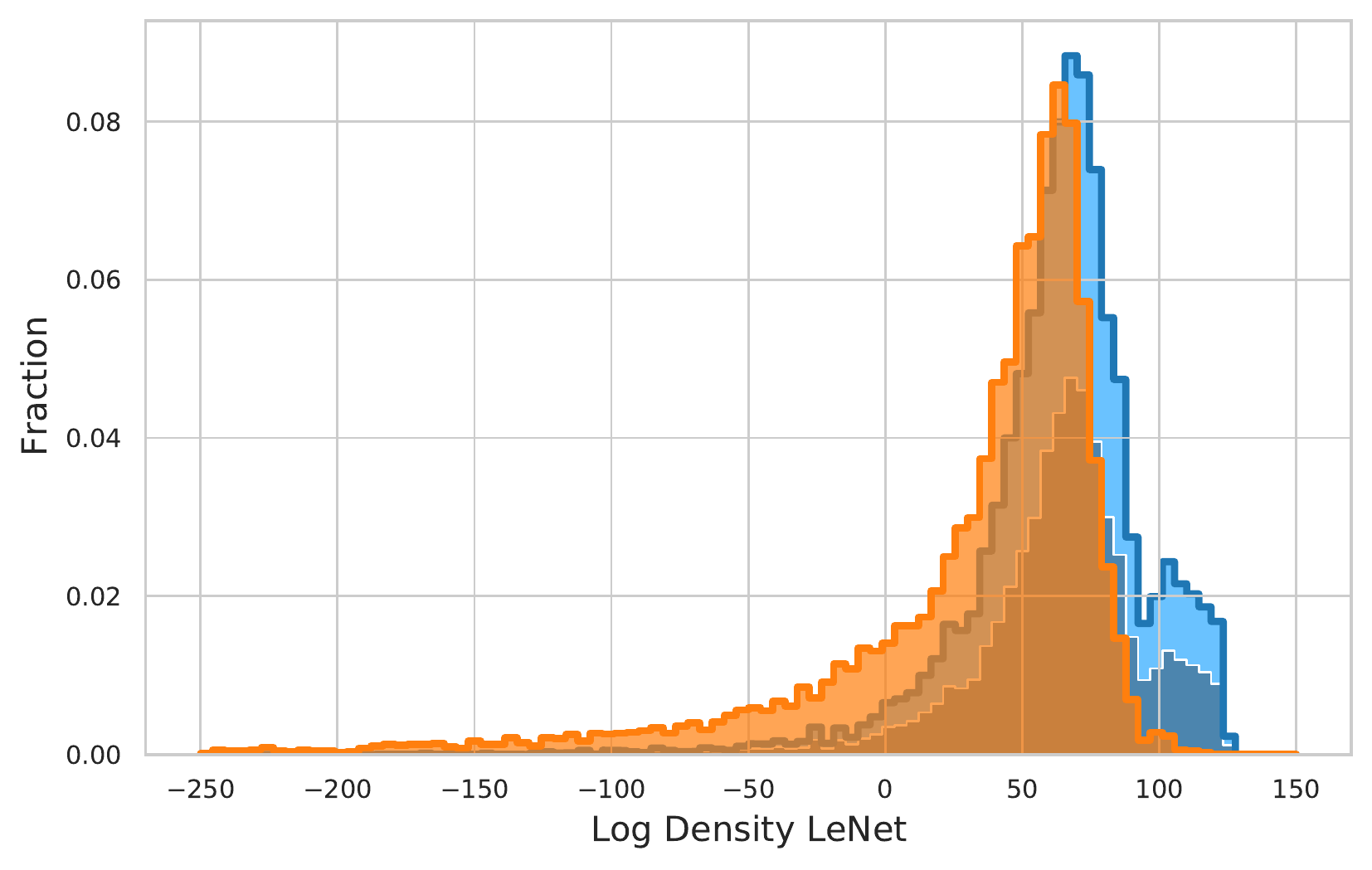}%
        \end{subfigure}%
        \begin{subfigure}{0.33\linewidth}
            \centering
            \includegraphics[width=\linewidth]{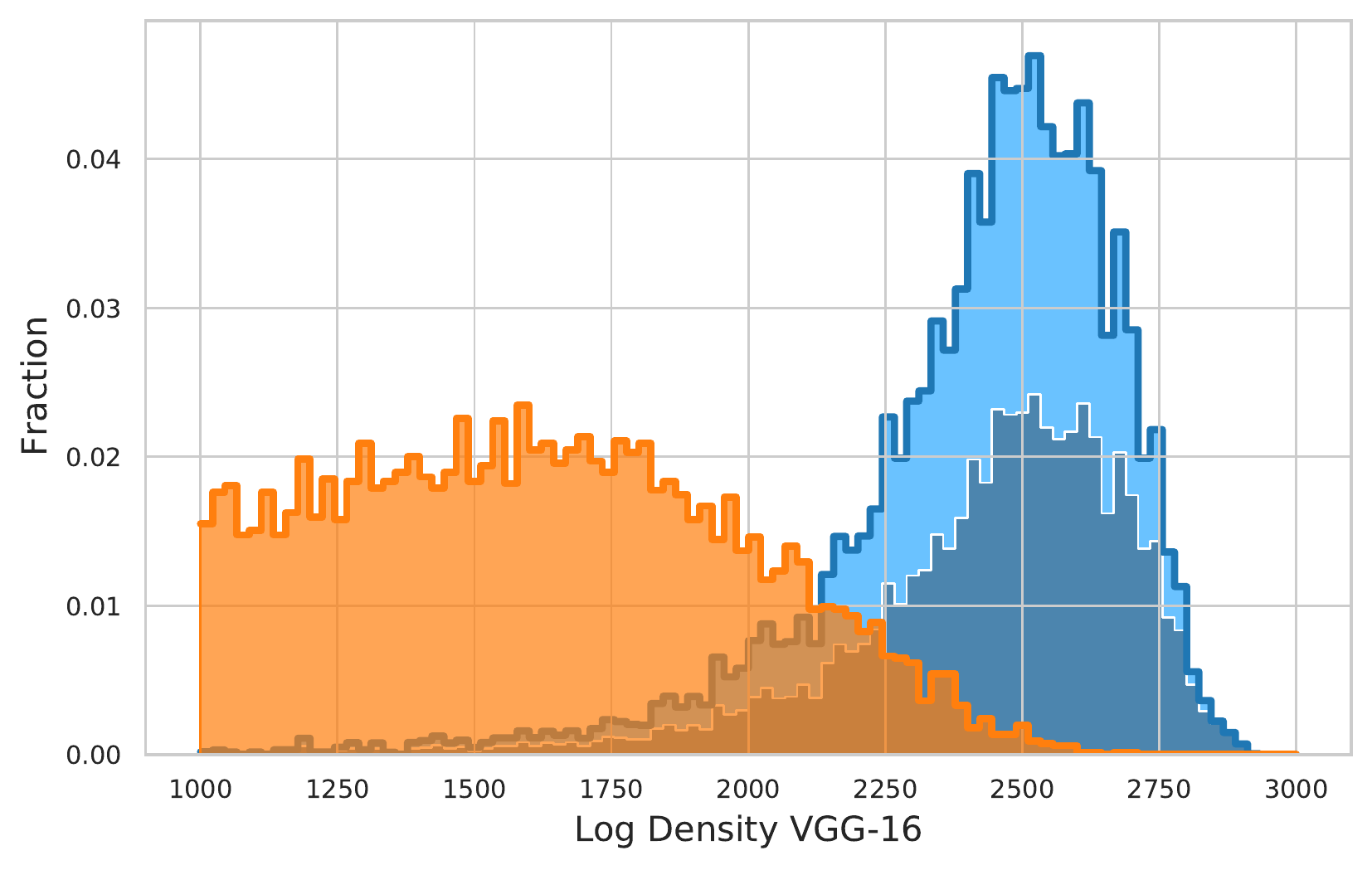}%
        \end{subfigure}%
        \begin{subfigure}{0.33\linewidth}
            \centering
            \includegraphics[width=\linewidth]{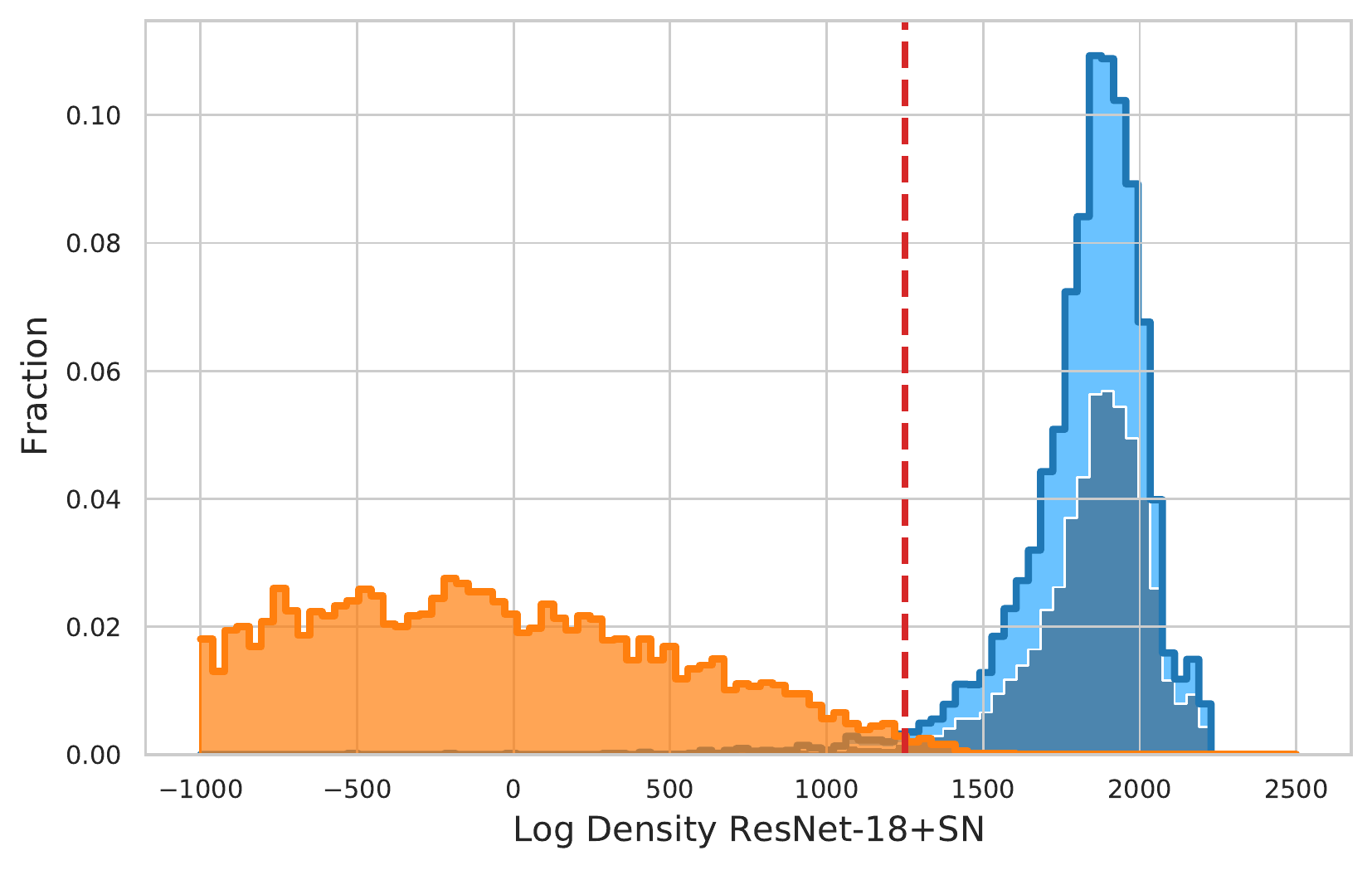}%
        \end{subfigure}%
        \subcaption{{Feature-space density}}
        \label{app_fig:intro_gmm}
    \end{subfigure}%
    \caption{
    \emph{
    Disentangling aleatoric and epistemic uncertainty on {\color{sns-blue}Dirty-MNIST (iD)} and {\color{sns-orange} Fashion-MNIST (OoD)} \textbf{\subref{app_fig:intro_sample_viz}} requires using \emph{softmax entropy} \textbf{\subref{app_fig:intro_softmax_ent}} and \emph{feature-space density (GMM)} \textbf{\subref{app_fig:intro_gmm}} with appropriate inductive biases (\emph{ResNet-18+SN} vs \emph{LeNet} \& \emph{VGG-16} without them). Enlarged version.
    }
    \textbf{\subref{app_fig:intro_softmax_ent}:}
    Softmax entropy captures aleatoric uncertainty for iD data (Dirty-MNIST), thereby separating {\color{sns-nonambiguous}unambiguous MNIST samples} and {\color{sns-ambiguous}Ambiguous-MNIST samples}. However, iD and OoD are confounded: softmax entropy has arbitrary values for OoD, indistinguishable from iD.
    \textbf{\subref{app_fig:intro_gmm}:}
    With appropriate inductive biases (DDU with ResNet-18+SN), iD and OoD densities do not overlap, capturing epistemic uncertainty. However, without appropriate inductive biases (LeNet \& VGG-16), feature density suffers from \emph{feature collapse}: iD and OoD densities overlap. Generally, feature-space density confounds unambiguous and ambiguous iD samples as their densities overlap.
    \textbf{Note:} {\color{sns-nonambiguous}Unambiguous MNIST samples} and {\color{sns-ambiguous}Ambiguous-MNIST samples} are shown as stacked histograms with the total fractions adding up to 1 for {\color{sns-blue}Dirty-MNIST}. 
    }
    \label{app_fig:intro_histograms}
\end{sidewaysfigure}
\end{document}